\documentclass{article} 
\usepackage{iclr2024_conference,times}

\usepackage{amsmath,amsfonts,bm}

\def\figref#1{figure~\ref{#1}}

\def\secref#1{section~\ref{#1}}

\def\eqref#1{equation~\ref{#1}}

\def\1{\bm{1}}

\def\ve{{\vect{e}}}

\def\vq{{\vect{q}}}
\def\vr{{\vect{r}}}

\def\vu{{\vect{u}}}
\def\vv{{\vect{v}}}
\def\vw{{\vect{w}}}

\def\vz{{\vect{z}}}

\def\mO{{\bm{O}}}

\def\mQ{{\bm{Q}}}

\def\mSigma{{\bm{\Sigma}}}

\DeclareMathAlphabet{\mathsfit}{\encodingdefault}{\sfdefault}{m}{sl}
\SetMathAlphabet{\mathsfit}{bold}{\encodingdefault}{\sfdefault}{bx}{n}

\newcommand{\R}{\mathbb{R}}

\DeclareMathOperator*{\argmin}{arg\,min}

\usepackage[utf8]{inputenc}
\usepackage{graphicx}
\usepackage{xcolor}
\usepackage{amsfonts}
\usepackage{amsmath}
\usepackage{amssymb}
\usepackage{mathtools}
\usepackage{amsthm}
\usepackage{hyperref}
\usepackage{url}
\usepackage{amsmath} 

\RequirePackage{algorithm}
\RequirePackage{algorithmic}
\usepackage{cancel}
\usepackage{nicefrac}
\usepackage{subfigure}
 \usepackage{enumitem}

\allowdisplaybreaks

\newcommand{\ie}{\textit{i.e., }}
\newcommand{\eg}{\textit{e.g., }}

\newcommand\nth{\textsuperscript{th} }

\newcommand{\vect}[1]{\mathbf{#1}}

\newcommand{\doubleN}{\mathbb{N}}

\newcommand{\naturals}{\doubleN^{+}}

\renewcommand{\vw}{\vect{w}}           

\newcommand{\0}{\mat{0}}
\newcommand{\teacher}{\vw^{\star}}
\newcommand{\mat}[1]{\mathbf{#1}}
\newcommand{\X}{\mat{X}}

\newcommand{\M}{\mat{M}}
\renewcommand{\mO}{\mat{O}}
\newcommand{\U}{\mat{U}}
\newcommand{\Q}{\mat{Q}}
\newcommand{\V}{\mat{V}}
\newcommand{\I}{\mat{I}}
\newcommand{\A}{\mat{A}}

\newcommand{\x}{\vect{x}}

\newcommand{\y}{\vect{y}}
\newcommand{\w}{\vw}

\newcommand{\rank}{\operatorname{rank}}

\def\mSigma{{\mat{\Sigma}}}

\def\mQ{{\mat{Q}}}

\def\reals{\mathbb{R}}

\newcommand{\eqmargin}{\hspace{1.em}}

\newcommand{\algmargin}{\hspace{-.5em}}
\floatname{algorithm}{Scheme}

\newcommand{\hfrac}[2]{{#1}/{#2}}
\newcommand{\norm}[1]{\left\Vert{#1}\right\Vert}

\newcommand{\cnt}[1]{\left[{#1}\right]}
\newcommand{\explain}[1]{\left[\substack{#1}\right]}
\newcommand{\expectation}{\mathop{\mathbb{E}}}

\newcommand{\prn}[1]{\left({#1}\right)}
\newcommand{\bigprn}[1]{\big({#1}\big)}
\newcommand{\Bigprn}[1]{\Big({#1}\Big)}
\newcommand{\biggprn}[1]{\bigg({#1}\bigg)}

\newcommand{\smallnorm}[1]{\Vert{#1}\Vert}
\newcommand{\tnorm}[1]{\smallnorm{#1}}
\newcommand{\bignorm}[1]{\big\Vert{#1}\big\Vert}

\newcommand{\tsum}{{\sum}}
\newcommand{\ball}[1]{\mathcal{B}^{d}}

\newcommand{\loss}{\mathcal{L}}

\renewcommand\dim{p}

\usepackage{xspace}

\newcommand{\bigO}{\mathcal{O}}

\newtheorem{theorem}{Theorem}
\newtheorem{assumption}[theorem]{Assumption}

\newtheorem{lemma}[theorem]{Lemma}
\newtheorem{corollary}[theorem]{Corollary}
\newtheorem{property}[theorem]{Property}

\newtheorem{definition}[theorem]{Definition}
\newtheorem{remark}[theorem]{Remark}
\newtheorem{proposition}[theorem]{Proposition}

\newcommand{\remove}[1]{REMOVE!}

\newenvironment{proof-sketch}{\noindent{\bf Proof sketch.}}{}

\def\figref#1{Figure~\ref{#1}}

\def\secref#1{Section~\ref{#1}}

\def\eqref#1{Eq.~(\ref{#1})}

\def\lemref#1{Lemma~\ref{#1}}

\def\thmref#1{Theorem~\ref{#1}}

\def\corref#1{Cor.~\ref{#1}}

\def\defref#1{Def.~\ref{#1}}
\def\propref#1{Prop.~\ref{#1}}

\def\rmkref#1{Remark~\ref{#1}}

\def\appref#1{Appendix~\ref{#1}}

\definecolor{residuals}{RGB}{200,27,80}

\definecolor{itay}{RGB}{210,30,220}
\definecolor{gon}{RGB}{40,180,220}
\definecolor{mnb}{RGB}{220,180,20}
\definecolor{red2}{RGB}{220,25,25}
\definecolor{daniel}{RGB}{30,200,30}

\definecolor{todo}{RGB}{50,50,200}

\newcommand{\deleted}[1]{TODO}%

\newcommand{\unnotice}[1]{TODO}

\usepackage{multirow}

\makeatletter
\newenvironment{recall}[1][\proofname]{\par
\normalfont \topsep6\p@\@plus6\p@\relax
\trivlist
\item\relax
{\bfseries
Recall #1}%
{\bfseries\@addpunct{.}}\hspace\labelsep\ignorespaces
}
\makeatother

\long\def\supptitle#1{

   \gdef\@runningheadingerrortitle{0}

   \ifnum\statePaper=0
    {
     \gdef\@runningtitle{Manuscript under review by AISTATS \@conferenceyear}
    }
   \fi

   \ifnum\statePaper=1
   {
   \ifx\undefined\@runningtitle
    {
    \gdef\@runningtitle{#1}
    }
   \fi
   }
   \fi

   \ifnum\@runningheadingerrortitle=0
         {
         \global\setbox\titrun=\vbox{\small\bfseries\@runningtitle}%
         \ifdim\wd\titrun>\textwidth%
            {\gdef\@runningheadingerrortitle{2}
             \gdef\@messagetitle{Running heading title too long}
            }%
         \else\ifdim\ht\titrun>10pt
              {\gdef\@runningheadingerrortitle{3}
              \gdef\@messagetitle{Running heading title breaks the line}
              }%
              \fi
          \fi
         }
    \fi

   \ifnum\@runningheadingerrortitle>0
     {
        \fancyhead[CE]{\small\bfseries\@messagetitle}
        \ifnum\@runningheadingerrortitle>1
           \typeout{}%
           \typeout{}%
           \typeout{*******************************************************}%
           \typeout{Running heading title exceeds size limitations for running head.}%
           \typeout{Please supply a shorter form for the running head}
           \typeout{with \string\runningtitle{...}\space just after \string\begin{document}}%
           \typeout{*******************************************************}%
           \typeout{}%
           \typeout{}%
        \fi
     }
  \else
     {
          \fancyhead[CE]{\small\bfseries\@runningtitle}
     }
  \fi

  \hsize\textwidth
  \linewidth\hsize \toptitlebar {\centering
  {\Large\bfseries #1 \par}}
 \bottomtitlebar
}

\newcommand*\samethanks[1][\value{footnote}]{\footnotemark[#1]}

\title{The Joint Effect of Task Similarity and Overparameterization on Catastrophic Forgetting ---
An Analytical Model}

\author{Daniel Goldfarb\thanks{ Equal contribution.
Correspondence to
$<$goldfarb.d@northeastern.edu$>$
or
$<$itay@evron.me$>$.
}~~\textsuperscript{1}, 
Itay Evron\samethanks~~\textsuperscript{2}, 
Nir Weinberger \textsuperscript{2}, 
Daniel Soudry \textsuperscript{2}, 
Paul Hand \textsuperscript{1}
\\
\textsuperscript{1} Northeastern University
\textsuperscript{2} Department of Electrical and Computer Engineering, Technion
}

\newcommand{\fix}{\marginpar{FIX}}
\newcommand{\new}{\marginpar{NEW}}

\iclrfinalcopy 
\begin{document}

\maketitle

\begin{abstract}
In continual learning, catastrophic forgetting is affected by multiple aspects of the tasks. 
Previous works have analyzed separately how forgetting is affected by either task similarity or overparameterization. In contrast, our paper examines how task similarity and overparameterization \emph{jointly} affect forgetting in an analyzable model. 
Specifically, we focus on two-task continual linear regression, where the second task is a random orthogonal transformation of an arbitrary first task (an abstraction of random permutation tasks). We derive an exact analytical expression for the expected forgetting --- and uncover a nuanced pattern. In highly overparameterized models, intermediate task similarity causes the most forgetting.  However, near the interpolation threshold, forgetting decreases monotonically with the expected task similarity. We validate our findings with linear regression on synthetic data, and with neural networks on established permutation task benchmarks.
\end{abstract}

\section{Introduction}
Modern neural networks achieve state-of-the-art performance in a wide range of applications, but when trained on multiple tasks in sequence, they typically suffer from a drop in performance on earlier tasks, known as the \textit{catastrophic forgetting problem} \citep{goodfellow2013empirical}. 
\linebreak
Continual learning research is mostly dedicated to designing neural architectures and optimization methods that better suit learning sequentially (\eg \citet{zenke2017continual,de2021continual}). 
Despite these efforts, it is still unclear in which regimes forgetting is most pronounced, even for elementary models.

A number of works have explored the relationship between task similarity and catastrophic forgetting \citep{ramasesh2020anatomy, bennani2020generalisation, doan2021theoretical, lee2021continual, evron2022catastrophic, lin2023theory}. 
While earlier works struggled to have a consensus on whether similar or different tasks are most prone to forgetting, recent works suggested that continual learning is the easiest when tasks have either high similarity or low similarity, and that it is most difficult for tasks that have an intermediate level of similarity \citep{evron2022catastrophic, lin2023theory}. 
The main experimental evidence of this claim so far focused on the similarity of learned feature representations of a neural network \citep{ramasesh2020anatomy}. 
Theoretically, \citet{lin2023theory} quantified task similarity using  Euclidean distances between underlying ``teacher'' models. 
However, this notion of similarity cannot capture the task similarity in standard benchmarks (e.g., permuted MNIST), where typically the input features are changing between tasks (and not the teacher).
Others
\citep{doan2021theoretical,evron2022catastrophic} interpreted task similarity as the principal angles between data matrices.

In practice, neural networks are typically extremely overparameterized. 
Several works have studied, empirically and analytically, the beneficial effect of overparameterization in continual learning, particularly on the commonly used permutation and rotation benchmarks \citep{goldfarb2023analysis, mirzadeh2022wide}.
In this paper, we propose a more refined analysis, which is able to show the \emph{combined} effect of overparameterization and task similarity for a general data model. We show that task similarity alone cannot explain the difficulty of a continual learning problem, rather it depends also on the model's overparameterization level.

Our main result is an exact analytical expression for the worst-case forgetting under a two-task linear {regression} model {trained using the simplest (S)GD scheme}. 
Data for each task is related by a random orthogonal transformation over a randomly chosen subspace, and the \textbf{D}imensionality \textbf{o}f the \textbf{T}ransformed \textbf{S}ubspace (DOTS) controls the task similarity --- the higher the DOTS, the lower the similarity between tasks.
This similarity notion provides a natural knob that controls how similar tasks are after a random transformation. 
This notion also closely characterizes popular permutation benchmarks, for which it was first suggested by the seminal work of \citet{kirkpatrick2017overcoming}.

\figref{fig:extreme-behaviors} informally illustrates the essence of our result. When the model is suitably overparameterized, the relationship between task similarity and expected forgetting is non-monotonic, where the most forgetting occurs for intermediately similar tasks. 
However, if the model is critically parameterized, then the behavior is monotonic and the continual learning problem is most difficult for the highest dissimilarity level. This behavior illustrates how hard it is to estimate the difficulty of continual learning in different regimes.

\vspace{-3mm}
    
\begin{figure}[ht!]
    \subfigure[
    A highly overparameterized regime.] {
        \includegraphics[width=.49\columnwidth]{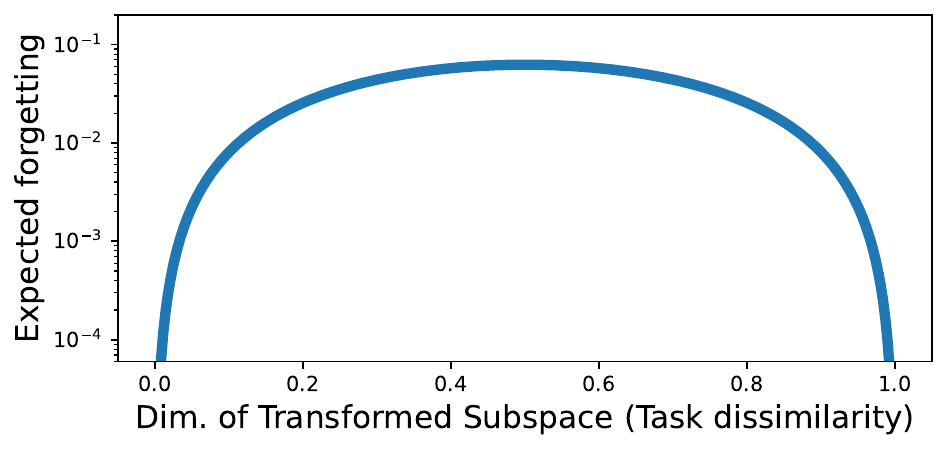}
        \label{fig:overparameterized}
    }
    \hfill
    \subfigure[
    Near the interpolation threshold.]
    {
        \includegraphics[width=.49\columnwidth]{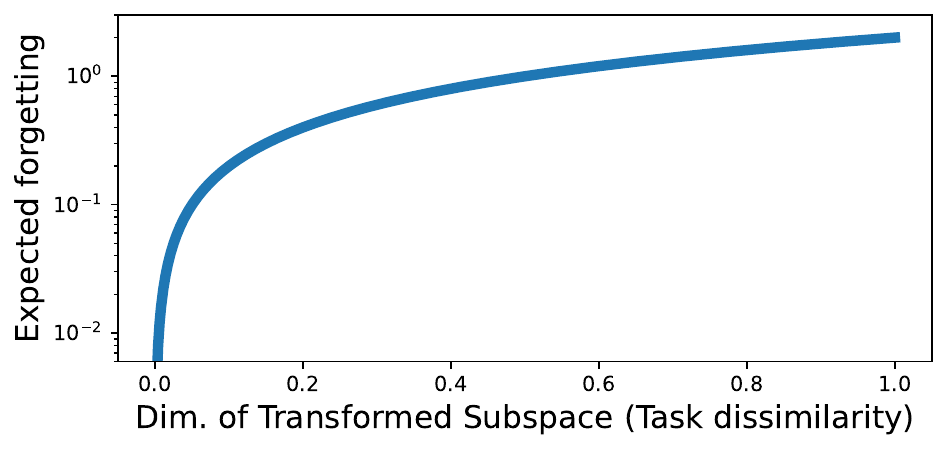}
        \label{fig:interpolation}
    }
    \vspace{-4mm}
    \caption{
    Informal illustration of our theoretical result. 
    Formal details are shared in \secref{sec:analysis}.
    \label{fig:extreme-behaviors}
    }
\end{figure}

The contributions of this paper are:
\vspace{-1mm}
\begin{itemize}[leftmargin=5mm]\itemsep.8pt
    \item We present a linear regression data model motivated by empirically-studied permutation tasks that exhibits a joint effect of task similarity and overparameterization on catastrophic forgetting.
    
    \item We derive an exact non-asymptotic expression for the worst-case expected forgetting under our model. 
    We reveal a non-monotonic behavior in task similarity when the model is suitably overparameterized, and a monotonic behavior when it is critically overparameterized.
    {We demonstrate that contrary to common belief, overparameterization alone cannot always prevent forgetting.}
    
    \item We replicate this theoretically-observed interaction of task similarity and overparameterization using a fully connected neural network in a permuted image setting.
\end{itemize}

\section{Analysis}
\label{sec:analysis}

We study a linear regression model trained continually on a sequence of two tasks under varying overparameterization and task-similarity levels.

\subsection{Data Model and Assumptions}
\label{sec:data-model}
Formally, we consider two regression tasks given by two $p$-dimensional data matrices ${\X_1,\X_2\in\reals^{n\times p}}$
and a single label vector $\y\in\reals^{n}$.
The first data matrix $\X_1$ can be \emph{arbitrary}. 
For ease of notation, we often denote $\X\triangleq\X_{1}$.

The second task's data matrix $\X_2$ is given by rotating the first task's $p$-dimensional data in a \emph{random} $m$-dimensional subspace for some $m\in\cnt{p}$.
Specifically, we set 
$\X_{2}=\X_{1}\mO=\X\mO$,
where $\mO\in\reals^{p\times p}$ is a random orthogonal operator defined as,
\begin{align}
\label{eq:data-model}
\mO	=\Q_{p}\left[\begin{array}{cc}
\Q_{m}\\
 & \I_{p-m}
\end{array}\right]\Q_{p}^{\top}\,,
\end{align}
and $\Q_{m}\sim{O}\left(m\right),\Q_{p}\sim{O}\left(p\right)$
are orthogonal operators, sampled ``uniformly'', \ie using the Haar measure of the orthogonal group. 
Our definition of the operator $\mO$ results in a more mathematically-tractable orthogonal transformation version of popular permutation task datasets (like the ones we use in \secref{sec:NN-experiments}), where the labels stay fixed while features permute.
This definition also reduces (when $m=p$) to the random rotations studied in \citet{goldfarb2023analysis}.

\pagebreak

To facilitate our analysis, we assume that the first task is realizable by a linear model (implying that the second one is realizable as well).
A similar assumption has been made in a previous theoretical paper
\citep{evron2022catastrophic} that analyzed, as we do, arbitrary data matrices (rather than assuming random isotropic data).
This assumption is especially reasonable in overparameterized regimes, such as in wide neural networks under the NTK regime.
\begin{assumption}[Realizability]
\label{asm:realizability}
    There exists a solution $\teacher\in\reals^{p}$ such that 
    $
    \X_1 \teacher = \X\teacher = \y
    $.
\end{assumption}
\vspace{-.1cm}
Note that this assumption implies that the second task is also realizable (since ${\X_2(\mO^\top\teacher)=\X\mO\mO^\top\teacher=\X\teacher=\y}$).
When $p\ge 2\rank(\X)$, it is readily seen that the tasks are also \emph{jointly}-realizable w.h.p.

\vspace{-.1cm}

\subsection{The Analyzed Learning Scheme and its  Learning Dynamics}

We analyze the most natural continual learning scheme. That is,
given two tasks $\prn{\X_1, \y}, \prn{\X_2,\y}$, 
we learn them sequentially with a gradient algorithm, without explicitly trying to prevent forgetting.

\vspace{-0.15in}

{\centering
\begin{minipage}{1\linewidth}
\begin{algorithm}[H]
   \caption{Continual learning of two tasks
    \label{proc:continual}}
\begin{algorithmic}
   \STATE {
   \algmargin Initialize} 
   $\w_0 = \0_{\dim}$
   \STATE {
   \algmargin Start from $\w_0$ and obtain $\w_1$ by minimizing 
   $\loss_1 (\w)\!\triangleq\!\norm{\X_1 \w \!-\! \y}^2$ with (S)GD (to convergence)
   } 
   \STATE {
   \algmargin Start from $\w_1$ and obtain $\w_2$ by minimizing 
   $\loss_2 (\w)\!\triangleq\! \norm{\X_2 \w \!-\! \y}^2$ with (S)GD (to convergence)
   } 
   \STATE {
   \algmargin 
   Output $\w_2$
   } 
\end{algorithmic}
\end{algorithm}
\end{minipage}
\par
}

The learning scheme above is known to mathematically converge\footnote{
In the realizable case, minimizing the squared loss of a linear model using (stochastic) gradient descent,
is known to 
converge to the projection of the initialization onto the solution space (see Sec.~2.1 in \citet{gunasekar2018characterizing}). 
This happens regardless of the batch size (as long as the learning rate is small enough).
Finally, the projections are given mathematically using the pseudoinverses 
(\eg see Sec.~1.3 in \citet{needell2014paved}).
} 
to the following iterates,
\begin{align}
    \w_1 \quad &
    = \quad 
    \Big(
     \argmin_{\w\in\reals^\dim}  \label{opt-prob}
     \norm{\w-\w_0}^2
     \text{ s.t. } \y \!=\! \X_1\w\,
     \Big)
     \quad = \quad
     \X_{1}^{+}\y~,
     \\[-0.1cm]
    \w_2 \quad &
    = \quad 
    \Big(
     \argmin_{\w\in\reals^\dim}  \label{opt-prob2}
     \norm{\w-\w_1}^2
     \text{ s.t. } \y \!=\! \X_2\w\,
     \Big)
     \quad = \quad
     \X_{2}^{+}\y+\left(\I_{\dim}-\X_{2}^{+}\X_{2}\right)\w_{1}~,
\end{align}
where $\X^{+}$
is the pseudoinverse of $\X$
{(we use $\X_{1}^{+},\X_{2}^{+}$ for analysis only; we do not compute them)}.

We now define our main quantity of interest.
\begin{definition}[Forgetting]
    \label{def:forgetting}
    The forgetting after learning the two tasks (parameterized by $\X,\teacher,\mO$),
    is defined as the degradation in the loss of the first task.
    More formally,
    $$
    F(\mO;\X, \teacher) 
    \triangleq
    \loss_1 (\w_2)
    -
    \loss_1 (\w_1)
    =
    \norm{\X_1 \w_2 - \y}^2 
    -
    \underbrace{\norm{\X_1 \w_1 - \y}^2 }_{=0}
    =
    \norm{\X \w_2 - \y}^2~. 
    $$
\end{definition}
\vspace{-0.25cm}
Our forgetting definition is natural and relates to definitions in previous papers (\eg \citet{doan2021theoretical}).
Notice that under our Assumption~\ref{asm:realizability},
the forgetting is the training loss of the first task (exactly as in \citet{evron2022catastrophic}).
Alternatively, one can study the degradation in the \emph{generalization} loss instead, but this often requires additional data-distribution assumptions 
(\eg random isotropic data as in 
\citet{goldfarb2023analysis,lin2023theory}), while our analysis is valid for any arbitrary data matrix.
Our analysis thus gives a better insight into the problem's expected \emph{worst-case} error.

To analyze the forgetting, we utilize our data model from \secref{sec:data-model} to show that,
\begin{align}
\label{eq:forgetting-inequality}
\begin{split}
    F(\mO;\X, \teacher)&=
    \left\Vert \X\w_{2}-\y\right\Vert ^{2}
    =\left\Vert 
    \X\left(\X_{2}^{+}\y
    +
    \left(\I_{p}-\X_{2}^{+}\X_{2}\right)\w_{1}\right)-\y\right\Vert ^{2}
    \\
    \explain{
    \X_2 = \X \mO
    }
    &= 
    \bignorm{
    \X\!\left(\X\mO\right)^{+}\!
    \y+
    \X\w_{1}
    -\X\!\left(\X\mO\right)^{+}
    \!
    \left(\X\mO\right)\w_{1}\!-\!\y} ^{2}
    \\
    \explain{
    \substack{
    \y=\X\w_{1},
    \\
    \w_{1}=\X^{+}\y
    =\X^{+}\X\teacher
    }
    }
    &=
    \bignorm{
    \X\!\left(\X\mO\right)^{+}\!
    \X\left(\I\!-\!\mO\right)\w_{1}}^{2}
    %
    =\bignorm{ \X\!\left(\X\mO\right)^{+}\!
    \X\left(\I\!-\!\mO\right)\X^{+}\X\teacher}^{2}
    \\
    \explain{
    \text{pseudoinverse properties}
    }
    &=\bignorm{
    \prn{\X\X^{+}\X}
    \prn{\mO^{\top}\X^{+}}
    \X\left(\mathbf{I}-\mO\right)
    \X^{+}\X\teacher}^{2}
    \\
    \explain{    \forall\X,\vv:\norm{\X\vv}_2\le\norm{\X}_2\norm{\vv}_2
    }
    &\le
    \bignorm{\X}^2_2
    \bignorm{
    \X^{+}\X
    {\mO^{\top}\X^{+}}
    \X\left(\mathbf{I}-\mO\right)
    \X^{+}\X\teacher}^{2}\,,
\end{split}
\end{align}
where $\norm{\X}_2 = \sigma_{\max}(\X)$.
{We used two pseudoinverse properties
(any matrix $\X$
and orthogonal matrix $\mO$
hold
$\X\!=\!\X\X^{+}\X$ and $\prn{\X\mO}^{+}\!\!=\!\mO^{\top}\X^{+}$).
} 
Our upper bound is \emph{sharp}, 
\ie it saturates when all nonzero singular values of $\X$ are identical. 
This allows for \emph{exact} worst-case forgetting~analysis.

\pagebreak

\subsection{Key result: Interplay between Task-Similarity and Overparameterization}

We now present our main theorem and illustrate it in 
\figref{fig:analytical}
on random synthetic data.

\begin{theorem}
\label{thm:main}
    Let ${\dim\ge 4}, d\in \left\{1,\dots,p\right\}, m\ge 2$.
    Define $\mathcal{X}_{p,d} \triangleq 
    \left\{\, \X\in\reals^{n\times p} \mid n\ge \rank(\X)=d \,\right\}$.
    \linebreak
    Define 
    the \textbf{D}imensionality \textbf{o}f \textbf{T}ransformed \textbf{S}ubspace 
    $\alpha\triangleq \frac{m}{p}$ as our proxy for task dissimilarity 
    and $\beta\triangleq 1-\frac{d}{p}$ as our proxy for overparameterization.
    Then, for any solution ${\teacher\in\reals^{p}}$ (Assumption~\ref{asm:realizability}),
    the (normalized) worst-case expected forgetting {per \defref{def:forgetting} 
    (obtained by Scheme~\ref{proc:continual}) } is
\begin{align*}
    &\max_{\substack{
    \X\in\mathcal{X}_{p,d}
    }}
    \!
    \frac{\mathbb{E}_{\mO} F(\mO;\X, \teacher)}{\tnorm{\X}^2_2 \tnorm{\X^{+}\X\teacher}^2}
    \!=\!
    \alpha
    \bigg(\!
    2+
    \beta\left(\alpha^{3}-6\alpha^{2}+11\alpha-8\right)
    +
    \beta^{2}\left(-5\alpha^{3}+22\alpha^{2}-30\alpha+12\right)+
    \\
    &
    \hspace{4.55cm}    
    \beta^{3}\left(5\alpha^{3}-18\alpha^{2}+20\alpha-6\right)
    \!
    \bigg)
    +
    \mathcal{O}\left(\frac{1}{p}\right),
\end{align*}
where $\X^{+}\X\teacher$ projects $\teacher$ onto the column space of $\X$.
Notice that $\tnorm{\X}^2_2 \tnorm{\X^{+}\X\teacher}^2$ is a necessary scaling factor, since the forgetting
$\norm{\X\w_2\!-\!\y}^2
=
\norm{\X\w_2\!-\!\X\teacher}^2$ 
naturally scales with $\tnorm{\X}^2_2$ and $\tnorm{\X^{+}\X\teacher}^2$.
The exact expression (without the $\bigO$ notation) appears
in \eqref{eq:full-expression} in \appref{app:analysis}.
\end{theorem}
The full proof is given
in \appref{app:analysis}.
Below we outline an informal sketch of the proof.

\paragraph{Proof sketch.}
{
In our proof, we show that the expected forgetting is controlled by two important terms, namely,
$\mathbb{E}_{\mO}\left(\ve_{i}^{\top}\mO^{\top}\mSigma^{+}\mSigma\left(\mathbf{I}-\mO\right)\ve_{i}\right)^{2}$
and 
$\mathbb{E}_{\mO}\left(\ve_{i}^{\top}\mO^{\top}\mSigma^{+}\mSigma\left(\mathbf{I}-\mO\right)\ve_{j}\right)^{2}$ for $i\neq j$,
where $\ve_i$ is the $i$\nth standard unit vector.
Each of these expectations is essentially a polynomial of the entries of our random 
$\Q_{p},\Q_{m}$ from \eqref{eq:data-model}.
To compute these expectations
(in Lemmas~\ref{lem:at-most-three-vectors}~and~\ref{lem:at-most-four-vectors}), 
we employ \emph{exact} formulas for the integrals of monomials over the orthogonal groups in $p$ and $m$ dimensions
\citep{gorin2002integrals}.
A more detailed proof outline is given in \appref{app:analysis}.
}%

\begin{figure}[ht!]
\begin{center}
        \includegraphics[width=.99\columnwidth]{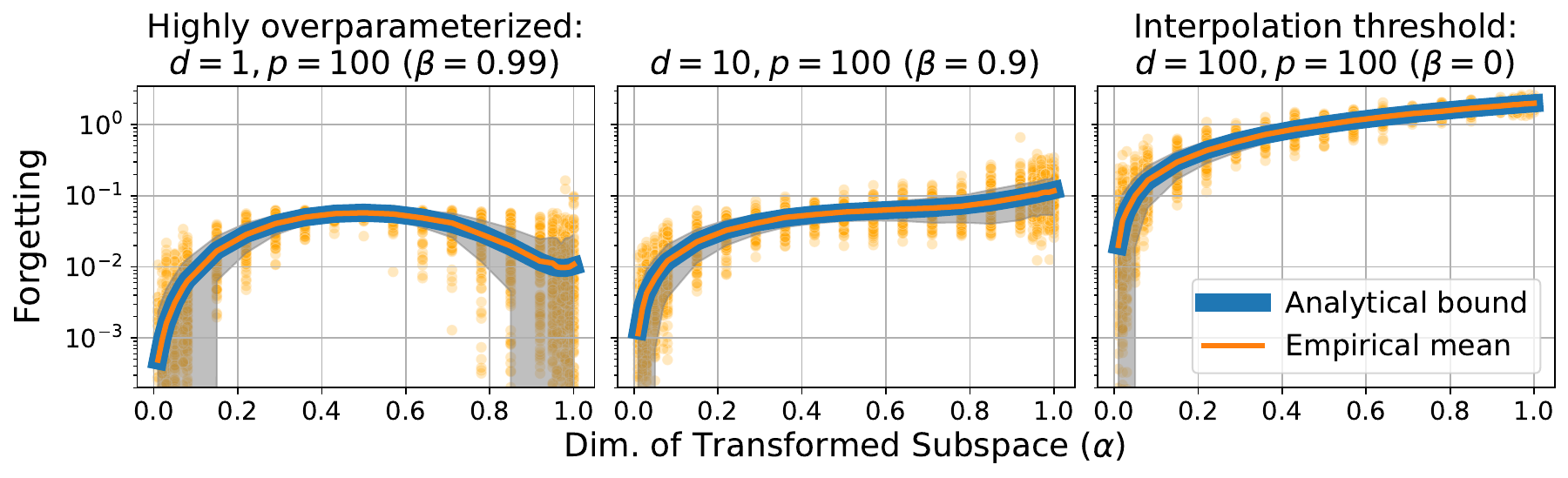}
    \vspace{-2mm}
    \caption{Empirically illustrating the worst-case forgetting under different overparameterization levels.
    Points indicate the forgetting under 1000 sampled random transformations applied on a (single) random data matrix $\X$. 
    Their mean is shown in the thin orange line, with the standard deviation represented by a gray band.
    The thick blue line depicts the analytical expression of \thmref{thm:main}.
    \linebreak
    {
    Here, we restrict the nonzero singular values of $\X$ to be identical, 
    saturating
    the inequality in \eqref{eq:forgetting-inequality}.
    Indeed, the analytical bound matches the empirical mean, thus exemplifying the tightness of our analysis.
    For completeness, in \appref{app:synthetic-figures}, we repeat this experiment with $p=10$ and $p=1000$.
    }
    \vspace{-2mm}
    }
    \label{fig:analytical}
\end{center}
\end{figure}


\subsubsection{Interesting extremal cases}
To help interpret our result, we exemplify it in several interesting regimes, taking either the task similarity proxy $\alpha$ or the overparameterization proxy $\beta$ to their extremes.

\paragraph{Highly overparameterized regime ($\beta=1-\frac{d}{p}\to1$).}
Plugging in $\beta=1$ into \thmref{thm:main}, we get
\begin{align}
\max_{\substack{
    \X\in\mathcal{X}_{p,d}
    }}
    \frac{\mathbb{E}_{\mO} F(\mO;\X, \teacher)}{\tnorm{\X}^2_2 \tnorm{\X^{+}\X\teacher}^2}
=\alpha^{2}\left(1-\alpha\right)^{2}
=
\prn{\tfrac{m}{p}}^{2}
\left(1-\tfrac{m}{p}\right)^{2}.
\label{eq:overparameterized}
\end{align}

\vspace{-1mm}

This behavior is illustrated in \figref{fig:overparameterized} and in the top part of \figref{fig:levelsets}. 
The non-monotonic nature of this behavior and its peak at $\alpha\!=\!0.5$ corresponding to \emph{intermediate} similarity, seem to agree with Figure 3(a) in \citet{evron2022catastrophic} 
(especially when no repetitions are made, \ie their $k\!=\!2$ curve).

\paragraph{At the interpolation threshold ($d=p\Longrightarrow\beta=1-\frac{d}{p}=0)$.}
In this extreme, the theorem asserts,
$$
\max_{\substack{
    \X\in\mathcal{X}_{p,d}
    }}
    \frac{\mathbb{E}_{\mO} F(\mO;\X, \teacher)}{\tnorm{\X}^2_2 \tnorm{\X^{+}\X\teacher}^2}
=
2\alpha
=
\tfrac{2m}{p}\,.
$$
This behavior is illustrated in \figref{fig:interpolation}
and in the rightmost plot of \figref{fig:analytical}. 
Notably, we get a monotonic decrease in forgetting as tasks become more similar.
This seems to contradict the conclusions of \citet{evron2022catastrophic}, according to which intermediate task similarity should be the worst.
We settle this alleged disagreement in our discussion in \secref{sec:discussion}.

\paragraph{Minimal task similarity ($m=p\Longrightarrow\alpha=\frac{m}{p} = 1$).}
Plugging in $\alpha=1$ into \thmref{thm:main}, we get
$$
\max_{\substack{
    \X\in\mathcal{X}_{p,d}
    }}
    \frac{\mathbb{E}_{\mO} F(\mO;\X, \teacher)}{\tnorm{\X}^2_2 \tnorm{\X^{+}\X\teacher}^2}
=
2-2\beta-\beta^{2}+\beta^{3}
=
\tfrac{d}{p}+2\left(\tfrac{d}{p}\right)^{2}-\left(\tfrac{d}{p}\right)^{3}\,.
$$

This minimal task similarity regime matches the (noiseless) model of \citet{goldfarb2023analysis}, where
a (generalization) risk bound with a scaling of $\sqrt{\frac{n}{p}}$ was proven under a particular data model.

\subsubsection{The entire interplay between task similarity and overparameterization}

The figure below illustrates our main result (\thmref{thm:main}).
While high task similarity consistently reduces forgetting, this effect becomes more evident with sufficient overparameterization.
Noticeably, non-monotonic effects of task similarity (our DOTS), only occur when the model turns highly overparameterized.
Importantly, we observe once more that even with extremely high overparameterization levels ($\beta\!=\!1\!-\!\frac{d}{p}\!\to\! 1$), forgetting does not vanish entirely. Instead, this outcome is contingent on task similarity, as captured by our DOTS measure.
For instance, \eqref{eq:overparameterized}
shows that when $\alpha=\frac{m}{p}=0.5$, the worst-case forgetting becomes $\prn{0.5}^4=0.0625$ when $\beta\!=\!1\!-\!\frac{d}{p}\!\to\! 1$.

\begin{figure}[h!]
    \centering    
    \includegraphics[width=.99\columnwidth]{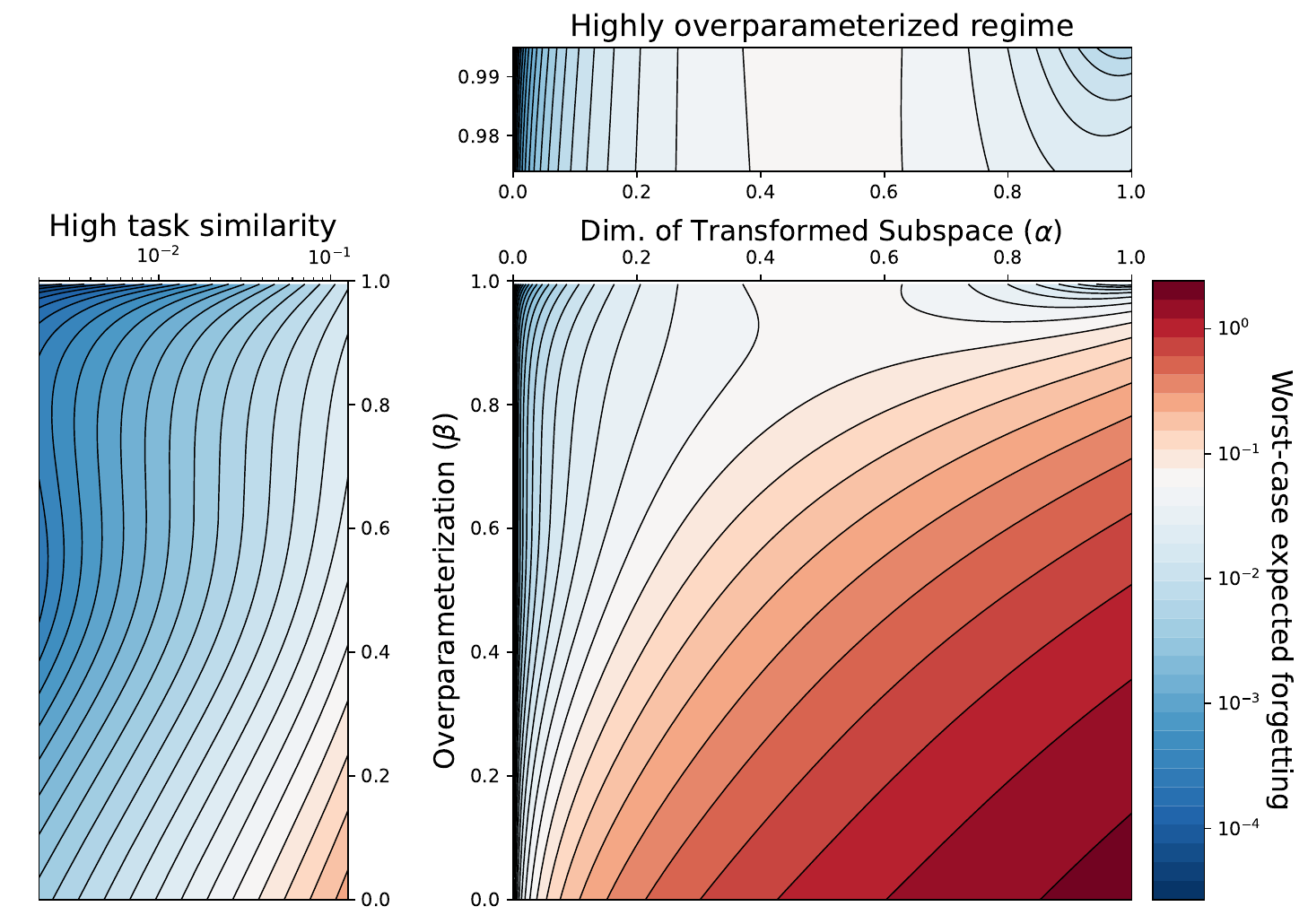}
    \vspace{-3mm}
    \caption{
    Levelsets depicting our main result from \thmref{thm:main}.
    The entire space (combinations of $\alpha,\beta$) appears on the lower-right subplot.
    We zoom into more interesting regimes, 
    \ie high task similarity and high overparameterization, 
    on the lower-left and upper-right subplots (respectively).
    \label{fig:levelsets}
    }
\end{figure}

\pagebreak

\subsection{Empirical Forgetting under an Average-Case Data Model}

We aim to simulate the (linear) model from Section \ref{sec:analysis} and show that the average-case behavior matches the joint effect that task similarity and overparameterization have on the worst-case forgetting as analyzed in our \thmref{thm:main}.
We choose the data model of \citet{goldfarb2023analysis} for its clear analogies to learning with neural networks.
Under this model, $n$ samples of effective dimensionality $d$ are sampled independently. 
The latent dimensionality $p$ of the samples controls the overparameterization level. 
These parameters are precisely defined in Section 2.1 therein. We also utilize the same hyperparameters and noise levels as defined in their numerical simulations in Section 2.3.
We compute the statistical risk and training error (MSE) of $\w_1$, $\w_2$, and the null estimator on task 1 as a function of $\alpha \in [0, 1]$ for $p \in [500, 1000, 3000]$ averaged over 100 runs.

\vspace{0.2em}

\begin{figure}[ht!]
\begin{center}   
    \includegraphics[width=0.99\columnwidth]{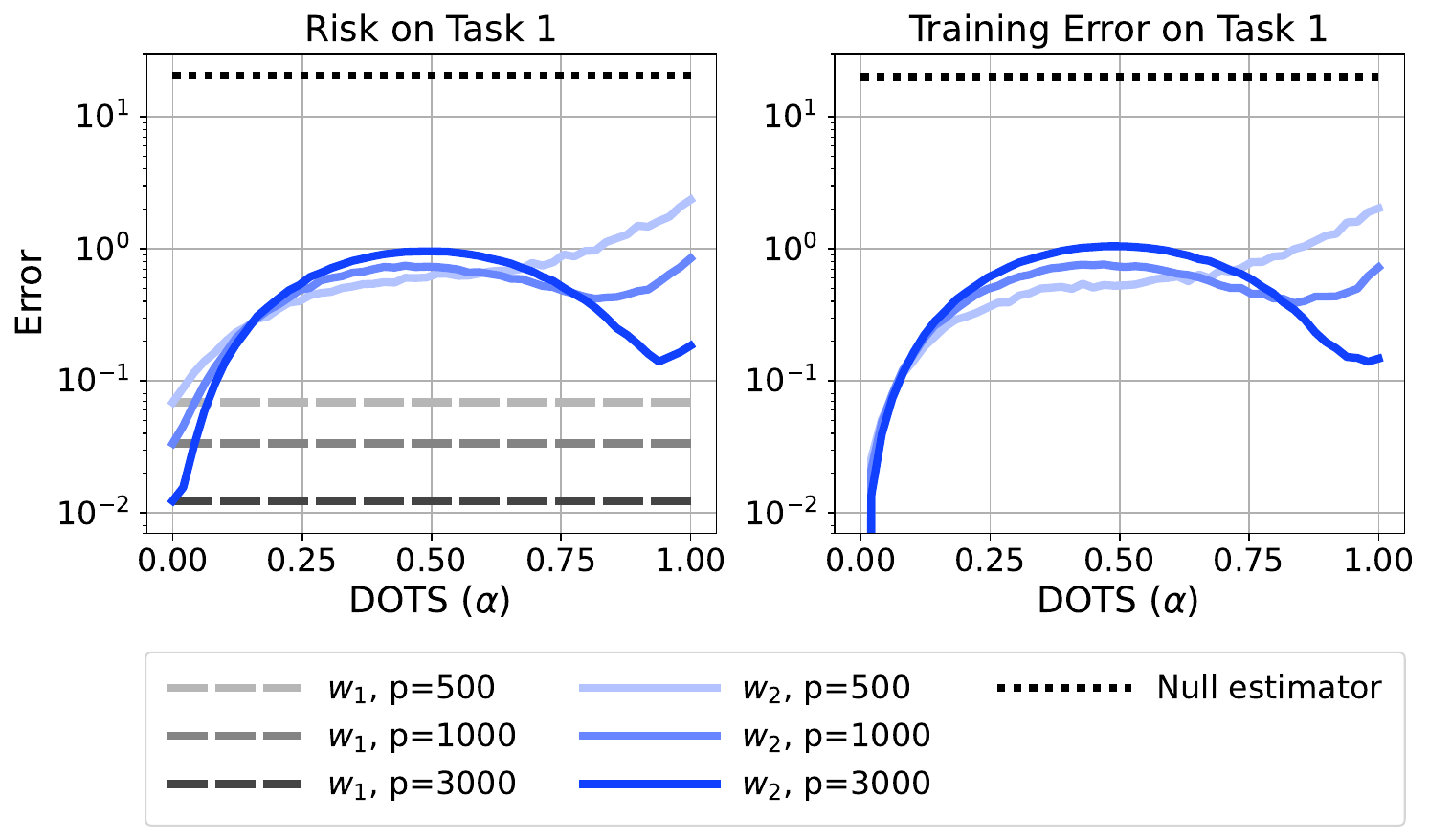}
    
    \vspace{-0.5em}
    
    \caption{Results of the numerical simulation. Risk and training error on task 1 are plotted as a function of $\alpha$ for various levels of $p$. The solid (blue) curves denote performance on task 1 of an estimator that is trained on task 1 and then on task 2.
    The dashed dark lines denote performance on task 1 of an estimator trained on task 1 only. 
    The dotted black line denotes the performance of the null estimator. 
    Training error curves for $w_1$ are omitted as these values are 0 for $p > d$.}
    \label{sim-risk}
\end{center}
\end{figure}

The results of the experiment are shown in \figref{sim-risk}. The dotted black line denotes performance of the null estimator (the parameter vector $\0_p$) providing a level for which any non-trivial estimator should beat. 
The dashed horizontal lines denote the risk of the single task estimator on task 1, providing a hypothetical best-case bound for $\w_2$. 
The forgetting of the model is then defined by the difference between the performance of $\w_1$ (grey curves) and the performance of $\w_2$ (blue curves). The grey curves are omitted for the training error plot as the single-task training error is 0 for $p>d$. Thus, forgetting in training error is controlled solely by the performance of $\w_2$.

Comparing \figref{fig:analytical} (for the worst-case forgetting) and \figref{sim-risk} here (for the average-case data model) reveals that the interplay between task similarity and overparameterization in both cases agrees with our analytical result in \thmref{thm:main}.
Here, for large $p$ ($3000$), we observe a $\cap$-shaped curve for the forgetting of $\w_2$, where the highest amount of forgetting occurs in the intermediate similarity regime. 
The model under small $p$ (500) has the greatest forgetting when tasks are most dissimilar. 
Forgetting risk at the extreme of $\alpha = 1$ reduces to the model of \citet{goldfarb2023analysis}, whose main result explains why overparameterization is beneficial in this setting. Additionally, it is not surprising that higher levels of overparameterization benefit the single task risk setting of $\w_1$, as this reduces to the model of \citet{hastie2022surprises} where the double descent behavior was observed.

\section{Neural Network Experiments}
\label{sec:NN-experiments}

In this section, we examine whether our analytical results apply to a continual learning benchmark using neural networks. The permuted MNIST \citep{lecun1998mnist} benchmark is a popular continual learning problem that has been used to measure the performance of many state-of-the-art continual learning algorithms \citep{kirkpatrick2017overcoming, zenke2017continual, li2019learn}. One performs a number of uniformly chosen permutations on the $28 \times 28$ images of the MNIST dataset. This results in equally difficult tasks (for an MLP): each task is a different pixel-shuffled version of MNIST. One then trains in sequence on these tasks and measures catastrophic forgetting. We consider a variant of the permuted MNIST benchmark, first suggested by \citet{kirkpatrick2017overcoming}, where instead of permuting the entire image, we only permute a square grid of pixels in the center of the image. Define the width/length of the permuted square to be the ``permutation size" ($PS$). When $PS=0$, each task is identical and thus extremely similar. When $PS=28$, each task is fully permuted and thus extremely different. Any value of $PS \in (0, 28)$ can be deemed to have some level of intermediate similarity. \figref{perm-mnist} shows examples of high, intermediate, and low similarities in this setup. We are interested in the relationship between $PS$ and catastrophic forgetting. Based on prior work \citep{ramasesh2020anatomy, evron2022catastrophic}, it seems we should expect the most forgetting to occur in the regime of intermediate similarity in a two-task scenario.

\vspace{-.2cm}
    
\begin{figure}[h]
\begin{center}
    \includegraphics[width=0.95\columnwidth]{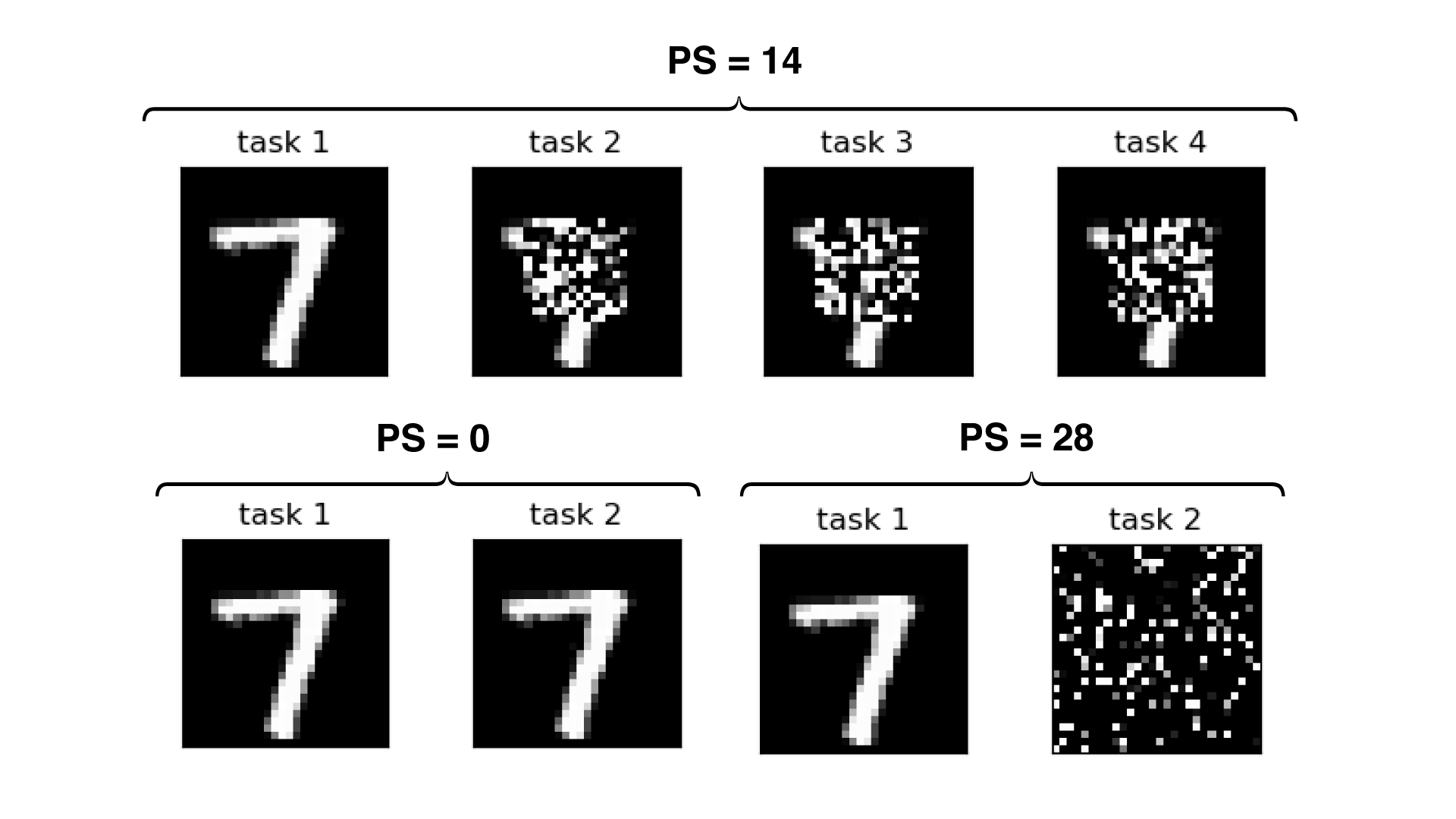}
    \vspace{-.7cm}
    \caption{Three versions of permuted MNIST for $PS=0$ (high similarity), $PS=14$ (intermediate similarity), $PS=28$ (low similarity).}
    \label{perm-mnist}
\end{center}
\end{figure}

We use vanilla SGD to train a 2-layer MLP of width 400 on sequences of up to 4 tasks of MNIST and EMNIST 
(a more difficult 26-class version of MNIST for handwritten English letters; see \citet{cohen2017emnist}). 
After each task is trained, we report the test error on all seen tasks. 
We compare forgetting across varying permutation sizes. 
See Appendix \ref{nn-details} for complete implementational details. 
The results are shown in \figref{mnist400}. 
The leftmost plots best illustrate the relationship between $PS$ and forgetting. The remaining plots are intended to ensure that each new task is being sufficiently fit. We observe a $\cap$-shape behavior where the most forgetting occurs around $PS \in [16, 24]$, agreeing with our hypothesis that forgetting is most severe for intermediately similar tasks.

\begin{figure}[h]
\begin{center}
    \includegraphics[width=0.99\columnwidth]{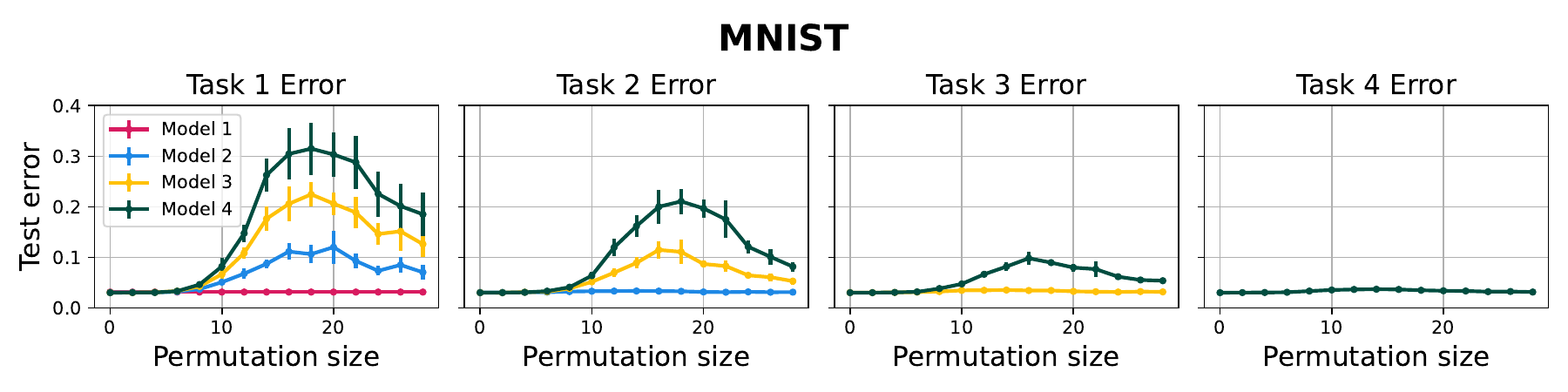} \\
    \vspace{.1cm}
    \includegraphics[width=0.99\columnwidth]{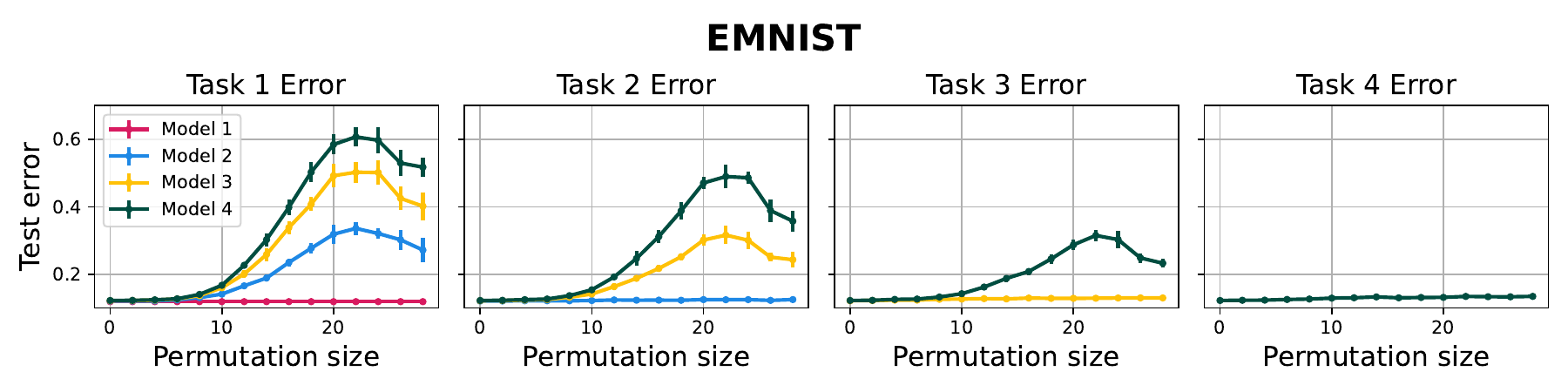}
    \caption{Results of the permutation experiment for varying levels of permutation size. The architecture is a 2-layer MLP of width 400. Model 1 corresponds to the net that is trained on task 1, Model 2 corresponds to the net that is trained on task 1 then task 2, and so on. Plotted curves have been averaged over 10 runs and error bars denote standard deviation of test error over the runs.}
    \label{mnist400}
\end{center}
\end{figure}

Now consider the experiment in \figref{mnist400} but using a 2-layer MLP with a lower overparameterization level. For each dataset, we find an overparameterization level that is significantly lower than before but still saturates to the training data and has comparable single-task test error as the highly overparameterized version. For MNIST we choose width 20 and for EMNIST we choose width 40. The results of this experiment are shown in \figref{mnist20}. We can observe that the $\cap$-shape behavior is now less pronounced and it appears that the continual learning problem has relatively equal difficulty for intermediately similar tasks as for extremely dissimilar tasks. This general behavior agrees with our previous results on the relationship of overparameterization and forgetting in Section \ref{sec:analysis}, both in the theoretical results and numerical simulations. See Appendix \ref{sec:ntk-features} for evidence of the connection between the notion of similarity in permuted MNIST and the NTK feature regime.

\begin{figure}[h]
\begin{center}
    \includegraphics[width=0.99\columnwidth]{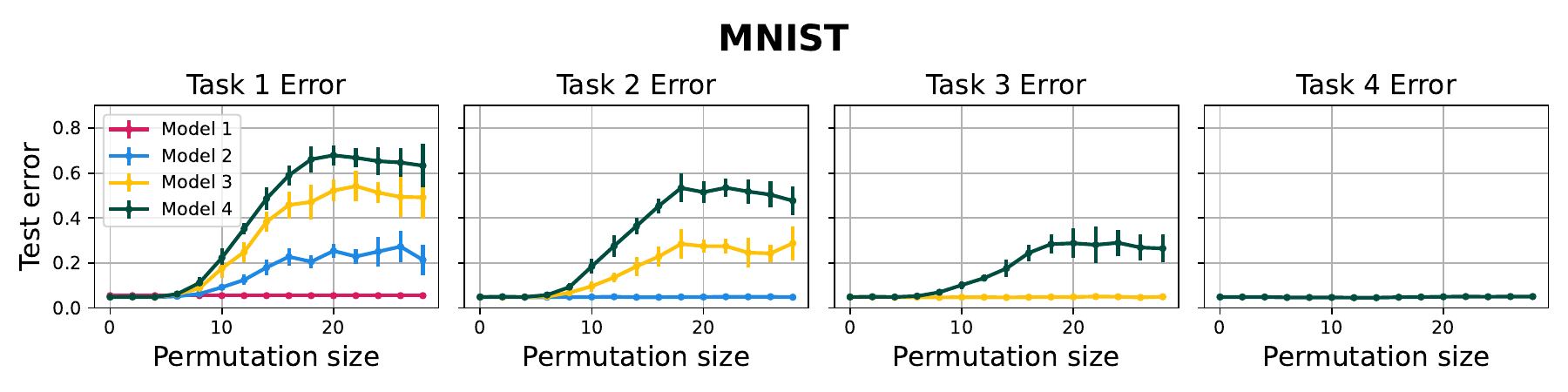} \\
    \includegraphics[width=0.99\columnwidth]{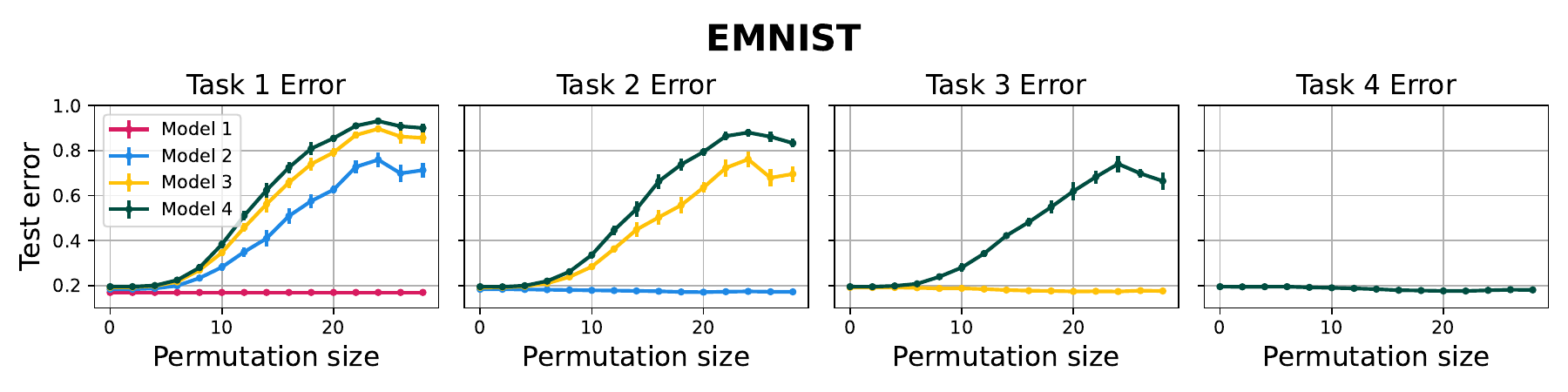}
    \caption{Replication of the experiment in \figref{mnist400} using less overparameterized 2-layer MLP models (width 20 for MNIST and width 40 for EMNIST).}
    \label{mnist20}
\end{center}
\end{figure}

\newpage

\section{Discussion}
\label{sec:discussion}

While several related works study continual learning theoretically \citep{kim2022theoretical, heckel2022provable}, only some of them study the relationship between task similarity and catastrophic forgetting.  \citet{bennani2020generalisation} prove generalization bounds which suggest that forgetting is more severe when tasks are dissimilar. 
\citet{lee2021continual} analyze a student-teacher setup, showing that intermediate tasks forget the most. 
\citet{li2023fixed} show regimes for which dissimilar tasks may be difficult and where performance on intermediately similar tasks can benefit from regularization.
Experimentally, \citet{ramasesh2020anatomy} provide evidence that intermediate task similarity is most difficult by studying learned feature representations during training on a number of modern neural networks.

\paragraph{Geometric interpretation and Comparison to \citet{evron2022catastrophic}.}
Previous studies utilize principal angles between the solution subspaces of the two data matrices ($\X_1,\X_2$) to quantify task similarity \citep{doan2021theoretical,evron2022catastrophic}.
\citet{evron2022catastrophic} show analytically that intermediate task similarity 
(an angle of $45^{\circ}$)
is most difficult in two-task linear regression.
Their analysis applies to \emph{any} two {arbitrary} tasks, and thus seemingly contradicts the behavior observed, \eg in our \figref{fig:interpolation}, 
where maximal dissimilarity is most difficult.
The key to settling this apparent disagreement is the \emph{randomness} of our transformations (\eqref{eq:data-model}).
Their analysis focuses on any two \emph{deterministic} tasks, while our second task is given by a \emph{random} transformation of the first, as done in many popular continual learning benchmarks (\eg permutation and rotation tasks).

To gain a geometric intuition, consider two tasks of rank $d\!=\!1$ ($\x_1, \x_2$).\footnote{
For simplicity, we discuss the principal angles between $\x_1,\x_2$ instead of between their nullspaces 
(\ie the solution spaces).
In two-task scenarios, these are essentially equivalent (see Claim~19 in \citet{evron2022catastrophic}).
}
Consider also a \emph{maximal} task dissimilarity (DOTS of $\alpha\!=\!\frac{m}{p}\!=\!1$).
Then,
$\x_2 \!=\! \mO\x_1$ is a completely random rotation of $\x_1$ in $p$ dimensions.
It is known that 
$\mathbb{E}
\big|
\big\langle
\frac{\x_1}{\tnorm{\x_1}},
\frac{\x_2}{\tnorm{\x_2}}
\big\rangle
\big|
\!\approx\!
\frac{1}{\sqrt{p}}
$
(Remark 3.2.5 in \citet{vershynin2018high}).
Near the interpolation threshold, \eg when $p\!=\! 2$ (recall that $d\!=\!1$), we get  
${\mathbb{E}
\big|
\big\langle
\frac{\x_1}{\tnorm{\x_1}},
\frac{\x_2}{\tnorm{\x_2}}
\big\rangle
\big|
\!\approx\!
\frac{1}{\sqrt{2}}}
\Longrightarrow 
{\mathbb{E}\angle(\x_1,\x_2)
\approx
45^{\circ}}$, 
corresponding to the \emph{intermediate} task dissimilarity in \citet{evron2022catastrophic}, where forgetting is \emph{maximal}.
Conversely, given high overparameterization levels ($p\to\infty$), 
we get 
${\mathbb{E}
\big|
\big\langle
\frac{\x_1}{\tnorm{\x_1}},
\frac{\x_2}{\tnorm{\x_2}}
\big\rangle
\big|
\!\approx\!
\frac{1}{\sqrt{p}}
\to
0}
\Longrightarrow 
{\mathbb{E}
\angle(\x_1,\x_2)\to 90^{\circ}}$, corresponding to the \emph{maximal} task dissimilarity in \citet{evron2022catastrophic}, where forgetting is \emph{minimal}.

\vspace{-0.5em}

\paragraph{Comparison to \citet{lin2023theory}.}
\citet{lin2023theory} prove generalization bounds on forgetting which suggest that forgetting may not change monotonically with task similarity.  However, their data model is not suitable for high overparameterization. For example, in the limit of high overparameterization, their model is performing as well as a null predictor. In contrast, we focus on the training error, and do not assume a specific data model for the first task, which allows us to generalize even in the highly overparameterized regime.

\vspace{-0.5em}

\paragraph{A starting point for analysis.}
Our work focuses on linear models and data permutation tasks. 
Exploring linear models using (stochastic) gradient descent is the most natural starting point for theoretical analysis, as an initial step towards understanding more complex systems. Moreover, recent work shows connections between extremely overparameterized neural networks and linear models via the neural tangent kernel (NTK) (\citet{jacot2018neural}; but also see \citet{wenger2023disconnect}). 
We choose to study data permutation tasks for their well-defined mathematical relationship and generation of equally difficult tasks (for a fully connected model). However there is criticism that permutation tasks are relatively easy to solve in practice and only provide a best-case for real-world problems \citep{farquhar2018towards, pfulb2019comprehensive}. Despite these critiques, we believe that the data permutation setting is the most amenable for initial theoretical results.

\vspace{-0.5em}

\subsection{Limitations and Future work}

Our analysis in \secref{sec:analysis}
has centered around a continual \emph{linear} regression model.
An immediate next step is to explore the extension of our analysis and empirical findings to more intricate non-linear models (\eg MLPs, CNNs, and transformers) or to other notions of task similarity.
Another avenue of investigation involves extending the analysis to continual \emph{classification} models,
possibly using the weak regularization approach suggested by 
\citet{evron23continualClassification}.

Our analysis has also primarily examined settings with $T=2$ tasks.
Extending these analytical results to $T\ge 3$ tasks poses an immediate challenge.
The complexity of our analysis, which already required intricate techniques and proofs, suggests that tackling the extension may be considerably difficult.
Moreover, the convergence analysis presented in a previous paper \citep{evron2022catastrophic} 
for learning $T\ge 3$ tasks cyclically has proven to be notably more challenging than that for $T=2$ tasks (and was further improved in a follow-up paper \citep{swartworth2023nearly}).

Finally, since our models are linear, our proxy for overparameterization, 
\ie $\beta=1-\frac{d}{p}$, directly controls the overlap between the task subspaces (see also the geometric interpretation above).
Clearly, this proxy is different than the width of deep networks.
On the other hand, there are still relations between these two proxies through the theory of the NTK regime \citep{jacot2018neural}.
A further examination of these relations, both theoretically and empirically (perhaps in the spirit of our \appref{sec:ntk-features} and {\citet{wenger2023disconnect}}), could benefit the continual learning literature.

\newpage

\subsubsection*{Acknowledgments}
We would like to thank the anonymous reviewers for their insightful feedback.
\linebreak
DG is partially supported by NSF DMS-2053448. PH acknowledges support from NSF DMS-2053448, DMS-2022205, and DMS-1848087.
The research of NW was partially supported by the Israel Science Foundation (ISF), grant no.~1782/22.
The research of DS was Funded by the European Union (ERC, A-B-C-Deep, 101039436). Views and opinions expressed are however those of the author only and do not necessarily reflect those of the European Union or the European Research Council Executive Agency (ERCEA). Neither the European Union nor the granting authority can be held responsible for them. DS also acknowledges the support of the Schmidt Career Advancement Chair in AI.

\bibliography{99_references}
\bibliographystyle{iclr2024_conference}

\newpage

\appendix

\newpage

\section{Neural Network Implementation Details for \secref{sec:NN-experiments}} \label{nn-details}

\begin{table}
\begin{center}
\caption{Hyperparameters for the neural network experiments}
\begin{tabular}{ |c|c|c|c| } 
\hline
Hyperparameter & Value \\
\hline
learning rate & 0.01 \\ 
batch size & 64 \\
epochs / task & 100 \\ 
momentum & 0 \\
dropout & 0 \\
\hline
\end{tabular}
\label{hyperparameters}
\end{center}
\end{table}

Table \ref{hyperparameters} reports the hyperparameters used in the neural network experiments. All architectures used ReLU activation functions for the hidden layers and softmax for the output layers. Weights were initialized as $Unif(\frac{-1}{\sqrt{i}}, \frac{1}{\sqrt{i}})$ where $i$ is the input dimension of the given layer. The experiment in Figure \ref{mnist400} used intermediate width 400 and the experiment in Figure \ref{mnist20} used intermediate width 20 for MNIST and 40 for EMNIST.

\section{NTK Feature Similarity Experiments} \label{sec:ntk-features}

\newcommand{\innerproduct}[2]{\langle #1, #2 \rangle}

Let $\vect{a}, \vect{b} \in \R^p$ be two sets of NTK features and define correlation as $\frac{|\innerproduct{\vect{a}}{\vect{b}}|}{\|\vect{a}\|\|\vect{b}\|}$. 
Then the average correlation between datasets $\mat{A}, \mat{B} \in \R^{n \times p}$ is $\sum_{i=1}^n \frac{|\innerproduct{\vect{a}_i}{\vect{b}_i}|}{\|\vect{a}_i\|\|\vect{b}_i\|}/n$. 
Figure \ref{ntk-features} plots the average correlation between original MNIST and permuted MNIST as a function of permutation size.
We see that the average correlation is monotonic in permutation size --- extremely high for low permutation size and extremely low for high permutation size. This provides evidence of the connection between the centered permuted MNIST problem and task similarity in the NTK regime.

\begin{figure}[ht!]
\begin{center}
    \includegraphics[scale=0.7]{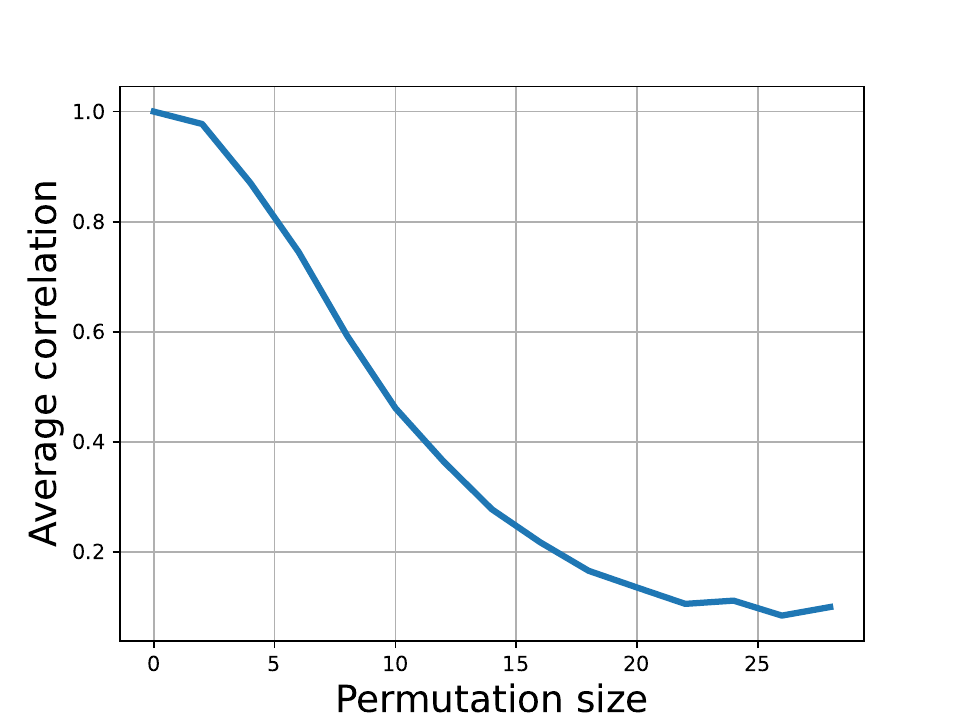}
    \caption{Results of the permuted MNIST NTK feature similarity experiment.}
    \label{ntk-features}
\end{center}
\end{figure}

\newpage

\section{Supplementary Materials for our Analytic Results (\secref{sec:analysis})}
\label{app:analysis}

Throughout the appendix, we denote the singular-value decomposition of a given $\X\in\reals^{n\times p}$ by $\X=\U\mSigma\V^{\top}$ for $\mSigma\in\reals^{n \times p}$. 
The number of nonzero entries on the diagonal of $\mSigma$ is $d=\rank\prn{\X}$.

\begin{recall}[\thmref{thm:main}]
\\
    Let ${\dim\ge 4}, d\in \left\{1,\dots,p\right\}, m\ge 2$.
    Define $\mathcal{X}_{p,d} \triangleq 
    \left\{\, \X\in\reals^{n\times p} \mid n\ge \rank(\X)=d \,\right\}$.
    \linebreak
    Define 
    the \textbf{D}imensionality \textbf{o}f \textbf{T}ransformed \textbf{S}ubspace 
    $\alpha\triangleq \frac{m}{p}$ as our proxy for task dissimilarity 
    and $\beta\triangleq 1-\frac{d}{p}$ as our proxy for overparameterization.
    Then, for any solution ${\teacher\in\reals^{p}}$ (Assumption~\ref{asm:realizability}),
    the (normalized) worst-case expected forgetting {per \defref{def:forgetting} 
    (obtained by Scheme~\ref{proc:continual}) } is
\begin{align*}
    &\max_{\substack{
    \X\in\mathcal{X}_{p,d}
    }}
    \!
    \frac{\mathbb{E}_{\mO} F(\mO;\X, \teacher)}{\tnorm{\X}^2 \tnorm{\X^{+}\X\teacher}^2}
    %
    \!=\!
    \alpha
    \bigg(
    2+
    \beta\left(\alpha^{3}+11\alpha-6\alpha^{2}-8\right)
    +
    \beta^{2}\left(-5\alpha^{3}+22\alpha^{2}-30\alpha+12\right)+
    \\
    &
    \hspace{4.5cm}    
    \beta^{3}\left(5\alpha^{3}-18\alpha^{2}+20\alpha-6\right)
    \bigg)
    +
    \mathcal{O}\left(\frac{1}{p}\right),
\end{align*}
where $\X^{+}\X\teacher$ projects $\teacher$ onto the column space of $\X$.
Notice that $\tnorm{\X}^2 \tnorm{\X^{+}\X\teacher}^2$ is a necessary scaling factor, since the forgetting
$\norm{\X\w_2\!-\!\y}^2
=
\norm{\X\w_2\!-\!\X\teacher}^2
$ naturally scales with $\tnorm{\X}^2$ and $\tnorm{\X^{+}\X\teacher}^2$.
The exact expression (without the $\bigO$ notation) appears
in \eqref{eq:full-expression} in \appref{app:analysis}.
\end{recall}

\bigskip

\paragraph{Proof outline.}

We start our proof (below) by showing that the expected forgetting is sharply upper bounded as,
\begin{align*}
    &
    \tfrac{1}{\norm{\X}^2}
    \mathbb{E}_{\mO} F(\mO;\X, \teacher)
    \le
    \sum_{i=1}^{d}
    \mathbb{E}_{\mO}
    \prn{ \ve_{i}^{\top}\mO^{\top}\mSigma^{+}\mSigma\left(\I-\mO\right)
    \mSigma^{+}\mSigma \V^\top\teacher
    }^{2}
    \\
    &=
    \biggprn{    \mathbb{E}_{\mO}\left(\ve_{1}^{\top}\mO^{\top}\mSigma^{+}\mSigma\left(\mathbf{I}-\mO\right)\ve_{1}\right)^{2}
    +
    \left(d-1\right)    \mathbb{E}_{\mO}\left(\ve_{1}^{\top}\mO^{\top}\mSigma^{+}\mSigma\left(\mathbf{I}-\mO\right)\ve_{2}\right)^{2}
    }
    \underbrace{\sum_{i=1}^{d}
    {\prn{\V^\top\teacher}_{i}}^{2}
    }_{=\norm{\X^{+}\X\teacher}^2}
    \,,
\end{align*}
where $\ve_1, \ve_2$ are the two first standard unit vectors in $\mathbb{R}^{\dim}$.
This directly implies that,
\begin{align*}
\frac{\mathbb{E}_{\mO} F(\mO;\X, \teacher)}{\norm{\X}^2 \norm{\X^{+}\X\teacher}^2}
\le
\underbrace{
\mathbb{E}_{\mO}\left(\ve_{1}^{\top}\mO^{\top}\mSigma^{+}\mSigma\left(\mathbf{I}-\mO\right)\ve_{1}\right)^{2}
}_{\text{solved in \lemref{lem:at-most-three-vectors}}}
+
\left(d-1\right)
\underbrace{
\mathbb{E}_{\mO}\left(\ve_{1}^{\top}\mO^{\top}\mSigma^{+}\mSigma\left(\mathbf{I}-\mO\right)\ve_{2}\right)^{2}
}_{\text{solved in \lemref{lem:at-most-four-vectors}}}\,.
\end{align*}
Each of these two expectations is essentially a polynomial of the entries of our random orthogonal
$\Q_{p}\in\reals^{\dim\times\dim},
\Q_{m}\in\reals^{m\times m}$ which form the random operator
$\mO	=\Q_{p}\left[\begin{array}{cc}
\Q_{m}\\
 & \I_{p-m}
\end{array}\right]\Q_{p}^{\top}$
as explained in \eqref{eq:data-model}.

We compute these two expectations
(in Lemmas~\ref{lem:at-most-three-vectors}~and~\ref{lem:at-most-four-vectors}) by employing \emph{exact} formulas for the integrals of monomials over the orthogonal groups in $p$ and $m$ dimensions
\citep{gorin2002integrals}.
Our derivations often get complicated, and so we split them into three appendices:
\begin{enumerate}[leftmargin=6mm]
    \item
    \appref{app:analysis} (below):
    Derivations of more complicated expressions involving the operator $\mO$.
    
    \item 
    \appref{app:monomials}:
    Derivations and properties of monomials of general random orthogonal matrices, mostly using results from \citet{gorin2002integrals}.

     \item \appref{app:auxiliary}:
     Derivations of auxiliary expressions (mostly involving $\mQ_m,\mQ_p$) that are used as building blocks in multiple derivations.
\end{enumerate}

\pagebreak

\begin{proof}[Proof for \thmref{thm:main}]
Starting from \eqref{eq:forgetting-inequality}, we show that for any given ${\X,\teacher}$
it holds that 
\begin{align}
\begin{split}
    \tfrac{1}{{\norm{\X}}^2}
    \mathbb{E}_{\mO} F(\mO;\X, \teacher)
    &=
    \tfrac{1}{{\norm{\X}}^2}
    \mathbb{E}_{\mO}
    \bignorm{ \X\X^{+}\X\mO^{\top}\X^{+}\X\left(\I-\mO\right)
    \X^{+}\X
    \teacher}^{2}
    \\
    &
    \le
    \mathbb{E}_{\mO}
    \bignorm{\X^{+}\X\mO^{\top}\X^{+}\X\left(\I-\mO\right)
    \X^{+}\X
    \teacher}^{2}\,,
\end{split}
\end{align}
where, importantly, the inequality saturates when all the nonzero singular values of $\X$ are identical.
Plugging in the SVD, 
\ie of $\X=\U\mSigma\V^\top$, we get
\begin{align*}
\tfrac{1}{{\norm{\X}}^2}
\mathbb{E}_{\mO} F(\mO;\X, \teacher)
    &\le
    \mathbb{E}_{\mO}
    \bignorm{
    \V \mSigma^{+}\mSigma \V^\top
    \mO^{\top}
    \V \mSigma^{+}\mSigma \V^\top
    \left(\I-\mO\right)
    \V \mSigma^{+}\mSigma \V^\top
    \teacher}^{2}
    \\
    &=
    \mathbb{E}_{\mO}
    \bignorm{\mSigma^{+}\mSigma \V^\top
    \mO^{\top}
    \V \mSigma^{+}\mSigma \V^\top
    \left(\I-\mO\right)
    \V \mSigma^{+}\mSigma \V^\top
    \teacher}^{2}\,,
\end{align*}
where the last equality stems from spectral norm properties (recall that $\V$ is an orthogonal matrix).

Following our definition of $\mO =\mathbf{Q}_{p}\left[\begin{smallmatrix}
\mathbf{Q}_{m}\\
 & \I_{p-m}
\end{smallmatrix}\right]\mathbf{Q}_{p}^{\top}$ 
(\eqref{eq:data-model}) and since $\Q_p$ is sampled \emph{uniformly} from the orthogonal group $O(p)$, we notice that $\mO$ and $\V^\top\mO\V$ are identically distributed.
Hence, we can rewrite the above expectation as
\begin{align}
\tfrac{1}{{\norm{\X}}^2}
\mathbb{E}_{\mO} F(\mO;\X, \teacher)
    &\le
    \mathbb{E}_{\mO}
    \bignorm{\mSigma^{+}\mSigma
    \mO^{\top}
    \mSigma^{+}\mSigma
    \left(\I-\mO\right)
    \underbrace{
    \mSigma^{+}\mSigma 
    \V^\top\teacher}_{\triangleq \vv}
    }^{2}
    \nonumber
    \\
    \explain{
    \mSigma^{+}\mSigma=
    \left(\mSigma^{+}\mSigma\right)^2}
    & =
    \mathbb{E}\left\Vert \sum_{i=1}^{d}\ve_{i}\ve_{i}^{\top}\mO^{\top}\mSigma^{+}\mSigma\left(\I-\mO\right)\mSigma^{+}\mSigma\vv\right\Vert ^{2}
    \nonumber
    \\
    \explain{\text{Pythagorean theorem}}
    & =
    \sum_{i=1}^{d}
    \mathbb{E}
    \prn{ \ve_{i}^{\top}\mO^{\top}\mSigma^{+}\mSigma\left(\I-\mO\right)\mSigma^{+}\mSigma\vv}^{2}
    \underbrace{\norm{\ve_{i}}^2}_{=1}\,.
    \label{eq:separate-directions}
\end{align}

\bigskip

\begin{remark}[Ease of notation]
\label{rmk:simplify}
    For simplicity, in the equation above and from now on, we omit the explicit subscript notation whenever it is clear from the context of our derivations what the expectation pertains to.
    For instance, instead of writing 
    $\mathbb{E}_\mO \left[F(\mO;\X, \teacher)\right]$ we can simply write
    $\mathbb{E} \left[F(\mO;\X, \teacher)\right]$.

    As another example, we often analyze (for $i\neq j$) expressions of the following spirit,
    \begin{align*}        \mathbb{E}_{\mO}\left(\ve_{i}^{\top}\mO\ve_{j}\right)^{2}
    &=
    \mathbb{E}_{\mathbf{Q}_{p}\sim O\left(p\right),\mathbf{Q}_{m}\sim O\left(m\right)}\left(\ve_{i}^{\top}\mathbf{Q}_{p}\left[\begin{array}{cc}
\mathbf{Q}_{m}\\
 & \mathbf{I}_{p-m}
\end{array}\right]\mathbf{Q}_{p}^{\top}\ve_{j}\right)^{2}
\\
&
=
\mathbb{E}_{\substack{\mathbf{u},\mathbf{v}\sim\mathcal{S}^{p-1}\left(p\right):\mathbf{u\perp v}\\
\mathbf{Q}_{m}\sim O\left(m\right)
}
}\left(\mathbf{u}^{\top}\left[\begin{array}{cc}
\mathbf{Q}_{m}\\
 & \mathbf{I}_{p-m}
\end{array}\right]\mathbf{v}\right)^{2}\,,
\end{align*}
where the last step follows from the angle-preserving property of orthogonal operators ($\Q_p$).
For the sake of simplicity, we take the liberty to also write the expectations above as,
\begin{align*}    \mathbb{E}\left(\ve_{i}^{\top}\mO\ve_{j}\right)^{2}
&=
\mathbb{E}\left(\ve_{i}^{\top}\mathbf{Q}_{p}\left[\begin{array}{cc}
\mathbf{Q}_{m}\\
 & \mathbf{I}_{p-m}
\end{array}\right]\mathbf{Q}_{p}^{\top}\ve_{j}\right)^{2}
=
\mathbb{E}_{\mathbf{u\perp v}}
\left(\mathbf{u}^{\top}\left[\begin{array}{cc}
\mathbf{Q}_{m}\\
 & \mathbf{I}_{p-m}
\end{array}\right]\mathbf{v}\right)^{2}\,.
\end{align*}

Finally, when the dimensions are clear from the context, we interchangeably write matrices in the two following forms:
\hfill
$
\left[\begin{array}{cc}
\mathbf{0}_{m}\\
 & \mathbf{I}_{p-m}
\end{array}\right]
=
\left[\begin{array}{cc}
\mathbf{0}\\
 & \mathbf{I}_{p-m}
\end{array}\right]
\text{, and }
\left[\begin{array}{cc}
\Q_{m}\\
 & \mathbf{0}_{p-m}
\end{array}\right]
=
\left[\begin{array}{cc}
\Q_{m}\\
 & \mathbf{0}
\end{array}\right]\,.
$
\end{remark}

\pagebreak

\paragraph{Back to the proof.}
Focusing on just one term from the above \eqref{eq:separate-directions},
we have,
\begin{align*}
&\mathbb{E}\left(\ve_{i}^{\top}\mO^{\top}\mSigma^{+}\mSigma\left(\I-\mO\right)\mSigma^{+}\mSigma\vv\right)^{2}
=
\mathbb{E}\left(\ve_{i}^{\top}\mO^{\top}\mSigma^{+}\mSigma\left(\I-\mO\right)\sum_{i=1}^{d}\ve_{j}\ve_{j}^{\top}\vv\right)^{2}
 \\
 &
 =
 \mathbb{E}\left(\sum_{j=1}^{d}\left(\ve_{i}^{\top}\mO^{\top}\mSigma^{+}\mSigma\left(\I-\mO\right)\ve_{j}\right)\left(\ve_{j}^{\top}\vv\right)\right)^{2}
 \\
 &=
 \sum_{j=1}^{d}\sum_{k=1}^{d}v_{j}v_{k}\mathbb{E}\left(\ve_{i}^{\top}\mO^{\top}\mSigma^{+}\mSigma\left(\I-\mO\right)\ve_{j}\right)\left(\ve_{i}^{\top}\mO^{\top}\mSigma^{+}\mSigma\left(\I-\mO\right)\ve_{k}\right)
 \\
 &= \sum_{j=1}^{d}
 v_{j}^{2}\mathbb{E}
 \left(\ve_{i}^{\top}\mO^{\top}\mSigma^{+}\mSigma\left(\I-\mO\right)\ve_{j}\right)^{2}
 +
 \\
 &
 \hspace{1cm}
 +
 \sum_{j\neq k=1}^{d}v_{j}v_{k}
 \underbrace{\mathbb{E}\left(\ve_{i}^{\top}\mO^{\top}\mSigma^{+}\mSigma\left(\I-\mO\right)\ve_{j}\right)\left(\ve_{i}^{\top}\mO^{\top}\mSigma^{+}\mSigma\left(\I-\mO\right)\ve_{k}\right)}_{
    =0,\text{ by \lemref{lem:zero_fifth}}
 }
 \\
 &=
 \sum_{j=1}^{d}v_{j}^{2}\mathbb{E}\left(\ve_{i}^{\top}\mO^{\top}\mSigma^{+}\mSigma\left(\I-\mO\right)\ve_{j}\right)^{2}
 \\
 &=
 \underbrace{v_{i}^{2}\mathbb{E}\left(\ve_{i}^{\top}\mO^{\top}\mSigma^{+}\mSigma\left(\I-\mO\right)\ve_{i}\right)^{2}}_{j=i}
 +
 \underbrace{\sum_{j\in\left[d\right]\setminus\left\{ i\right\} }v_{j}^{2}\mathbb{E}\left(\ve_{i}^{\top}\mO^{\top}\mSigma^{+}\mSigma\left(\I-\mO\right)\ve_{j}\right)^{2}}_{j\neq i}
 \\
 &=
 v_{i}^{2} \mathbb{E}\left(\ve_{1}^{\top}\mO^{\top}\mSigma^{+}\mSigma\left(\I-\mO\right)\ve_{1}\right)^{2}+\mathbb{E}\left(\ve_{1}^{\top}\mO^{\top}\mSigma^{+}\mSigma\left(\I-\mO\right)\ve_{2}\right)^{2}\sum_{j\in\left[d\right]\setminus\left\{ i\right\} }v_{j}^{2}\,.
\end{align*}
We are free to use $\ve_1,\ve_2$ instead of $\ve_i,\ve_j$ (for $i,j\in \cnt{d}$)
due to the exchangeability of different rows of $\mQ_p$.
For instance, notice that
\begin{align*}
\mathbb{E}\left(\mathbf{e}_{i}^{\top}\mathbf{O}^{\top}\bm{\mathbf{\Sigma}}^{+}\bm{\mathbf{\Sigma}}\left(\mathbf{I}\!-\!\mathbf{O}\right)\mathbf{e}_{j}\right)^{2}
&=\mathbb{E}\!\left(\sum_{k=1}^{d}\mathbf{e}_{i}^{\top}\mathbf{Q}_{p}
\!
\left[\begin{smallmatrix}
\mathbf{Q}_{m}^{\top}\\
 & \mathbf{I}_{p-m}
\end{smallmatrix}\right]
\!
\mathbf{Q}_{p}^{\top}\mathbf{e}_{k}\mathbf{e}_{k}^{\top}\mathbf{Q}_{p}\left(\mathbf{I}\!-\!\left[\begin{smallmatrix}
\mathbf{Q}_{m}\\
 & \mathbf{I}_{p-m}
\end{smallmatrix}\right]\right)\mathbf{Q}_{p}^{\top}\mathbf{e}_{j}\right)^{\!2}\!.
\end{align*}
That is, the expression is a function of $\mQ_p^\top\ve_1,\dots,\mQ_p^\top\ve_d$ which are the first $d$ rows of the random $\mQ_p$ and are entirely exchangeable (see \propref{prop:invariance}).

\newpage

Going back to \eqref{eq:separate-directions},
\begin{align*}
    &
    \tfrac{1}{\norm{\X}^2}
    \mathbb{E}_{\mO} F(\mO;\X, \teacher)
    \le
    \sum_{i=1}^{d}
    \mathbb{E}
    \prn{ \ve_{i}^{\top}\mO^{\top}\mSigma^{+}\mSigma\left(\I-\mO\right)\mSigma^{+}\mSigma\vv}^{2}
    \\
    &=
    \sum_{i=1}^{d}\left(v_{i}^{2}\mathbb{E}\left(\ve_{1}^{\top}\mO^{\top}\mSigma^{+}\mSigma\left(\mathbf{I}-\mO\right)\ve_{1}\right)^{2}+\mathbb{E}\left(\ve_{1}^{\top}\mO^{\top}\mSigma^{+}\mSigma\left(\mathbf{I}-\mO\right)\ve_{2}\right)^{2}\sum_{j\in\left[d\right]\setminus\left\{ i\right\} }v_{j}^{2}\right)
    \\
    &=
    \mathbb{E}\left(\ve_{1}^{\top}\mO^{\top}\mSigma^{+}\mSigma\left(\mathbf{I}-\mO\right)\ve_{1}\right)^{2}\sum_{i=1}^{d}v_{i}^{2}
    +
    \mathbb{E}\left(\ve_{1}^{\top}\mO^{\top}\mSigma^{+}\mSigma\left(\mathbf{I}-\mO\right)\ve_{2}\right)^{2}\sum_{i=1}^{d}\sum_{j\in\left[d\right]\setminus\left\{ i\right\} }v_{j}^{2}
    \\
    &=
    \mathbb{E}\left(\ve_{1}^{\top}\mO^{\top}\mSigma^{+}\mSigma\left(\mathbf{I}-\mO\right)\ve_{1}\right)^{2}
    \sum_{i=1}^{d}v_{i}^{2}
    +\mathbb{E}\left(\ve_{1}^{\top}\mO^{\top}\mSigma^{+}\mSigma\left(\mathbf{I}-\mO\right)\ve_{2}\right)^{2}\sum_{i=1}^{d}\left(\left(\sum_{j\in\left[d\right]}v_{j}^{2}\right)-v_{i}^{2}\right)
    \\
    &=
    \mathbb{E}\left(\ve_{1}^{\top}\mO^{\top}\mSigma^{+}\mSigma\left(\mathbf{I}-\mO\right)\ve_{1}\right)^{2}
    \sum_{i=1}^{d}v_{i}^{2}
    +\mathbb{E}\left(\ve_{1}^{\top}\mO^{\top}\mSigma^{+}\mSigma\left(\mathbf{I}-\mO\right)\ve_{2}\right)^{2}
    \left(d\sum_{j=1}^{d}v_{j}^{2}
    -
    \sum_{i=1}^{d}v_{i}^{2}\right)
    \\
    &=
    \mathbb{E}\left(\ve_{1}^{\top}\mO^{\top}\mSigma^{+}\mSigma\left(\mathbf{I}-\mO\right)\ve_{1}\right)^{2}
    \sum_{i=1}^{d}v_{i}^{2}
    +\mathbb{E}\left(\ve_{1}^{\top}\mO^{\top}\mSigma^{+}\mSigma\left(\mathbf{I}-\mO\right)\ve_{2}\right)^{2}
    \left(d-1\right)
    \sum_{i=1}^{d}v_{i}^{2}
    \\
    &=
    \biggprn{    \mathbb{E}\left(\ve_{1}^{\top}\mO^{\top}\mSigma^{+}\mSigma\left(\mathbf{I}-\mO\right)\ve_{1}\right)^{2}
    +
    \left(d-1\right)    \mathbb{E}\left(\ve_{1}^{\top}\mO^{\top}\mSigma^{+}\mSigma\left(\mathbf{I}-\mO\right)\ve_{2}\right)^{2}
    }
    \sum_{i=1}^{d}v_{i}^{2}
    \,.
\end{align*}
Interestingly, in this worst-case scenario, 
the direction of $\teacher$ 
that lies in the column span of $\X$,
\ie $\X^{+}\X\teacher=\V\underbrace{\mSigma^{+}\mSigma\V^{\top} \teacher}_{=\vv}$
does not play a role, but only its scale
$\norm{\V\mSigma^{+}\mSigma\V^{\top} \teacher}
=
\norm{\V\vv}
=
\norm{\vv}
=\sum_{i=1}^{d}v_{i}^{2}$
(since $\vv$ is only nonzero in its first $d$ entries).
%
%
Normalizing by this scale, we get,
\begin{align*}
\frac{\mathbb{E}_{\mO} F(\mO;\X, \teacher)}{\norm{\X}^2 \norm{\X^{+}\X\teacher}^2}
\le
\underbrace{
\mathbb{E}\left(\ve_{1}^{\top}\mO^{\top}\mSigma^{+}\mSigma\left(\mathbf{I}-\mO\right)\ve_{1}\right)^{2}
}_{\text{solved in \lemref{lem:at-most-three-vectors}}}
+
\left(d-1\right)
\underbrace{
\mathbb{E}\left(\ve_{1}^{\top}\mO^{\top}\mSigma^{+}\mSigma\left(\mathbf{I}-\mO\right)\ve_{2}\right)^{2}
}_{\text{solved in \lemref{lem:at-most-four-vectors}}}\,,
\end{align*}
where we remind the reader that the inequality saturates when all the nonzero singular values of $\X$ are identical.

\newpage

By plugging in the two lemmata and after some tedious algebraic steps, we get the final bound of,
\begin{align}
\label{eq:full-expression}
\begin{split}
&
\max_{\substack{
\X\in\mathcal{X}_{p,d}
}}
\frac{\mathbb{E}_{\mO} F(\mO;\X, \teacher)}{\tnorm{\X}^2 \tnorm{\X^{+}\X\teacher}^2}
\\
&
=\tfrac{m^{4}p^{2}\left(2+p+p^{2}\right)-2m^{3}p\left(24+10p+13p^{2}+p^{4}\right)+m^{2}p\left(240+230p-15p^{2}+50p^{3}-2p^{4}+p^{5}\right)}{\left(p-3\right)\left(p-2\right)\left(p-1\right)p\left(p+1\right)\left(p+2\right)\left(p+4\right)\left(p+6\right)}+
\\
&
=\tfrac{m\left(p^{6}-28p^{5}+47p^{4}-324p^{3}-60p^{2}-240p-576\right)+2p\left(288+120p-90p^{2}+71p^{3}-6p^{4}+p^{5}\right)}{\left(p-3\right)\left(p-2\right)\left(p-1\right)p\left(p+1\right)\left(p+2\right)\left(p+4\right)\left(p+6\right)}+
\\&
\eqmargin
d\cdot
\tfrac{m^{4}p\left(-7p-2-6p^{2}\right)+m^{3}\left(16p^{4}+18p^{3}+74p^{2}+92p+48\right)+m^{2}\left(-11p^{5}-223p^{3}-267p^{2}-302p-240\right)}{\left(p-3\right)\left(p-2\right)\left(p-1\right)p\left(p+1\right)\left(p+2\right)\left(p+4\right)\left(p+6\right)}
+
\\&
\eqmargin
d\cdot
\tfrac{m\left(2p^{6}-3p^{5}+168p^{4}-33p^{3}+728p^{2}+1124p+192\right)+\left(74p^{4}-240p-576-490p^{3}-252p^{2}-24p^{5}\right)}{\left(p-3\right)\left(p-2\right)\left(p-1\right)p\left(p+1\right)\left(p+2\right)\left(p+4\right)\left(p+6\right)}
+
\\&
\eqmargin
2d^{2}\cdot\tfrac{m^{4}p\left(6+5p\right)-2m^{3}\left(8p^{3}+13p^{2}+9p+18\right)+m^{2}\left(15p^{4}+24p^{3}+106p^{2}+162p+36\right)}{\left(p-3\right)\left(p-2\right)\left(p-1\right)p\left(p+1\right)\left(p+2\right)\left(p+4\right)\left(p+6\right)}
+
\\&
\eqmargin
2d^{2}\cdot\tfrac{m\left(-3p^{5}+3p^{4}-129p^{3}-191p^{2}-198p-144\right)+p\left(288+246p-15p^{2}+26p^{3}-p^{4}\right)}{\left(p-3\right)\left(p-2\right)\left(p-1\right)p\left(p+1\right)\left(p+2\right)\left(p+4\right)\left(p+6\right)}
+
\\&
\eqmargin
d^{3}\cdot\tfrac{m^{4}\left(-5p-6\right)+m^{3}\left(18p^{2}+34p-12\right)+m^{2}\left(-20p^{3}-46p^{2}-39p-42\right)}{\left(p-3\right)\left(p-2\right)\left(p-1\right)p\left(p+1\right)\left(p+2\right)\left(p+4\right)\left(p+6\right)}
+
\\&
\eqmargin
d^{3}\cdot\tfrac{m\left(6p^{4}+10p^{3}+96p^{2}+154p+60\right)+\left(2p^{4}-30p^{3}-32p^{2}-144p-144\right)}{\left(p-3\right)\left(p-2\right)\left(p-1\right)p\left(p+1\right)\left(p+2\right)\left(p+4\right)\left(p+6\right)}\,.
\end{split}
\end{align}

The above is the exact expression for the worst-case forgetting.
To reach the $\bigO$ notation, we assume that $p\gg 1$, and so we are left with the most significant elements of each product. 
That is, we show that,
\begin{align*}
&\approx\tfrac{m^{4}p^{2}\left(p^{2}\right)-2m^{3}p\left(p^{4}\right)+m^{2}p\left(p^{5}\right)+m\left(p^{6}\right)+2p\left(p^{5}\right)}{p^{8}}\!+\!
d\tfrac{m^{4}p\left(-6p^{2}\right)+m^{3}\left(16p^{4}\right)+m^{2}\left(-11p^{5}\right)+m\left(2p^{6}\right)-24p^{5}}{p^{8}}
+
\\
&\eqmargin
2d^{2}\tfrac{m^{4}p\left(5p\right)-2m^{3}\left(8p^{3}\right)+m^{2}\left(15p^{4}\right)+m\left(-3p^{5}\right)+p\left(-p^{4}\right)}{p^{8}}+
\\
&\eqmargin
d^{3}\tfrac{m^{4}\left(-5p\right)+m^{3}\left(18p^{2}\right)+m^{2}\left(-20p^{3}\right)+m\left(6p^{4}\right)+\left(2p^{4}\right)}{p^{8}}
\\&
=\tfrac{m^{4}-2m^{3}p+m^{2}p^{2}+mp^{2}+2p^{2}}{p^{4}}+d\tfrac{-6m^{4}+16m^{3}p-11m^{2}p^{2}+2mp^{3}-24p^{2}}{p^{5}}
+
\\
&\eqmargin
2d^{2}\tfrac{5m^{4}-16m^{3}p+15m^{2}p^{2}-3mp^{3}-p^{3}}{p^{6}}+d^{3}\tfrac{-5m^{4}+18m^{3}-20m^{2}p^{2}+6mp^{3}+2p^{3}}{p^{7}}
\\&
=\tfrac{m^{4}-2m^{3}p+m^{2}p^{2}+mp^{2}+2p^{2}}{p^{4}}+\tfrac{d}{p}\tfrac{-6m^{4}+16m^{3}p-11m^{2}p^{2}+2mp^{3}-24p^{2}}{p^{4}}+
\\
&\eqmargin
2\left(\tfrac{d}{p}\right)^{2}\tfrac{5m^{4}-16m^{3}p+15m^{2}p^{2}-3mp^{3}-p^{3}}{p^{4}}+\left(\tfrac{d}{p}\right)^{3}\tfrac{-5m^{4}+18m^{3}-20m^{2}p^{2}+6mp^{3}+2p^{3}}{p^{4}}
\\&
=\tfrac{m^{4}-2m^{3}p+\left(m^{2}+m+2\right)p^{2}}{p^{4}}+\left(\tfrac{d}{p}\right)\tfrac{-6m^{4}+16m^{3}p-11m^{2}p^{2}+2mp^{3}}{p^{4}}+
\\
&\eqmargin
2\left(\tfrac{d}{p}\right)^{2}\tfrac{5m^{4}-16m^{3}p+15m^{2}p^{2}-\left(3m+1\right)p^{3}}{p^{4}}+\left(\tfrac{d}{p}\right)^{3}\tfrac{-5m^{4}+18m^{3}p-20m^{2}p^{2}+\left(6m+2\right)p^{3}}{p^{4}}
\\&
=\left(\tfrac{m}{p}\right)^{4}\!-\!2\left(\tfrac{m}{p}\right)^{3}+\left(\tfrac{m}{p}\right)^{2}
\!+\!
\left(\tfrac{m}{p}\right)\tfrac{1}{p}
\!+\!
\tfrac{2}{p^{2}}
\!+\!
\left(\tfrac{d}{p}\right)
\!
\left(\!
-6\left(\tfrac{m}{p}\right)^{4}
\!+\!
16\left(\tfrac{m}{p}\right)^{3}\!-\!11\left(\tfrac{m}{p}\right)^{2}\!+\!2\left(\tfrac{m}{p}\right)\right)\!+
\\
&\eqmargin
2\left(\tfrac{d}{p}\right)^{2}\left(5\left(\tfrac{m}{p}\right)^{4}-16\left(\tfrac{m}{p}\right)^{3}+15\left(\tfrac{m}{p}\right)^{2}-3\left(\tfrac{m}{p}\right)-\tfrac{1}{p}\right)+
\\
&\eqmargin
\left(\tfrac{d}{p}\right)^{3}\left(-5\left(\tfrac{m}{p}\right)^{4}+18\left(\tfrac{m}{p}\right)^{3}-20\left(\tfrac{m}{p}\right)^{2}+6\left(\tfrac{m}{p}\right)+\tfrac{2}{p}\right)
\\&
\triangleq\alpha^{4}-2\alpha^{3}+\alpha^{2}+\left(1-\beta\right)\left(-6\alpha^{4}+16\alpha^{3}-11\alpha^{2}+2\alpha\right)+
\\
&\eqmargin
2\left(1-\beta\right)^{2}\left(5\alpha^{4}-16\alpha^{3}+15\alpha^{2}-3\alpha\right)+\left(1-\beta\right)^{3}\left(-5\alpha^{4}+18\alpha^{3}-20\alpha^{2}+6\alpha\right)+\mathcal{O}\left(\tfrac{1}{p}\right)
\\&
=\alpha
\Big(2+\beta^{3}\left(5\alpha^{3}-18\alpha^{2}+20\alpha-6\right)+\beta^{2}\left(-5\alpha^{3}+22\alpha^{2}-30\alpha+12\right)+
\\
&\eqmargin
\beta\left(\alpha^{3}-6\alpha^{2}+11\alpha-8\right)\Big)+
\mathcal{O}\left(\tfrac{1}{p}\right)\,.
\end{align*}
\end{proof}

\newpage

\begin{lemma}
\label{lem:zero_fifth}
Let $i\in{\cnt{d}}$ and let $j,k\in\cnt{d}$ such that $j\neq k$. 
It holds that
$$
\mathbb{E}\left(\ve_{i}^{\top}\mO^{\top}\mSigma^{+}\mSigma\left(\I-\mO\right)\ve_{j}\right)\left(\ve_{i}^{\top}\mO^{\top}\mSigma^{+}\mSigma\left(\I-\mO\right)\ve_{k}\right)
=
0\,.
$$
\end{lemma}

\begin{proof}
The expectation can be decomposed as,
\begin{align*}
&
\mathbb{E}\left(\ve_{i}^{\top}\mO^{\top}\mSigma^{+}\mSigma\left(\I-\mO\right)\ve_{j}\right)\left(\ve_{i}^{\top}\mO^{\top}\mSigma^{+}\mSigma\left(\I-\mO\right)\ve_{k}\right)
\\
&=
\mathbb{E}
\left[
\left(\ve_{i}^{\top}\mO^{\top}\mSigma^{+}\mSigma
\ve_{j}\right)\left(\ve_{i}^{\top}\mO^{\top}\mSigma^{+}\mSigma\ve_{k}\right)
\right]
-
\mathbb{E}
\left[
\left(\ve_{i}^{\top}\mO^{\top}\mSigma^{+}\mSigma
\ve_{j}\right)\left(\ve_{i}^{\top}\mO^{\top}\mSigma^{+}\mSigma\mO\ve_{k}\right)
\right]
\\
&
\eqmargin
-
\mathbb{E}
\left[
\left(\ve_{i}^{\top}\mO^{\top}\mSigma^{+}\mSigma
\mO
\ve_{j}\right)\left(\ve_{i}^{\top}\mO^{\top}\mSigma^{+}\mSigma\ve_{k}\right)
\right]
+
\mathbb{E}
\left[
\left(\ve_{i}^{\top}\mO^{\top}\mSigma^{+}\mSigma
\mO\ve_{j}\right)
\left(\ve_{i}^{\top}\mO^{\top}\mSigma^{+}\mSigma\mO\ve_{k}\right)
\right]
\\
&=
\mathbb{E}
\left[
\left(\ve_{i}^{\top}\mO^{\top}\ve_{j}\right)
\left(\ve_{i}^{\top}\mO^{\top}\ve_{k}\right)
\right]
-
\mathbb{E}
\left[
\left(\ve_{i}^{\top}\mO^{\top}\ve_{j}\right)
\left(\ve_{i}^{\top}\mO^{\top}\mSigma^{+}\mSigma\mO\ve_{k}\right)
\right]
\\
&
\eqmargin
-
\mathbb{E}
\left[
\left(\ve_{i}^{\top}\mO^{\top}\mSigma^{+}\mSigma
\mO\ve_{j}\right)
\left(\ve_{i}^{\top}\mO^{\top}\ve_{k}\right)
\right]
+
\mathbb{E}
\left[
\left(\ve_{i}^{\top}\mO^{\top}\mSigma^{+}\mSigma
\mO\ve_{j}\right)
\left(\ve_{i}^{\top}\mO^{\top}\mSigma^{+}\mSigma\mO\ve_{k}\right)
\right]\,,
\end{align*}
where the last step holds because $j,k\in\cnt{d}$ and therefore
$\mSigma^{+}\mSigma\ve_{j}=\ve_{j},\,
\mSigma^{+}\mSigma\ve_{k}=\ve_{k}$.

\medskip

Following the definition of $\mO$ (\eqref{eq:data-model}), the first expectation becomes
\begin{align*}
&\mathbb{E}
\left[
\left(\ve_{i}^{\top}\mO^{\top}\ve_{j}\right)
\left(\ve_{i}^{\top}\mO^{\top}\ve_{k}\right)
\right]
=
\mathbb{E}_{\Q_p,\Q_m}
\bigg[
\ve_{j}^{\top}
\Q_{p}\left[\begin{smallmatrix}
\Q_{m}\\
 & \I_{p-m}
\end{smallmatrix}\right]
\Q_{p}^{\top}
\ve_{i}
\cdot
\ve_{k}^{\top}
\Q_{p}
\underbrace{
\left[\begin{smallmatrix}
\Q_{m}\\
 & \I_{p-m}
\end{smallmatrix}\right]
}_{\triangleq \A}
\Q_{p}^{\top}
\ve_{i}
\bigg]\,.
\end{align*}
Since $j\neq k$, we must have either $j\notin \left\{i,k\right\}$ or 
$k\notin \left\{i,j\right\}$ (or both).
Denote the relevant rows of $\Q_p$ by 
${\vq_i \triangleq\Q_{p}^{\top}\ve_{i}}$,
${\vq_j \triangleq\Q_{p}^{\top}\ve_{j}}$,
${\vq_k \triangleq\Q_{p}^{\top}\ve_{k}}$ 
and notice that they are independent of $\Q_m$ (or $\A$).
%
Without loss of generality, 
$j\notin \left\{i,k\right\}$.
The above expectation becomes
$\mathbb{E}_{\Q_p , \Q_m}
\big[
\vq_{j}^{\top}
\A
\vq_{i}
\cdot
\vq_{k}^{\top}
\A
\vq_{i}
\big]$,
where $\vq_j$ appears only once (an odd number).
By \corref{cor:odd-column-or-row}, the expectation vanishes.

Quite similarly, the second expectation becomes
\begin{align*}
\mathbb{E}
\left[
\left(\ve_{i}^{\top}\mO^{\top}
\ve_{j}\right)\left(\ve_{i}^{\top}\mO^{\top}\mSigma^{+}\mSigma\mO\ve_{k}\right)
\right]
&=
\sum_{t=1}^{d}
\mathbb{E}
\left[
\ve_{i}^{\top}\mO^{\top}
\ve_{j}\cdot
\ve_{i}^{\top}\mO^{\top}
\ve_{t}\cdot
\ve_{t}^{\top}\mO\ve_{k}
\right]
\\
&=
\sum_{t=1}^{d}
\mathbb{E}
\left[
\vq_{j}^{\top}\A\vq_{i}
\cdot
\vq_{t}^{\top}\A\vq_{i}
\cdot
\vq_{t}^{\top}\A\vq_{k}
\right]\,.
\end{align*}
Notice that both $i,t$ appear an \emph{even} number of times in (each of) the above expectation(s).
\linebreak
Since $j\neq k$,
at least one out of $j,k$ appears an odd number of times 
(either one, three, or five)
in each of the above expectations.
Again, by \corref{cor:odd-column-or-row}, this expectation vanishes.
Clearly, the same holds for the third expectation.

The fourth expectation is,
\begin{align*}
\mathbb{E}
\left[
\left(\ve_{i}^{\top}\mO^{\top}\mSigma^{+}\mSigma
\mO\ve_{j}\right)
\left(\ve_{i}^{\top}\mO^{\top}\mSigma^{+}\mSigma\mO\ve_{k}\right)
\right]
&
=
\sum_{\ell,t=1}^{d}
\mathbb{E}
\left[
\left(\ve_{i}^{\top}\mO^{\top}
\ve_{\ell}\ve_{\ell}^{\top}
\mO\ve_{j}\right)
\left(\ve_{i}^{\top}\mO^{\top}
\ve_{t}\ve_{t}^{\top}
\mO\ve_{k}\right)
\right]
\\
&=
\sum_{\ell,t=1}^{d}
\mathbb{E}
\left[
\vq_{\ell}^{\top}\A\vq_{i}
\cdot
\vq_{\ell}^{\top}\A\vq_{j}
\cdot
\vq_{t}^{\top}\A\vq_{i}
\cdot
\vq_{t}^{\top}\A\vq_{k}
\right]\,.
\end{align*}
Notice that $i,\ell, t$ appear an \emph{even} number of times in (each of) the above expectation(s).
\linebreak
Since $j\neq k$,
at least one out of $j,k$ appears an odd number of times 
(either one, three, five, or seven)
in each of the above expectations.
Again, by \corref{cor:odd-column-or-row}, this expectation vanishes.

\end{proof}

\newpage

\subsection{Deriving $\mathbb{E}\left(\ve_{i}^{\top}\mO^{\top}\mSigma^{+}\mSigma\left(\mathbf{I}-\mO\right)\ve_{i}\right)^{2}$}

\begin{lemma}
\label{lem:at-most-three-vectors}
Let $p\ge 4, m\in\{2,\dots,p\}, {d\in\cnt{p}}$,
and let $\mO$ be a random transformation sampled as described in \eqref{eq:data-model}.
Then, $\forall i\in\cnt{d}$, it holds that
\begin{align*}
 &
\mathbb{E}\left(\ve_{i}^{\top}\mO^{\top}\mSigma^{+}\mSigma\left(\mathbf{I}-\mO\right)\ve_{i}\right)^{2}
\\
&
=\tfrac{m^{4}\left(p^{2}+2p\right)-2m^{3}p^{3}+m^{2}\left(p^{4}-4p^{3}+20p^{2}-24\right)+m\left(3p^{4}-10p^{3}+63p^{2}+6p-72\right)+2\left(p+1\right)\left(p^{3}-11p^{2}+38p-24\right)}{\left(p-1\right)p\left(p+1\right)\left(p+2\right)\left(p+4\right)\left(p+6\right)}+
\\&
\eqmargin\tfrac{d\left(p+1\right)\left(-2m^{4}+6m^{3}p-2m^{2}\left(p-1\right)\left(2p-5\right)-12m\left(p^{2}-3p+3\right)-8\left(p^{2}-8p+6\right)\right)}{\left(p-1\right)p\left(p+1\right)\left(p+2\right)\left(p+4\right)\left(p+6\right)}+
\\&
\eqmargin\tfrac{d^{2}\left(4mp\left(-m^{2}+m+4\right)+4\left(m+1\right)\left(m+2\right)p^{2}+\left(m-6\right)\left(m-1\right)m\left(m+1\right)-8\left(2p+3\right)\right)}{\left(p-1\right)p\left(p+1\right)\left(p+2\right)\left(p+4\right)\left(p+6\right)}    \end{align*}
\end{lemma}

\begin{proof}
    We decompose the expectation as,
    \begin{align*}        &\mathbb{E}\left(\ve_{i}^{\top}\mO^{\top}\bm{\mathbf{\Sigma}}^{+}\bm{\mathbf{\Sigma}}\left(\mathbf{I}-\mO\right)\ve_{i}\right)^{2}
    \\
    &=
    \mathbb{E}\left(\ve_{i}^{\top}\mO\ve_{i}\right)^{2}-2\mathbb{E}\left[\left(\ve_{i}^{\top}\mO\ve_{i}\right)\left(\ve_{i}^{\top}\mO^{\top}\bm{\mathbf{\Sigma}}^{+}\bm{\mathbf{\Sigma}}\mO\ve_{i}\right)\right]+\mathbb{E}\left(\ve_{i}^{\top}\mO^{\top}\bm{\mathbf{\Sigma}}^{+}\bm{\mathbf{\Sigma}}\mO\ve_{i}\right)^{2}\,,
    \end{align*}
    and derive each of its three terms separately in the following subsections.

    By adding these three terms and by simple algebra, we get the required result,
\begin{align*}
    &=\tfrac{m^{2}-2mp-m+p^{2}+2p+2}{p\left(p+2\right)}-2\tfrac{\left(p-m\right)\left(d\left(-m^{2}+2mp+3m-6\right)+m^{2}p-2mp^{2}-3mp+p^{3}+5p^{2}+8p-8\right)}{\left(p-1\right)p\left(p+2\right)\left(p+4\right)}+
\\&
\eqmargin\tfrac{3\left(m+4\right)\left(m+6\right)+\left(p-m\right)\left(p-m+2\right)\left(m^{2}-2mp-4m+p^{2}+10p+36\right)}{p\left(p+2\right)\left(p+4\right)\left(p+6\right)}+
\\&
\eqmargin\tfrac{\left(d-1\right)\left(-72+m^{4}\left(-1-2p\right)-72p-20p^{2}-20p^{3}+m^{3}\left(6+16p+8p^{2}\right)+m^{2}\left(-47-70p-34p^{2}-10p^{3}\right)\right)}{\left(p-1\right)p\left(p+1\right)\left(p+2\right)\left(p+4\right)\left(p+6\right)}+
\\&
\eqmargin\tfrac{\left(d-1\right)\left(m\left(42+136p+114p^{2}+22p^{3}+4p^{4}\right)\right)}{\left(p-1\right)p\left(p+1\right)\left(p+2\right)\left(p+4\right)\left(p+6\right)}+
\\&
\eqmargin\tfrac{\left(d-1\right)d\left(4mp\left(-m^{2}+m+4\right)+4\left(m+1\right)\left(m+2\right)p^{2}+\left(m-6\right)\left(m-1\right)m\left(m+1\right)-8\left(2p+3\right)\right)}{\left(p-1\right)p\left(p+1\right)\left(p+2\right)\left(p+4\right)\left(p+6\right)}
\\&
=\tfrac{m^{4}\left(p^{2}+2p\right)-2m^{3}p^{3}+m^{2}\left(p^{4}-4p^{3}+20p^{2}-24\right)+m\left(3p^{4}-10p^{3}+63p^{2}+6p-72\right)+2\left(p+1\right)\left(p^{3}-11p^{2}+38p-24\right)}{\left(p-1\right)p\left(p+1\right)\left(p+2\right)\left(p+4\right)\left(p+6\right)}+
\\&
\eqmargin\tfrac{d\left(p+1\right)\left(-2m^{4}+6m^{3}p-2m^{2}\left(p-1\right)\left(2p-5\right)-12m\left(p^{2}-3p+3\right)-8\left(p^{2}-8p+6\right)\right)}{\left(p-1\right)p\left(p+1\right)\left(p+2\right)\left(p+4\right)\left(p+6\right)}+
\\&
\eqmargin\tfrac{d^{2}\left(4mp\left(-m^{2}+m+4\right)+4\left(m+1\right)\left(m+2\right)p^{2}+\left(m-6\right)\left(m-1\right)m\left(m+1\right)-8\left(2p+3\right)\right)}{\left(p-1\right)p\left(p+1\right)\left(p+2\right)\left(p+4\right)\left(p+6\right)}
\end{align*}
\end{proof}

\bigskip

\begin{remark}[Explaining proof techniques]
In this subsection, we explain the proof steps more thoroughly than in other places, since most of the techniques repeat themselves throughout the appendices.
\end{remark}

\newpage

\subsubsection{Term 1: $\mathbb{E}\left(\ve_{i}^{\top}\mO\ve_{i}\right)^{2}
$}

Recalling \rmkref{rmk:simplify} on our simplified notations, we show that,
\begin{align*}
&\mathbb{E}\left(\ve_{i}^{\top}\mO\ve_{\text{i}}\right)^{2}
=\mathbb{E}\left(\ve_{i}^{\top}\mathbf{Q}_{p}\left[\begin{array}{cc}
\mathbf{Q}_{m}\\
 & \mathbf{I}_{p-m}
\end{array}\right]\mathbf{Q}_{p}^{\top}\ve_{i}\right)^{2}
=
\mathbb{E}_{\mathbf{u}\sim\mathcal{S}^{p-1}}\left(\mathbf{u}^{\top}\left[\begin{array}{cc}
\mathbf{Q}_{m}\\
 & \mathbf{I}_{p-m}
\end{array}\right]\mathbf{u}\right)^{2}
\\
&=
\mathbb{E}\left[\mathbf{u}^{\top}\left(\left[\begin{array}{cc}
\mathbf{Q}_{m}\\
 & \mathbf{0}
\end{array}\right]+\left[\begin{array}{cc}
\mathbf{0}\\
 & \mathbf{I}_{p-m}
\end{array}\right]\right)
\mathbf{u}\cdot\mathbf{u}^{\top}\left(\left[\begin{array}{cc}
\mathbf{Q}_{m}\\
 & \mathbf{0}
\end{array}\right]+\left[\begin{array}{cc}
\mathbf{0}\\
 & \mathbf{I}_{p-m}
\end{array}\right]\right)\mathbf{u}\right]\,.
\end{align*}
Opening the product above, 
by \corref{cor:odd-Q_m} we are only left with the following terms:
\begin{align*}
&=
\mathbb{E}\left[\mathbf{u}^{\top}
\left[\begin{array}{cc}
\mathbf{0}\\
 & \mathbf{I}_{p-m}
\end{array}\right]
\mathbf{u}\mathbf{u}^{\top}
\left[\begin{array}{cc}
\mathbf{0}\\
 & \mathbf{I}_{p-m}
\end{array}\right]\mathbf{u}\right]
+\mathbb{E}\left[\mathbf{u}^{\top}
\left[\begin{array}{cc}
\mathbf{Q}_{m}\\
 & \mathbf{0}
\end{array}\right]\mathbf{u}\mathbf{u}^{\top}\left[\begin{array}{cc}
\mathbf{Q}_{m}\\
 & \mathbf{0}
\end{array}\right]\mathbf{u}\right]
\\
&
=\mathbb{E}\left[\left(\mathbf{u}_{b}^{\top}\mathbf{u}_{b}\right)^{2}\right]+\mathbb{E}\left[\mathbf{u}_{a}^{\top}\mathbf{Q}_{m}\mathbf{u}_{a}\mathbf{u}_{a}^{\top}\mathbf{Q}_{m}\mathbf{u}_{a}\right]
\,,
\end{align*}
where, like we frequently do throughout the appendix, we decomposed $\vu$ into
${\vu=
\left[\begin{array}{c}
\vu_a \\
\vu_b 
\end{array}\right]\in\reals^{p}}$
with $\vu_a\in\reals^{m}$ and
$\vu_b \in\reals^{p-m}$.
This decomposition is often useful, since for two orthogonal unit vectors $\vu,\vv$, it holds that
${0=\vu^{\top}\vv=
\vu_{a}^{\top}\vv_{a}
+
\vu_{b}^{\top}\vv_{b}
\Longrightarrow
\vu_{a}^{\top}\vv_{a}
= -\vu_{b}^{\top}\vv_{b}}
$
and
${1=\norm{\vu}^2=
\norm{\vu_a}^2+\norm{\vu_b}^2
\Longrightarrow
\norm{\vu_a}^2 = 1-\norm{\vu_b}^2}
$.

Another ``trick'' that we use often, is to reparameterize $\Q_m \vu_a$ (for $\Q_m\sim O(m)$) as $\norm{\vu_a} \vr$ for $\vr\sim\mathcal{S}^{m-1}$.
Then, the expectation above becomes,
\begin{align*}
&=
\mathbb{E}\left[\left\Vert \mathbf{u}_{b}\right\Vert ^{4}\right]+\mathbb{E}\left[\left\Vert \mathbf{u}_{a}\right\Vert ^{2}\mathbf{u}_{a}^{\top}\left(\frac{1}{m}\mathbf{I}_{m}\right)\mathbf{u}_{a}\right]=\mathbb{E}\left[\left\Vert \mathbf{u}_{b}\right\Vert ^{4}\right]+\frac{1}{m}\mathbb{E}\left[\left\Vert \mathbf{u}_{a}\right\Vert ^{4}\right]\\&=\sum_{i=m+1}^{p}\sum_{j=m+1}^{p}\mathbb{E}\left[u_{i}^{2}u_{j}^{2}\right]+\frac{1}{m}\sum_{i=1}^{m}\sum_{j=1}^{m}\mathbb{E}\left[u_{i}^{2}u_{j}^{2}\right]
\\&
=\left(p-m\right)\mathbb{E}\left[u_{p}^{4}\right]+\left(p-m\right)\left(p-m-1\right)\mathbb{E}\left[u_{p-1}^{2}u_{p}^{2}\right]+\frac{1}{m}\left(m\mathbb{E}\left[u_{1}^{4}\right]+m\left(m-1\right)\mathbb{E}\left[u_{1}^{2}u_{2}^{2}\right]\right)
\\&
=\left(p-m\right)\mathbb{E}\left[u_{1}^{4}\right]+\left(p-m\right)\left(p-m-1\right)\mathbb{E}\left[u_{1}^{2}u_{2}^{2}\right]+\mathbb{E}\left[u_{1}^{4}\right]+\left(m-1\right)\mathbb{E}\left[u_{1}^{2}u_{2}^{2}\right]\,,
\end{align*}
where in the last step we used the fact that different entries of $\vu$ are identically distributed (see also \propref{prop:invariance}).
Using simple algebraic steps and employing the notations presented in 
\appref{app:monomials} for monomials of entries of random orthogonal matrices,
we have
\begin{align*}
&=\left(p-m+1\right)\mathbb{E}\left[u_{1}^{4}\right]+\left(m^{2}-2mp+2m+p^{2}-p-1\right)\mathbb{E}\left[u_{1}^{2}u_{2}^{2}\right]\\&=\left(p-m+1\right)\left\langle \begin{array}{c}
4\\
\overrightarrow{0}
\end{array}\right\rangle +\left(m^{2}-2mp+2m+p^{2}-p-1\right)\left\langle \begin{array}{c}
2\\
2\\
\overrightarrow{0}
\end{array}\right\rangle \,.
\end{align*}
Finally, we plug in the expectations (computed in \appref{app:monomials}), and get
\begin{align*}
    &=\frac{3\left(p-m+1\right)}{p\left(p+2\right)}+\frac{m^{2}-2mp+2m+p^{2}-p-1}{p\left(p+2\right)}
    =\frac{m^{2}-2mp-m+p^{2}+2p+2}{p\left(p+2\right)}\,.
\end{align*}

\newpage

\subsubsection{Term 2: $\mathbb{E}\left[\left(\ve_{i}^{\top}\mO\ve_{i}\right)\left(\ve_{i}^{\top}\mO^{\top}\mSigma^{+}\mSigma\mO\ve_{i}\right)\right]$}

It holds that,
\begin{align*}
&
\mathbb{E}\left[\left(\ve_{i}^{\top}\mO\ve_{i}\right)\left(\ve_{i}^{\top}\mO^{\top}\mSigma^{+}\mSigma\mO\ve_{i}\right)\right]
=
\mathbb{E}\left[\left(\ve_{1}^{\top}\mO\ve_{1}\right)\left(\ve_{1}^{\top}\mO^{\top}\mSigma^{+}\mSigma\mO\ve_{1}\right)\right]
\\
&
=
\mathbb{E}\left[\ve_{1}^{\top}\mO\ve_{1}\ve_{1}^{\top}\mO^{\top}\sum_{k=1}^{d}\ve_{k}\ve_{k}^{\top}\mO\ve_{1}\right]
=
\sum_{k=1}^{d}\mathbb{E}\left[\ve_{1}^{\top}\mO\ve_{1}\ve_{1}^{\top}\mO^{\top}\ve_{k}\ve_{k}^{\top}\mO\ve_{1}\right]
\\
&=\sum_{k=1}^{d}\mathbb{E}\left[\left(\ve_{k}^{\top}\mO\ve_{1}\right)^{2}\ve_{1}^{\top}\mO\ve_{1}\right]=\underbrace{\mathbb{E}\left[\left(\ve_{1}^{\top}\mO\ve_{1}\right)^{3}\right]}_{\text{solved in \propref{prop:(e1Oe1)3}}}+\left(d-1\right)\underbrace{\mathbb{E}\left[\left(\ve_{2}^{\top}\mO\ve_{1}\right)^{2}\ve_{1}^{\top}\mO\ve_{1}\right]}_{\text{solved in \propref{prop:(e2Oe1)2_e2Oe_2}}}
\\
&=
\frac{\left(p-m\right)\left(m^{2}-2mp-3m+p^{2}+6p+14\right)}{p\left(p+2\right)\left(p+4\right)}+\left(d-1\right)\frac{\left(p-m\right)\left(-m^{2}+2mp+3m-6\right)}{\left(p-1\right)p\left(p+2\right)\left(p+4\right)}
\\
&=
\frac{\left(p-m\right)\left(\left(p-1\right)\left(m^{2}-2mp-3m+p^{2}+6p+14\right)+\left(d-1\right)\left(-m^{2}+2mp+3m-6\right)\right)}{\left(p-1\right)p\left(p+2\right)\left(p+4\right)}
\\
&=
\frac{\left(p-m\right)\left(d\left(-m^{2}+2mp+3m-6\right)+m^{2}p-2mp^{2}-3mp+p^{3}+5p^{2}+8p-8\right)}{\left(p-1\right)p\left(p+2\right)\left(p+4\right)}
\end{align*}

\subsubsection{Term 3: $\mathbb{E}\left(\mathbf{e}_{i}^{\top}\mathbf{O}^{\top}\mSigma^{+}\mSigma\mathbf{O}\mathbf{e}_{i}\right)^{2}$}

It holds that,
\begin{align*}
&\mathbb{E}\!\left(\mathbf{e}_{i}^{\top}\mathbf{O}^{\top}\bm{\mathbf{\Sigma}}^{+}\bm{\mathbf{\Sigma}}\mathbf{O}\mathbf{e}_{i}\right)^{2}
\!\!
=
\mathbb{E}\left(\mathbf{e}_{1}^{\top}\mathbf{O}^{\top}\bm{\mathbf{\Sigma}}^{+}\bm{\mathbf{\Sigma}}\mathbf{O}\mathbf{e}_{1}\right)^{2}
\!\!
=
\mathbb{E}\!\left(
\!\mathbf{e}_{1}^{\top}
\mathbf{O}^{\top}\!\sum_{k=1}^{d}\mathbf{e}_{k}\mathbf{e}_{k}^{\top}\mathbf{O}\mathbf{e}_{1}
\!\right)^{\!2}
\!\!=\!
\mathbb{E}\!
\left(
\sum_{k=1}^{d}
\!\left(\mathbf{e}_{k}^{\top}\mathbf{O}\mathbf{e}_{1}
\!\right)^{2}\right)^{\!2}
\\&
=\sum_{k=1}^{d}\sum_{\ell=1}^{d}\mathbb{E}\left(\mathbf{e}_{k}^{\top}\mathbf{O}\mathbf{e}_{1}\right)^{2}\left(\mathbf{e}_{\ell}^{\top}\mathbf{O}\mathbf{e}_{1}\right)^{2}
=\underbrace{\sum_{k=1}^{d}\mathbb{E}\left(\mathbf{e}_{k}^{\top}\mathbf{O}\mathbf{e}_{1}\right)^{4}}_{k=\ell}+\underbrace{\sum_{k\neq\ell=1}^{d}\mathbb{E}\left(\mathbf{e}_{k}^{\top}\mathbf{O}\mathbf{e}_{1}\right)^{2}\left(\mathbf{e}_{\ell}^{\top}\mathbf{O}\mathbf{e}_{1}\right)^{2}}_{k\neq\ell}
\\&
=\underbrace{\underbrace{\mathbb{E}\left(\mathbf{e}_{1}^{\top}\mathbf{O}\mathbf{e}_{1}\right)^{4}}_{k=1\text{, solved in \propref{prop:(e1Oe1)4}}}
+
\underbrace{\left(d-1\right)\mathbb{E}\left(\mathbf{e}_{2}^{\top}\mathbf{O}\mathbf{e}_{1}\right)^{4}}_{k\ge2\text{, solved in \propref{prop:(e1Oe2)4}}}}_{k=\ell}
+
\\
&
\eqmargin
\underbrace{\underbrace{2\left(d-1\right)\mathbb{E}\left(\mathbf{e}_{1}^{\top}\mathbf{O}\mathbf{e}_{1}\right)^{2}\left(\mathbf{e}_{2}^{\top}\mathbf{O}\mathbf{e}_{1}\right)^{2}}_{\ell\neq k=1\,\vee\,k\neq\ell=1\text{, solved in \propref{prop:e2Oe2*e2Oe1}}}+
\underbrace{\left(d-1\right)\left(d-2\right)\mathbb{E}\left(\mathbf{e}_{2}^{\top}\mathbf{O}\mathbf{e}_{1}\right)^{2}\left(\mathbf{e}_{3}^{\top}\mathbf{O}\mathbf{e}_{1}\right)^{2}}_{k,\ell\ge2,\,k\neq\ell\text{ solved in \propref{prop:(e2Oe3*e2Oe1)^2}}}}_{k\neq\ell}    
\end{align*}
We note in passing that since $d\le p$,
then when $p=2\Rightarrow d\le 2$, the rightmost term is necessarily zero.
Therefore, we can use \propref{prop:(e2Oe3*e2Oe1)^2} freely $\forall p\ge 2$, even though it requires that $p\ge3$.
\begin{align*}
\mathbb{E}\!
\left(\mathbf{e}_{i}^{\top}\mathbf{O}^{\top}\bm{\mathbf{\Sigma}}^{+}\bm{\mathbf{\Sigma}}\mathbf{O}\mathbf{e}_{i}\right)^{2}
\!
&=
\tfrac{3\left(m+4\right)\left(m+6\right)+\left(p-m\right)\left(p-m+2\right)\left(m^{2}-2mp-4m+p^{2}+10p+36\right)}{p\left(p+2\right)\left(p+4\right)\left(p+6\right)}+
\\
&\eqmargin
\tfrac{3\left(d-1\right)\left(m^{4}-2m^{3}\left(2p+3\right)+m^{2}\left(4p^{2}+4p-1\right)+2m\left(6p^{2}+8p+3\right)+8\left(p^{2}-2p-3\right)\right)}{\left(p-1\right)p\left(p+1\right)\left(p+2\right)\left(p+4\right)\left(p+6\right)}+
\\&
\eqmargin 2\left(d-1\right)\tfrac{\left(m+4\right)\left(2mp+4p+m-m^{2}-6\right)-\left(p-m\right)\left(p-m+2\right)\left(m\left(m-2p-5\right)+10\right)}{\left(p-1\right)p\left(p+2\right)\left(p+4\right)\left(p+6\right)}
+
\\
&\eqmargin
\tfrac{\left(d-1\right)\left(d-2\right)\left(4mp\left(-m^{2}+m+4\right)+4\left(m+1\right)\left(m+2\right)p^{2}+\left(m-6\right)\left(m-1\right)m\left(m+1\right)-8\left(2p+3\right)\right)}{\left(p-1\right)p\left(p+1\right)\left(p+2\right)\left(p+4\right)\left(p+6\right)}
\\
&
=
\tfrac{3\left(m+4\right)\left(m+6\right)+\left(p-m\right)\left(p-m+2\right)\left(m^{2}-2mp-4m+p^{2}+10p+36\right)}{p\left(p+2\right)\left(p+4\right)\left(p+6\right)}
+
\\
&\eqmargin
\tfrac{\left(d-1\right)\left(-72+m^{4}\left(-1-2p\right)-72p-20p^{2}-20p^{3}+m^{3}\left(6+16p+8p^{2}\right)\right)}{\left(p-1\right)p\left(p+1\right)\left(p+2\right)\left(p+4\right)\left(p+6\right)}
+
\\
&\eqmargin
\tfrac{\left(d-1\right)\left(m^{2}\left(-47-70p-34p^{2}-10p^{3}\right)+m\left(42+136p+114p^{2}+22p^{3}+4p^{4}\right)\right)}{\left(p-1\right)p\left(p+1\right)\left(p+2\right)\left(p+4\right)\left(p+6\right)}
+
\\
&\eqmargin
\tfrac{\left(d-1\right)d\left(4mp\left(-m^{2}+m+4\right)+4\left(m+1\right)\left(m+2\right)p^{2}+\left(m-6\right)\left(m-1\right)m\left(m+1\right)-8\left(2p+3\right)\right)}{\left(p-1\right)p\left(p+1\right)\left(p+2\right)\left(p+4\right)\left(p+6\right)}
\end{align*}

\newpage

\subsection{Deriving $\mathbb{E}\left(\ve_{i}^{\top}\mO^{\top}\mSigma^{+}\mSigma\left(\mathbf{I}-\mO\right)\ve_{j}\right)^{2}$}

\begin{lemma}
\label{lem:at-most-four-vectors}
Let $p\ge 4, m\in\{2,\dots,p\}, d\in\cnt{p}$,
and let $\mO$ be a random transformation sampled as described in \eqref{eq:data-model}.
Then, $\forall i,j\in\cnt{d}$ such that $i\neq j$, it holds that
    \begin{align*}
    &\mathbb{E}
    \left(\ve_{i}^{\top}\mO^{\top}\mSigma^{+}\mSigma\left(\mathbf{I}-\mO\right)\ve_{j}\right)^{2}
    \\
    &
    =
    \tfrac{\left(-4p^{3}-6p^{2}+12p\right)m^{4}+
    m^{3}\left(10p^{4}+14p^{3}+20p^{2}+48p\right)+
    m^{2}\left(-7p^{5}-4p^{4}-109p^{3}-134p^{2}-120p-144\right)
    }{\left(p-3\right)\left(p-2\right)\left(p-1\right)p\left(p+1\right)\left(p+2\right)\left(p+4\right)\left(p+6\right)}
    +
    \\
    &
    \eqmargin
    \tfrac{m\left(2p^{6}+3p^{5}+84p^{4}-45p^{3}+336p^{2}+636p+144\right)+\left(-18p^{5}+24p^{4}-182p^{3}-104p^{2}-168p-288\right)}{\left(p-3\right)\left(p-2\right)\left(p-1\right)p\left(p+1\right)\left(p+2\right)\left(p+4\right)\left(p+6\right)}
    +
    \\
    &
    \eqmargin
    d\cdot
    \tfrac{3m^{4}\left(-4+4p+3p^{2}\right)-4m^{3}\left(p+2\right)\left(7p^{2}-2p+6\right)+m^{2}\left(26p^{4}+44p^{3}+163p^{2}+256p+36\right)}{\left(p-3\right)\left(p-2\right)\left(p-1\right)p\left(p+1\right)\left(p+2\right)\left(p+4\right)\left(p+6\right)}
    +
    \\
    &
    \eqmargin
    d\cdot
    \tfrac{-2m\left(3p^{5}+102p^{3}+142p^{2}+154p+132\right)-2p\left(p+1\right)\left(p^{3}-24p^{2}+26p-204\right)}{\left(p-3\right)\left(p-2\right)\left(p-1\right)p\left(p+1\right)\left(p+2\right)\left(p+4\right)\left(p+6\right)}
    +
    \\
    &
    \eqmargin
    d^{2}\cdot
    \tfrac{m^{4}\left(-5p-6\right)+m^{3}\left(18p^{2}+34p-12\right)+m^{2}\left(-20p^{3}-46p^{2}-39p-42\right)}{\left(p-3\right)\left(p-2\right)\left(p-1\right)p\left(p+1\right)\left(p+2\right)\left(p+4\right)\left(p+6\right)}
    +
    \\
    &
    \eqmargin
    d^{2}\cdot
    \tfrac{m\left(6p^{4}+10p^{3}+96p^{2}+154p+60\right)+
    \left(2p^{4}-30p^{3}-32p^{2}-144p-144\right)}{\left(p-3\right)\left(p-2\right)\left(p-1\right)p\left(p+1\right)\left(p+2\right)\left(p+4\right)\left(p+6\right)}
    \,.
    \end{align*}
\end{lemma}

\begin{proof}
    We decompose the expectation as,
    \begin{align*}        &\mathbb{E}\left(\ve_{i}^{\top}\mO^{\top}\bm{\mathbf{\Sigma}}^{+}\bm{\mathbf{\Sigma}}\left(\mathbf{I}-\mO\right)\ve_{j}\right)^{2}
    \\
    &=
    \mathbb{E}\left(\ve_{i}^{\top}\mO\ve_{j}\right)^{2}
    -
    2\mathbb{E}\left[\left(\ve_{j}^{\top}\mO\ve_{i}\right)\left(\ve_{i}^{\top}\mO^{\top}\bm{\mathbf{\Sigma}}^{+}\bm{\mathbf{\Sigma}}\mO\ve_{j}\right)\right]
    +
    \mathbb{E}\left(\ve_{i}^{\top}\mO^{\top}\bm{\mathbf{\Sigma}}^{+}\bm{\mathbf{\Sigma}}\mO\ve_{j}\right)^{2}\,,
    \end{align*}
and derive each of its three terms separately in the following subsections.
The final result in the lemma is obtained by summing these three terms.
\end{proof}

\bigskip

\subsubsection{Term 1: $\mathbb{E}\left(\ve_{i}^{\top}\mO\ve_{j}\right)^{2}$}

It holds that,
\begin{align*}
\mathbb{E}\left(\mathbf{e}_{i}^{\top}\mathbf{O}\mathbf{e}_{j}\right)^{2}=&\mathbb{E}\left(\mathbf{u}^{\top}\left(\left[\begin{smallmatrix}
\mathbf{Q}_{m}\\
& \mathbf{0}_{p-m}
\end{smallmatrix}\right]+\left[\begin{smallmatrix}
\mathbf{0}_{m}\\
& \mathbf{I}_{p-m}
\end{smallmatrix}\right]\right)\mathbf{v}\right)^{2}
=
\mathbb{E}\left(\mathbf{u}_{a}^{\top}\mathbf{Q}_{m}\mathbf{v}_{a}+\mathbf{u}_{b}^{\top}\mathbf{v}_{b}\right)^{2}
\\&
=\underbrace{\mathbb{E}\left(\mathbf{u}_{a}^{\top}\mathbf{Q}_{m}\mathbf{v}_{a}\right)^{2}}_{\text{solved in \eqref{eq:(uaQmva)^2}}}
+
\underbrace{2\mathbb{E}\left(\mathbf{u}_{a}^{\top}\mathbf{Q}_{m}\mathbf{v}_{a}\mathbf{u}_{b}^{\top}\mathbf{v}_{b}\right)}_{=0\text{, by \corref{cor:odd-Q_m}}}
+
\underbrace{\mathbb{E}\left(\mathbf{u}_{b}^{\top}\mathbf{v}_{b}\right)^{2}}_{\text{solved in \eqref{eq:(uava)2}}}
\\&
=\frac{mp+m-2}{\left(p-1\right)p\left(p+2\right)}+\frac{m\left(p-m\right)}{\left(p-1\right)p\left(p+2\right)}=\frac{mp+m-2+m\left(p-m\right)}{\left(p-1\right)p\left(p+2\right)}
\\&
=
\frac{2mp+m-m^{2}-2}{\left(p-1\right)p\left(p+2\right)}
\end{align*}

\newpage

\subsubsection{Term 2: $2\left(\mathbf{e}_{j}^{\top}\mathbf{O}\mathbf{e}_{i}\right)
\left(\mathbf{e}_{i}^{\top}\mathbf{O}^{\top}\mSigma^{+}\mSigma\mathbf{O}\mathbf{e}_{j}\right)$}

It holds that,
\begin{align*}
&2\mathbb{E}\left[\left(\mathbf{e}_{j}^{\top}\mathbf{O}\mathbf{e}_{i}\right)\left(\mathbf{e}_{i}^{\top}\mathbf{O}^{\top}\bm{\mathbf{\Sigma}}^{+}\bm{\mathbf{\Sigma}}\mathbf{O}\mathbf{e}_{j}\right)\right]=2\mathbb{E}\left[\mathbf{e}_{2}^{\top}\mathbf{O}\mathbf{e}_{1}\mathbf{e}_{1}^{\top}\mathbf{O}^{\top}\sum_{k=1}^{d}\mathbf{e}_{k}\mathbf{e}_{k}^{\top}\mathbf{O}\mathbf{e}_{2}\right]
\\&
=2\sum_{k=1}^{d}\mathbb{E}\left[\mathbf{e}_{2}^{\top}\mathbf{O}\mathbf{e}_{1}\mathbf{e}_{1}^{\top}\mathbf{O}^{\top}\mathbf{e}_{k}\mathbf{e}_{k}^{\top}\mathbf{O}\mathbf{e}_{2}\right]=2\sum_{k=1}^{d}\mathbb{E}\left[\mathbf{e}_{2}^{\top}\mathbf{O}\mathbf{e}_{1}\mathbf{e}_{k}^{\top}\mathbf{O}\mathbf{e}_{1}\mathbf{e}_{k}^{\top}\mathbf{O}\mathbf{e}_{2}\right]
\\&
=
2\left(\underbrace{\mathbb{E}\!
\left[\mathbf{e}_{2}^{\top}\mathbf{O}\mathbf{e}_{1}\mathbf{e}_{1}^{\top}\mathbf{O}\mathbf{e}_{1}\mathbf{e}_{1}^{\top}\mathbf{O}\mathbf{e}_{2}\right]}_{\text{solved in \propref{prop:(e2Oe1)(e1Oe1)(e1Oe2)}}}
+
\underbrace{\mathbb{E}\!
\left[\left(\mathbf{e}_{2}^{\top}\mathbf{O}\mathbf{e}_{1}\right)^{2}\mathbf{e}_{2}^{\top}\mathbf{O}\mathbf{e}_{2}\right]}_{\text{solved in \propref{prop:(e2Oe1)2_e2Oe_2}}}
+
\left(d-2\right)\underbrace{\mathbb{E}\!
\left[\mathbf{e}_{2}^{\top}\mathbf{O}\mathbf{e}_{1}\mathbf{e}_{3}^{\top}\mathbf{O}\mathbf{e}_{1}\mathbf{e}_{3}^{\top}\mathbf{O}\mathbf{e}_{2}\right]}_{\text{solved in \propref{prop:(e2Oe1)(e3Oe1)(e3Oe2)}}}\right)
\\&
=2\left(p-m\right)\left(\tfrac{-m^{2}-m+\left(m+1\right)p-2}{\left(p-1\right)p\left(p+2\right)\left(p+4\right)}
+
\tfrac{-m^{2}+2mp+3m-6}{\left(p-1\right)p\left(p+2\right)\left(p+4\right)}+\tfrac{\left(d-2\right)\left(2m^{2}-3mp-2m-p+8\right)}{\left(p-2\right)\left(p-1\right)p\left(p+2\right)\left(p+4\right)}\right)
\\&
=2\left(p-m\right)\left(\frac{-2m^{2}+2m+\left(3m+1\right)p-8}{\left(p-1\right)p\left(p+2\right)\left(p+4\right)}+\frac{\left(d-2\right)\left(2m^{2}-3mp-2m-p+8\right)}{\left(p-2\right)\left(p-1\right)p\left(p+2\right)\left(p+4\right)}\right)
\\&
=\frac{2\left(p-m\right)\left(d\left(2m^{2}-3mp-2m-p+8\right)+p\left(-2m^{2}+3mp+2m+p-8\right)\right)}{\left(p-2\right)\left(p-1\right)p\left(p+2\right)\left(p+4\right)}
\end{align*}

\newpage

\subsubsection{Term 3: $\mathbb{E}\left(\mathbf{e}_{i}^{\top}\mathbf{O}^{\top}\mSigma^{+}\mSigma\mathbf{O}\mathbf{e}_{j}\right)^{2}$}

It holds that,
\begin{align*}
&\mathbb{E}\left(\mathbf{e}_{i}^{\top}\mathbf{O}^{\top}\bm{\mathbf{\Sigma}}^{+}\bm{\mathbf{\Sigma}}\mathbf{O}\mathbf{e}_{j}\right)^{2}=\mathbb{E}\left(\mathbf{e}_{1}^{\top}\mathbf{O}^{\top}\bm{\mathbf{\Sigma}}^{+}\bm{\mathbf{\Sigma}}\mathbf{O}\mathbf{e}_{2}\right)^{2}=\mathbb{E}\left(\mathbf{e}_{1}^{\top}\mathbf{O}^{\top}\sum_{k=1}^{d}\mathbf{e}_{k}\mathbf{e}_{k}^{\top}\mathbf{O}\mathbf{e}_{2}\right)^{2}
\\&
=\sum_{k,\ell=1}^{d}\!
\mathbb{E}\left(\mathbf{e}_{1}^{\top}\mathbf{O}^{\top}\mathbf{e}_{k}\mathbf{e}_{k}^{\top}\mathbf{O}\mathbf{e}_{2}\right)\left(\mathbf{e}_{1}^{\top}\mathbf{O}^{\top}\mathbf{e}_{\ell}\mathbf{e}_{\ell}^{\top}\mathbf{O}\mathbf{e}_{2}\right)
=
\sum_{k,\ell=1}^{d}\!\mathbb{E}\left(\mathbf{e}_{k}^{\top}\mathbf{O}\mathbf{e}_{1}\cdot\mathbf{e}_{k}^{\top}\mathbf{O}\mathbf{e}_{2}\right)\left(\mathbf{e}_{\ell}^{\top}\mathbf{O}\mathbf{e}_{1}\cdot\mathbf{e}_{\ell}^{\top}\mathbf{O}\mathbf{e}_{2}\right)
\\&
=\underbrace{\sum_{k=1}^{d}\mathbb{E}\left(\mathbf{e}_{k}^{\top}\mathbf{O}\mathbf{e}_{1}\cdot\mathbf{e}_{k}^{\top}\mathbf{O}\mathbf{e}_{2}\right)^{2}}_{k=\ell}+\underbrace{\sum_{k\neq\ell=1}^{d}\mathbb{E}\left(\mathbf{e}_{k}^{\top}\mathbf{O}\mathbf{e}_{1}\cdot\mathbf{e}_{k}^{\top}\mathbf{O}\mathbf{e}_{2}\right)\left(\mathbf{e}_{\ell}^{\top}\mathbf{O}\mathbf{e}_{1}\cdot\mathbf{e}_{\ell}^{\top}\mathbf{O}\mathbf{e}_{2}\right)}_{k\neq\ell}
\end{align*}

We now show that,
\begin{align*}
&\sum_{k=1}^{d}\mathbb{E}\left(\mathbf{e}_{k}^{\top}\mathbf{O}\mathbf{e}_{1}\cdot\mathbf{e}_{k}^{\top}\mathbf{O}\mathbf{e}_{2}\right)^{2}
=
\underbrace{2\mathbb{E}\left(\mathbf{e}_{1}^{\top}\mathbf{O}\mathbf{e}_{1}\cdot\mathbf{e}_{1}^{\top}\mathbf{O}\mathbf{e}_{2}\right)^{2}}_{k=1,2\text{, solved in \propref{prop:e2Oe2*e2Oe1}}}
+
\underbrace{\left(d-2\right)\mathbb{E}\left(\mathbf{e}_{3}^{\top}\mathbf{O}\mathbf{e}_{1}\cdot\mathbf{e}_{3}^{\top}\mathbf{O}\mathbf{e}_{2}\right)^{2}}_{k\ge3\text{, solved in \propref{prop:(e2Oe3*e2Oe1)^2}}}
\\&
=\tfrac{2\left(\left(m+4\right)\left(2mp+4p+m-m^{2}-6\right)-\left(p-m\right)\left(p-m+2\right)\left(m\left(m-2p-5\right)+10\right)\right)}{\left(p-1\right)p\left(p+2\right)\left(p+4\right)\left(p+6\right)}
+
\\
&
\eqmargin
\tfrac{\left(d-2\right)\left(4mp\left(-m^{2}+m+4\right)+4\left(m+1\right)\left(m+2\right)p^{2}
+
\left(m-6\right)\left(m-1\right)m\left(m+1\right)-8\left(2p+3\right)\right)}{\left(p-1\right)p\left(p+1\right)\left(p+2\right)\left(p+4\right)\left(p+6\right)}
\\&
=
\tfrac{-2\left(m-p-1\right)\left(m-p\right)\left(m\left(p+2\right)\left(m-2p-5\right)+2\left(5p+6\right)\right)}{\left(p-1\right)p\left(p+1\right)\left(p+2\right)\left(p+4\right)\left(p+6\right)}
+
\\
&
\eqmargin
d\cdot
\tfrac{4m\left(-m^{2}+m+4\right)p+4\left(m+1\right)\left(m+2\right)p^{2}+\left(m-6\right)\left(m-1\right)m\left(m+1\right)-8\left(2p+3\right)}{\left(p-1\right)p\left(p+1\right)\left(p+2\right)\left(p+4\right)\left(p+6\right)}
\end{align*}

Moreover, we have that,
\begin{align*}
&\sum_{k\neq\ell=1}^{d}\mathbb{E}\left(\mathbf{e}_{k}^{\top}\mathbf{O}\mathbf{e}_{1}\cdot\mathbf{e}_{k}^{\top}\mathbf{O}\mathbf{e}_{2}\right)\left(\mathbf{e}_{\ell}^{\top}\mathbf{O}\mathbf{e}_{1}\cdot\mathbf{e}_{\ell}^{\top}\mathbf{O}\mathbf{e}_{2}\right)
\\
&=
\underbrace{2\mathbb{E}\left(\mathbf{e}_{1}^{\top}\mathbf{O}\mathbf{e}_{1}\mathbf{e}_{1}^{\top}\mathbf{O}\mathbf{e}_{2}\right)\left(\mathbf{e}_{2}^{\top}\mathbf{O}\mathbf{e}_{1}\mathbf{e}_{2}^{\top}\mathbf{O}\mathbf{e}_{2}\right)}_{k=1,\ell=2\,\vee\,k=2,\ell=1\text{, solved in \propref{prop:(e1Oe1)(e1Oe2)(e2Oe1)(e2Oe2)}}}
+
\underbrace{4\left(d-2\right)\mathbb{E}\left(\mathbf{e}_{1}^{\top}\mathbf{O}\mathbf{e}_{1}\mathbf{e}_{1}^{\top}\mathbf{O}\mathbf{e}_{2}\right)\left(\mathbf{e}_{3}^{\top}\mathbf{O}\mathbf{e}_{1}\mathbf{e}_{3}^{\top}\mathbf{O}\mathbf{e}_{2}\right)}_{k\le2,\ell\ge3\,\vee\,k\ge3,\ell\le2\text{, solved in \propref{prop:(e1Oe1)(e1Oe2)(e3Oe1)(e3Oe2)}.}}
+
\\
&
\eqmargin
\underbrace{\left(d-2\right)\left(d-3\right)\mathbb{E}\left(\mathbf{e}_{3}^{\top}\mathbf{O}\mathbf{e}_{1}\mathbf{e}_{3}^{\top}\mathbf{O}\mathbf{e}_{2}\right)\left(\mathbf{e}_{4}^{\top}\mathbf{O}\mathbf{e}_{1}\mathbf{e}_{4}^{\top}\mathbf{O}\mathbf{e}_{2}\right)}_{
k\neq\ell\ge3\text{, solved in \propref{prop:(e3Oe4)(e3Oe1)(e2Oe4)(e2Oe1)}}}    
\end{align*}

By summing all of the above and by some tedious algebraic steps,
we finally get that,
\begin{align*}
&
\mathbb{E}\left(\mathbf{e}_{i}^{\top}\mathbf{O}^{\top}\bm{\mathbf{\Sigma}}^{+}\bm{\mathbf{\Sigma}}\mathbf{O}\mathbf{e}_{j}\right)^{2}
\\
&=
\tfrac{2p\left(m^{4}\left(6-3p-2p^{2}\right)+m^{3}\left(-12-20p+15p^{2}+7p^{3}\right)+m^{2}\left(18+13p+5p^{2}-21p^{3}-8p^{4}\right)\right)}{\left(p-3\right)\left(p-2\right)\left(p-1\right)p\left(p+1\right)\left(p+2\right)\left(p+4\right)\left(p+6\right)}
+
\\&
\eqmargin
\tfrac{2p\left(m\left(3p^{5}+8p^{4}+21p^{3}+28p^{2}-14p-12\right)+p\left(12-4p-29p^{2}-12p^{3}+p^{4}\right)\right)}{\left(p-3\right)\left(p-2\right)\left(p-1\right)p\left(p+1\right)\left(p+2\right)\left(p+4\right)\left(p+6\right)}
+
\\&
\eqmargin
d\cdot\tfrac{3m^{4}\left(-4+4p+3p^{2}\right)-4m^{3}\left(-6-13p+16p^{2}+8p^{3}\right)+m^{2}\left(-36+16p+29p^{2}+88p^{3}+36p^{4}\right)}{\left(p-3\right)\left(p-2\right)\left(p-1\right)p\left(p+1\right)\left(p+2\right)\left(p+4\right)\left(p+6\right)}
+
\\&
\eqmargin
d\cdot\tfrac{-2m\left(6p^{5}+13p^{4}+69p^{3}+105p^{2}+16p-12\right)+\left(-4p^{5}+54p^{4}+90p^{3}+152p^{2}+120p\right)}{\left(p-3\right)\left(p-2\right)\left(p-1\right)p\left(p+1\right)\left(p+2\right)\left(p+4\right)\left(p+6\right)}
+
\\&
\eqmargin
d^{2}\cdot
\tfrac{m^{4}\left(-5p-6\right)+m^{3}\left(18p^{2}+34p-12\right)+m^{2}\left(-20p^{3}-46p^{2}-39p-42\right)}{\left(p-3\right)\left(p-2\right)\left(p-1\right)p\left(p+1\right)\left(p+2\right)\left(p+4\right)\left(p+6\right)}
+
\\&
\eqmargin
d^{2}\cdot
\tfrac{m\left(6p^{4}+10p^{3}+96p^{2}+154p+60\right)+\left(2p^{4}-30p^{3}-32p^{2}-144p-144\right)}{\left(p-3\right)\left(p-2\right)\left(p-1\right)p\left(p+1\right)\left(p+2\right)\left(p+4\right)\left(p+6\right)}\,.
\end{align*}

\newpage

\subsection{Extending \figref{fig:analytical}: More random synthetic data experiments}
\label{app:synthetic-figures}

\begin{figure}[ht!]
\begin{center}
    \subfigure[Experiments for $p=10$.]{
        \includegraphics[width=.99\columnwidth]{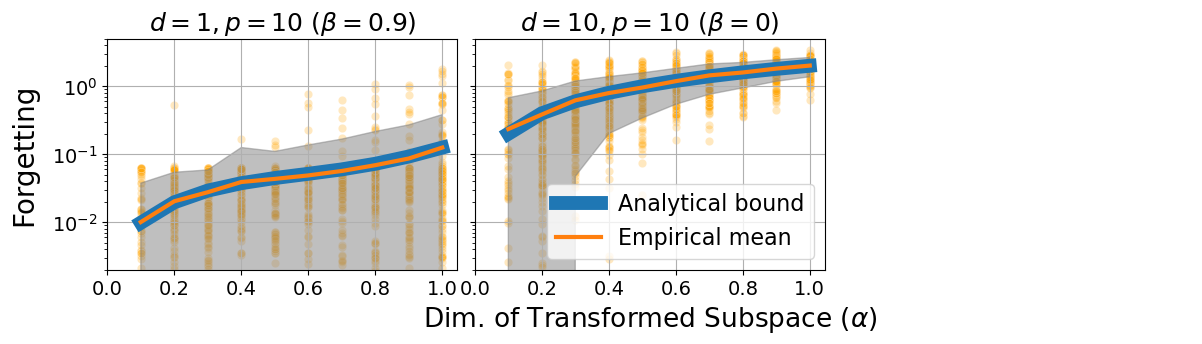}
    }
    
    \vspace{2em}
    
    \subfigure[Experiments for $p=100$.]{
        \includegraphics[width=.99\columnwidth]{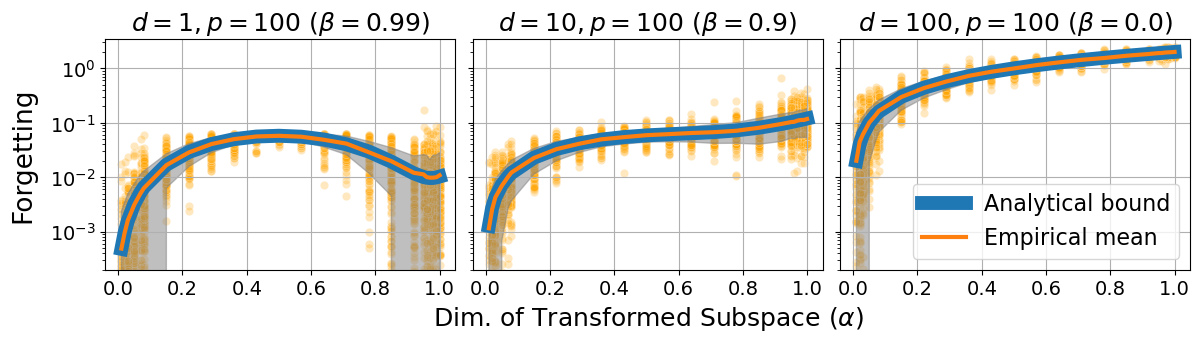}
    }
    
    \vspace{2em}
    
    \subfigure[Experiments for $p=1000$.]{
        \includegraphics[width=.99\columnwidth]{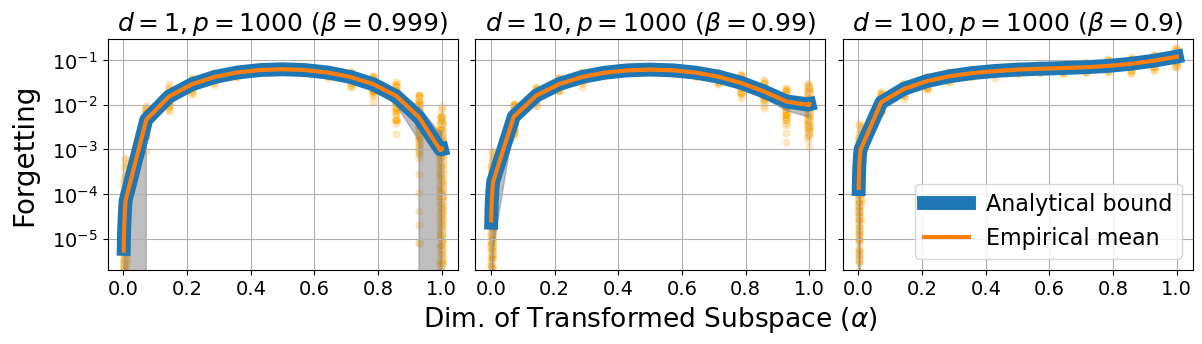}
    }
    \caption{
    Empirically illustrating the worst-case forgetting under different overparameterization levels.
    Points indicate the forgetting under many sampled random transformations applied on a (single) random data matrix $\X$. 
    Their mean is shown in the thin orange line, with the standard deviation represented by a gray band.
    The thick blue line depicts the analytical expression of \thmref{thm:main}.
    The analytical bound matches the empirical mean, thus exemplifying the tightness of our analysis.
    }
\end{center}
\end{figure}

\newpage

\section{Monomials of Entries of Random Orthogonal Matrices}
\label{app:monomials}

Throughout the supplementary materials, we often wish to compute the expectation of an arbitrary monomial of the entries
of a random orthogonal matrix $\Q\sim O(p)$ sampled uniformly from the orthogonal group,
\eg $\expectation_\Q [q_{1,1}^2 q_{1,2}^4 q_{2,2}^6]$.
Following \citet{gorin2002integrals}, we define a ``power matrix'' ${\M\in\mathbb{Z}_{\ge 0}^{p\times R}}$
(for some $R\le p$)
that maps into a monomial 
$\prod_{i,j=1}^{p,R} q_{i,j}^{m_{i,j}}$ 
constructed from the entries
of the first $R$ columns of
the random $\Q\in\reals^{p\times p}$.
We denote the expected value of this monomial by
$$
\mathbb{E}_{\Q} \left[{\prod}_{i,j=1}^{p,R} q_{i,j}^{m_{i,j}}\right] 
\triangleq 
\left\langle\M\right\rangle
\,,
\quad
\text{for example, }
\quad
\mathbb{E}_{\Q} [q_{1,1}^2 q_{2,1}^4 q_{2,2}^6]
\triangleq
\left\langle \begin{smallmatrix}
2 & 0\\
4 & 6\\
\overrightarrow{0} & \overrightarrow{0}
\end{smallmatrix}\right\rangle
\,.
$$
We employ the notation $\overrightarrow{0}$ to complement $\M$ to have $p$ rows.
For instance, in the example above, $\overrightarrow{0}$ is a vector with $p\!-\!2$ zero entries.

\bigskip

The following are helpful properties of the integral (expectation) over the orthogonal group, as mentioned in \citet{gorin2002integrals}.

\begin{property}[Invariance of the integral over the orthogonal group]
\label{prop:invariance}
Let $\Q\sim O(p)$ and $\M\in\mathbb{Z}_{\ge 0}^{p\times R}$.

\begin{enumerate}[leftmargin=.6cm]
    \item \textbf{Invariance w.r.t. transpose.} 
    Since $\Q$ and $\Q^\top$ are identically distributed, it holds that ${\left\langle\M\right\rangle=\langle\M^\top \rangle}$. For example,
    $$
    \left\langle \begin{smallmatrix}
    2 & 0\\
    4 & 2\\
    4 & 6\\
    \overrightarrow{0} & \overrightarrow{0}
    \end{smallmatrix}\right\rangle
    =
    \mathbb{E}_{\Q} [q_{1,1}^2 q_{2,1}^4 q_{2,2}^2 q_{3,1}^4 q_{3,2}^6]
    =
    \mathbb{E}_{\Q} [q_{1,1}^2 q_{1,2}^4 q_{1,3}^4 q_{2,2}^2 q_{2,3}^6]
    =
    \left\langle \begin{smallmatrix}
    2 & 4 & 4\\
    0 & 2 & 6\\
    \overrightarrow{0} & \overrightarrow{0} & \overrightarrow{0}
    \end{smallmatrix}\right\rangle
    \,.
    $$

    \item \textbf{Invariance w.r.t. row and column permutations.} 
    Since different rows/columns of $\Q$ are identically distributed, the integral over the orthogonal group $O(p)$ is invariant under permutations of columns or rows of the power matrix $\M$ (see also \citet{ullah1964invariance}). 
    For example,
    $$
    \left\langle \begin{smallmatrix}
    2 & 0\\
    4 & 6\\
    \overrightarrow{0} & \overrightarrow{0}
    \end{smallmatrix}\right\rangle
    =
    \mathbb{E}_{\Q} [q_{1,1}^2 q_{2,1}^4 q_{2,2}^6]
    =
    \mathbb{E}_{\Q} [q_{1,1}^6 q_{1,2}^4 q_{2,2}^2]
    =
    \left\langle \begin{smallmatrix}
    6 & 4\\
    0 & 2\\
    \overrightarrow{0} & \overrightarrow{0}
    \end{smallmatrix}\right\rangle
    \,.
    $$
\end{enumerate}
\end{property}

\bigskip

The following is a known property of odd moments in integrals (expectations) over the orthogonal group
(see \citet{brody1981random,gorin2002integrals}).
\begin{property}
\label{prop:odd}
    If the sum over \emph{any} row or column of the power matrix $\M$ is odd, the integral vanishes, \ie $\langle\M\rangle=0$.
    For example,
    \hfill$
    \mathbb{E}_{\Q} [q_{1,1} q_{2,1}^4 q_{2,2}^6]
    =
    \left\langle \begin{smallmatrix}
    1 & 0\\
    4 & 6\\
    \overrightarrow{0} & \overrightarrow{0}
    \end{smallmatrix}\right\rangle
    =
    0
    $.
    \\
    (In contrast, generally, it holds that 
    $\left\langle \begin{smallmatrix}
    1 & 1\\
    1 & 1\\
    \overrightarrow{0} & \overrightarrow{0} 
    \end{smallmatrix}\right\rangle
    \neq 0$).
\end{property}

\bigskip

\begin{corollary}
\label{cor:odd-column-or-row}
    Let $\vq_i$ be the $i$\nth column (or row) of $\Q$.
    Let $\A^{(1)},\dots, \A^{(N)}\in\reals^{p\times p}$ be $N$ matrices independent on $\Q$ and let $i_1,\dots,i_{2N}\in\cnt{p}$ be $2N$ indices.
    Then, if there exists an index $i\in\cnt{p}$ that appears an \emph{odd} number of times in $i_1,\dots,i_{2N}$, the following expectation vanishes:
    $$
    \mathbb{E}
    [
    \vq_{i_1} \A^{(i_1)} \vq_{i_{N+1}}
    \cdot 
    \vq_{i_2} \A^{(i_2)} \vq_{i_{N+2}}
    \cdots
    \vq_{i_N} \A^{(i_N)} \vq_{i_{2N}}
    ]
    =0\,.
    $$

    For example, the above dictates that,
    $$
    \mathbb{E}_{\Q_p,\Q_m}
    \bigg[
    \ve_{1}^{\top}
    \Q_{p}\left[\begin{smallmatrix}
    \Q_{m}\\
     & \I_{p-m}
    \end{smallmatrix}\right]
    \Q_{p}^{\top}
    \ve_{1}
    \cdot
    \ve_{1}^{\top}
    \Q_{p}
    \left[\begin{smallmatrix}
    \0_{m}\\
     & \I_{p-m}
    \end{smallmatrix}\right]
    \Q_{p}^{\top}
    \ve_{2}
    \cdot
    \ve_{3}^{\top}
    \Q_{p}
    \left[\begin{smallmatrix}
    \Q_{m}\\
     & \0_{p-m}
    \end{smallmatrix}\right]
    \Q_{p}^{\top}
    \ve_{2}
    \bigg]=0
    \,,
    $$
    since $\vq_{1}$ appears $3$ times (also because $\vq_3$ appears once).
\end{corollary}

\begin{proof}
    The expectation can be rewritten as,
    \begin{align*}
    \mathbb{E}
    [
    \vq_{i_1} \A^{(i_1)} \vq_{i_{N+1}}
    \cdots
    \vq_{i_N} \A^{(i_N)} \vq_{i_{2N}}
    ]
    =
    \sum_{j_1,\dots, j_N=1}^{p}
    \sum_{k_1,\dots, k_N=1}^{p}
    \mathbb{E}
    \left[
    {\prod}_{\ell=1}^{N}
    q_{j_{\ell},i_{\ell}}
    a^{(i_{\ell})}_{j_{\ell},k_{\ell}}
    q_{k_{\ell},i_{N+{\ell}}}
    \right]\,.
    \end{align*}
Let $t\in\cnt{p}$ be an index that appears an odd number of times in $i_1,\dots, i_{2N}$.
The $t$\nth column (or row) appears the same number of times in each of the (summand) expectations.
Thus, the sum of the $t$\nth column (or row) corresponding to $\vq_{t}$ in the power matrix $\M$ corresponding to that expectation will be an odd number. 
Thus, by \propref{prop:odd}, the expectation vanishes.
\end{proof}

\bigskip

\begin{corollary}
\label{cor:odd-Q_m}
    Let $\vv_{1},\dots, \vv_{2N}\in\reals^{p}$ 
    and $c\in \reals$ 
    be random variables independent of $\Q_m\sim O(m)$.
    Then, if $N$ is an odd number, the following expectation vanishes:
    $$
    \mathbb{E}
    [
    c\cdot
    \vv_{1}^\top 
    \left[\begin{smallmatrix}
    \Q_{m}\\
     & \0_{p-m}
    \end{smallmatrix}\right]
    \vv_{N+1}
    \cdot
    \vv_{2}^\top 
    \left[\begin{smallmatrix}
    \Q_{m}\\
     & \0_{p-m}
    \end{smallmatrix}\right]
    \vv_{N+2}
    \cdots
    \vv_{N}^\top 
    \left[\begin{smallmatrix}
    \Q_{m}\\
     & \0_{p-m}
    \end{smallmatrix}\right]
    \vv_{2N}
    ]
    =0\,.
    $$

    For example, the above dictates that the following expectation vanishes,
    $$
    \mathbb{E}_{\Q_p,\Q_m}
    \bigg[
    \ve_{1}^{\top}
    \Q_{p}\left[\begin{smallmatrix}
    \Q_{m}\\
     & \0_{p-m}
    \end{smallmatrix}\right]
    \Q_{p}^{\top}
    \ve_{1}
    \cdot
    \ve_{1}^{\top}
    \Q_{p}
    \left[\begin{smallmatrix}
    \Q_{m}\\
     & \0_{p-m}
    \end{smallmatrix}\right]
    \Q_{p}^{\top}
    \ve_{1}
    \cdot
    \ve_{2}^{\top}
    \Q_{p}
    \left[\begin{smallmatrix}
    \Q_{m}\\
     & \0_{p-m}
    \end{smallmatrix}\right]
    \Q_{p}^{\top}
    \ve_{2}
    \bigg]=0
    \,,
    $$
    since here $N=3$.
    \\
    In contrast, the above does \emph{not} imply that the following expectation vanishes,
    $$
    \mathbb{E}_{\Q_p,\Q_m}
    \bigg[
    \ve_{1}^{\top}
    \Q_{p}\left[\begin{smallmatrix}
    \Q_{m}\\
     & \0_{p-m}
    \end{smallmatrix}\right]
    \Q_{p}^{\top}
    \ve_{1}
    \cdot
    \underbrace{\ve_{1}^{\top}
    \Q_{p}
    \left[\begin{smallmatrix}
    \0_{m}\\
     & \I_{p-m}
    \end{smallmatrix}\right]
    \Q_{p}^{\top}
    \ve_{1}}_{
    \text{here, this is considered as $c$}
    }
    \cdot
    \,
    \ve_{2}^{\top}
    \Q_{p}
    \left[\begin{smallmatrix}
    \Q_{m}\\
     & \0_{p-m}
    \end{smallmatrix}\right]
    \Q_{p}^{\top}
    \ve_{2}
    \bigg]
    \,,
    $$
    since here $N=2$.
\end{corollary}

\begin{proof}
The expectations become, 
\begin{align*}
    \mathbb{E}
    \left[
    c\cdot
    {\prod}_{\ell=1}^{N}
    \vv_{\ell}^\top 
    \left[\begin{smallmatrix}
    \Q_{m}\\
     & \0_{\ell}
    \end{smallmatrix}\right]
    \vv_{N+\ell}
    \right]
    &
    =
    \mathbb{E}
    \bigg[
    c\cdot
    {\prod}_{\ell=1}^{N}
    \biggprn{
    \tsum_{i,j=1}^{m}
        (\vv_{\ell})_{i}
        \underbrace{(\Q_m)_{i,j}}_{\triangleq q_{i,j}}
        (\vv_{N+\ell})_{j}
    }
    \bigg]
    \\
    &
    =
    \mathbb{E}
    \left[
    c\cdot
    {\prod}_{\ell=1}^{N}
    \prn{
    \tsum_{i,j=1}^{m}
        q_{i,j}
        (\vv_{\ell})_{i}
        (\vv_{N+\ell})_{j}
    }
    \right]
    \\
    &
    =
    \sum_{i_1,\dots,i_N=1}^{m}
    \sum_{j_1,\dots,j_N=1}^{m}
    \mathbb{E}
    \left[
    c\cdot
    {\prod}_{\ell=1}^{N}
    q_{i_\ell,j_\ell}
    (\vv_{\ell})_{i_\ell}
    (\vv_{N+\ell})_{j_\ell}
    \right]\,.
\end{align*}
We notice that in each of the (summand) expectations the entries of $\Q_m$ appear exactly $N$ times. However, since $N$ is odd, at least one row or column of the power matrix corresponding to the monomial in the expectation must have an odd sum.
Then, by \propref{prop:odd}, all these expectations vanish.
\end{proof}

\newpage

The main result we need from \citet{gorin2002integrals} is their Eq.~(23), providing a recursive formula to compute $\langle\M\rangle$ for any power matrix $\M$.
We bring this formula here for the sake of completeness.
\begin{lemma}[Recursive formula for computing expectations of monomials over the orthogonal group]
\label{lem:recursive}
    Define the Pochhammer symbol $(z)_n = \frac{\Gamma(z+n)}{\Gamma(z)}$.
    Denote
    $\binom{\overrightarrow{n}}{\overrightarrow{k}}
    =
    \prod_{i=1}^{p}\binom{n_i}{k_i}$.
    \linebreak
    Denote $(\overrightarrow{n}|K)
    = \prod_{i=1}^{p}
    \left(
    n_i
    \,\mid\,
    K_{i,1},\dots, K_{i,R-1}
    \right)
    $.
    
    Then, the one-vector integral is given by
    $\langle\overrightarrow{m}\rangle
    =
    \prn{\frac{p}{2}}^{-1}_{\overline{m}/2}
    \prod_{i=1}^{p}
    \prn{\frac{1}{2}}_{m_i / 2}.
    $

    Moreover, the $R$-vector integral (corresponding to a matrix $\M\in\mathbb{Z}_{\ge 0}^{p\times R}$, is given by
    \begin{align*}
    \langle\M\rangle
    &=
    \prn{\tfrac{p-R+1}{2}}^{-1}_{\hfrac{\overline{m}_R}{2}}
    \cdot 
    \\
    &\cdot 
    \sum_{\overrightarrow{\kappa}}
    \!
    \prn{
    \!
        \binom{\overline{m}_R}{\overrightarrow{\kappa}}
        \prn{-1}^{\prn{\overline{m}_R - \overline{\kappa}}/2}
        \cdot
        {
        \prod_{i=1}^{p}\!
        \prn{\tfrac{1}{2}}_{\hfrac{\kappa_i}{2}}
        }
        \cdot
        \sum_{K}
        (\overrightarrow{m}_R \!-\! \overrightarrow{\kappa} |K)
        \cdot
        {
        \prod_{j=1}^{R-1}\!
        \prn{\tfrac{1}{2}}_{\hfrac{\overline{k}_j}{2}}
        }
        \cdot
        \langle
        \M^{(R-1)} \!+\! K
        \rangle
        \!
    }
    \end{align*}
    where the first sum runs over all $\overrightarrow{\kappa}$ with all even entries (less or equal to the corresponding entries of the last column $\overrightarrow{m}_R$).
    The second sum runs over all $K\in\mathbb{Z}_{\ge 0}^{p,R-1}$ for which all sums of columns are even, 
    \ie $\overline{k}_j \triangleq \sum_{i=1}^{p} k_{i,j}$ is even $\forall j\in\cnt{R-1}$.
    Finally, 
    $\M^{(R-1)}$ stands for the first $R-1$ columns of $\M$,
    and $\overline{m}_R=\sum_{i=1}^{p} m_{i,R}$
    and $\overline{\kappa}=\sum_{i=1}^{p} \kappa_{i}$.
\end{lemma}

\bigskip
\medskip

In the pages to come, we present many calculations of expectations of different monomials that we use throughout the supplementary materials.
We include these calculations since the recursive formula of \lemref{lem:recursive} is somewhat complicated to apply, and we wish our derivations to be reproducible and easily followed. 

\medskip

\subsection{Monomials of One Orthogonal Vector}

\begin{align*}
 \mathbb{E}[u_{1}^{4}]
 &
 =
\left\langle \begin{smallmatrix}
4 \\
\overrightarrow{0}
\end{smallmatrix}\right\rangle 
=
\prn{\tfrac{p}{2}}^{-1}_{2}
\prn{\tfrac{1}{2}}_{2}
=
\left(\tfrac{p}{2}\left(\tfrac{p+2}{2}\right)\right)^{-1}
\tfrac{3}{4}
=
\tfrac{3}{p(p+2)}
\\
&
\\
 \mathbb{E}[u_{1}^{6}]
 &
 =
\left\langle \begin{smallmatrix}
6 \\
\overrightarrow{0}
\end{smallmatrix}\right\rangle 
=
\prn{\tfrac{p}{2}}^{-1}_{3}
\prn{\tfrac{1}{2}}_{3}
=
\tfrac{15}{p(p+2)(p+4)}
\\
&
\\
 \mathbb{E}[u_{1}^{2}u_{2}^{2}]
 &
 =
\left\langle \begin{smallmatrix}
2 \\
2 \\
\overrightarrow{0}
\end{smallmatrix}\right\rangle 
=
\prn{\tfrac{p}{2}}^{-1}_{2}
\prn{\tfrac{1}{2}}_{1}
\prn{\tfrac{1}{2}}_{1}
=
\left(\tfrac{p}{2}\left(\tfrac{p+2}{2}\right)\right)^{-1}
\tfrac{1}{4}
=
\tfrac{1}{p(p+2)}
\\
&
\\
 \mathbb{E}[u_{1}^{4}u_{2}^{2}]
 &
 =
\left\langle \begin{smallmatrix}
4 \\
2 \\
\overrightarrow{0}
\end{smallmatrix}\right\rangle 
=
\prn{\tfrac{p}{2}}^{-1}_{3}
\prn{\tfrac{1}{2}}_{2}
\prn{\tfrac{1}{2}}_{1}
=
\tfrac{3}{p\prn{p+2}\prn{p+4}}
\\
&
\\
 \mathbb{E}[u_{1}^{2}u_{2}^{2}u_{3}^{2}]
 &
 =
\left\langle \begin{smallmatrix}
2 \\
2 \\
2 \\
\overrightarrow{0}
\end{smallmatrix}\right\rangle 
=
\prn{\tfrac{p}{2}}^{-1}_{3}
\prn{\tfrac{1}{2}}_{1}^3
=
\tfrac{1}{p\prn{p+2}\prn{p+4}}
\end{align*}

\newpage

\subsection{Monomials of Two Orthogonal Vector}
\subsubsection{One index (row)}
\begin{align*}
\left\langle \begin{smallmatrix}
4 & 4\\
\overrightarrow{0} & \overrightarrow{0}
\end{smallmatrix}\right\rangle &=
\left(\tfrac{p-1}{2}\left(\tfrac{p+1}{2}\right)\right)^{-1}\left(\left(-1\right)^{2}\left(\tfrac{1}{2}\right)_{2}%
\left\langle \begin{smallmatrix}
8\\
\overrightarrow{0}
\end{smallmatrix}\right\rangle +
\binom{4}{2}\left(-1\right)^{1}\left(\tfrac{1}{2}\right)_{1}^{2}%
\left\langle \begin{smallmatrix}
6\\
\overrightarrow{0}
\end{smallmatrix}\right\rangle +\left(-1\right)^{0}\left(\tfrac{1}{2}\right)_{2}%
\left\langle \begin{smallmatrix}
4\\
\overrightarrow{0}
\end{smallmatrix}\right\rangle \right)\\&=\tfrac{1}{\left(p-1\right)\left(p+1\right)}\left(3%
\left\langle \begin{smallmatrix}
8\\
\overrightarrow{0}
\end{smallmatrix}\right\rangle -6%
\left\langle \begin{smallmatrix}
6\\
\overrightarrow{0}
\end{smallmatrix}\right\rangle +3%
\left\langle \begin{smallmatrix}
4\\
\overrightarrow{0}
\end{smallmatrix}\right\rangle \right)
\\
&=
\tfrac{3}{\left(p-1\right)\left(p+1\right)}\left(\left(\tfrac{p}{2}\right)_{4}^{-1}\left(\tfrac{1}{2}\right)_{4}-2\left(\tfrac{p}{2}\right)_{3}^{-1}\left(\tfrac{1}{2}\right)_{3}+\left(\tfrac{p}{2}\right)_{2}^{-1}\left(\tfrac{1}{2}\right)_{2}\right)\\&=\tfrac{3}{\left(p-1\right)\left(p+1\right)}\left(\tfrac{105}{p\left(p+2\right)\left(p+4\right)\left(p+6\right)}-\tfrac{30}{p\left(p+2\right)\left(p+4\right)}+\tfrac{3}{p\left(p+2\right)}\right)
\\
&=\tfrac{3}{\left(p-1\right)\left(p+1\right)}\left(\tfrac{3\left(p-1\right)\left(p+1\right)}{p\left(p+2\right)\left(p+4\right)\left(p+6\right)}\right)
=
\tfrac{9}{p\left(p+2\right)\left(p+4\right)\left(p+6\right)}
\\
&
\\
\left\langle \begin{smallmatrix}
2 & 2\\
\overrightarrow{0} & \overrightarrow{0}
\end{smallmatrix}\right\rangle &=\tfrac{1}{p\left(p+2\right)}\\
&\\
\left\langle \begin{smallmatrix}
4 & 2\\
\overrightarrow{0} & \overrightarrow{0}
\end{smallmatrix}\right\rangle &=\left(\tfrac{p-1}{2}\right)^{-1}\left(\left(-1\right)^{1}\left(\tfrac{1}{2}\right)_{1}%
\left\langle \begin{smallmatrix}
6\\
\overrightarrow{0}
\end{smallmatrix}\right\rangle +\left(-1\right)^{0}\left(\tfrac{1}{2}\right)_{1}%
\left\langle \begin{smallmatrix}
4\\
\overrightarrow{0}
\end{smallmatrix}\right\rangle \right)
\\
&=
\tfrac{1}{\left(p-1\right)}\left(%
\left\langle \begin{smallmatrix}
4\\
\overrightarrow{0}
\end{smallmatrix}\right\rangle -\left\langle 
\begin{smallmatrix}
6\\
\overrightarrow{0}
\end{smallmatrix}
\right\rangle \right)=\tfrac{1}{\left(p-1\right)}\left(\left(\tfrac{p}{2}\right)_{2}^{-1}\left(\tfrac{1}{2}\right)_{2}-\left(\tfrac{p}{2}\right)_{3}^{-1}\left(\tfrac{1}{2}\right)_{3}\right)\\&=\tfrac{1}{\left(p-1\right)}\left(\tfrac{3}{p\left(p+2\right)}-\tfrac{15}{p\left(p+2\right)\left(p+4\right)}\right)=\tfrac{3\left(p-1\right)}{\left(p-1\right)p\left(p+2\right)\left(p+4\right)}=\tfrac{3}{p\left(p+2\right)\left(p+4\right)}\\&\\%
\left\langle \begin{smallmatrix}
6 & 2\\
\overrightarrow{0} & \overrightarrow{0}
\end{smallmatrix}\right\rangle &=%
\left\langle \begin{smallmatrix}
6\\
2\\
\overrightarrow{0}
\end{smallmatrix}\right\rangle =\left(\tfrac{p}{2}\right)_{4}^{-1}\left(\tfrac{1}{2}\right)_{3}\left(\tfrac{1}{2}\right)_{1}=\tfrac{15}{p\left(p+2\right)\left(p+4\right)\left(p+6\right)}
\end{align*}

\newpage

\subsubsection{Two indices (rows)}

\begin{align*}
\left\langle \begin{smallmatrix}
2 & 2\\
2 & 2\\
\overrightarrow{0} & \overrightarrow{0}
\end{smallmatrix}\right\rangle &=\tfrac{1}{4}\left(\tfrac{p-1}{2}\left(\tfrac{p+1}{2}\right)\right)^{-1}\left(3%
\left\langle \begin{smallmatrix}
4\\
4\\
\overrightarrow{0}
\end{smallmatrix}\right\rangle +%
\left\langle \begin{smallmatrix}
2\\
2\\
\overrightarrow{0}
\end{smallmatrix}\right\rangle -%
\left\langle \begin{smallmatrix}
2\\
4\\
\overrightarrow{0}
\end{smallmatrix}\right\rangle -%
\left\langle \begin{smallmatrix}
4\\
2\\
\overrightarrow{0}
\end{smallmatrix}\right\rangle \right)\\&=\tfrac{1}{\left(p-1\right)\left(p+1\right)}\left(3\left(\tfrac{p}{2}\right)_{4}^{-1}\left(\tfrac{1}{2}\right)_{2}\left(\tfrac{1}{2}\right)_{2}+\left(\tfrac{p}{2}\right)_{2}^{-1}\left(\tfrac{1}{2}\right)_{1}\left(\tfrac{1}{2}\right)_{1}-2\left(\tfrac{p}{2}\right)_{3}^{-1}\left(\tfrac{1}{2}\right)_{1}\left(\tfrac{1}{2}\right)_{2}\right)\\&=\tfrac{1}{\left(p-1\right)\left(p+1\right)}\left(3\tfrac{9}{16}\left(\tfrac{p}{2}\right)_{4}^{-1}+\tfrac{1}{4}\left(\tfrac{p}{2}\right)_{2}^{-1}-\tfrac{2}{2}\cdot\tfrac{3}{4}\left(\tfrac{p}{2}\right)_{3}^{-1}\right)\\&=\tfrac{1}{\left(p-1\right)\left(p+1\right)}\left(\tfrac{27}{16}\left(\tfrac{p}{2}\cdot\tfrac{p+2}{2}\cdot\tfrac{p+4}{2}\cdot\tfrac{p+6}{2}\right)^{-1}+\tfrac{1}{4}\left(\tfrac{p}{2}\cdot\tfrac{p+2}{2}\right)^{-1}-\tfrac{3}{4}\left(\tfrac{p}{2}\cdot\tfrac{p+2}{2}\cdot\tfrac{p+4}{2}\right)^{-1}\right)\\&=\tfrac{1}{\left(p-1\right)\left(p+1\right)}\left(\tfrac{27}{p\left(p+2\right)\left(p+4\right)\left(p+6\right)}+\tfrac{1}{p\left(p+2\right)}-\tfrac{3}{4}\tfrac{8}{p\left(p+2\right)\left(p+4\right)}\right)\\&=\tfrac{1}{\left(p-1\right)\left(p+1\right)}\left(\tfrac{27+\left(p+4\right)\left(p+6\right)-6\left(p+6\right)}{p\left(p+2\right)\left(p+4\right)\left(p+6\right)}\right)=\tfrac{p^{2}+4p+15}{\left(p-1\right)p\left(p+1\right)\left(p+2\right)\left(p+4\right)\left(p+6\right)}\\&\\%
\left\langle \begin{smallmatrix}
4 & 0\\
0 & 4\\
\overrightarrow{0} & \overrightarrow{0}
\end{smallmatrix}\right\rangle &=\left(\tfrac{p-1}{2}\left(\tfrac{p+1}{2}\right)\right)^{-1}\left(\left(-1\right)^{2}\left(\tfrac{1}{2}\right)_{2}%
\left\langle \begin{smallmatrix}
4\\
4\\
\overrightarrow{0}
\end{smallmatrix}\right\rangle +
\binom{4}{2}\left(-1\right)^{1}\left(\tfrac{1}{2}\right)_{1}^{2}%
\left\langle \begin{smallmatrix}
4\\
2\\
\overrightarrow{0}
\end{smallmatrix}\right\rangle +\left(-1\right)^{0}\left(\tfrac{1}{2}\right)_{2}%
\left\langle \begin{smallmatrix}
4\\
0\\
\overrightarrow{0}
\end{smallmatrix}\right\rangle \right)\\&=\tfrac{1}{\left(p-1\right)\left(p+1\right)}\left(3%
\left\langle \begin{smallmatrix}
4\\
4\\
\overrightarrow{0}
\end{smallmatrix}\right\rangle -6%
\left\langle \begin{smallmatrix}
4\\
2\\
\overrightarrow{0}
\end{smallmatrix}\right\rangle +3%
\left\langle \begin{smallmatrix}
4\\
0\\
\overrightarrow{0}
\end{smallmatrix}\right\rangle \right)\\&=\tfrac{3}{\left(p-1\right)\left(p+1\right)}\left(\left(\tfrac{p}{2}\right)_{4}^{-1}\left(\tfrac{1}{2}\right)_{2}^{2}-2\left(\tfrac{p}{2}\right)_{3}^{-1}\left(\tfrac{1}{2}\right)_{2}\left(\tfrac{1}{2}\right)_{1}+\left(\tfrac{p}{2}\right)_{2}^{-1}\left(\tfrac{1}{2}\right)_{2}\right)\\&=\tfrac{3}{\left(p-1\right)\left(p+1\right)}\left(\tfrac{9}{p\left(p+2\right)\left(p+4\right)\left(p+6\right)}-\tfrac{6}{p\left(p+2\right)\left(p+4\right)}+\tfrac{3}{p\left(p+2\right)}\right)\\&=\tfrac{3}{\left(p-1\right)\left(p+1\right)}\left(\tfrac{3\left(p+3\right)\left(p+5\right)}{p\left(p+2\right)\left(p+4\right)\left(p+6\right)}\right)\\&=\tfrac{9\left(p+3\right)\left(p+5\right)}{\left(p-1\right)p\left(p+1\right)\left(p+2\right)\left(p+4\right)\left(p+6\right)}\\&\\%
\left\langle \begin{smallmatrix}
4 & 0\\
0 & 2\\
\overrightarrow{0} & \overrightarrow{0}
\end{smallmatrix}\right\rangle &=\left(\tfrac{p-1}{2}\right)^{-1}\left(\left(-1\right)^{1}\left(\tfrac{1}{2}\right)_{1}%
\left\langle \begin{smallmatrix}
4\\
2\\
\overrightarrow{0}
\end{smallmatrix}\right\rangle +\left(-1\right)^{0}\left(\tfrac{1}{2}\right)_{1}%
\left\langle \begin{smallmatrix}
4\\
0\\
\overrightarrow{0}
\end{smallmatrix}\right\rangle \right)\\&=\tfrac{1}{\left(p-1\right)}\left(%
\left\langle \begin{smallmatrix}
4\\
0\\
\overrightarrow{0}
\end{smallmatrix}\right\rangle -%
\left\langle \begin{smallmatrix}
4\\
2\\
\overrightarrow{0}
\end{smallmatrix}\right\rangle \right)=\tfrac{1}{\left(p-1\right)}\left(\left(\tfrac{p}{2}\right)_{2}^{-1}\left(\tfrac{1}{2}\right)_{2}-\left(\tfrac{p}{2}\right)_{3}^{-1}\left(\tfrac{1}{2}\right)_{2}\left(\tfrac{1}{2}\right)_{1}\right)\\&=\tfrac{1}{\left(p-1\right)}\left(\tfrac{3}{p\left(p+2\right)}-\tfrac{3}{p\left(p+2\right)\left(p+4\right)}\right)=\tfrac{3\left(p+3\right)}{\left(p-1\right)p\left(p+2\right)\left(p+4\right)}\\&\\%
\left\langle \begin{smallmatrix}
6 & 0\\
0 & 2\\
\overrightarrow{0} & \overrightarrow{0}
\end{smallmatrix}\right\rangle &=\left(\tfrac{p-1}{2}\right)^{-1}\left(\left(-1\right)^{1}\left(\tfrac{1}{2}\right)_{1}%
\left\langle \begin{smallmatrix}
6\\
2\\
\overrightarrow{0}
\end{smallmatrix}\right\rangle +\left(-1\right)^{0}\left(\tfrac{1}{2}\right)_{1}%
\left\langle \begin{smallmatrix}
6\\
0\\
\overrightarrow{0}
\end{smallmatrix}\right\rangle \right)\\&=\tfrac{1}{p-1}\left(%
\left\langle \begin{smallmatrix}
6\\
0\\
\overrightarrow{0}
\end{smallmatrix}\right\rangle -%
\left\langle \begin{smallmatrix}
6\\
2\\
\overrightarrow{0}
\end{smallmatrix}\right\rangle \right)=\tfrac{1}{p-1}\left(\left(\tfrac{p}{2}\right)_{3}^{-1}\left(\tfrac{1}{2}\right)_{3}-\left(\tfrac{p}{2}\right)_{4}^{-1}\left(\tfrac{1}{2}\right)_{3}\left(\tfrac{1}{2}\right)_{1}\right)\\&=\tfrac{1}{p-1}\left(\tfrac{15}{p\left(p+2\right)\left(p+4\right)}-\tfrac{15}{p\left(p+2\right)\left(p+4\right)\left(p+6\right)}\right)
=\tfrac{15\left(p+5\right)}{\left(p-1\right)p\left(p+2\right)\left(p+4\right)\left(p+6\right)}\\&\\%
\left\langle \begin{smallmatrix}
3 & 3\\
1 & 1\\
\overrightarrow{0} & \overrightarrow{0}
\end{smallmatrix}\right\rangle 
&=
\left\langle \begin{smallmatrix}
3 & 1\\
3 & 1\\
\overrightarrow{0} & \overrightarrow{0}
\end{smallmatrix}\right\rangle =\left(\tfrac{p-1}{2}\left(\tfrac{p+1}{2}\right)\right)^{-1}\left(\left(-1\right)^{2}\left(\tfrac{1}{2}\right)_{2}%
\left\langle \begin{smallmatrix}
6\\
2\\
\overrightarrow{0}
\end{smallmatrix}\right\rangle +
\binom{3}{2}
\left(-1\right)^{1}\left(\tfrac{1}{2}\right)_{1}\left(\tfrac{1}{2}\right)_{1}%
\left\langle \begin{smallmatrix}
4\\
2\\
\overrightarrow{0}
\end{smallmatrix}\right\rangle \right)\\&=\tfrac{4}{\left(p-1\right)\left(p+1\right)}\left(\tfrac{3}{4}\left(\tfrac{p}{2}\right)_{4}^{-1}\left(\tfrac{1}{2}\right)_{3}\left(\tfrac{1}{2}\right)_{1}-\tfrac{3!}{2!1!}\cdot\tfrac{3}{8}\left(\tfrac{p}{2}\right)_{3}^{-1}\left(\tfrac{1}{2}\right)_{1}\left(\tfrac{1}{2}\right)_{1}\right)\\&=\tfrac{4}{\left(p-1\right)\left(p+1\right)}\left(\tfrac{3}{4}\cdot\tfrac{15}{16}\left(\tfrac{p}{2}\right)_{4}^{-1}-\tfrac{9}{8}\cdot\tfrac{1}{4}\left(\tfrac{p}{2}\right)_{3}^{-1}\right)\\&=\tfrac{4}{\left(p-1\right)\left(p+1\right)}\left(\tfrac{45}{64}\cdot\left(\tfrac{p}{2}\cdot\tfrac{p+2}{2}\cdot\tfrac{p+4}{2}\cdot\tfrac{p+6}{2}\right)^{-1}-\tfrac{9}{32}\cdot\left(\tfrac{p}{2}\cdot\tfrac{p+2}{2}\cdot\tfrac{p+4}{2}\right)^{-1}\right)\\&=\tfrac{1}{\left(p-1\right)\left(p+1\right)}\left(\tfrac{45}{p\left(p+2\right)\left(p+4\right)\left(p+6\right)}-\tfrac{9}{p\left(p+2\right)\left(p+4\right)}\right)
=\tfrac{-9\left(p+1\right)}{\left(p-1\right)p\left(p+1\right)\left(p+2\right)\left(p+4\right)\left(p+6\right)}
\end{align*}

\newpage

\begin{align*}
\left\langle \begin{smallmatrix}
3 & 1\\
1 & 3\\
\overrightarrow{0} & \overrightarrow{0}
\end{smallmatrix}\right\rangle &=\left(\tfrac{p-1}{2}\left(\tfrac{p+1}{2}\right)\right)^{-1}\left(\left(-1\right)^{2}\left(\tfrac{1}{2}\right)_{2}%
\left\langle \begin{smallmatrix}
4\\
4\\
\overrightarrow{0}
\end{smallmatrix}\right\rangle +
\binom{3}{2}\left(-1\right)^{1}\left(\tfrac{1}{2}\right)_{1}^{2}%
\left\langle \begin{smallmatrix}
4\\
2\\
\overrightarrow{0}
\end{smallmatrix}\right\rangle \right)\\&=\tfrac{3}{\left(p-1\right)\left(p+1\right)}\left(%
\left\langle \begin{smallmatrix}
4\\
4\\
\overrightarrow{0}
\end{smallmatrix}\right\rangle -%
\left\langle \begin{smallmatrix}
4\\
2\\
\overrightarrow{0}
\end{smallmatrix}\right\rangle \right)=\tfrac{3}{\left(p-1\right)\left(p+1\right)}\left(\left(\tfrac{p}{2}\right)_{4}^{-1}\left(\tfrac{1}{2}\right)_{2}^{2}-\left(\tfrac{p}{2}\right)_{3}^{-1}\left(\tfrac{1}{2}\right)_{1}\left(\tfrac{1}{2}\right)_{2}\right)\\&=\tfrac{3}{\left(p-1\right)\left(p+1\right)}\left(\tfrac{9}{p\left(p+2\right)\left(p+4\right)\left(p+6\right)}-\tfrac{3}{p\left(p+2\right)\left(p+4\right)}\right)
=\tfrac{-9\left(p+3\right)}{\left(p-1\right)p\left(p+1\right)\left(p+2\right)\left(p+4\right)\left(p+6\right)}
\\
&
\\
\left\langle \begin{smallmatrix}
3 & 1\\
1 & 1\\
\overrightarrow{0} & \overrightarrow{0}
\end{smallmatrix}\right\rangle &=\tfrac{1}{p-1}\left(\left(-1\right)^{1}\left(\tfrac{1}{2}\right)_{1}%
\left\langle \begin{smallmatrix}
4\\
2\\
\overrightarrow{0}
\end{smallmatrix}\right\rangle \right)=\tfrac{-1}{p-1}\left(\left(\tfrac{p}{2}\right)_{3}^{-1}\left(\tfrac{1}{2}\right)_{2}\left(\tfrac{1}{2}\right)_{1}\right)
=\tfrac{-3}{\left(p-1\right)p\left(p+2\right)\left(p+4\right)}\\&\\%
\left\langle \begin{smallmatrix}
5 & 1\\
1 & 1\\
\overrightarrow{0} & \overrightarrow{0}
\end{smallmatrix}\right\rangle &=\left(\tfrac{p-1}{2}\right)^{-1}\left(\left(-1\right)^{1}\left(\tfrac{1}{2}\right)_{1}%
\left\langle \begin{smallmatrix}
6\\
2\\
\overrightarrow{0}
\end{smallmatrix}\right\rangle \right)=\tfrac{1}{p-1}\left(-\left(\tfrac{p}{2}\right)_{4}^{-1}\left(\tfrac{1}{2}\right)_{1}\left(\tfrac{1}{2}\right)_{3}\right)
=\tfrac{-15}{\left(p-1\right)p\left(p+2\right)\left(p+4\right)\left(p+6\right)}\\&\\%
\left\langle \begin{smallmatrix}
4 & 2\\
0 & 2\\
\overrightarrow{0} & \overrightarrow{0}
\end{smallmatrix}\right\rangle 
&=
\tfrac{4}{\left(p-1\right)\left(p+1\right)}
\bigg(
\left(-1\right)^{2}\left(\tfrac{1}{2}\right)_{2}%
\left\langle \begin{smallmatrix}
6\\
2\\
\overrightarrow{0}
\end{smallmatrix}\right\rangle +\left(-1\right)^{1}\left(\tfrac{1}{2}\right)_{1}^{2}%
\left\langle \begin{smallmatrix}
4\\
2\\
\overrightarrow{0}
\end{smallmatrix}\right\rangle +
\\
&\eqmargin
\hspace{2cm}
\left(-1\right)^{1}\left(\tfrac{1}{2}\right)_{1}^{2}%
\left\langle \begin{smallmatrix}
6\\
0\\
\overrightarrow{0}
\end{smallmatrix}\right\rangle +\left(-1\right)^{0}\left(\tfrac{1}{2}\right)_{1}^{2}%
\left\langle \begin{smallmatrix}
4\\
0\\
\overrightarrow{0}
\end{smallmatrix}\right\rangle \bigg)
\\
&=\tfrac{1}{\left(p-1\right)\left(p+1\right)}\left(3%
\left\langle \begin{smallmatrix}
6\\
2\\
\overrightarrow{0}
\end{smallmatrix}\right\rangle -%
\left\langle \begin{smallmatrix}
4\\
2\\
\overrightarrow{0}
\end{smallmatrix}\right\rangle -%
\left\langle \begin{smallmatrix}
6\\
0\\
\overrightarrow{0}
\end{smallmatrix}\right\rangle +%
\left\langle \begin{smallmatrix}
4\\
0\\
\overrightarrow{0}
\end{smallmatrix}\right\rangle \right)
\\&
=\tfrac{1}{\left(p-1\right)\left(p+1\right)}\left(3\left(\tfrac{p}{2}\right)_{4}^{-1}\left(\tfrac{1}{2}\right)_{3}\left(\tfrac{1}{2}\right)_{1}-\left(\tfrac{p}{2}\right)_{3}^{-1}\left(\tfrac{1}{2}\right)_{2}\left(\tfrac{1}{2}\right)_{1}-\left(\tfrac{p}{2}\right)_{3}^{-1}\left(\tfrac{1}{2}\right)_{3}+\left(\tfrac{p}{2}\right)_{2}^{-1}\left(\tfrac{1}{2}\right)_{2}\right)\\&=\tfrac{1}{\left(p-1\right)\left(p+1\right)}\left(\tfrac{45}{p\left(p+2\right)\left(p+4\right)\left(p+6\right)}+\tfrac{-3-15}{p\left(p+2\right)\left(p+4\right)}+\tfrac{3}{p\left(p+2\right)}\right)\\&=\tfrac{3\left(p+1\right)\left(p+3\right)}{\left(p-1\right)p\left(p+1\right)\left(p+2\right)\left(p+4\right)\left(p+6\right)}=\tfrac{3\left(p+3\right)}{\left(p-1\right)p\left(p+2\right)\left(p+4\right)\left(p+6\right)}\\&\\%
\left\langle \begin{smallmatrix}
4 & 2\\
2 & 0\\
\overrightarrow{0} & \overrightarrow{0}
\end{smallmatrix}\right\rangle 
&=
\left(\tfrac{p-1}{2}\right)^{-1}\left(\left(-1\right)^{1}\left(\tfrac{1}{2}\right)_{1}%
\left\langle \begin{smallmatrix}
6\\
2\\
\overrightarrow{0}
\end{smallmatrix}\right\rangle +\left(-1\right)^{0}\left(\tfrac{1}{2}\right)_{1}%
\left\langle \begin{smallmatrix}
4\\
2\\
\overrightarrow{0}
\end{smallmatrix}\right\rangle \right)=\tfrac{1}{p-1}\left(%
\left\langle \begin{smallmatrix}
4\\
2\\
\overrightarrow{0}
\end{smallmatrix}\right\rangle -%
\left\langle \begin{smallmatrix}
6\\
2\\
\overrightarrow{0}
\end{smallmatrix}\right\rangle \right)
\\&
=\tfrac{1}{p-1}\left(\left(\tfrac{p}{2}\right)_{3}^{-1}\left(\tfrac{1}{2}\right)_{2}\left(\tfrac{1}{2}\right)_{1}-\left(\tfrac{p}{2}\right)_{4}^{-1}\left(\tfrac{1}{2}\right)_{3}\left(\tfrac{1}{2}\right)_{1}\right)
\\&
=\tfrac{1}{p-1}\left(\tfrac{3}{p\left(p+2\right)\left(p+4\right)}-\tfrac{15}{p\left(p+2\right)\left(p+4\right)\left(p+6\right)}\right)
=\tfrac{3\left(p+1\right)}{\left(p-1\right)p\left(p+2\right)\left(p+4\right)\left(p+6\right)}
\\&\\%
\left\langle \begin{smallmatrix}
2 & 1\\
0 & 1\\
\overrightarrow{0} & \overrightarrow{0}
\end{smallmatrix}\right\rangle &=\tfrac{1}{p-1}\left(-%
\left\langle \begin{smallmatrix}
3\\
1\\
\overrightarrow{0}
\end{smallmatrix}\right\rangle \right)=0\\&\\%
\left\langle \begin{smallmatrix}
2 & 0\\
0 & 2\\
\overrightarrow{0} & \overrightarrow{0}
\end{smallmatrix}\right\rangle &=\tfrac{2}{p-1}\left(\left(-1\right)^{1}\left(\tfrac{1}{2}\right)_{1}%
\left\langle \begin{smallmatrix}
2\\
2\\
\overrightarrow{0}
\end{smallmatrix}\right\rangle +\left(-1\right)^{0}\left(\tfrac{1}{2}\right)_{1}%
\left\langle \begin{smallmatrix}
2\\
0\\
\overrightarrow{0}
\end{smallmatrix}\right\rangle \right)=\tfrac{1}{p-1}\left(%
\left\langle \begin{smallmatrix}
2\\
0\\
\overrightarrow{0}
\end{smallmatrix}\right\rangle -%
\left\langle \begin{smallmatrix}
2\\
2\\
\overrightarrow{0}
\end{smallmatrix}\right\rangle \right)\\&=\tfrac{1}{p-1}\left(\tfrac{1}{p}-\tfrac{1}{p\left(p+2\right)}\right)=\tfrac{p+1}{\left(p-1\right)p\left(p+2\right)}\\&\\%
\left\langle \begin{smallmatrix}
1 & 1\\
1 & 1\\
\overrightarrow{0} & \overrightarrow{0}
\end{smallmatrix}\right\rangle &=\tfrac{1}{p-1}\left(-%
\left\langle \begin{smallmatrix}
2\\
2\\
\overrightarrow{0}
\end{smallmatrix}\right\rangle \right)=\tfrac{1}{p-1}\left(-\left(\tfrac{p}{2}\right)_{2}^{-1}\left(\tfrac{1}{2}\right)_{1}^{2}\right)
=\tfrac{-1}{\left(p-1\right)p\left(p+2\right)}
\\&\\%
\left\langle \begin{smallmatrix}
2 & 2\\
2 & 0\\
\overrightarrow{0} & \overrightarrow{0}
\end{smallmatrix}\right\rangle &=\left(\tfrac{p-1}{2}\right)_{1}^{-1}\left(\left(-1\right)^{1}\left(\tfrac{1}{2}\right)_{1}%
\left\langle \begin{smallmatrix}
4\\
2\\
\overrightarrow{0}
\end{smallmatrix}\right\rangle +\left(-1\right)^{0}\left(\tfrac{1}{2}\right)_{1}\left(\tfrac{1}{2}\right)_{0}%
\left\langle \begin{smallmatrix}
2\\
2\\
\overrightarrow{0}
\end{smallmatrix}\right\rangle \right)\\&=\tfrac{1}{p-1}\left(%
\left\langle \begin{smallmatrix}
2\\
2\\
\overrightarrow{0}
\end{smallmatrix}\right\rangle -%
\left\langle \begin{smallmatrix}
4\\
2\\
\overrightarrow{0}
\end{smallmatrix}\right\rangle \right)=\tfrac{1}{p-1}\left(\left(\tfrac{p}{2}\right)_{2}^{-1}\left(\tfrac{1}{2}\right)_{1}^{2}-\left(\tfrac{p}{2}\right)_{3}^{-1}\left(\tfrac{1}{2}\right)_{1}\left(\tfrac{1}{2}\right)_{2}\right)\\&=\tfrac{\left(p+4\right)-3}{\left(p-1\right)p\left(p+2\right)\left(p+4\right)}=\tfrac{p+1}{\left(p-1\right)p\left(p+2\right)\left(p+4\right)}\\&\\%
\left\langle \begin{smallmatrix}
2 & 2\\
4 & 0\\
\overrightarrow{0} & \overrightarrow{0}
\end{smallmatrix}\right\rangle &=%
\left\langle \begin{smallmatrix}
4 & 2\\
0 & 2\\
\overrightarrow{0} & \overrightarrow{0}
\end{smallmatrix}\right\rangle =\tfrac{3p+9}{\left(p-1\right)p\left(p+2\right)\left(p+4\right)\left(p+6\right)}
\end{align*}

\newpage

\subsubsection{Three indices (rows)}

\begin{align*}
    %
\left\langle \begin{smallmatrix}
4 & 0\\
0 & 2\\
0 & 2\\
\overrightarrow{0} & \overrightarrow{0}
\end{smallmatrix}\right\rangle &=
\left(\tfrac{p-1}{2}\left(\tfrac{p+1}{2}\right)\right)^{-1}
\bigg(\left(-1\right)^{2}\left(\tfrac{1}{2}\right)_{2}%
\left\langle \begin{smallmatrix}
4\\
2\\
2\\
\overrightarrow{0}
\end{smallmatrix}\right\rangle +
\\
&
\eqmargin\hspace{3cm}
\left(-1\right)^{1}\left(\tfrac{1}{2}\right)_{1}^{2}%
\left\langle \begin{smallmatrix}
4\\
0\\
2\\
\overrightarrow{0}
\end{smallmatrix}\right\rangle +\left(-1\right)^{1}\left(\tfrac{1}{2}\right)_{1}^{2}%
\left\langle \begin{smallmatrix}
4\\
2\\
0\\
\overrightarrow{0}
\end{smallmatrix}\right\rangle +\left(-1\right)^{0}\left(\tfrac{1}{2}\right)_{1}^{2}%
\left\langle \begin{smallmatrix}
4\\
0\\
0\\
\overrightarrow{0}
\end{smallmatrix}\right\rangle \bigg)\\&=\tfrac{1}{\left(p-1\right)\left(p+1\right)}\left(3%
\left\langle \begin{smallmatrix}
4\\
2\\
2\\
\overrightarrow{0}
\end{smallmatrix}\right\rangle -2%
\left\langle \begin{smallmatrix}
4\\
2\\
0\\
\overrightarrow{0}
\end{smallmatrix}\right\rangle +%
\left\langle \begin{smallmatrix}
4\\
0\\
0\\
\overrightarrow{0}
\end{smallmatrix}\right\rangle \right)\\&=\tfrac{1}{\left(p-1\right)\left(p+1\right)}\left(3\left(\tfrac{p}{2}\right)_{4}^{-1}\left(\tfrac{1}{2}\right)_{1}^{2}\left(\tfrac{1}{2}\right)_{2}-2\left(\tfrac{p}{2}\right)_{3}^{-1}\left(\tfrac{1}{2}\right)_{1}\left(\tfrac{1}{2}\right)_{2}+\left(\tfrac{p}{2}\right)_{2}^{-1}\left(\tfrac{1}{2}\right)_{2}\right)\\&=\tfrac{1}{\left(p-1\right)\left(p+1\right)}\left(\tfrac{9}{p\left(p+2\right)\left(p+4\right)\left(p+6\right)}-\tfrac{6}{p\left(p+2\right)\left(p+4\right)}+\tfrac{3}{p\left(p+2\right)}\right)\\&=\tfrac{3\left(p+3\right)\left(p+5\right)}{\left(p-1\right)p\left(p+1\right)\left(p+2\right)\left(p+4\right)\left(p+6\right)}\\&\\%
\left\langle \begin{smallmatrix}
2 & 0\\
4 & 0\\
0 & 2\\
\overrightarrow{0} & \overrightarrow{0}
\end{smallmatrix}\right\rangle &=\left(\tfrac{p-1}{2}\right)^{-1}\left(\left(-1\right)^{1}\left(\tfrac{1}{2}\right)_{1}%
\left\langle \begin{smallmatrix}
2\\
4\\
2\\
\overrightarrow{0}
\end{smallmatrix}\right\rangle +\left(-1\right)^{0}\left(\tfrac{1}{2}\right)_{1}%
\left\langle \begin{smallmatrix}
2\\
4\\
0\\
\overrightarrow{0}
\end{smallmatrix}\right\rangle \right)=\tfrac{1}{p-1}\left(%
\left\langle \begin{smallmatrix}
2\\
4\\
0\\
\overrightarrow{0}
\end{smallmatrix}\right\rangle -%
\left\langle \begin{smallmatrix}
2\\
4\\
2\\
\overrightarrow{0}
\end{smallmatrix}\right\rangle \right)\\&=\tfrac{1}{p-1}\left(\left(\tfrac{p}{2}\right)_{3}^{-1}\left(\tfrac{1}{2}\right)_{1}\left(\tfrac{1}{2}\right)_{2}-\left(\tfrac{p}{2}\right)_{4}^{-1}\left(\tfrac{1}{2}\right)_{1}^{2}\left(\tfrac{1}{2}\right)_{2}\right)=\tfrac{3\left(p+5\right)}{\left(p-1\right)p\left(p+2\right)\left(p+4\right)\left(p+6\right)}\\&\\%
\left\langle \begin{smallmatrix}
2 & 2\\
1 & 1\\
1 & 1\\
\overrightarrow{0} & \overrightarrow{0}
\end{smallmatrix}\right\rangle &=\left(\tfrac{p-1}{2}\right)_{2}^{-1}\left(\left(-1\right)^{2}\left(\tfrac{1}{2}\right)_{2}%
\left\langle \begin{smallmatrix}
4\\
2\\
2\\
\overrightarrow{0}
\end{smallmatrix}\right\rangle +\left(-1\right)^{1}\left(\tfrac{1}{2}\right)_{1}^{2}%
\left\langle \begin{smallmatrix}
2\\
2\\
2\\
\overrightarrow{0}
\end{smallmatrix}\right\rangle \right)
\\&
=\tfrac{4}{\left(p-1\right)\left(p+1\right)}\left(\tfrac{3}{4}%
\left\langle \begin{smallmatrix}
4\\
2\\
2\\
\overrightarrow{0}
\end{smallmatrix}\right\rangle -\tfrac{1}{4}%
\left\langle \begin{smallmatrix}
2\\
2\\
2\\
\overrightarrow{0}
\end{smallmatrix}\right\rangle \right)
=\tfrac{1}{\left(p-1\right)\left(p+1\right)}\left(3%
\left\langle \begin{smallmatrix}
4\\
2\\
2\\
\overrightarrow{0}
\end{smallmatrix}\right\rangle -%
\left\langle \begin{smallmatrix}
2\\
2\\
2\\
\overrightarrow{0}
\end{smallmatrix}\right\rangle \right)
\\&=\tfrac{1}{\left(p-1\right)\left(p+1\right)}\left(3\left(\tfrac{p}{2}\right)_{4}^{-1}\left(\tfrac{1}{2}\right)_{2}\left(\tfrac{1}{2}\right)_{1}^{2}-\left(\tfrac{p}{2}\right)_{3}^{-1}\left(\tfrac{1}{2}\right)_{1}^{3}\right)
=\tfrac{-p+3}{\left(p-1\right)p\left(p+1\right)\left(p+2\right)\left(p+4\right)\left(p+6\right)}
\\&\\%
\left\langle \begin{smallmatrix}
0 & 2\\
2 & 0\\
2 & 0\\
\overrightarrow{0} & \overrightarrow{0}
\end{smallmatrix}\right\rangle &=\left(\tfrac{p-1}{2}\right)^{-1}\left(\left(-1\right)^{1}\left(\tfrac{1}{2}\right)_{1}%
\left\langle \begin{smallmatrix}
2\\
2\\
2\\
\overrightarrow{0}
\end{smallmatrix}\right\rangle +\left(-1\right)^{0}\left(\tfrac{1}{2}\right)_{1}%
\left\langle \begin{smallmatrix}
0\\
2\\
2\\
\overrightarrow{0}
\end{smallmatrix}\right\rangle \right)=\tfrac{1}{\left(p-1\right)}\left(%
\left\langle \begin{smallmatrix}
0\\
2\\
2\\
\overrightarrow{0}
\end{smallmatrix}\right\rangle -%
\left\langle \begin{smallmatrix}
2\\
2\\
2\\
\overrightarrow{0}
\end{smallmatrix}\right\rangle \right)\\&=\tfrac{1}{\left(p-1\right)}\left(\left(\tfrac{p}{2}\right)_{2}^{-1}\left(\tfrac{1}{2}\right)_{1}^{2}-\left(\tfrac{p}{2}\right)_{3}^{-1}\left(\tfrac{1}{2}\right)_{1}^{3}\right)=\tfrac{1}{\left(p-1\right)}\left(\tfrac{\left(p+4\right)-1}{p\left(p+2\right)\left(p+4\right)}\right)=\tfrac{p+3}{\left(p-1\right)p\left(p+2\right)\left(p+4\right)}\\&\\%
\left\langle \begin{smallmatrix}
2 & 2\\
2 & 0\\
2 & 0\\
\overrightarrow{0} & \overrightarrow{0}
\end{smallmatrix}\right\rangle &=\left(\tfrac{p-1}{2}\right)_{1}^{-1}\left(\left(-1\right)^{1}\left(\tfrac{1}{2}\right)_{1}%
\left\langle \begin{smallmatrix}
4\\
2\\
2\\
\overrightarrow{0}
\end{smallmatrix}\right\rangle +\left(-1\right)^{0}\left(\tfrac{1}{2}\right)_{1}\left(\tfrac{1}{2}\right)_{0}%
\left\langle \begin{smallmatrix}
2\\
2\\
2\\
\overrightarrow{0}
\end{smallmatrix}\right\rangle \right)=\tfrac{1}{p-1}\left(%
\left\langle \begin{smallmatrix}
2\\
2\\
2\\
\overrightarrow{0}
\end{smallmatrix}\right\rangle 
\!-\!
\left\langle \begin{smallmatrix}
4\\
2\\
2\\
\overrightarrow{0}
\end{smallmatrix}\right\rangle \right)
\\&=
\tfrac{1}{p-1}\left(\left(\tfrac{p}{2}\right)_{3}^{-1}\left(\tfrac{1}{2}\right)_{1}^{3}-\left(\tfrac{p}{2}\right)_{4}^{-1}\left(\tfrac{1}{2}\right)_{1}^{2}\left(\tfrac{1}{2}\right)_{2}\right)
=
\tfrac{p+3}{\left(p-1\right)p\left(p+2\right)\left(p+4\right)\left(p+6\right)}
\\&\\%
\left\langle \begin{smallmatrix}
3 & 1\\
1 & 1\\
2 & 0\\
\overrightarrow{0} & \overrightarrow{0}
\end{smallmatrix}\right\rangle &=\left(\tfrac{p-1}{2}\right)^{-1}\left(\left(-1\right)^{1}\left(\tfrac{1}{2}\right)_{1}%
\left\langle \begin{smallmatrix}
4\\
2\\
2\\
\overrightarrow{0}
\end{smallmatrix}\right\rangle \right)=-\tfrac{3}{p-1}\left(\left(\tfrac{p}{2}\right)_{4}^{-1}\left(\tfrac{1}{2}\right)_{2}\left(\tfrac{1}{2}\right)_{1}^{2}\right)
\\&
=-\tfrac{3}{\left(p-1\right)p\left(p+2\right)\left(p+4\right)\left(p+6\right)}
\\&\\%
\left\langle \begin{smallmatrix}
3 & 1\\
1 & 1\\
0 & 2\\
\overrightarrow{0} & \overrightarrow{0}
\end{smallmatrix}\right\rangle &=\left(\tfrac{p-1}{2}\left(\tfrac{p+1}{2}\right)\right)^{-1}\left(\left(-1\right)^{2}\left(\tfrac{1}{2}\right)_{2}%
\left\langle \begin{smallmatrix}
4\\
2\\
2\\
\overrightarrow{0}
\end{smallmatrix}\right\rangle +\left(-1\right)^{1}\left(\tfrac{1}{2}\right)_{1}^{2}%
\left\langle \begin{smallmatrix}
4\\
2\\
0\\
\overrightarrow{0}
\end{smallmatrix}\right\rangle \right)\\&=\tfrac{1}{\left(p-1\right)\left(p+1\right)}\left(3\left(\tfrac{p}{2}\right)_{4}^{-1}\left(\tfrac{1}{2}\right)_{2}\left(\tfrac{1}{2}\right)_{1}^{2}-\left(\tfrac{p}{2}\right)_{3}^{-1}\left(\tfrac{1}{2}\right)_{2}\left(\tfrac{1}{2}\right)_{1}\right)\\&=\tfrac{1}{\left(p-1\right)\left(p+1\right)}\left(\tfrac{9}{p\left(p+2\right)\left(p+4\right)\left(p+6\right)}-\tfrac{3}{p\left(p+2\right)\left(p+4\right)}\right)\\&=\tfrac{-3\left(p+3\right)}{\left(p-1\right)p\left(p+1\right)\left(p+2\right)\left(p+4\right)\left(p+6\right)}
\end{align*}

\newpage

\begin{align*}
&\\%
\left\langle \begin{smallmatrix}
4 & 0\\
1 & 1\\
1 & 1\\
\overrightarrow{0} & \overrightarrow{0}
\end{smallmatrix}\right\rangle &=\left(\tfrac{p-1}{2}\right)_{1}^{-1}\left(-1\right)^{1}\left(\tfrac{1}{2}\right)_{1}%
\left\langle \begin{smallmatrix}
4\\
2\\
2\\
\overrightarrow{0}
\end{smallmatrix}\right\rangle =-\tfrac{1}{p-1}%
\left\langle \begin{smallmatrix}
4\\
2\\
2\\
\overrightarrow{0}
\end{smallmatrix}\right\rangle =-\tfrac{1}{p-1}\left(\tfrac{p}{2}\right)_{4}^{-1}\left(\tfrac{1}{2}\right)_{2}\left(\tfrac{1}{2}\right)_{1}^{2}
\\&
=-\tfrac{3}{\left(p-1\right)p\left(p+2\right)\left(p+4\right)\left(p+6\right)}\\&\\%
\left\langle \begin{smallmatrix}
2 & 2\\
2 & 0\\
0 & 2\\
\overrightarrow{0} & \overrightarrow{0}
\end{smallmatrix}\right\rangle &=
\left(\tfrac{p-1}{2}\right)_{2}^{-1}
\bigg(\left(-1\right)^{2}\left(\tfrac{1}{2}\right)_{2}%
\left\langle \begin{smallmatrix}
4\\
2\\
2\\
\overrightarrow{0}
\end{smallmatrix}\right\rangle +
\\
&
\eqmargin\hspace{3cm}
\left(-1\right)^{1}\left(\tfrac{1}{2}\right)_{1}^{2}%
\left\langle \begin{smallmatrix}
2\\
2\\
2\\
\overrightarrow{0}
\end{smallmatrix}\right\rangle +\left(-1\right)^{1}\left(\tfrac{1}{2}\right)_{1}^{2}%
\left\langle \begin{smallmatrix}
4\\
2\\
0\\
\overrightarrow{0}
\end{smallmatrix}\right\rangle +\left(-1\right)^{0}\left(\tfrac{1}{2}\right)_{1}^{2}%
\left\langle \begin{smallmatrix}
2\\
2\\
0\\
\overrightarrow{0}
\end{smallmatrix}\right\rangle \bigg)
\\&
=\tfrac{1}{\left(p-1\right)\left(p+1\right)}\left(3%
\left\langle \begin{smallmatrix}
4\\
2\\
2\\
\overrightarrow{0}
\end{smallmatrix}\right\rangle -%
\left\langle \begin{smallmatrix}
2\\
2\\
2\\
\overrightarrow{0}
\end{smallmatrix}\right\rangle -%
\left\langle \begin{smallmatrix}
4\\
2\\
0\\
\overrightarrow{0}
\end{smallmatrix}\right\rangle +%
\left\langle \begin{smallmatrix}
2\\
2\\
0\\
\overrightarrow{0}
\end{smallmatrix}\right\rangle \right)\\&=\tfrac{1}{\left(p-1\right)\left(p+1\right)}\left(3\left(\tfrac{p}{2}\right)_{4}^{-1}\left(\tfrac{1}{2}\right)_{2}\left(\tfrac{1}{2}\right)_{1}^{2}-\left(\tfrac{p}{2}\right)_{3}^{-1}\left(\tfrac{1}{2}\right)_{1}^{3}-\left(\tfrac{p}{2}\right)_{3}^{-1}\left(\tfrac{1}{2}\right)_{1}\left(\tfrac{1}{2}\right)_{2}+\left(\tfrac{p}{2}\right)_{2}^{-1}\left(\tfrac{1}{2}\right)_{1}^{2}\right)\\&=\tfrac{9-\left(p+6\right)-3\left(p+6\right)+\left(p+4\right)\left(p+6\right)}{\left(p-1\right)p\left(p+1\right)\left(p+2\right)\left(p+4\right)\left(p+6\right)}=\tfrac{\left(p+3\right)^{2}}{\left(p-1\right)p\left(p+1\right)\left(p+2\right)\left(p+4\right)\left(p+6\right)}
\\&\\%
\left\langle \begin{smallmatrix}
2 & 0\\
1 & 1\\
1 & 1\\
\overrightarrow{0} & \overrightarrow{0}
\end{smallmatrix}\right\rangle &=\left(\tfrac{p-1}{2}\right)_{1}^{-1}\left(-1\right)^{1}\left(\tfrac{1}{2}\right)_{1}%
\left\langle \begin{smallmatrix}
2\\
2\\
2\\
\overrightarrow{0}
\end{smallmatrix}\right\rangle =-\tfrac{1}{p-1}%
\left\langle \begin{smallmatrix}
2\\
2\\
2\\
\overrightarrow{0}
\end{smallmatrix}\right\rangle =-\tfrac{1}{p-1}\left(\tfrac{p}{2}\right)_{3}^{-1}\left(\tfrac{1}{2}\right)_{1}^{3}
\\
&
=-\tfrac{1}{\left(p-1\right)p\left(p+2\right)\left(p+4\right)}
\end{align*}

\newpage

\subsubsection{Four indices (rows)}

\begin{align*}
\left\langle \begin{smallmatrix}
2 & 0\\
0 & 2\\
1 & 1\\
1 & 1\\
\overrightarrow{0} & \overrightarrow{0}
\end{smallmatrix}\right\rangle &=\left(\tfrac{p-1}{2}\right)_{2}^{-1}\left(\left(-1\right)^{2}\left(\tfrac{1}{2}\right)_{2}%
\left\langle \begin{smallmatrix}
2\\
2\\
2\\
2\\
\overrightarrow{0}
\end{smallmatrix}\right\rangle +\left(-1\right)^{1}\left(\tfrac{1}{2}\right)_{1}^{2}%
\left\langle \begin{smallmatrix}
2\\
0\\
2\\
2\\
\overrightarrow{0}
\end{smallmatrix}\right\rangle \right)
\\&
=\tfrac{1}{\left(p-1\right)\left(p+1\right)}\left(3%
\left\langle \begin{smallmatrix}
2\\
2\\
2\\
2\\
\overrightarrow{0}
\end{smallmatrix}\right\rangle -%
\left\langle \begin{smallmatrix}
2\\
0\\
2\\
2\\
\overrightarrow{0}
\end{smallmatrix}\right\rangle \right)
=
\tfrac{1}{\left(p-1\right)\left(p+1\right)}\left(3\left(\tfrac{p}{2}\right)_{4}^{-1}\left(\tfrac{1}{2}\right)_{1}^{4}-\left(\tfrac{p}{2}\right)_{3}^{-1}\left(\tfrac{1}{2}\right)_{1}^{3}\right)
\\&
=\tfrac{-p-3}{\left(p-1\right)p\left(p+1\right)\left(p+2\right)\left(p+4\right)\left(p+6\right)}\\&\\%
\left\langle \begin{smallmatrix}
2 & 0\\
2 & 0\\
1 & 1\\
1 & 1\\
\overrightarrow{0} & \overrightarrow{0}
\end{smallmatrix}\right\rangle &=\left(\tfrac{p-1}{2}\right)_{1}^{-1}\left(-1\right)^{1}\left(\tfrac{1}{2}\right)_{1}%
\left\langle \begin{smallmatrix}
2\\
2\\
2\\
2\\
\overrightarrow{0}
\end{smallmatrix}\right\rangle =-\tfrac{1}{p-1}%
\left\langle \begin{smallmatrix}
2\\
2\\
2\\
2\\
\overrightarrow{0}
\end{smallmatrix}\right\rangle =-\tfrac{1}{p-1}\left(\tfrac{p}{2}\right)_{4}^{-1}\left(\tfrac{1}{2}\right)_{1}^{4}
\\&
=-\tfrac{1}{\left(p-1\right)p\left(p+2\right)\left(p+4\right)\left(p+6\right)}
\\&\\%
\left\langle \begin{smallmatrix}
2 & 0\\
2 & 0\\
0 & 2\\
0 & 2\\
\overrightarrow{0} & \overrightarrow{0}
\end{smallmatrix}\right\rangle &=\left(\tfrac{p-1}{2}\left(\tfrac{p+1}{2}\right)\right)^{-1}\left(\left(-1\right)^{2}\left(\tfrac{1}{2}\right)_{2}%
\left\langle \begin{smallmatrix}
2\\
2\\
2\\
2\\
\overrightarrow{0}
\end{smallmatrix}\right\rangle +2\left(-1\right)^{1}\left(\tfrac{1}{2}\right)_{1}^{2}%
\left\langle \begin{smallmatrix}
2\\
2\\
2\\
0\\
\overrightarrow{0}
\end{smallmatrix}\right\rangle +\left(-1\right)^{0}\left(\tfrac{1}{2}\right)_{1}^{2}%
\left\langle \begin{smallmatrix}
2\\
2\\
0\\
0\\
\overrightarrow{0}
\end{smallmatrix}\right\rangle \right)
\\&=\tfrac{1}{\left(p-1\right)\left(p+1\right)}\left(3%
\left\langle \begin{smallmatrix}
2\\
2\\
2\\
2\\
\overrightarrow{0}
\end{smallmatrix}\right\rangle -2%
\left\langle \begin{smallmatrix}
2\\
2\\
2\\
0\\
\overrightarrow{0}
\end{smallmatrix}\right\rangle +%
\left\langle \begin{smallmatrix}
2\\
2\\
0\\
0\\
\overrightarrow{0}
\end{smallmatrix}\right\rangle \right)
\\ &
=\tfrac{1}{\left(p-1\right)\left(p+1\right)}\left(3\left(\tfrac{p}{2}\right)_{4}^{-1}\left(\tfrac{1}{2}\right)_{1}^{4}-2\left(\tfrac{p}{2}\right)_{3}^{-1}\left(\tfrac{1}{2}\right)_{1}^{3}+\left(\tfrac{p}{2}\right)_{2}^{-1}\left(\tfrac{1}{2}\right)_{1}^{2}\right)\\&=\tfrac{1}{\left(p-1\right)\left(p+1\right)}\left(\tfrac{3}{p\left(p+2\right)\left(p+4\right)\left(p+6\right)}-\tfrac{2}{p\left(p+2\right)\left(p+4\right)}+\tfrac{1}{p\left(p+2\right)}\right)\\&=\tfrac{\left(p+3\right)\left(p+5\right)}{\left(p-1\right)p\left(p+1\right)\left(p+2\right)\left(p+4\right)\left(p+6\right)}\\&\\%
\left\langle \begin{smallmatrix}
1 & 1\\
1 & 1\\
1 & 1\\
1 & 1\\
\overrightarrow{0} & \overrightarrow{0}
\end{smallmatrix}\right\rangle &=\left(\tfrac{p-1}{2}\right)_{2}^{-1}
\left(-1\right)^{2}\left(\tfrac{1}{2}\right)_{2}%
\left\langle \begin{smallmatrix}
2\\
2\\
2\\
2\\
\overrightarrow{0}
\end{smallmatrix}\right\rangle
=\tfrac{3}{\left(p-1\right)\left(p+1\right)}
\left(\tfrac{p}{2}\right)_{4}^{-1}\left(\tfrac{1}{2}\right)_{1}^{4}
=
\tfrac{3}{\left(p-1\right)p\left(p+1\right)\left(p+2\right)\left(p+4\right)\left(p+6\right)}
\\&\\%
\left\langle \begin{smallmatrix}
2 & 0\\
2 & 0\\
2 & 0\\
0 & 2\\
\overrightarrow{0} & \overrightarrow{0}
\end{smallmatrix}\right\rangle &=%
\left\langle \begin{smallmatrix}
2 & 0\\
0 & 2\\
2 & 0\\
2 & 0\\
\overrightarrow{0} & \overrightarrow{0}
\end{smallmatrix}\right\rangle 
=\left(\tfrac{p-1}{2}\right)^{-1}\left(\left(-1\right)^{1}\left(\tfrac{1}{2}\right)_{1}%
\left\langle \begin{smallmatrix}
2\\
2\\
2\\
2\\
\overrightarrow{0}
\end{smallmatrix}\right\rangle +\left(-1\right)^{0}\left(\tfrac{1}{2}\right)_{1}%
\left\langle \begin{smallmatrix}
2\\
0\\
2\\
2\\
\overrightarrow{0}
\end{smallmatrix}\right\rangle \right)
\\ &
=
\tfrac{1}{p-1}\left(%
\left\langle \begin{smallmatrix}
2\\
0\\
2\\
2\\
\overrightarrow{0}
\end{smallmatrix}\right\rangle -%
\left\langle \begin{smallmatrix}
2\\
2\\
2\\
2\\
\overrightarrow{0}
\end{smallmatrix}\right\rangle \right)=\tfrac{1}{p-1}\left(\left(\tfrac{p}{2}\right)_{3}^{-1}\left(\tfrac{1}{2}\right)_{1}^{3}-\left(\tfrac{p}{2}\right)_{4}^{-1}\left(\tfrac{1}{2}\right)_{1}^{4}\right)\\&=\tfrac{1}{p-1}\left(\tfrac{1}{p\left(p+2\right)\left(p+4\right)}-\tfrac{1}{p\left(p+2\right)\left(p+4\right)\left(p+6\right)}\right)=\tfrac{p+5}{\left(p-1\right)p\left(p+2\right)\left(p+4\right)\left(p+6\right)}
\end{align*}

\newpage

\subsection{Monomials of Three Orthogonal Vector}

\subsubsection{One index (row)}
\begin{align*}
    %
\left\langle \begin{smallmatrix}
2 & 2 & 2\\
\overrightarrow{0} & \overrightarrow{0} & \overrightarrow{0}
\end{smallmatrix}\right\rangle &=\tfrac{1}{p\left(p+2\right)\left(p+4\right)}\\&\\%
\left\langle \begin{smallmatrix}
4 & 2 & 2\\
\overrightarrow{0} & \overrightarrow{0} & \overrightarrow{0}
\end{smallmatrix}\right\rangle &=\tfrac{1}{p-2}\left(\left(-1\right)^{1}\left(%
\left\langle \begin{smallmatrix}
6 & 2\\
\overrightarrow{0} & \overrightarrow{0}
\end{smallmatrix}\right\rangle +%
\left\langle \begin{smallmatrix}
4 & 4\\
\overrightarrow{0} & \overrightarrow{0}
\end{smallmatrix}\right\rangle \right)+\left(-1\right)^{0}%
\left\langle \begin{smallmatrix}
4 & 2\\
\overrightarrow{0} & \overrightarrow{0}
\end{smallmatrix}\right\rangle \right)\\&=\tfrac{1}{p-2}\left(%
\left\langle \begin{smallmatrix}
4 & 2\\
\overrightarrow{0} & \overrightarrow{0}
\end{smallmatrix}\right\rangle -%
\left\langle \begin{smallmatrix}
6 & 2\\
\overrightarrow{0} & \overrightarrow{0}
\end{smallmatrix}\right\rangle -%
\left\langle \begin{smallmatrix}
4 & 4\\
\overrightarrow{0} & \overrightarrow{0}
\end{smallmatrix}\right\rangle \right)
=\tfrac{1}{p-2}\left(\tfrac{3}{p\left(p+2\right)\left(p+4\right)}
+\tfrac{-15-9}{p\left(p+2\right)\left(p+4\right)\left(p+6\right)}\right)
\\&
=\tfrac{3\left(p-2\right)}{\left(p-2\right)p\left(p+2\right)\left(p+4\right)\left(p+6\right)}=\tfrac{3}{p\left(p+2\right)\left(p+4\right)\left(p+6\right)}
\end{align*}

\subsubsection{Two indices (rows)}
\begin{align*}
    %
\left\langle \begin{smallmatrix}
2 & 2 & 2\\
2 & 0 & 0\\
\overrightarrow{0} & \overrightarrow{0} & \overrightarrow{0}
\end{smallmatrix}\right\rangle &=%
\left\langle \begin{smallmatrix}
2 & 2\\
2 & 0\\
2 & 0\\
\overrightarrow{0} & \overrightarrow{0}
\end{smallmatrix}\right\rangle =\tfrac{p+3}{\left(p-1\right)p\left(p+2\right)\left(p+4\right)\left(p+6\right)}\\&\\%
\left\langle \begin{smallmatrix}
2 & 2 & 0\\
0 & 0 & 2\\
\overrightarrow{0} & \overrightarrow{0} & \overrightarrow{0}
\end{smallmatrix}\right\rangle &=%
\left\langle \begin{smallmatrix}
0 & 2\\
2 & 0\\
2 & 0\\
\overrightarrow{0} & \overrightarrow{0}
\end{smallmatrix}\right\rangle =\tfrac{\left(p+3\right)}{\left(p-1\right)p\left(p+2\right)\left(p+4\right)}
\\&
\\
\left\langle \begin{smallmatrix}
0 & 2 & 2\\
4 & 0 & 0\\
\overrightarrow{0} & \overrightarrow{0} & \overrightarrow{0}
\end{smallmatrix}\right\rangle &=%
\left\langle \begin{smallmatrix}
4 & 0\\
0 & 2\\
0 & 2\\
\overrightarrow{0} & \overrightarrow{0}
\end{smallmatrix}\right\rangle =\tfrac{3\left(p+3\right)\left(p+5\right)}{\left(p-1\right)p\left(p+1\right)\left(p+2\right)\left(p+4\right)\left(p+6\right)}\\&\\%
\left\langle \begin{smallmatrix}
2 & 1 & 1\\
0 & 1 & 1\\
\overrightarrow{0} & \overrightarrow{0} & \overrightarrow{0}
\end{smallmatrix}\right\rangle &=\tfrac{1}{p-2}\left(-%
\left\langle \begin{smallmatrix}
2 & 2\\
0 & 2\\
\overrightarrow{0} & \overrightarrow{0}
\end{smallmatrix}\right\rangle -%
\left\langle \begin{smallmatrix}
3 & 1\\
1 & 1\\
\overrightarrow{0} & \overrightarrow{0}
\end{smallmatrix}\right\rangle \right)\\&=\tfrac{1}{p-2}\left(-\tfrac{p+1}{\left(p-1\right)p\left(p+2\right)\left(p+4\right)}+\tfrac{3}{\left(p-1\right)p\left(p+2\right)\left(p+4\right)}\right)
=
\tfrac{-1}{\left(p-1\right)p\left(p+2\right)\left(p+4\right)}
\\&\\%
\left\langle \begin{smallmatrix}
4 & 1 & 1\\
0 & 1 & 1\\
\overrightarrow{0} & \overrightarrow{0} & \overrightarrow{0}
\end{smallmatrix}\right\rangle &=%
\left\langle \begin{smallmatrix}
4 & 0\\
1 & 1\\
1 & 1\\
\overrightarrow{0} & \overrightarrow{0}
\end{smallmatrix}\right\rangle =\tfrac{-3}{\left(p-1\right)p\left(p+2\right)\left(p+4\right)\left(p+6\right)}\\&\\%
\left\langle \begin{smallmatrix}
2 & 1 & 1\\
2 & 1 & 1\\
\overrightarrow{0} & \overrightarrow{0} & \overrightarrow{0}
\end{smallmatrix}\right\rangle &=%
\left\langle \begin{smallmatrix}
2 & 2\\
1 & 1\\
1 & 1\\
\overrightarrow{0} & \overrightarrow{0}
\end{smallmatrix}\right\rangle =\tfrac{-\left(p-3\right)}{\left(p-1\right)p\left(p+1\right)\left(p+2\right)\left(p+4\right)\left(p+6\right)}\\&\\%
\left\langle \begin{smallmatrix}
3 & 1 & 2\\
1 & 1 & 0\\
\overrightarrow{0} & \overrightarrow{0} & \overrightarrow{0}
\end{smallmatrix}\right\rangle &=%
\left\langle \begin{smallmatrix}
3 & 1\\
1 & 1\\
2 & 0\\
\overrightarrow{0} & \overrightarrow{0}
\end{smallmatrix}\right\rangle =\tfrac{-3}{\left(p-1\right)p\left(p+2\right)\left(p+4\right)\left(p+6\right)}\\&\\%
\left\langle \begin{smallmatrix}
1 & 1 & 2\\
3 & 1 & 0\\
\overrightarrow{0} & \overrightarrow{0} & \overrightarrow{0}
\end{smallmatrix}\right\rangle &=%
\left\langle \begin{smallmatrix}
3 & 1\\
1 & 1\\
0 & 2\\
\overrightarrow{0} & \overrightarrow{0}
\end{smallmatrix}\right\rangle =\tfrac{-3\left(p+3\right)}{\left(p-1\right)p\left(p+1\right)\left(p+2\right)\left(p+4\right)\left(p+6\right)}\\&\\%
\left\langle \begin{smallmatrix}
2 & 2 & 0\\
2 & 0 & 2\\
\overrightarrow{0} & \overrightarrow{0} & \overrightarrow{0}
\end{smallmatrix}\right\rangle &=%
\left\langle \begin{smallmatrix}
2 & 2\\
2 & 0\\
0 & 2\\
\overrightarrow{0} & \overrightarrow{0}
\end{smallmatrix}\right\rangle =\tfrac{\left(p+3\right)^{2}}{\left(p-1\right)p\left(p+1\right)\left(p+2\right)\left(p+4\right)\left(p+6\right)}\\&\\%
\left\langle \begin{smallmatrix}
4 & 2 & 0\\
0 & 0 & 2\\
\overrightarrow{0} & \overrightarrow{0} & \overrightarrow{0}
\end{smallmatrix}\right\rangle &=\tfrac{1}{p-2}\left(\left(-1\right)^{1}\left(%
\left\langle \begin{smallmatrix}
4 & 2\\
0 & 2\\
\overrightarrow{0} & \overrightarrow{0}
\end{smallmatrix}\right\rangle +%
\left\langle \begin{smallmatrix}
4 & 2\\
2 & 0\\
\overrightarrow{0} & \overrightarrow{0}
\end{smallmatrix}\right\rangle \right)+\left(-1\right)^{0}%
\left\langle \begin{smallmatrix}
4 & 2\\
0 & 0\\
\overrightarrow{0} & \overrightarrow{0}
\end{smallmatrix}\right\rangle \right)\\&=\tfrac{1}{p-2}\left(%
\left\langle \begin{smallmatrix}
4 & 2\\
0 & 0\\
\overrightarrow{0} & \overrightarrow{0}
\end{smallmatrix}\right\rangle -%
\left\langle \begin{smallmatrix}
4 & 2\\
0 & 2\\
\overrightarrow{0} & \overrightarrow{0}
\end{smallmatrix}\right\rangle -%
\left\langle \begin{smallmatrix}
4 & 2\\
2 & 0\\
\overrightarrow{0} & \overrightarrow{0}
\end{smallmatrix}\right\rangle \right)\\&=\tfrac{1}{p-2}\left(\tfrac{3\left(p-1\right)\left(p+6\right)-3\left(p+3\right)-3\left(p+1\right)}{\left(p-1\right)p\left(p+2\right)\left(p+4\right)\left(p+6\right)}\right)
=\tfrac{3\left(p+5\right)}{\left(p-1\right)p\left(p+2\right)\left(p+4\right)\left(p+6\right)}
\end{align*}

\newpage

\subsubsection{Three indices (rows)}
\begin{align*}
    %
    %
\left\langle \begin{smallmatrix}
2 & 0 & 0\\
0 & 2 & 0\\
0 & 0 & 2\\
\overrightarrow{0} & \overrightarrow{0} & \overrightarrow{0}
\end{smallmatrix}\right\rangle &=\tfrac{1}{p-2}\left(%
\left\langle \begin{smallmatrix}
2 & 0\\
0 & 2\\
0 & 0\\
\overrightarrow{0} & \overrightarrow{0}
\end{smallmatrix}\right\rangle -%
\left\langle \begin{smallmatrix}
2 & 0\\
0 & 2\\
2 & 0\\
\overrightarrow{0} & \overrightarrow{0}
\end{smallmatrix}\right\rangle -%
\left\langle \begin{smallmatrix}
2 & 0\\
0 & 2\\
0 & 2\\
\overrightarrow{0} & \overrightarrow{0}
\end{smallmatrix}\right\rangle \right)=\tfrac{1}{p-2}\left(%
\left\langle \begin{smallmatrix}
2 & 0\\
0 & 2\\
0 & 0\\
\overrightarrow{0} & \overrightarrow{0}
\end{smallmatrix}\right\rangle -2%
\left\langle \begin{smallmatrix}
2 & 0\\
0 & 2\\
0 & 2\\
\overrightarrow{0} & \overrightarrow{0}
\end{smallmatrix}\right\rangle \right)
\\&=\tfrac{1}{p-2}\left(\tfrac{p+1}{\left(p-1\right)p\left(p+2\right)}-\tfrac{2\left(p+3\right)}{\left(p-1\right)p\left(p+2\right)\left(p+4\right)}\right)
=\tfrac{p^{2}+3p-2}{\left(p-2\right)\left(p-1\right)p\left(p+2\right)\left(p+4\right)}
\\&
\\%
\left\langle \begin{smallmatrix}
4 & 0 & 0\\
0 & 2 & 0\\
0 & 0 & 2\\
\overrightarrow{0} & \overrightarrow{0} & \overrightarrow{0}
\end{smallmatrix}\right\rangle &=\tfrac{1}{p-2}\left(%
\left\langle \begin{smallmatrix}
4 & 0\\
0 & 2\\
0 & 0\\
\overrightarrow{0} & \overrightarrow{0}
\end{smallmatrix}\right\rangle -%
\left\langle \begin{smallmatrix}
4 & 0\\
0 & 2\\
2 & 0\\
\overrightarrow{0} & \overrightarrow{0}
\end{smallmatrix}\right\rangle -%
\left\langle \begin{smallmatrix}
4 & 0\\
0 & 2\\
0 & 2\\
\overrightarrow{0} & \overrightarrow{0}
\end{smallmatrix}\right\rangle \right)\\&=\tfrac{1}{p-2}\left(\tfrac{3\left(p+3\right)}{\left(p-1\right)p\left(p+2\right)\left(p+4\right)}-\tfrac{3\left(p+5\right)}{\left(p-1\right)p\left(p+2\right)\left(p+4\right)\left(p+6\right)}-\tfrac{3\left(p+3\right)\left(p+5\right)}{\left(p-1\right)p\left(p+1\right)\left(p+2\right)\left(p+4\right)\left(p+6\right)}\right)\\&=\tfrac{1}{p-2}\left(\tfrac{3\left(p+1\right)\left(p+3\right)\left(p+6\right)-3\left(p+1\right)\left(p+5\right)-3\left(p+3\right)\left(p+5\right)}{\left(p-1\right)p\left(p+1\right)\left(p+2\right)\left(p+4\right)\left(p+6\right)}\right)\\&=\tfrac{3\left(p^{3}+8p^{2}+13p-2\right)}{\left(p-2\right)\left(p-1\right)p\left(p+1\right)\left(p+2\right)\left(p+4\right)\left(p+6\right)}\\
&\\%
\left\langle \begin{smallmatrix}
0 & 2 & 2\\
2 & 0 & 0\\
2 & 0 & 0\\
\overrightarrow{0} & \overrightarrow{0} & \overrightarrow{0}
\end{smallmatrix}\right\rangle 
&
=\tfrac{1}{p-2}\left(\left(-1\right)^{1}\left(%
\left\langle \begin{smallmatrix}
2 & 2\\
2 & 0\\
2 & 0\\
\overrightarrow{0} & \overrightarrow{0}
\end{smallmatrix}\right\rangle +%
\left\langle \begin{smallmatrix}
0 & 4\\
2 & 0\\
2 & 0\\
\overrightarrow{0} & \overrightarrow{0}
\end{smallmatrix}\right\rangle \right)+\left(-1\right)^{0}%
\left\langle \begin{smallmatrix}
0 & 2\\
2 & 0\\
2 & 0\\
\overrightarrow{0} & \overrightarrow{0}
\end{smallmatrix}\right\rangle \right)
\\
&
=\tfrac{1}{p-2}\left(%
\left\langle \begin{smallmatrix}
0 & 2\\
2 & 0\\
2 & 0\\
\overrightarrow{0} & \overrightarrow{0}
\end{smallmatrix}\right\rangle -%
\left\langle \begin{smallmatrix}
2 & 2\\
2 & 0\\
2 & 0\\
\overrightarrow{0} & \overrightarrow{0}
\end{smallmatrix}\right\rangle -%
\left\langle \begin{smallmatrix}
0 & 4\\
2 & 0\\
2 & 0\\
\overrightarrow{0} & \overrightarrow{0}
\end{smallmatrix}\right\rangle \right)\\&=\tfrac{1}{p-2}\left(\tfrac{p+3}{\left(p-1\right)p\left(p+2\right)\left(p+4\right)}-\tfrac{p+3}{\left(p-1\right)p\left(p+2\right)\left(p+4\right)\left(p+6\right)}-\tfrac{3\left(p+3\right)\left(p+5\right)}{\left(p-1\right)p\left(p+1\right)\left(p+2\right)\left(p+4\right)\left(p+6\right)}\right)\\&=\tfrac{\left(p+3\right)\left(\left(p+1\right)\left(p+6\right)-\left(p+1\right)-3\left(p+5\right)\right)}{\left(p-2\right)\left(p-1\right)p\left(p+1\right)\left(p+2\right)\left(p+4\right)\left(p+6\right)}=\tfrac{\left(p-2\right)\left(p+3\right)\left(p+5\right)}{\left(p-2\right)\left(p-1\right)p\left(p+1\right)\left(p+2\right)\left(p+4\right)\left(p+6\right)}\\&=\tfrac{\left(p+3\right)\left(p+5\right)}{\left(p-1\right)p\left(p+1\right)\left(p+2\right)\left(p+4\right)\left(p+6\right)}\\&\\%
\left\langle \begin{smallmatrix}
2 & 0 & 0\\
0 & 2 & 0\\
2 & 0 & 2\\
\overrightarrow{0} & \overrightarrow{0} & \overrightarrow{0}
\end{smallmatrix}\right\rangle &=\tfrac{1}{p-2}\left(%
\left\langle \begin{smallmatrix}
2 & 0\\
0 & 2\\
2 & 0\\
\overrightarrow{0} & \overrightarrow{0}
\end{smallmatrix}\right\rangle -%
\left\langle \begin{smallmatrix}
2 & 0\\
0 & 2\\
4 & 0\\
\overrightarrow{0} & \overrightarrow{0}
\end{smallmatrix}\right\rangle -%
\left\langle \begin{smallmatrix}
2 & 0\\
0 & 2\\
2 & 2\\
\overrightarrow{0} & \overrightarrow{0}
\end{smallmatrix}\right\rangle \right)\\&=\tfrac{1}{p-2}\left(\tfrac{p+3}{\left(p-1\right)p\left(p+2\right)\left(p+4\right)}-\tfrac{3\left(p+5\right)}{\left(p-1\right)p\left(p+2\right)\left(p+4\right)\left(p+6\right)}-\tfrac{\left(p+3\right)^{2}}{\left(p-1\right)p\left(p+1\right)\left(p+2\right)\left(p+4\right)\left(p+6\right)}\right)\\&=\tfrac{\left(p+1\right)\left(\left(p+3\right)\left(p+6\right)-3\left(p+5\right)\right)-\left(p+3\right)^{2}}{\left(p-2\right)\left(p-1\right)p\left(p+1\right)\left(p+2\right)\left(p+4\right)\left(p+6\right)}=\tfrac{p^{3}+6p^{2}+3p-6}{\left(p-2\right)\left(p-1\right)p\left(p+1\right)\left(p+2\right)\left(p+4\right)\left(p+6\right)}
\\
&
\\
\left\langle \begin{smallmatrix}
2 & 1 & 1\\
0 & 1 & 1\\
2 & 0 & 0\\
\overrightarrow{0} & \overrightarrow{0} & \overrightarrow{0}
\end{smallmatrix}\right\rangle &=\tfrac{1}{p-2}\left(-%
\left\langle \begin{smallmatrix}
3 & 1\\
1 & 1\\
2 & 0\\
\overrightarrow{0} & \overrightarrow{0}
\end{smallmatrix}\right\rangle -%
\left\langle \begin{smallmatrix}
2 & 2\\
0 & 2\\
2 & 0\\
\overrightarrow{0} & \overrightarrow{0}
\end{smallmatrix}\right\rangle \right)\\&=\tfrac{1}{p-2}\left(\tfrac{3}{\left(p-1\right)p\left(p+2\right)\left(p+4\right)\left(p+6\right)}-\tfrac{\left(p+3\right)^{2}}{\left(p-1\right)p\left(p+1\right)\left(p+2\right)\left(p+4\right)\left(p+6\right)}\right)\\&=\tfrac{3\left(p+1\right)-\left(p+3\right)^{2}}{\left(p-2\right)\left(p-1\right)p\left(p+1\right)\left(p+2\right)\left(p+4\right)\left(p+6\right)}
=\tfrac{-\left(p^{2}+3p+6\right)}{\left(p-2\right)\left(p-1\right)p\left(p+1\right)\left(p+2\right)\left(p+4\right)\left(p+6\right)}
\\
&\\%
\left\langle \begin{smallmatrix}
0 & 1 & 1\\
0 & 1 & 1\\
2 & 0 & 0\\
\overrightarrow{0} & \overrightarrow{0} & \overrightarrow{0}
\end{smallmatrix}\right\rangle &=\tfrac{1}{p-2}\left(-%
\left\langle \begin{smallmatrix}
1 & 1\\
1 & 1\\
2 & 0\\
\overrightarrow{0} & \overrightarrow{0}
\end{smallmatrix}\right\rangle -%
\left\langle \begin{smallmatrix}
0 & 2\\
0 & 2\\
2 & 0\\
\overrightarrow{0} & \overrightarrow{0}
\end{smallmatrix}\right\rangle \right)
\\&=\tfrac{1}{p-2}\left(\tfrac{1}{\left(p-1\right)p\left(p+2\right)\left(p+4\right)}-\tfrac{p+3}{\left(p-1\right)p\left(p+2\right)\left(p+4\right)}\right)
=\tfrac{-\left(p+2\right)}{\left(p-2\right)\left(p-1\right)p\left(p+2\right)\left(p+4\right)}
\\
&
\\%
\left\langle \begin{smallmatrix}
0 & 1 & 1\\
0 & 1 & 1\\
4 & 0 & 0\\
\overrightarrow{0} & \overrightarrow{0} & \overrightarrow{0}
\end{smallmatrix}\right\rangle &=\tfrac{1}{p-2}\left(-%
\left\langle \begin{smallmatrix}
1 & 1\\
1 & 1\\
4 & 0\\
\overrightarrow{0} & \overrightarrow{0}
\end{smallmatrix}\right\rangle -%
\left\langle \begin{smallmatrix}
0 & 2\\
0 & 2\\
4 & 0\\
\overrightarrow{0} & \overrightarrow{0}
\end{smallmatrix}\right\rangle \right)\\&=\tfrac{1}{p-2}\left(\tfrac{3}{\left(p-1\right)p\left(p+2\right)\left(p+4\right)\left(p+6\right)}-\tfrac{3\left(p+3\right)\left(p+5\right)}{\left(p-1\right)p\left(p+1\right)\left(p+2\right)\left(p+4\right)\left(p+6\right)}\right)\\&=\tfrac{1}{p-2}\left(\tfrac{3\left(p+1\right)-3\left(p+3\right)\left(p+5\right)}{\left(p-1\right)p\left(p+1\right)\left(p+2\right)\left(p+4\right)\left(p+6\right)}\right)
=\tfrac{-3\left(p^{2}+7p+14\right)}{\left(p-2\right)\left(p-1\right)p\left(p+1\right)\left(p+2\right)\left(p+4\right)\left(p+6\right)}
\end{align*}

\begin{align*}
\left\langle \begin{smallmatrix}
1 & 2 & 1\\
1 & 0 & 1\\
0 & 0 & 2\\
\overrightarrow{0} & \overrightarrow{0} & \overrightarrow{0}
\end{smallmatrix}\right\rangle &=\tfrac{1}{p-2}\left(-%
\left\langle \begin{smallmatrix}
3 & 1\\
1 & 1\\
0 & 2\\
\overrightarrow{0} & \overrightarrow{0}
\end{smallmatrix}\right\rangle -%
\left\langle \begin{smallmatrix}
2 & 2\\
0 & 2\\
0 & 2\\
\overrightarrow{0} & \overrightarrow{0}
\end{smallmatrix}\right\rangle \right)\\&=\tfrac{1}{p-2}\left(\tfrac{3\left(p+3\right)}{\left(p-1\right)p\left(p+1\right)\left(p+2\right)\left(p+4\right)\left(p+6\right)}-\tfrac{p+3}{\left(p-1\right)p\left(p+2\right)\left(p+4\right)\left(p+6\right)}\right)\\&=\tfrac{-\left(p+3\right)}{\left(p-1\right)p\left(p+1\right)\left(p+2\right)\left(p+4\right)\left(p+6\right)}\\&\\%
\left\langle \begin{smallmatrix}
1 & 3 & 0\\
1 & 1 & 0\\
0 & 0 & 2\\
\overrightarrow{0} & \overrightarrow{0} & \overrightarrow{0}
\end{smallmatrix}\right\rangle &=\tfrac{1}{p-2}\left(-%
\left\langle \begin{smallmatrix}
4 & 0\\
2 & 0\\
0 & 2\\
\overrightarrow{0} & \overrightarrow{0}
\end{smallmatrix}\right\rangle -%
\left\langle \begin{smallmatrix}
3 & 1\\
1 & 1\\
0 & 2\\
\overrightarrow{0} & \overrightarrow{0}
\end{smallmatrix}\right\rangle \right)\\&=\tfrac{1}{p-2}\left(-\tfrac{3\left(p+5\right)}{\left(p-1\right)p\left(p+2\right)\left(p+4\right)\left(p+6\right)}+\tfrac{3\left(p+3\right)}{\left(p-1\right)p\left(p+1\right)\left(p+2\right)\left(p+4\right)\left(p+6\right)}\right)\\&=\tfrac{1}{p-2}\left(\tfrac{3\left(p+3\right)-3\left(p+1\right)\left(p+5\right)}{\left(p-1\right)p\left(p+1\right)\left(p+2\right)\left(p+4\right)\left(p+6\right)}\right)\\&=\tfrac{-3\left(p^{2}+5p+2\right)}{\left(p-2\right)\left(p-1\right)p\left(p+1\right)\left(p+2\right)\left(p+4\right)\left(p+6\right)}\\&\\%
\left\langle \begin{smallmatrix}
2 & 2 & 0\\
0 & 1 & 1\\
0 & 1 & 1\\
\overrightarrow{0} & \overrightarrow{0} & \overrightarrow{0}
\end{smallmatrix}\right\rangle &=%
\left\langle \begin{smallmatrix}
2 & 1 & 1\\
0 & 1 & 1\\
2 & 0 & 0\\
\overrightarrow{0} & \overrightarrow{0} & \overrightarrow{0}
\end{smallmatrix}\right\rangle =\tfrac{-\left(p^{2}+3p+6\right)}{\left(p-2\right)\left(p-1\right)p\left(p+1\right)\left(p+2\right)\left(p+4\right)\left(p+6\right)}\\&\\%
\left\langle \begin{smallmatrix}
1 & 1 & 0\\
1 & 0 & 1\\
2 & 1 & 1\\
\overrightarrow{0} & \overrightarrow{0} & \overrightarrow{0}
\end{smallmatrix}\right\rangle &=%
\left\langle \begin{smallmatrix}
2 & 1 & 1\\
1 & 1 & 0\\
1 & 0 & 1\\
\overrightarrow{0} & \overrightarrow{0} & \overrightarrow{0}
\end{smallmatrix}\right\rangle =\tfrac{1}{p-2}\left(-%
\left\langle \begin{smallmatrix}
3 & 1\\
1 & 1\\
2 & 0\\
\overrightarrow{0} & \overrightarrow{0}
\end{smallmatrix}\right\rangle -%
\left\langle \begin{smallmatrix}
2 & 2\\
1 & 1\\
1 & 1\\
\overrightarrow{0} & \overrightarrow{0}
\end{smallmatrix}\right\rangle \right)\\&=\tfrac{1}{p-2}\left(\tfrac{3}{\left(p-1\right)p\left(p+2\right)\left(p+4\right)\left(p+6\right)}-\tfrac{-p+3}{\left(p-1\right)p\left(p+1\right)\left(p+2\right)\left(p+4\right)\left(p+6\right)}\right)\\&=\tfrac{p-3+3\left(p+1\right)}{\left(p-2\right)\left(p-1\right)p\left(p+1\right)\left(p+2\right)\left(p+4\right)\left(p+6\right)}\\&=\tfrac{4p}{\left(p-2\right)\left(p-1\right)p\left(p+1\right)\left(p+2\right)\left(p+4\right)\left(p+6\right)}\\&\\%
\left\langle \begin{smallmatrix}
3 & 1 & 0\\
0 & 1 & 1\\
1 & 0 & 1\\
\overrightarrow{0} & \overrightarrow{0} & \overrightarrow{0}
\end{smallmatrix}\right\rangle &=\tfrac{1}{p-2}\left(-%
\left\langle \begin{smallmatrix}
3 & 1\\
1 & 1\\
2 & 0\\
\overrightarrow{0} & \overrightarrow{0}
\end{smallmatrix}\right\rangle -%
\left\langle \begin{smallmatrix}
3 & 1\\
0 & 2\\
1 & 1\\
\overrightarrow{0} & \overrightarrow{0}
\end{smallmatrix}\right\rangle \right)\\&=\tfrac{1}{p-2}\left(\tfrac{3}{\left(p-1\right)p\left(p+2\right)\left(p+4\right)\left(p+6\right)}+\tfrac{3\left(p+3\right)}{\left(p-1\right)p\left(p+1\right)\left(p+2\right)\left(p+4\right)\left(p+6\right)}\right)\\&=\tfrac{1}{p-2}\left(\tfrac{3\left(p+1\right)+3\left(p+3\right)}{\left(p-1\right)p\left(p+1\right)\left(p+2\right)\left(p+4\right)\left(p+6\right)}\right)\\&=\tfrac{6\left(p+2\right)}{\left(p-2\right)\left(p-1\right)p\left(p+1\right)\left(p+2\right)\left(p+4\right)\left(p+6\right)}\\&\\%
\left\langle \begin{smallmatrix}
0 & 1 & 1\\
1 & 1 & 0\\
1 & 0 & 1\\
\overrightarrow{0} & \overrightarrow{0} & \overrightarrow{0}
\end{smallmatrix}\right\rangle &=\tfrac{1}{p-2}\left(-%
\left\langle \begin{smallmatrix}
1 & 1\\
1 & 1\\
2 & 0\\
\overrightarrow{0} & \overrightarrow{0}
\end{smallmatrix}\right\rangle -%
\left\langle \begin{smallmatrix}
0 & 2\\
1 & 1\\
1 & 1\\
\overrightarrow{0} & \overrightarrow{0}
\end{smallmatrix}\right\rangle \right)=\tfrac{1}{p-2}\left(-2%
\left\langle \begin{smallmatrix}
1 & 1\\
1 & 1\\
2 & 0\\
\overrightarrow{0} & \overrightarrow{0}
\end{smallmatrix}\right\rangle \right)=\tfrac{2}{\left(p-2\right)\left(p-1\right)p\left(p+2\right)\left(p+4\right)}
\end{align*}

\newpage

\subsubsection{Four indices (rows)}

\begin{align*}
    %
    %
\left\langle \begin{smallmatrix}
0 & 2 & 0\\
0 & 0 & 2\\
2 & 0 & 0\\
2 & 0 & 0\\
\overrightarrow{0} & \overrightarrow{0} & \overrightarrow{0}
\end{smallmatrix}\right\rangle &=\tfrac{1}{p-2}\left(%
\left\langle \begin{smallmatrix}
0 & 2\\
0 & 0\\
2 & 0\\
2 & 0\\
\overrightarrow{0} & \overrightarrow{0}
\end{smallmatrix}\right\rangle -%
\left\langle \begin{smallmatrix}
0 & 2\\
2 & 0\\
2 & 0\\
2 & 0\\
\overrightarrow{0} & \overrightarrow{0}
\end{smallmatrix}\right\rangle -%
\left\langle \begin{smallmatrix}
0 & 2\\
0 & 2\\
2 & 0\\
2 & 0\\
\overrightarrow{0} & \overrightarrow{0}
\end{smallmatrix}\right\rangle \right)\\&=\tfrac{1}{p-2}\left(\tfrac{p+3}{\left(p-1\right)p\left(p+2\right)\left(p+4\right)}-\tfrac{p+5}{\left(p-1\right)p\left(p+2\right)\left(p+4\right)\left(p+6\right)}-\tfrac{\left(p+3\right)\left(p+5\right)}{\left(p-1\right)p\left(p+1\right)\left(p+2\right)\left(p+4\right)\left(p+6\right)}\right)\\&=\tfrac{\left(p+1\right)\left(p+3\right)\left(p+6\right)-\left(p+1\right)\left(p+5\right)-\left(p+3\right)\left(p+5\right)}{\left(p-2\right)\left(p-1\right)p\left(p+1\right)\left(p+2\right)\left(p+4\right)\left(p+6\right)}\\&=\tfrac{p^{3}+8p^{2}+13p-2}{\left(p-2\right)\left(p-1\right)p\left(p+1\right)\left(p+2\right)\left(p+4\right)\left(p+6\right)}\\&\\%
\left\langle \begin{smallmatrix}
0 & 1 & 1\\
0 & 1 & 1\\
2 & 0 & 0\\
2 & 0 & 0\\
\overrightarrow{0} & \overrightarrow{0} & \overrightarrow{0}
\end{smallmatrix}\right\rangle &=\tfrac{1}{p-2}\left(-%
\left\langle \begin{smallmatrix}
1 & 1\\
1 & 1\\
2 & 0\\
2 & 0\\
\overrightarrow{0} & \overrightarrow{0}
\end{smallmatrix}\right\rangle -%
\left\langle \begin{smallmatrix}
0 & 2\\
0 & 2\\
2 & 0\\
2 & 0\\
\overrightarrow{0} & \overrightarrow{0}
\end{smallmatrix}\right\rangle \right)=\tfrac{1}{p-2}\left(\tfrac{\left(p+1\right)-\left(p+3\right)\left(p+5\right)}{\left(p-1\right)p\left(p+1\right)\left(p+2\right)\left(p+4\right)\left(p+6\right)}\right)\\&=\tfrac{-\left(p^{2}+7p+14\right)}{\left(p-2\right)\left(p-1\right)p\left(p+1\right)\left(p+2\right)\left(p+4\right)\left(p+6\right)}\\&\\%
\left\langle \begin{smallmatrix}
2 & 0 & 0\\
0 & 2 & 0\\
1 & 0 & 1\\
1 & 0 & 1\\
\overrightarrow{0} & \overrightarrow{0} & \overrightarrow{0}
\end{smallmatrix}\right\rangle &=\tfrac{1}{p-2}\left(-%
\left\langle \begin{smallmatrix}
2 & 0\\
0 & 2\\
2 & 0\\
2 & 0\\
\overrightarrow{0} & \overrightarrow{0}
\end{smallmatrix}\right\rangle -%
\left\langle \begin{smallmatrix}
2 & 0\\
0 & 2\\
1 & 1\\
1 & 1\\
\overrightarrow{0} & \overrightarrow{0}
\end{smallmatrix}\right\rangle \right)\\&=\tfrac{1}{p-2}\left(\tfrac{-\left(p+5\right)}{\left(p-1\right)p\left(p+2\right)\left(p+4\right)\left(p+6\right)}+\tfrac{p+3}{\left(p-1\right)p\left(p+1\right)\left(p+2\right)\left(p+4\right)\left(p+6\right)}\right)\\&=\tfrac{1}{p-2}\left(\tfrac{p+3-\left(p+5\right)\left(p+1\right)}{\left(p-1\right)p\left(p+1\right)\left(p+2\right)\left(p+4\right)\left(p+6\right)}\right)\\&=\tfrac{-\left(p^{2}+5p+2\right)}{\left(p-2\right)\left(p-1\right)p\left(p+1\right)\left(p+2\right)\left(p+4\right)\left(p+6\right)}\\&\\%
\left\langle \begin{smallmatrix}
1 & 0 & 1\\
1 & 0 & 1\\
1 & 1 & 0\\
1 & 1 & 0\\
\overrightarrow{0} & \overrightarrow{0} & \overrightarrow{0}
\end{smallmatrix}\right\rangle &=\tfrac{1}{p-2}\left(-%
\left\langle \begin{smallmatrix}
2 & 0\\
2 & 0\\
1 & 1\\
1 & 1\\
\overrightarrow{0} & \overrightarrow{0}
\end{smallmatrix}\right\rangle -%
\left\langle \begin{smallmatrix}
1 & 1\\
1 & 1\\
1 & 1\\
1 & 1\\
\overrightarrow{0} & \overrightarrow{0}
\end{smallmatrix}\right\rangle \right)\\&=\tfrac{1}{p-2}\left(\tfrac{1}{\left(p-1\right)p\left(p+2\right)\left(p+4\right)\left(p+6\right)}-\tfrac{3}{\left(p-1\right)p\left(p+1\right)\left(p+2\right)\left(p+4\right)\left(p+6\right)}\right)
\\
&=\tfrac{p-2}{\left(p-2\right)\left(p-1\right)p\left(p+1\right)\left(p+2\right)\left(p+4\right)\left(p+6\right)}
=\tfrac{1}{\left(p-1\right)p\left(p+1\right)\left(p+2\right)\left(p+4\right)\left(p+6\right)}\\&\\%
\left\langle \begin{smallmatrix}
2 & 0 & 0\\
0 & 1 & 1\\
1 & 1 & 0\\
1 & 0 & 1\\
\overrightarrow{0} & \overrightarrow{0} & \overrightarrow{0}
\end{smallmatrix}\right\rangle &=\tfrac{1}{p-2}\left(\tfrac{1}{\left(p-1\right)p\left(p+2\right)\left(p+4\right)\left(p+6\right)}+\tfrac{p+3}{\left(p-1\right)p\left(p+1\right)\left(p+2\right)\left(p+4\right)\left(p+6\right)}\right)
\\
&=\tfrac{2\left(p+2\right)}{\left(p-2\right)\left(p-1\right)p\left(p+1\right)\left(p+2\right)\left(p+4\right)\left(p+6\right)}
\end{align*}

\newpage

\subsection{Monomials of Four Orthogonal Vector}

\begin{align*}
    %
    %
\left\langle \begin{smallmatrix}
2 & 2 & 2 & 2\\
\overrightarrow{0} & \overrightarrow{0} & \overrightarrow{0} & \overrightarrow{0}
\end{smallmatrix}\right\rangle &=\tfrac{1}{p-3}\left(\left(-1\right)^{1}\left(%
\left\langle \begin{smallmatrix}
4 & 2 & 2\\
\overrightarrow{0} & \overrightarrow{0} & \overrightarrow{0}
\end{smallmatrix}\right\rangle +%
\left\langle \begin{smallmatrix}
2 & 4 & 2\\
\overrightarrow{0} & \overrightarrow{0} & \overrightarrow{0}
\end{smallmatrix}\right\rangle +%
\left\langle \begin{smallmatrix}
2 & 2 & 4\\
\overrightarrow{0} & \overrightarrow{0} & \overrightarrow{0}
\end{smallmatrix}\right\rangle \right)+\left(-1\right)^{0}%
\left\langle \begin{smallmatrix}
2 & 2 & 2\\
\overrightarrow{0} & \overrightarrow{0} & \overrightarrow{0}
\end{smallmatrix}\right\rangle \right)\\&=\tfrac{1}{p-3}\left(%
\left\langle \begin{smallmatrix}
2 & 2 & 2\\
\overrightarrow{0} & \overrightarrow{0} & \overrightarrow{0}
\end{smallmatrix}\right\rangle -3%
\left\langle \begin{smallmatrix}
4 & 2 & 2\\
\overrightarrow{0} & \overrightarrow{0} & \overrightarrow{0}
\end{smallmatrix}\right\rangle \right)\\&=\tfrac{1}{p-3}\left(\tfrac{1}{p\left(p+2\right)\left(p+4\right)}-\tfrac{9}{p\left(p+2\right)\left(p+4\right)\left(p+6\right)}\right)\\&=\tfrac{p-3}{\left(p-3\right)p\left(p+2\right)\left(p+4\right)\left(p+6\right)}=\tfrac{1}{p\left(p+2\right)\left(p+4\right)\left(p+6\right)}\\&\\%
\left\langle \begin{smallmatrix}
2 & 2 & 2 & 0\\
0 & 0 & 0 & 2\\
\overrightarrow{0} & \overrightarrow{0} & \overrightarrow{0} & \overrightarrow{0}
\end{smallmatrix}\right\rangle &=%
\left\langle \begin{smallmatrix}
2 & 0\\
0 & 2\\
2 & 0\\
2 & 0\\
\overrightarrow{0} & \overrightarrow{0}
\end{smallmatrix}\right\rangle =\tfrac{p+5}{\left(p-1\right)p\left(p+2\right)\left(p+4\right)\left(p+6\right)}\\&\\%
\left\langle \begin{smallmatrix}
2 & 2 & 0 & 0\\
0 & 0 & 2 & 2\\
\overrightarrow{0} & \overrightarrow{0} & \overrightarrow{0} & \overrightarrow{0}
\end{smallmatrix}\right\rangle &=%
\left\langle \begin{smallmatrix}
2 & 0\\
2 & 0\\
0 & 2\\
0 & 2\\
\overrightarrow{0} & \overrightarrow{0}
\end{smallmatrix}\right\rangle =\tfrac{\left(p+3\right)\left(p+5\right)}{\left(p-1\right)p\left(p+1\right)\left(p+2\right)\left(p+4\right)\left(p+6\right)}\\&\\%
\left\langle \begin{smallmatrix}
2 & 2 & 0 & 0\\
0 & 0 & 2 & 0\\
0 & 0 & 0 & 2\\
\overrightarrow{0} & \overrightarrow{0} & \overrightarrow{0} & \overrightarrow{0}
\end{smallmatrix}\right\rangle &=%
\left\langle \begin{smallmatrix}
0 & 2 & 0\\
0 & 0 & 2\\
2 & 0 & 0\\
2 & 0 & 0\\
\overrightarrow{0} & \overrightarrow{0} & \overrightarrow{0}
\end{smallmatrix}\right\rangle =\tfrac{\left(p^{3}+8p^{2}+13p-2\right)}{\left(p-2\right)\left(p-1\right)p\left(p+1\right)\left(p+2\right)\left(p+4\right)\left(p+6\right)}
\\
&
\\%
\left\langle \begin{smallmatrix}
2 & 0 & 0 & 0\\
0 & 2 & 0 & 0\\
0 & 0 & 2 & 0\\
0 & 0 & 0 & 2\\
\overrightarrow{0} & \overrightarrow{0} & \overrightarrow{0} & \overrightarrow{0}
\end{smallmatrix}\right\rangle 
&=
\tfrac{1}{p-3}\left(
\!\left(-1\right)^{1}
\!
\left(%
\left\langle \begin{smallmatrix}
2 & 0 & 0\\
0 & 2 & 0\\
0 & 0 & 2\\
2 & 0 & 0\\
\overrightarrow{0} & \overrightarrow{0} & \overrightarrow{0}
\end{smallmatrix}\right\rangle \!+\!%
\left\langle \begin{smallmatrix}
2 & 0 & 0\\
0 & 2 & 0\\
0 & 0 & 2\\
0 & 2 & 0\\
\overrightarrow{0} & \overrightarrow{0} & \overrightarrow{0}
\end{smallmatrix}\right\rangle \!+\!%
\left\langle \begin{smallmatrix}
2 & 0 & 0\\
0 & 2 & 0\\
0 & 0 & 2\\
0 & 0 & 2\\
\overrightarrow{0} & \overrightarrow{0} & \overrightarrow{0}
\end{smallmatrix}\right\rangle \right)
\!+\!
\left(-1\right)^{0}%
\left\langle \begin{smallmatrix}
2 & 0 & 0\\
0 & 2 & 0\\
0 & 0 & 2\\
0 & 0 & 0\\
\overrightarrow{0} & \overrightarrow{0} & \overrightarrow{0}
\end{smallmatrix}\right\rangle \right)\\&=\tfrac{1}{p-3}\left(%
\left\langle \begin{smallmatrix}
2 & 0 & 0\\
0 & 2 & 0\\
0 & 0 & 2\\
0 & 0 & 0\\
\overrightarrow{0} & \overrightarrow{0} & \overrightarrow{0}
\end{smallmatrix}\right\rangle -3%
\left\langle \begin{smallmatrix}
2 & 0 & 0\\
0 & 2 & 0\\
0 & 0 & 2\\
0 & 0 & 2\\
\overrightarrow{0} & \overrightarrow{0} & \overrightarrow{0}
\end{smallmatrix}\right\rangle \right)\\&=\tfrac{1}{p-3}\left(\tfrac{p^{2}+3p-2}{\left(p-2\right)\left(p-1\right)p\left(p+2\right)\left(p+4\right)}-\tfrac{3\left(p^{3}+8p^{2}+13p-2\right)}{\left(p-2\right)\left(p-1\right)p\left(p+1\right)\left(p+2\right)\left(p+4\right)\left(p+6\right)}\right)\\&=\tfrac{\left(p-2\right)\left(p+3\right)\left(p^{2}+6p+1\right)}{\left(p-3\right)\left(p-2\right)\left(p-1\right)p\left(p+1\right)\left(p+2\right)\left(p+4\right)\left(p+6\right)}\\&\\%
\left\langle \begin{smallmatrix}
1 & 1 & 2 & 2\\
1 & 1 & 0 & 0\\
\overrightarrow{0} & \overrightarrow{0} & \overrightarrow{0} & \overrightarrow{0}
\end{smallmatrix}\right\rangle &=%
\left\langle \begin{smallmatrix}
2 & 0\\
2 & 0\\
1 & 1\\
1 & 1\\
\overrightarrow{0} & \overrightarrow{0}
\end{smallmatrix}\right\rangle =\tfrac{-1}{\left(p-1\right)p\left(p+2\right)\left(p+4\right)\left(p+6\right)}\\&\\%
\left\langle \begin{smallmatrix}
1 & 1 & 2 & 0\\
1 & 1 & 0 & 2\\
\overrightarrow{0} & \overrightarrow{0} & \overrightarrow{0} & \overrightarrow{0}
\end{smallmatrix}\right\rangle &=%
\left\langle \begin{smallmatrix}
2 & 0\\
0 & 2\\
1 & 1\\
1 & 1\\
\overrightarrow{0} & \overrightarrow{0}
\end{smallmatrix}\right\rangle =\tfrac{-\left(p+3\right)}{\left(p-1\right)p\left(p+1\right)\left(p+2\right)\left(p+4\right)\left(p+6\right)}\\&\\%
\left\langle \begin{smallmatrix}
1 & 1 & 2 & 0\\
1 & 1 & 0 & 0\\
0 & 0 & 0 & 2\\
\overrightarrow{0} & \overrightarrow{0} & \overrightarrow{0} & \overrightarrow{0}
\end{smallmatrix}\right\rangle &=%
\left\langle \begin{smallmatrix}
2 & 0 & 0\\
0 & 2 & 0\\
1 & 0 & 1\\
1 & 0 & 1\\
\overrightarrow{0} & \overrightarrow{0} & \overrightarrow{0}
\end{smallmatrix}\right\rangle =\tfrac{-\left(p^{2}+5p+2\right)}{\left(p-2\right)\left(p-1\right)p\left(p+1\right)\left(p+2\right)\left(p+4\right)\left(p+6\right)}\\&\\%
\left\langle \begin{smallmatrix}
1 & 1 & 0 & 0\\
1 & 1 & 0 & 0\\
0 & 0 & 2 & 2\\
\overrightarrow{0} & \overrightarrow{0} & \overrightarrow{0} & \overrightarrow{0}
\end{smallmatrix}\right\rangle &=%
\left\langle \begin{smallmatrix}
0 & 1 & 1\\
0 & 1 & 1\\
2 & 0 & 0\\
2 & 0 & 0\\
\overrightarrow{0} & \overrightarrow{0} & \overrightarrow{0}
\end{smallmatrix}\right\rangle =\tfrac{-\left(p^{2}+7p+14\right)}{\left(p-2\right)\left(p-1\right)p\left(p+1\right)\left(p+2\right)\left(p+4\right)\left(p+6\right)}
\\
&\\%
\left\langle \begin{smallmatrix}
2 & 1 & 1 & 0\\
0 & 1 & 0 & 1\\
0 & 0 & 1 & 1\\
\overrightarrow{0} & \overrightarrow{0} & \overrightarrow{0} & \overrightarrow{0}
\end{smallmatrix}\right\rangle &=%
\left\langle \begin{smallmatrix}
2 & 0 & 0\\
0 & 1 & 1\\
1 & 1 & 0\\
1 & 0 & 1\\
\overrightarrow{0} & \overrightarrow{0} & \overrightarrow{0}
\end{smallmatrix}\right\rangle =\tfrac{2\left(p+2\right)}{\left(p-2\right)\left(p-1\right)p\left(p+1\right)\left(p+2\right)\left(p+4\right)\left(p+6\right)}\\
\end{align*}

\newpage

\begin{align*}
\left\langle \begin{smallmatrix}
1 & 1 & 0 & 0\\
1 & 1 & 0 & 0\\
0 & 0 & 2 & 0\\
0 & 0 & 0 & 2\\
\overrightarrow{0} & \overrightarrow{0} & \overrightarrow{0} & \overrightarrow{0}
\end{smallmatrix}\right\rangle &=\tfrac{1}{p-3}\left(-%
\left\langle \begin{smallmatrix}
2 & 0 & 0\\
2 & 0 & 0\\
0 & 2 & 0\\
0 & 0 & 2\\
\overrightarrow{0} & \overrightarrow{0} & \overrightarrow{0}
\end{smallmatrix}\right\rangle -%
\left\langle \begin{smallmatrix}
1 & 1 & 0\\
1 & 1 & 0\\
0 & 2 & 0\\
0 & 0 & 2\\
\overrightarrow{0} & \overrightarrow{0} & \overrightarrow{0}
\end{smallmatrix}\right\rangle -%
\left\langle \begin{smallmatrix}
1 & 0 & 1\\
1 & 0 & 1\\
0 & 2 & 0\\
0 & 0 & 2\\
\overrightarrow{0} & \overrightarrow{0} & \overrightarrow{0}
\end{smallmatrix}\right\rangle \right)\\&=\tfrac{1}{p-3}\left(-%
\left\langle \begin{smallmatrix}
2 & 0 & 0\\
2 & 0 & 0\\
0 & 2 & 0\\
0 & 0 & 2\\
\overrightarrow{0} & \overrightarrow{0} & \overrightarrow{0}
\end{smallmatrix}\right\rangle -2%
\left\langle \begin{smallmatrix}
1 & 1 & 0\\
1 & 1 & 0\\
0 & 2 & 0\\
0 & 0 & 2\\
\overrightarrow{0} & \overrightarrow{0} & \overrightarrow{0}
\end{smallmatrix}\right\rangle \right)\\&=\tfrac{1}{p-3}\left(-\tfrac{p^{3}+8p^{2}+13p-2}{\left(p-2\right)\left(p-1\right)p\left(p+1\right)\left(p+2\right)\left(p+4\right)\left(p+6\right)}+\tfrac{2\left(p^{2}+5p+2\right)}{\left(p-2\right)\left(p-1\right)p\left(p+1\right)\left(p+2\right)\left(p+4\right)\left(p+6\right)}\right)\\&=\tfrac{-\left(p^{3}+6p^{2}+3p-6\right)}{\left(p-3\right)\left(p-2\right)\left(p-1\right)p\left(p+1\right)\left(p+2\right)\left(p+4\right)\left(p+6\right)}
\\
&\\%
\left\langle \begin{smallmatrix}
2 & 0 & 0 & 0\\
0 & 1 & 1 & 0\\
0 & 1 & 0 & 1\\
0 & 0 & 1 & 1\\
\overrightarrow{0} & \overrightarrow{0} & \overrightarrow{0} & \overrightarrow{0}
\end{smallmatrix}\right\rangle &=\tfrac{1}{p-3}\left(-%
\left\langle \begin{smallmatrix}
2 & 0 & 0\\
0 & 1 & 1\\
1 & 1 & 0\\
1 & 0 & 1\\
\overrightarrow{0} & \overrightarrow{0} & \overrightarrow{0}
\end{smallmatrix}\right\rangle -\underbrace{%
\left\langle \begin{smallmatrix}
2 & 0 & 0\\
0 & 1 & 1\\
0 & 2 & 0\\
0 & 1 & 1\\
\overrightarrow{0} & \overrightarrow{0} & \overrightarrow{0}
\end{smallmatrix}\right\rangle -%
\left\langle \begin{smallmatrix}
2 & 0 & 0\\
0 & 1 & 1\\
0 & 1 & 1\\
0 & 0 & 2\\
\overrightarrow{0} & \overrightarrow{0} & \overrightarrow{0}
\end{smallmatrix}\right\rangle }_{\text{equal}}\right)\\&=\tfrac{1}{p-3}\left(\tfrac{2\left(p^{2}+5p+2\right)-2\left(p+2\right)}{\left(p-2\right)\left(p-1\right)p\left(p+1\right)\left(p+2\right)\left(p+4\right)\left(p+6\right)}\right)\\&=\tfrac{2p\left(p+4\right)}{\left(p-3\right)\left(p-2\right)\left(p-1\right)p\left(p+1\right)\left(p+2\right)\left(p+4\right)\left(p+6\right)}\\&\\%
\left\langle \begin{smallmatrix}
1 & 1 & 1 & 1\\
1 & 1 & 1 & 1\\
\overrightarrow{0} & \overrightarrow{0} & \overrightarrow{0} & \overrightarrow{0}
\end{smallmatrix}\right\rangle &=%
\left\langle \begin{smallmatrix}
1 & 1\\
1 & 1\\
1 & 1\\
1 & 1\\
\overrightarrow{0} & \overrightarrow{0}
\end{smallmatrix}\right\rangle =\tfrac{3}{\left(p-1\right)p\left(p+1\right)\left(p+2\right)\left(p+4\right)\left(p+6\right)}\\&\\%
\left\langle \begin{smallmatrix}
1 & 1 & 1 & 1\\
1 & 1 & 0 & 0\\
0 & 0 & 1 & 1\\
\overrightarrow{0} & \overrightarrow{0} & \overrightarrow{0} & \overrightarrow{0}
\end{smallmatrix}\right\rangle &=\tfrac{1}{p-3}\left(-%
\left\langle \begin{smallmatrix}
2 & 1 & 1\\
1 & 1 & 0\\
1 & 0 & 1\\
\overrightarrow{0} & \overrightarrow{0} & \overrightarrow{0}
\end{smallmatrix}\right\rangle -%
\left\langle \begin{smallmatrix}
1 & 2 & 1\\
1 & 1 & 0\\
0 & 1 & 1\\
\overrightarrow{0} & \overrightarrow{0} & \overrightarrow{0}
\end{smallmatrix}\right\rangle -%
\left\langle \begin{smallmatrix}
1 & 1 & 2\\
1 & 1 & 0\\
0 & 0 & 2\\
\overrightarrow{0} & \overrightarrow{0} & \overrightarrow{0}
\end{smallmatrix}\right\rangle \right)
\\
&=
\tfrac{1}{p-3}\left(-2%
\left\langle \begin{smallmatrix}
2 & 1 & 1\\
1 & 1 & 0\\
1 & 0 & 1\\
\overrightarrow{0} & \overrightarrow{0} & \overrightarrow{0}
\end{smallmatrix}\right\rangle -%
\left\langle \begin{smallmatrix}
1 & 1 & 2\\
1 & 1 & 0\\
0 & 0 & 2\\
\overrightarrow{0} & \overrightarrow{0} & \overrightarrow{0}
\end{smallmatrix}\right\rangle \right)\\&=\tfrac{1}{p-3}\left(\tfrac{\left(p^{2}+3p+6\right)-8p}{\left(p-2\right)\left(p-1\right)p\left(p+1\right)\left(p+2\right)\left(p+4\right)\left(p+6\right)}\right)
=\tfrac{1}{\left(p-1\right)p\left(p+1\right)\left(p+2\right)\left(p+4\right)\left(p+6\right)}\\&\\%
\left\langle \begin{smallmatrix}
1 & 1 & 0 & 0\\
1 & 1 & 0 & 0\\
0 & 0 & 1 & 1\\
0 & 0 & 1 & 1\\
\overrightarrow{0} & \overrightarrow{0} & \overrightarrow{0} & \overrightarrow{0}
\end{smallmatrix}\right\rangle &=\tfrac{1}{p-3}\left(-%
\left\langle \begin{smallmatrix}
1 & 1 & 0\\
1 & 1 & 0\\
1 & 0 & 1\\
1 & 0 & 1\\
\overrightarrow{0} & \overrightarrow{0} & \overrightarrow{0}
\end{smallmatrix}\right\rangle -%
\left\langle \begin{smallmatrix}
1 & 1 & 0\\
1 & 1 & 0\\
0 & 1 & 1\\
0 & 1 & 1\\
\overrightarrow{0} & \overrightarrow{0} & \overrightarrow{0}
\end{smallmatrix}\right\rangle -%
\left\langle \begin{smallmatrix}
1 & 1 & 0\\
1 & 1 & 0\\
0 & 0 & 2\\
0 & 0 & 2\\
\overrightarrow{0} & \overrightarrow{0} & \overrightarrow{0}
\end{smallmatrix}\right\rangle \right)
\\
&=
\tfrac{1}{p-3}\left(-2%
\left\langle \begin{smallmatrix}
1 & 1 & 0\\
1 & 1 & 0\\
0 & 1 & 1\\
0 & 1 & 1\\
\overrightarrow{0} & \overrightarrow{0} & \overrightarrow{0}
\end{smallmatrix}\right\rangle -%
\left\langle \begin{smallmatrix}
1 & 1 & 0\\
1 & 1 & 0\\
0 & 0 & 2\\
0 & 0 & 2\\
\overrightarrow{0} & \overrightarrow{0} & \overrightarrow{0}
\end{smallmatrix}\right\rangle \right)\\&=\tfrac{1}{p-3}\left(\tfrac{-2}{\left(p-1\right)p\left(p+1\right)\left(p+2\right)\left(p+4\right)\left(p+6\right)}+\tfrac{p^{2}+7p+14}{\left(p-2\right)\left(p-1\right)p\left(p+1\right)\left(p+2\right)\left(p+4\right)\left(p+6\right)}\right)\\&=\tfrac{p^{2}+5p+18}{\left(p-3\right)\left(p-2\right)\left(p-1\right)p\left(p+1\right)\left(p+2\right)\left(p+4\right)\left(p+6\right)}\\&\\%
\left\langle \begin{smallmatrix}
1 & 0 & 1 & 0\\
0 & 1 & 0 & 1\\
1 & 1 & 0 & 0\\
0 & 0 & 1 & 1\\
\overrightarrow{0} & \overrightarrow{0} & \overrightarrow{0} & \overrightarrow{0}
\end{smallmatrix}\right\rangle &=\tfrac{1}{p-3}\left(-%
\left\langle \begin{smallmatrix}
1 & 0 & 1\\
1 & 1 & 0\\
1 & 1 & 0\\
1 & 0 & 1\\
\overrightarrow{0} & \overrightarrow{0} & \overrightarrow{0}
\end{smallmatrix}\right\rangle -%
\left\langle \begin{smallmatrix}
1 & 0 & 1\\
0 & 2 & 0\\
1 & 1 & 0\\
0 & 1 & 1\\
\overrightarrow{0} & \overrightarrow{0} & \overrightarrow{0}
\end{smallmatrix}\right\rangle -%
\left\langle \begin{smallmatrix}
1 & 0 & 1\\
0 & 1 & 1\\
1 & 1 & 0\\
0 & 0 & 2\\
\overrightarrow{0} & \overrightarrow{0} & \overrightarrow{0}
\end{smallmatrix}\right\rangle \right)\\&=\tfrac{1}{p-3}\left(-\tfrac{1}{\left(p-1\right)p\left(p+1\right)\left(p+2\right)\left(p+4\right)\left(p+6\right)}+\tfrac{-4\left(p+2\right)}{\left(p-2\right)\left(p-1\right)p\left(p+1\right)\left(p+2\right)\left(p+4\right)\left(p+6\right)}\right)\\&=\tfrac{-5p-6}{\left(p-3\right)\left(p-2\right)\left(p-1\right)p\left(p+1\right)\left(p+2\right)\left(p+4\right)\left(p+6\right)}
\end{align*}

\newpage

\section{Auxiliary Expectation Derivations}
\label{app:auxiliary}
In this section, we attach many auxiliary derivations of simple and complicated polynomials that we need in our main propositions and lemmas.

\begin{proposition}
\label{prop:(e1Oe1)3}
For $p\ge 2, m\in\{2,\dots,p\}$ and a random transformation $\mO$ sampled as described in \eqref{eq:data-model},
it holds that,
    $$\mathbb{E}\left[\left(\ve_{1}^{\top}\mO\ve_{1}\right)^{3}\right]
=\frac{\left(p-m\right)\left(m^{2}-2mp-3m+p^{2}+6p+14\right)}{p\left(p+2\right)\left(p+4\right)}\,.
    $$
\end{proposition}

\begin{proof}	
We notice that the expectation can be written as
$$
\mathbb{E}\left[\left(\ve_{1}^{\top}\mO\ve_{1}\right)^{3}\right]
=
\mathbb{E}\left[
\prn{\vu^{\top}\left[\begin{smallmatrix}
\Q_{m}\\
 & \mathbf{0}
\end{smallmatrix}\right]\vu
+
\vu^{\top}
\left[\begin{smallmatrix}
\mathbf{0}\\
 & \I_{p-m}
\end{smallmatrix}\right]\vu
}^3\right]
\,.
$$
Employing the algebraic identity that $\prn{a+b}^3=a^3 + 3 a^2 b + 3 a b^2 + b^3$
and \corref{cor:odd-Q_m} (canceling terms with an odd number of $\Q_m$ appearances), 
it can be readily seen that in our case we are left with $\prn{a+b}^3=\cancel{a^3} + {3 a^2 b} + \cancel{3 a b^2} + b^3$.
We thus get,
\begin{align*}
&
\mathbb{E}\left[\left(\ve_{1}^{\top}\mO\ve_{1}\right)^{3}\right]
=
3\mathbb{E}\left[\left(\vu^{\top}\left[\begin{smallmatrix}
\Q_{m}\\
 & \mathbf{0}
\end{smallmatrix}\right]\vu\right)^{2}\vu^{\top}\left[\begin{smallmatrix}
\mathbf{0}\\
 & \I_{p-m}
\end{smallmatrix}\right]\vu\right]+\mathbb{E}\left[\left(\vu^{\top}\left[\begin{smallmatrix}
\mathbf{0}\\
 & \I_{p-m}
\end{smallmatrix}\right]\vu\right)^{3}\right]\\&=3\mathbb{E}\left[\left(\vu^{\top}\left[\begin{smallmatrix}
\Q_{m}\\
 & \mathbf{0}
\end{smallmatrix}\right]\vu\right)^{2}\sum_{i=m+1}^{p}u_{i}^{2}\right]+\mathbb{E}\left[\left(\sum_{i=m+1}^{p}u_{i}^{2}\right)^{3}\right]
\\
&
=
3\sum_{i=m+1}^{p}\mathbb{E}\left[\left(\vu^{\top}\left[\begin{smallmatrix}
\Q_{m}\\
 & \mathbf{0}
\end{smallmatrix}\right]\vu\right)^{2}u_{i}^{2}\right]+\sum_{i=m+1}^{p}\sum_{j=m+1}^{p}\sum_{k=m+1}^{p}\mathbb{E}\left[u_{i}^{2}u_{j}^{2}u_{k}^{2}\right]\\&=\left(p-m\right)\left(3\mathbb{E}\left[u_{p}^{2}\cdot\vu_{a}^{\top}\Q_{m}\vu_{a}\cdot\vu_{a}^{\top}\Q_{m}\vu_{a}\right]+\sum_{j=m+1}^{p}\sum_{k=m+1}^{p}\mathbb{E}\left[u_{p}^{2}u_{j}^{2}u_{k}^{2}\right]\right)
\end{align*}

The inner left term becomes
\begin{align*}
&
\mathbb{E}\left[u_{p}^{2}\left\Vert \vu_{a}\right\Vert ^{2}\cdot\vu_{a}^{\top}\mathbb{E}_{\mathbf{r}\sim\mathcal{S}^{m-1}}\left[\mathbf{r}\mathbf{r}^{\top}\right]\vu_{a}\right]
=
\frac{3}{m}\mathbb{E}\left[u_{p}^{2}\left\Vert \vu_{a}\right\Vert ^{4}\right]
=
\frac{3}{m}\sum_{i=1}^{m}\sum_{j=1}^{m}\mathbb{E}\left[u_{p}^{2}u_{i}^{2}u_{j}^{2}\right]
\\&
=
\frac{3}{m}
\left(\underbrace{m\mathbb{E}\left[u_{p}^{2}u_{1}^{4}\right]}_{i=j}+\underbrace{m\left(m-1\right)\mathbb{E}\left[u_{p}^{2}u_{1}^{2}u_{2}^{2}\right]}_{i\neq j}\right)
=
3\left(\left\langle \begin{smallmatrix}
4\\
2\\
\overrightarrow{0}
\end{smallmatrix}\right\rangle +\left(m-1\right)\left\langle \begin{smallmatrix}
2\\
2\\
2\\
\overrightarrow{0}
\end{smallmatrix}\right\rangle \right)\,,
\end{align*}
while the inner right term becomes
\begin{align*}
&
\underbrace{\sum_{k=m+1}^{p}\mathbb{E}\left[u_{p}^{4}u_{k}^{2}\right]}_{j=p}+\underbrace{\left(p-m-1\right)\sum_{k=m+1}^{p}\mathbb{E}\left[u_{p-1}^{2}u_{p}^{2}u_{k}^{2}\right]}_{m+1\le j\le p-1}
\\&
=
\underbrace{\mathbb{E}\left[u_{p}^{6}\right]}_{k=p}
+
\underbrace{\left(p\!-\!m\!-\!1\right)\mathbb{E}\left[u_{p-1}^{2}u_{p}^{4}\right]}_{m+1\le k\le p-1}+\left(p\!-\!m\!-\!1\right)
\left(\underbrace{2\mathbb{E}\left[u_{p-1}^{2}u_{p}^{4}\right]}_{k=p\,\vee\,k=p-1}+
\underbrace{\left(p\!-\!m\!-\!2\right)
\mathbb{E}\left[u_{p-2}^{2}u_{p-1}^{2}u_{p}^{2}\right]}_{m+1\le j\le p-2}\right)
\\
&=
\left\langle \begin{smallmatrix}
6\\
\overrightarrow{0}
\end{smallmatrix}\right\rangle +\left(p-m-1\right)\left(3\left\langle \begin{smallmatrix}
4\\
2\\
\overrightarrow{0}
\end{smallmatrix}\right\rangle +\left(p-m-2\right)\left\langle \begin{smallmatrix}
2\\
2\\
2\\
\overrightarrow{0}
\end{smallmatrix}\right\rangle \right)\,.
\end{align*}

Combining these two inner terms, we get,
\begin{align*}
&\mathbb{E}\left[\left(\ve_{1}^{\top}\mO\ve_{1}\right)^{3}\right]
\\
&
=\left(p-m\right)\left(\left\langle \begin{smallmatrix}
6\\
\overrightarrow{0}
\end{smallmatrix}\right\rangle +3\left(p-m\right)\left\langle \begin{smallmatrix}
4\\
2\\
\overrightarrow{0}
\end{smallmatrix}\right\rangle +\left(m^{2}-2mp+6m+p^{2}-3p-1\right)\left\langle \begin{smallmatrix}
2\\
2\\
2\\
\overrightarrow{0}
\end{smallmatrix}\right\rangle \right)\\&=\left(p-m\right)\left(\frac{15}{p\left(p+2\right)\left(p+4\right)}+\frac{9\left(p-m\right)}{p\left(p+2\right)\left(p+4\right)}+\frac{m^{2}-2mp+6m+p^{2}-3p-1}{p\left(p+2\right)\left(p+4\right)}\right)
\\&
=\frac{\left(p-m\right)\left(m^{2}-2mp-3m+p^{2}+6p+14\right)}{p\left(p+2\right)\left(p+4\right)}\,.
\end{align*}
\end{proof}

\bigskip

\begin{proposition} 
\label{prop:(e2Oe1)2_e2Oe_2}
For $p\ge 2, m\in\{2,\dots,p\}$ and a random transformation $\mO$ sampled as described in \eqref{eq:data-model},
it holds that,
    $$    \mathbb{E}\left[\left(\ve_{2}^{\top}\mO\ve_{1}\right)^{2}\ve_{2}^{\top}\mO\ve_{2}\right]=\frac{\left(p-m\right)\left(-m^{2}+2mp+3m-6\right)}{\left(p-1\right)p\left(p+2\right)\left(p+4\right)}
    \,.
    $$
\end{proposition}

\begin{proof}
By decomposing the expectation into the two additive terms below, 
we get that
\begin{align*}
&\mathbb{E}\left[\left(\ve_{2}^{\top}\mO\ve_{1}\right)^{2}\ve_{2}^{\top}\mO\ve_{2}\right]
=
\mathbb{E}\left[\left(\vu^{\top}\left[\begin{smallmatrix}
\Q_{m}
\\
& \mathbf{I}_{p-m}
\end{smallmatrix}\right]\mathbf{v}\right)^{2}\vu^{\top}\left[\begin{smallmatrix}
\Q_{m}\\
& \mathbf{I}_{p-m}
\end{smallmatrix}\right]\vu\right]
\\
&=\mathbb{E}\left[\left(\vu^{\top}\left[\begin{smallmatrix}\Q_{m}\\
 & \mathbf{0}
\end{smallmatrix}\right]\mathbf{v}+\vu^{\top}\left[\begin{smallmatrix}\mathbf{0}\\
 & \mathbf{I}_{p-m}
\end{smallmatrix}\right]\mathbf{v}\right)^{2}\left(\vu^{\top}\left[\begin{smallmatrix}\Q_{m}\\
 & \mathbf{0}
\end{smallmatrix}\right]\mathbf{u}+\vu^{\top}\left[\begin{smallmatrix}\mathbf{0}\\
 & \mathbf{I}_{p-m}
\end{smallmatrix}\right]\mathbf{u}\right)\right]
\\&
=\mathbb{E}\left[\left(\mathbf{u}_{a}\Q_{m}\mathbf{v}_{a}+\vu_{b}^{\top}\mathbf{v}_{b}\right)^{2}\left(\mathbf{u}_{a}\Q_{m}\mathbf{u}_{a}+\vu_{b}^{\top}\mathbf{u}_{b}\right)\right]
\\&
=
\underbrace{\mathbb{E}\left[\mathbf{u}_{a}\Q_{m}\mathbf{v}_{a}\left(\mathbf{u}_{a}\Q_{m}\mathbf{v}_{a}+\vu_{b}^{\top}\mathbf{v}_{b}\right)\left(\mathbf{u}_{a}\Q_{m}\mathbf{u}_{a}+\vu_{b}^{\top}\mathbf{u}_{b}\right)\right]}_{\text{below}}+
\\
&
\eqmargin
\underbrace{
\mathbb{E}\left[\vu_{b}^{\top}\mathbf{v}_{b}\left(\mathbf{u}_{a}\Q_{m}\mathbf{v}_{a}+\vu_{b}^{\top}\mathbf{v}_{b}\right)\left(\mathbf{u}_{a}\Q_{m}\mathbf{u}_{a}+\vu_{b}^{\top}\mathbf{u}_{b}\right)\right]
}_{\text{below}}
\\
%
%
&=\frac{\left(p-m\right)\left(m\left(p+2\right)-4-m^{2}+mp+m-2\right)}{\left(p-1\right)p\left(p+2\right)\left(p+4\right)}=\frac{\left(p-m\right)\left(-m^{2}+2mp+3m-6\right)}{\left(p-1\right)p\left(p+2\right)\left(p+4\right)}\,.
\end{align*}

Using \corref{cor:odd-Q_m} in the first step below, we show that the first term is,
\begin{align*}
&\mathbb{E}\left[\mathbf{u}_{a}\Q_{m}\mathbf{v}_{a}\left(\mathbf{u}_{a}\Q_{m}\mathbf{v}_{a}+\vu_{b}^{\top}\mathbf{v}_{b}\right)\left(\mathbf{u}_{a}\Q_{m}\mathbf{u}_{a}+\vu_{b}^{\top}\mathbf{u}_{b}\right)\right]
\\
&
=
\mathbb{E}\left[\mathbf{u}_{a}\Q_{m}\mathbf{v}_{a}
\vu_{b}^{\top}\mathbf{v}_{b}
\mathbf{u}_{a}\Q_{m}\mathbf{u}_{a}\right]
+
\mathbb{E}\left[\mathbf{u}_{a}\Q_{m}\mathbf{v}_{a}
\mathbf{u}_{a}\Q_{m}\mathbf{v}_{a}
\vu_{b}^{\top}\mathbf{u}_{b}\right]
\\&
=\frac{1}{m}\mathbb{E}\left[\left\Vert \vu_{a}\right\Vert ^{2}\vu_{a}^{\top}\vv_{a}\cdot\vu_{b}^{\top}\vv_{b}\right]
+
\frac{1}{m}\mathbb{E}\left[\left\Vert \vu_{a}\right\Vert ^{2}\left\Vert \vv_{a}\right\Vert ^{2}\left\Vert \vu_{b}\right\Vert ^{2}\right]
\\&
=
-\frac{1}{m}
\underbrace{\mathbb{E}\left[\left\Vert \vu_{a}\right\Vert ^{2}\left(\vu_{b}^{\top}\mathbf{v}_{b}\right)^{2}\right]}_{
\text{solved in \eqref{eq:ua2(ubvb)2}}
}
+
\frac{1}{m}
\underbrace{\mathbb{E}\left[\left\Vert \vu_{a}\right\Vert ^{2}\left\Vert \mathbf{v}_{a}\right\Vert ^{2}\left\Vert \vu_{b}\right\Vert ^{2}\right]}_{
\text{solved in \eqref{eq:ua2va2ub2}}
}
\\&
=-\tfrac{1}{m}\tfrac{\left(p-m\right)m\left(m+2\right)}{\left(p-1\right)p\left(p+2\right)\left(p+4\right)}
+
\tfrac{1}{m}\tfrac{m\left(p-m\right)\left(m\left(p+3\right)-2\right)}{\left(p-1\right)p\left(p+2\right)\left(p+4\right)}
%
=\frac{\left(p-m\right)\left(m\left(p+2\right)-4\right)}{\left(p-1\right)p\left(p+2\right)\left(p+4\right)}
\,.
\end{align*}

From \eqref{eq:(ubvb)(uaQmua+ubub)(uaQmva+ubvb)}, 
we know that the second term is
\begin{align*}
&
\mathbb{E}\left[\vu_{b}^{\top}\mathbf{v}_{b}\left(\mathbf{u}_{a}\Q_{m}\mathbf{v}_{a}+\vu_{b}^{\top}\mathbf{v}_{b}\right)\left(\mathbf{u}_{a}\Q_{m}\mathbf{u}_{a}+\vu_{b}^{\top}\mathbf{u}_{b}\right)\right]
\\
&
=
\underbrace{\mathbb{E}\left[\left(\mathbf{u}_{b}^{\top}\mathbf{v}_{b}\right)^{2}\right]}_{\text{solved in \eqref{eq:(uava)2}}}-\left(1+\frac{1}{m}\right)\underbrace{\mathbb{E}\left[\left\Vert \mathbf{u}_{a}\right\Vert ^{2}\left(\mathbf{u}_{b}^{\top}\mathbf{v}_{b}\right)^{2}\right]}_{\text{solved in \eqref{eq:ua2(ubvb)2}}}
=
\tfrac{\left(p-m\right)m}{\left(p-1\right)p\left(p+2\right)}-\tfrac{m+1}{m}\tfrac{\left(p-m\right)m\left(m+2\right)}{\left(p-1\right)p\left(p+2\right)\left(p+4\right)}
\\&
=
\frac{\left(p-m\right)\left(-m^{2}+mp+m-2\right)}{\left(p-1\right)p\left(p+2\right)\left(p+4\right)}\,.
\end{align*}

\end{proof}

\newpage

\begin{proposition}
\label{prop:(e1Oe1)4}
For $p\ge 2, m\in\{2,\dots,p\}$ and a random transformation $\mO$ sampled as described in \eqref{eq:data-model},
it holds that,
%
%
$$
\mathbb{E}\left(\ve_{1}^{\top}\mO\ve_{1}\right)^{4}
=
\frac{3\left(m+4\right)\left(m+6\right)+\left(p-m\right)\left(p-m+2\right)\left(m^{2}-2mp-4m+p^{2}+10p+36\right)}{p\left(p+2\right)\left(p+4\right)\left(p+6\right)}
$$
\end{proposition}

\begin{proof}
We start by showing,
\begin{align*}
\mathbb{E}\left(\ve_{1}^{\top}\mO\ve_{1}\right)^{4}
=\mathbb{E}\left(\vu^{\top}\left[\begin{smallmatrix}
\Q_{m}\\
 & \mathbf{I}_{p-m}
\end{smallmatrix}\right]\vu\right)^{4}=\mathbb{E}\left(\vu_{a}^{\top}\Q_{m}\vu_{a}+\vu_{b}^{\top}\vu_{b}\right)^{4}\,.
\end{align*}

Employing an algebraic identity and \corref{cor:odd-Q_m} to cancel terms with an odd number of $\Q_m$ appearances,
it can be readily seen that in our case we are left with 
$\prn{a+b}^4=
a^{4}+\cancel{4a^{3}b}+6a^{2}b^{2}+\cancel{4ab^{3}}+b^{4}$.
We thus get,
\begin{align*}
&=\mathbb{E}\left(\vu_{a}^{\top}\Q_{m}\vu_{a}\right)^{4}
+
6\mathbb{E}\left(\vu_{a}^{\top}\Q_{m}\vu_{a}\right)^{2}\left(\vu_{a}^{\top}\vu_{a}\right)^{2}
+
\mathbb{E}\left(\vu_{b}^{\top}\vu_{b}\right)^{4}
\\
&=\mathbb{E}_{\mathbf{r}\sim\mathcal{S}^{m-1},\mathbf{u}\sim\mathcal{S}^{p-1}}\left(\left\Vert \mathbf{u}_{a}\right\Vert \mathbf{r}^{\top}\mathbf{u}_{a}\right)^{4}+6\mathbb{E}\left[\left(\mathbf{u}_{b}^{\top}\mathbf{u}_{b}\right)^{2}\mathbb{E}_{\mathbf{r}\sim\mathcal{S}^{m-1}}\left(\left\Vert \mathbf{u}_{a}\right\Vert \mathbf{r}^{\top}\mathbf{u}_{a}\right)^{2}\right]+\mathbb{E}\left(\mathbf{u}_{b}^{\top}\mathbf{u}_{b}\right)^{4}
\\&
=\mathbb{E}\left[\left\Vert \mathbf{u}_{a}\right\Vert ^{4}\mathbb{E}_{\mathbf{r}\sim\mathcal{S}^{m-1}}\left(\mathbf{r}^{\top}\mathbf{u}_{a}\right)^{4}\right]+6\mathbb{E}\left[\left\Vert \mathbf{u}_{b}\right\Vert ^{4}\left\Vert \mathbf{u}_{a}\right\Vert ^{2}\mathbf{u}_{a}^{\top}\mathbb{E}_{\mathbf{r}\sim\mathcal{S}^{m-1}}\left[\mathbf{r}\mathbf{r}^{\top}\right]\mathbf{u}_{a}\right]+\mathbb{E}\left\Vert \mathbf{u}_{b}\right\Vert ^{8}
\end{align*}
We employ the isotropy of $\mathbf{r}$ and show that,
\begin{align*}
    &=\mathbb{E}\left[\left\Vert \mathbf{u}_{a}\right\Vert ^{4}\mathbb{E}_{\mathbf{r}\sim\mathcal{S}^{m-1}}\left(\left\Vert \mathbf{u}_{a}\right\Vert \mathbf{r}^{\top}\ve_{1}\right)^{4}\right]+\frac{6}{m}\mathbb{E}\left[\left\Vert \mathbf{u}_{b}\right\Vert ^{4}\left\Vert \mathbf{u}_{a}\right\Vert ^{2}\mathbf{u}_{a}^{\top}\mathbf{u}_{a}\right]+\mathbb{E}\left\Vert \mathbf{u}_{b}\right\Vert ^{8}
\\&
=\mathbb{E}_{\mathbf{r}\sim\mathcal{S}^{m-1}}\left[r_{1}^{4}\right]\mathbb{E}\left\Vert \mathbf{u}_{a}\right\Vert ^{8}+\frac{6}{m}\mathbb{E}\left[\left\Vert \mathbf{u}_{b}\right\Vert ^{4}\left\Vert \mathbf{u}_{a}\right\Vert ^{4}\right]+\mathbb{E}\left\Vert \mathbf{u}_{b}\right\Vert ^{8}
\\&
=\left\langle \begin{smallmatrix}
4\\
\overrightarrow{0}
\end{smallmatrix}\right\rangle _{m}\mathbb{E}\left\Vert \mathbf{u}_{a}\right\Vert ^{8}+\frac{6}{m}\mathbb{E}\left[\left(1-\left\Vert \mathbf{u}_{a}\right\Vert ^{2}\right)^{2}\left\Vert \mathbf{u}_{a}\right\Vert ^{4}\right]+\mathbb{E}\left\Vert \mathbf{u}_{b}\right\Vert ^{8}
\\&
=\frac{3}{m\left(m+2\right)}\mathbb{E}\left\Vert \mathbf{u}_{a}\right\Vert ^{8}+\frac{6}{m}\left(\mathbb{E}\left\Vert \mathbf{u}_{a}\right\Vert ^{4}-2\mathbb{E}\left[\left\Vert \mathbf{u}_{a}\right\Vert ^{2}\left\Vert \mathbf{u}_{a}\right\Vert ^{4}\right]+\mathbb{E}\left\Vert \mathbf{u}_{a}\right\Vert ^{8}\right)+\mathbb{E}\left\Vert \mathbf{u}_{b}\right\Vert ^{8}
\\&
=\left(\frac{3}{m\left(m+2\right)}+\frac{6}{m}\right)\mathbb{E}\left\Vert \mathbf{u}_{a}\right\Vert ^{8}+\frac{6}{m}\mathbb{E}\left\Vert \mathbf{u}_{a}\right\Vert ^{4}-\frac{12}{m}\mathbb{E}\left\Vert \mathbf{u}_{a}\right\Vert ^{6}+\mathbb{E}\left\Vert \mathbf{u}_{b}\right\Vert ^{8}
\\&
=\frac{6m+15}{m\left(m+2\right)}\mathbb{E}\left\Vert \mathbf{u}_{a}\right\Vert ^{8}+\frac{6}{m}\mathbb{E}\left\Vert \mathbf{u}_{a}\right\Vert ^{4}-\frac{12}{m}\mathbb{E}\left\Vert \mathbf{u}_{a}\right\Vert ^{6}+\mathbb{E}\left\Vert \mathbf{u}_{b}\right\Vert ^{8}\,,
\end{align*}
where we use the subscript in $\left\langle \begin{smallmatrix}
4\\
\overrightarrow{0}
\end{smallmatrix}\right\rangle_{m}$
to indicate that, unlike in most places, the corresponding random vector is in $\mathcal{S}^{m-1}$  rather than $\mathcal{S}^{p-1}$.

Next, we derive these expected norms in \propref{prop:u_norms},
and obtain
\begin{align*}
&=
\frac{6m+15}{m\left(m+2\right)}
\frac{m\left(m+2\right)\left(m+4\right)\left(m+6\right)}{p\left(p+2\right)\left(p+4\right)\left(p+6\right)}
+
\frac{6}{m}
\frac{m\left(m+2\right)}{p\left(p+2\right)}
-
\frac{12}{m}
\frac{m\left(m+2\right)\left(m+4\right)}{p\left(p+2\right)\left(p+4\right)}
+
\\
&
\eqmargin
\frac{\left(p-m\right)\left(p-m+2\right)\left(p-m+4\right)\left(p-m+6\right)}{p\left(p+2\right)\left(p+4\right)\left(p+6\right)}
\\&
=\frac{\left(6m+15\right)\left(m+4\right)\left(m+6\right)}{p\left(p+2\right)\left(p+4\right)\left(p+6\right)}+\frac{6\left(m+2\right)}{p\left(p+2\right)}-\frac{12\left(m+2\right)\left(m+4\right)}{p\left(p+2\right)\left(p+4\right)}
+
\\
&
\eqmargin
\frac{\left(p-m\right)\left(p-m+2\right)\left(p-m+4\right)\left(p-m+6\right)}{p\left(p+2\right)\left(p+4\right)\left(p+6\right)}
\\
&=
\frac{\left(6m+15\right)\left(m+4\right)\left(m+6\right)+6\left(m+2\right)\left(p+4\right)\left(p+6\right)-12\left(m+2\right)\left(m+4\right)\left(p+6\right)}
{p\left(p+2\right)\left(p+4\right)\left(p+6\right)}
\\
&
\eqmargin
\frac{\left(p-m\right)\left(p-m+2\right)\left(p-m+4\right)\left(p-m+6\right)}{p\left(p+2\right)\left(p+4\right)\left(p+6\right)}
\\
&=
\frac{3\left(m+4\right)\left(m+6\right)+\left(p-m\right)\left(p-m+2\right)\left(m^{2}-2mp-4m+p^{2}+10p+36\right)}{p\left(p+2\right)\left(p+4\right)\left(p+6\right)}\,.
\end{align*}
\end{proof}

\newpage

\begin{proposition}
\label{prop:(e1Oe2)4}
For $p\ge 2, m\in\{2,\dots,p\}$ and a random transformation $\mO$ sampled as described in \eqref{eq:data-model},
it holds that,
%
%
\begin{align*}
&
\mathbb{E}\left(\ve_{1}^{\top}\mO\ve_{2}\right)^{4}
=
\tfrac{3\left(m^{4}-2m^{3}\left(2p+3\right)+m^{2}\left(4p^{2}+4p-1\right)+2m\left(6p^{2}+8p+3\right)+8\left(p^{2}-2p-3\right)\right)}{\left(p-1\right)p\left(p+1\right)\left(p+2\right)\left(p+4\right)\left(p+6\right)}\,.
\end{align*}
\end{proposition}

\begin{proof}
We begin by showing that,
\begin{align*}
&
\mathbb{E}\left(\ve_{1}^{\top}\mO\ve_{2}\right)^{4}
=
\mathbb{E}\left(\ve_{1}^{\top}\mathbf{Q}_{p}\left[\begin{smallmatrix}
\mathbf{Q}_{m}\\
 & \mathbf{I}_{p-m}
\end{smallmatrix}\right]\mathbf{Q}_{p}^{\top}\ve_{2}\right)^{4}
=
\mathbb{E}_{\mathbf{u}\perp\mathbf{v}}\left(\mathbf{u}^{\top}\left[\begin{smallmatrix}
\mathbf{Q}_{m}\\
 & \mathbf{I}_{p-m}
\end{smallmatrix}\right]\mathbf{v}\right)^{4}
\\
&=
\mathbb{E}_{\mathbf{u}\perp\mathbf{v}}\left(\mathbf{u}_{a}^{\top}\mathbf{Q}_{m}\mathbf{v}_{a}+\mathbf{u}_{b}^{\top}\mathbf{v}_{b}\right)^{4}
=
\mathbb{E}\left(\mathbf{u}_{a}^{\top}\mathbf{Q}_{m}\mathbf{v}_{a}+\mathbf{u}_{b}^{\top}\mathbf{v}_{b}\right)^{4}\,,
\end{align*}
where the last step simply employs our simplifying notations 
(\rmkref{rmk:simplify}).

Employing an algebraic identity and \corref{cor:odd-Q_m} to cancel terms with an odd number of $\Q_m$ appearances,
it can be readily seen that in our case we are left with 
$\prn{a+b}^4=
a^{4}+\cancel{4a^{3}b}+6a^{2}b^{2}+\cancel{4ab^{3}}+b^{4}$.
We thus get,
\begin{align*}
&
=
\underbrace{\mathbb{E}\left[\left(\mathbf{u}_{a}^{\top}\mathbf{Q}_{m}\mathbf{v}_{a}\right)^{4}\right]
}_{\text{solved in \eqref{eq:(uaQmva)4}}}
+
6\underbrace{\mathbb{E}\left[\left(\mathbf{u}_{a}^{\top}\mathbf{Q}_{m}\mathbf{v}_{a}\right)^{2}\left(\mathbf{u}_{b}^{\top}\mathbf{v}_{b}\right)^{2}\right]
}_{\text{solved in \eqref{eq:(uaQmva)2(uv)2}}}
+
\underbrace{
\mathbb{E}\left[\left(\mathbf{u}_{b}^{\top}\mathbf{v}_{b}\right)^{4}\right]
}_{\text{solved in \eqref{eq:(uava)4}}}
\\
&
=
\tfrac{3\left(m^{2}\left(p+3\right)\left(p+5\right)+2m\left(p+1\right)\left(p+3\right)-8\left(2p+3\right)\right)}{\left(p-1\right)p\left(p+1\right)\left(p+2\right)\left(p+4\right)\left(p+6\right)}
+
6
\tfrac{\left(m+2\right)\left(mp+2p+3m\right)\left(p-m\right)}{\left(p-1\right)p\left(p+1\right)\left(p+2\right)\left(p+4\right)\left(p+6\right)}
+
\\
&
\eqmargin
\tfrac{3m^{4}-6m^{3}p+3m^{2}p^{2}-6m^{2}p-12m^{2}+6mp^{2}+12mp}{\left(p-1\right)p\left(p+1\right)\left(p+2\right)\left(p+4\right)\left(p+6\right)}
\,.
\end{align*}
After some tedious algebra,
we get, as required,
\begin{align*}
&
\mathbb{E}\left(\ve_{1}^{\top}\mO\ve_{2}\right)^{4}
=
\tfrac{3\left(m^{4}-2m^{3}\left(2p+3\right)+m^{2}\left(4p^{2}+4p-1\right)+2m\left(6p^{2}+8p+3\right)+8\left(p^{2}-2p-3\right)\right)}{\left(p-1\right)p\left(p+1\right)\left(p+2\right)\left(p+4\right)\left(p+6\right)}\,.
\end{align*}
\end{proof}

\newpage

\begin{proposition}
\label{prop:e2Oe2*e2Oe1}
For $p\ge 2, m\in\{2,\dots,p\}$ and a random transformation $\mO$ sampled as described in \eqref{eq:data-model},
it holds that,
\begin{align*}
&
\mathbb{E}
\left(\ve_{1}^{\top}\mO\ve_{1}\cdot\ve_{2}^{\top}\mO\ve_{1}\right)^{2}
=
\mathbb{E}
\left(\ve_{2}^{\top}\mO\ve_{2}\cdot\ve_{2}^{\top}\mO\ve_{1}\right)^{2}
\\
&
=
\frac{\left(m+4\right)\left(2mp+4p+m-m^{2}-6\right)-\left(p-m\right)\left(p-m+2\right)\left(m\left(m-2p-5\right)+10\right)}{\left(p-1\right)p\left(p+2\right)\left(p+4\right)\left(p+6\right)}\,.
\end{align*}
\end{proposition}

\begin{proof}
First, due to the exchangeability (\propref{prop:invariance}), we have,
\begin{align*}
&
\mathbb{E}
\left(\ve_{1}^{\top}\mO\ve_{1}\cdot\ve_{2}^{\top}\mO\ve_{1}\right)^{2}
\!=\!
\mathbb{E}
\left(\ve_{2}^{\top}\mO\ve_{2}\ve_{1}^{\top}\mO\ve_{2}\right)^{2}
\!=\!
\mathbb{E}
\left(\ve_{2}^{\top}\mO^{\top}\ve_{2}\ve_{2}^{\top}\mO^{\top}\ve_{1}\right)^{2}
\!=\!
\mathbb{E}
\left(\ve_{2}^{\top}\mO\ve_{2}\cdot\ve_{2}^{\top}\mO\ve_{1}\right)^{2}
\end{align*}

Then, we show that,
\begin{align*}
\mathbb{E}\left(\ve_{2}^{\top}\mO\ve_{2}\cdot\ve_{2}^{\top}\mO\ve_{1}\right)^{2}
&
=\mathbb{E}\left(\vu^{\top}\left(\left[\begin{smallmatrix}
\Q_{m}\\
& \mathbf{0}
\end{smallmatrix}\right]+\left[\begin{smallmatrix}
\mathbf{0}\\
& \mathbf{I}_{p-m}
\end{smallmatrix}\right]\right)\vu\cdot\vu^{\top}\left(\left[\begin{smallmatrix}
\Q_{m}\\
& \mathbf{0}
\end{smallmatrix}\right]+\left[\begin{smallmatrix}
\mathbf{0}\\
& \mathbf{I}_{p-m}
\end{smallmatrix}\right]\right)\vv\right)^{2}
\\&
=\mathbb{E}\left(\vu_{a}^{\top}\Q_{m}\vu_{a}+\vu_{b}^{\top}\vu_{b}\right)^{2}
\left(\vu_{a}^{\top}\Q_{m}\vv_{a}+\vu_{b}^{\top}\vv_{b}\right)^{2}
\\&
=\mathbb{E}\left(\left(\vu_{a}^{\top}\Q_{m}\vu_{a}\right)^{2}+2\vu_{a}^{\top}\Q_{m}\vu_{a}\vu_{b}^{\top}\vu_{b}+\norm{\vu_{b}}^{4}\right)
\left(\vu_{a}^{\top}\Q_{m}\vv_{a}+\vu_{b}^{\top}\vv_{b}\right)^{2}
\end{align*}

We now partition the above into three terms that we solve separately.

The first term is,
\begin{align*} 
&
\mathbb{E}\left(\vu_{a}^{\top}\Q_{m}\vu_{a}\right)^{2}
\left(\left(\vu_{a}^{\top}\Q_{m}\vv_{a}\right)^{2}+2\vu_{a}^{\top}\Q_{m}\vv_{a}\vu_{b}^{\top}\vv_{b}+\left(\vu_{b}^{\top}\vv_{b}\right)^{2}\right)
\\
&
=
\underbrace{\mathbb{E}\left(\vu_{a}^{\top}\Q_{m}\vu_{a}\right)^{2}
\left(\vu_{a}^{\top}\Q_{m}\vv_{a}\right)^{2}
}_{\text{solved in \eqref{eq:(uaQmua)2(uaQmva)2}}}
+
2
\underbrace{
\mathbb{E}\left(\vu_{a}^{\top}\Q_{m}\vu_{a}\right)^{2}
\vu_{a}^{\top}\Q_{m}\vv_{a}\vu_{b}^{\top}\vv_{b}
}_{=0\text{, by \propref{cor:odd-Q_m}}}
+
\underbrace{\mathbb{E}\left(\vu_{a}^{\top}\Q_{m}\vu_{a}\right)^{2}
\left(\vu_{b}^{\top}\vv_{b}\right)^{2}
}_{\text{solved in \eqref{eq:(uaQmua)2(ubvb)2}}}
\\
&=\frac{\left(m+4\right)\left(m\left(p+3\right)+2\left(p-3\right)\right)}{\left(p-1\right)p\left(p+2\right)\left(p+4\right)\left(p+6\right)}+\frac{\left(m+2\right)\left(m+4\right)\left(p-m\right)}{\left(p-1\right)p\left(p+2\right)\left(p+4\right)\left(p+6\right)}
\\
&=\frac{\left(m+4\right)\left(2mp+4p+m-m^{2}-6\right)}{\left(p-1\right)p\left(p+2\right)\left(p+4\right)\left(p+6\right)}\,.
\end{align*}

The second term is,
\begin{align*} 
&\mathbb{E}\left[\left(2\mathbf{u}_{a}^{\top}\mathbf{Q}_{m}\mathbf{u}_{a}\mathbf{u}_{b}^{\top}\mathbf{u}_{b}\right)\left(\left(\mathbf{u}_{a}^{\top}\mathbf{Q}_{m}\mathbf{v}_{a}\right)^{2}+2\mathbf{u}_{a}^{\top}\mathbf{Q}_{m}\mathbf{v}_{a}\mathbf{u}_{b}^{\top}\mathbf{v}_{b}+\left(\mathbf{u}_{b}^{\top}\mathbf{v}_{b}\right)^{2}\right)\right]
\\
\explain{\text{\propref{cor:odd-Q_m}}}
&
=4\mathbb{E}
\left[\mathbf{u}_{a}^{\top}\mathbf{Q}_{m}\mathbf{u}_{a}\mathbf{u}_{b}^{\top}\mathbf{u}_{b}\cdot\mathbf{u}_{a}^{\top}\mathbf{Q}_{m}\mathbf{v}_{a}\mathbf{u}_{b}^{\top}\mathbf{v}_{b}\right]
\\&
=4\mathbb{E}_{\mathbf{u}\perp\mathbf{v}}\left[\mathbb{E}_{\mathbf{Q}_{m}}\left[\mathbf{u}_{a}^{\top}\mathbf{Q}_{m}\mathbf{u}_{a}\mathbf{u}_{a}^{\top}\mathbf{Q}_{m}\mathbf{v}_{a}\right]\cdot\left\Vert \mathbf{u}_{b}\right\Vert ^{2}\cdot\mathbf{u}_{b}^{\top}\mathbf{v}_{b}\right]
\\&
=4\mathbb{E}_{\mathbf{u}\perp\mathbf{v}}\left[\mathbb{E}_{\mathbf{r}\sim\mathcal{S}^{m-1}}\left[\left(\left\Vert \mathbf{u}_{a}\right\Vert \mathbf{u}_{a}^{\top}\mathbf{r}\right)\left(\left\Vert \mathbf{u}_{a}\right\Vert \mathbf{r}^{\top}\mathbf{v}_{a}\right)\right]\cdot\left\Vert \mathbf{u}_{b}\right\Vert ^{2}\cdot\mathbf{u}_{b}^{\top}\mathbf{v}_{b}\right]
\\&
=4\mathbb{E}_{\mathbf{u}\perp\mathbf{v}}\left[\left\Vert \mathbf{u}_{a}\right\Vert ^{2}\left\Vert \mathbf{u}_{b}\right\Vert ^{2}\cdot\mathbf{u}_{a}^{\top}\mathbb{E}_{\mathbf{r}\sim\mathcal{S}^{m-1}}\left[\mathbf{r}\mathbf{r}^{\top}\right]\mathbf{v}_{a}\cdot\mathbf{u}_{b}^{\top}\mathbf{v}_{b}\right]
\\&
=\frac{4}{m}\mathbb{E}_{\mathbf{u}\perp\mathbf{v}}\left[\left\Vert \mathbf{u}_{a}\right\Vert ^{2}\left\Vert \mathbf{u}_{b}\right\Vert ^{2}\cdot\mathbf{u}_{a}^{\top}\mathbf{v}_{a}\cdot\mathbf{u}_{b}^{\top}\mathbf{v}_{b}\right]=\frac{4}{m}\mathbb{E}\left[\left\Vert \mathbf{u}_{a}\right\Vert ^{2}\left(1-\left\Vert \mathbf{u}_{a}\right\Vert ^{2}\right)
\left(-\mathbf{u}_{b}^{\top}\mathbf{v}_{b}\right)
\mathbf{u}_{b}^{\top}\mathbf{v}_{b}\right]
\\&
=\frac{4}{m}\underbrace{\mathbb{E}\left[\left\Vert \mathbf{u}_{a}\right\Vert ^{4}\left(\mathbf{u}_{b}^{\top}\mathbf{v}_{b}\right)^{2}\right]}_{\text{solved in \eqref{eq:(ua)4(ubvb)2}}}
-
\frac{4}{m}\underbrace{\mathbb{E}\left[\left\Vert \mathbf{u}_{a}\right\Vert ^{2}\left(\mathbf{u}_{b}^{\top}\mathbf{v}_{b}\right)^{2}\right]}_{\text{solved in \eqref{eq:ua2(ubvb)2}}}
\\&
=\frac{4}{m}\frac{m\left(m+2\right)\left(m+4\right)\left(p-m\right)}{\left(p-1\right)p\left(p+2\right)\left(p+4\right)\left(p+6\right)}-\frac{4}{m}\frac{\left(p-m\right)m\left(m+2\right)}{\left(p-1\right)p\left(p+2\right)\left(p+4\right)}
\\&
=\frac{4\left(p-m\right)\left(m+2\right)\left(\left(m+4\right)-\left(p+6\right)\right)}{\left(p-1\right)p\left(p+2\right)\left(p+4\right)\left(p+6\right)}=\frac{-4\left(p-m\right)\left(mp+2p+4-m^{2}\right)}{\left(p-1\right)p\left(p+2\right)\left(p+4\right)\left(p+6\right)}
\end{align*}
The third term is,
\begin{align*} 
&
\mathbb{E}\left[\left\Vert \mathbf{u}_{b}\right\Vert ^{4}\left(\left(\mathbf{u}_{a}^{\top}\mathbf{Q}_{m}\mathbf{v}_{a}\right)^{2}+2\mathbf{u}_{a}^{\top}\mathbf{Q}_{m}\mathbf{v}_{a}\mathbf{u}_{b}^{\top}\mathbf{v}_{b}+\left(\mathbf{u}_{b}^{\top}\mathbf{v}_{b}\right)^{2}\right)\right]
\\
&
=\mathbb{E}\left[\left\Vert \mathbf{u}_{b}\right\Vert ^{4}\left(\mathbf{u}_{a}^{\top}\mathbf{Q}_{m}\mathbf{v}_{a}\right)^{2}\right]
+
2\underbrace{\mathbb{E}\left[\left\Vert \mathbf{u}_{b}\right\Vert ^{4}\mathbf{u}_{a}^{\top}\mathbf{Q}_{m}\mathbf{v}_{a}\mathbf{u}_{b}^{\top}\mathbf{v}_{b}\right]}_{=0\text{, by \propref{cor:odd-Q_m}}}
+
\mathbb{E}\left[\left\Vert \mathbf{u}_{b}\right\Vert ^{4}\left(\mathbf{u}_{b}^{\top}\mathbf{v}_{b}\right)^{2}\right]
\\
&
=\underbrace{\mathbb{E}\left[\left\Vert \mathbf{u}_{b}\right\Vert ^{4}\left(\mathbf{u}_{a}^{\top}\mathbf{Q}_{m}\mathbf{v}_{a}\right)^{2}\right]}_{\text{solved in \eqref{eq:(ub)4(uaQmva)2}}}
+
\underbrace{\mathbb{E}\left[\left\Vert \mathbf{u}_{b}\right\Vert ^{4}\left(\mathbf{u}_{a}^{\top}\mathbf{v}_{a}\right)^{2}\right]}_{
\substack{
\text{solved in \eqref{eq:(ua)4(ubvb)2}}\\
\text{(plugging in \ensuremath{m\leftarrow p-m}})
}
}
\\&
=\tfrac{\left(p-m\right)\left(mp^{2}-m^{2}p-5m^{2}+7mp+12m-2p-4\right)}{\left(p-1\right)p\left(p+2\right)\left(p+4\right)\left(p+6\right)}+\tfrac{\left(p-m\right)m\left(m^{2}-2mp-6m+p^{2}+6p+8\right)}{\left(p-1\right)p\left(p+2\right)\left(p+4\right)\left(p+6\right)}
\\&
=\frac{\left(p-m\right)\left(m^{3}-3m^{2}p-11m^{2}+2mp^{2}+13mp+20m-2p-4\right)}{\left(p-1\right)p\left(p+2\right)\left(p+4\right)\left(p+6\right)}
\end{align*}

Finally, summing the three terms (and after some tedious algebra), we get the required
\begin{align*}
&
\mathbb{E}\left(\ve_{2}^{\top}\mO\ve_{2}\cdot\ve_{2}^{\top}\mO\ve_{1}\right)^{2}
=
\tfrac{\left(m+4\right)\left(2mp+4p+m-m^{2}-6\right)-\left(p-m\right)\left(p-m+2\right)\left(m\left(m-2p-5\right)+10\right)}{\left(p-1\right)p\left(p+2\right)\left(p+4\right)\left(p+6\right)}\,.
\end{align*}
\end{proof}

\newpage

\begin{proposition} 
\label{prop:(e2Oe3*e2Oe1)^2}
For $p\ge 3, m\in\{2,3,\dots,p\}$ and a random transformation $\mO$ sampled as described in \eqref{eq:data-model},
it holds that,
\begin{align*}
&
\mathbb{E}
\left(\ve_{2}^{\top}\mO\ve_{3}\cdot\ve_{2}^{\top}\mO\ve_{1}\right)^{2}
=
\tfrac{4mp\left(-m^{2}+m+4\right)+4\left(m+1\right)\left(m+2\right)p^{2}+\left(m-6\right)\left(m-1\right)m\left(m+1\right)-8\left(2p+3\right)}{\left(p-1\right)p\left(p+1\right)\left(p+2\right)\left(p+4\right)\left(p+6\right)}\,.
\end{align*}
\end{proposition}

\begin{proof}
We show that,
\begin{align*}
\mathbb{E}\left(\ve_{2}^{\top}\mO\ve_{3}\cdot\ve_{2}^{\top}\mO\ve_{1}\right)^{2}
\!&
=\mathbb{E}\left(\vu^{\top}\left(\left[\begin{smallmatrix}
\Q_{m}\\
& \mathbf{0}
\end{smallmatrix}\right]+\left[\begin{smallmatrix}
\mathbf{0}\\
& \mathbf{I}_{p-m}
\end{smallmatrix}\right]\right)\vz
\cdot\vu^{\top}\left(\left[\begin{smallmatrix}
\Q_{m}\\
& \mathbf{0}
\end{smallmatrix}\right]+\left[\begin{smallmatrix}
\mathbf{0}\\
& \mathbf{I}_{p-m}
\end{smallmatrix}\right]\right)\vv\right)^{2}
\\&
=
\mathbb{E}\left(\vu_{a}^{\top}\Q_{m}\vz_{a}+\vu_{b}^{\top}\vz_{b}\right)^{2}
\left(\vu_{a}^{\top}\Q_{m}\vv_{a}+\vu_{b}^{\top}\vv_{b}\right)^{2}
\\&
=\mathbb{E}\left(\left(\vu_{a}^{\top}\Q_{m}\vz_{a}\right)^{2}+2\vu_{a}^{\top}\Q_{m}\vz_{a}\vu_{b}^{\top}\vz_{b}+\left(\vu_{b}^{\top}\vz_{b}\right)^{2}\right)
\left(\vu_{a}^{\top}\Q_{m}\vv_{a}+\vu_{b}^{\top}\vv_{b}\right)^{2},
\end{align*}
where the expectation is computed over the orthogonal matrix $\Q_m\sim O(m)$ and three random (isotropic) orthogonal unit vectors $\mathbf{u}\perp\mathbf{v}\perp\mathbf{z}\in\mathcal{S}^{p-1}$.

Again, starting from the first additive term,
\begin{align*}
&\mathbb{E}\left[\left(\mathbf{u}_{a}^{\top}\mathbf{Q}_{m}\mathbf{z}_{a}\right)^{2}\left(\left(\mathbf{u}_{a}^{\top}\mathbf{Q}_{m}\mathbf{v}_{a}\right)^{2}+2\mathbf{u}_{a}^{\top}\mathbf{Q}_{m}\mathbf{v}_{a}\mathbf{u}_{b}^{\top}\mathbf{v}_{b}+\left(\mathbf{u}_{b}^{\top}\mathbf{v}_{b}\right)^{2}\right)\right]
\\&
=\mathbb{E}\left[\underbrace{\left(\mathbf{u}_{a}^{\top}\mathbf{Q}_{m}\mathbf{z}_{a}\right)^{2}\left(\mathbf{u}_{a}^{\top}\mathbf{Q}_{m}\mathbf{v}_{a}\right)^{2}}_{\text{solved in \eqref{eq:(uaQmza)^2(uaQmva)^2}}}
+
\underbrace{2\left(\mathbf{u}_{a}^{\top}\mathbf{Q}_{m}\mathbf{z}_{a}\right)^{2}\mathbf{u}_{a}^{\top}\mathbf{Q}_{m}\mathbf{v}_{a}\mathbf{u}_{b}^{\top}\mathbf{v}_{b}}_{=0\text{, by \corref{cor:odd-Q_m}}}
+
\underbrace{\left(\mathbf{u}_{a}^{\top}\mathbf{Q}_{m}\mathbf{z}_{a}\right)^{2}\left(\mathbf{u}_{b}^{\top}\mathbf{v}_{b}\right)^{2}}_{\text{solved in \eqref{eq:(uaQmza)^2(ubvb)^2}}}\right]
\\&
=\tfrac{m\left(p+3\right)\left(\left(m+2\right)p+5m+2\right)-16p-24}{\left(p-1\right)p\left(p+1\right)\left(p+2\right)\left(p+4\right)\left(p+6\right)}
+
\tfrac{\left(p-m\right)\left(m+2\right)\left(m\left(p^{2}+5p+2\right)-6p-4\right)}{\left(p-2\right)\left(p-1\right)p\left(p+1\right)\left(p+2\right)\left(p+4\right)\left(p+6\right)}
\\&
=\tfrac{\left(p-2\right)\left(m\left(p+3\right)\left(\left(m+2\right)p+5m+2\right)-16p-24\right)+\left(p-m\right)\left(m+2\right)\left(m\left(p^{2}+5p+2\right)-6p-4\right)}{\left(p-2\right)\left(p-1\right)p\left(p+1\right)\left(p+2\right)\left(p+4\right)\left(p+6\right)}
\,.
\end{align*}
The second term is,
\begin{align*}
&\mathbb{E}\left[2\mathbf{u}_{a}^{\top}\mathbf{Q}_{m}\mathbf{z}_{a}\mathbf{u}_{b}^{\top}\mathbf{z}_{b}\left(\left(\mathbf{u}_{a}^{\top}\mathbf{Q}_{m}\mathbf{v}_{a}\right)^{2}+2\mathbf{u}_{a}^{\top}\mathbf{Q}_{m}\mathbf{v}_{a}\mathbf{u}_{b}^{\top}\mathbf{v}_{b}+\left(\mathbf{u}_{b}^{\top}\mathbf{v}_{b}\right)^{2}\right)\right]
\\&
=\underbrace{\mathbb{E}\left[2\mathbf{u}_{a}^{\top}\mathbf{Q}_{m}\mathbf{z}_{a}\mathbf{u}_{b}^{\top}\mathbf{z}_{b}\left(\mathbf{u}_{a}^{\top}\mathbf{Q}_{m}\mathbf{v}_{a}\right)^{2}\right]}_{\text{=0,\text{ by \corref{cor:odd-Q_m}}}}
+
4\mathbb{E}\left[\mathbf{u}_{a}^{\top}\mathbf{Q}_{m}\mathbf{z}_{a}\mathbf{u}_{b}^{\top}\mathbf{z}_{b}\mathbf{u}_{a}^{\top}\mathbf{Q}_{m}\mathbf{v}_{a}\mathbf{u}_{b}^{\top}\mathbf{v}_{b}\right]
+
\\
&
\eqmargin
\underbrace{\mathbb{E}\left[2\mathbf{u}_{a}^{\top}\mathbf{Q}_{m}\mathbf{z}_{a}\mathbf{u}_{b}^{\top}\mathbf{z}_{b}\left(\mathbf{u}_{b}^{\top}\mathbf{v}_{b}\right)^{2}\right]}_{\text{=0,\text{ by \corref{cor:odd-Q_m}}}}
\\&
=4\mathbb{E}_{\mathbf{u}\perp\mathbf{v}\perp\mathbf{z},\mathbf{Q}_{m}}\left[\mathbf{u}_{a}^{\top}\mathbf{Q}_{m}\mathbf{z}_{a}\mathbf{u}_{b}^{\top}\mathbf{z}_{b}\mathbf{u}_{a}^{\top}\mathbf{Q}_{m}\mathbf{v}_{a}\mathbf{u}_{b}^{\top}\mathbf{v}_{b}\right]
\\
&=4\mathbb{E}_{\mathbf{u}\perp\mathbf{v}\perp\mathbf{z}}\left[\left\Vert \mathbf{u}_{a}\right\Vert ^{2}
\mathbb{E}_{\mathbf{r}\sim\mathcal{S}^{m-1}}\!
\left(\mathbf{v}_{a}^{\top}\mathbf{r}\mathbf{r}^{\top}\mathbf{z}_{a}\right)
\mathbf{u}_{b}^{\top}\mathbf{z}_{b}\mathbf{u}_{b}^{\top}\mathbf{v}_{b}\right]
%
=\frac{4}{m}\underbrace{\mathbb{E}_{\mathbf{u}\perp\mathbf{v}\perp\mathbf{z}}\left[\left\Vert \mathbf{u}_{a}\right\Vert ^{2}\mathbf{v}_{a}^{\top}\mathbf{z}_{a}\mathbf{u}_{b}^{\top}\mathbf{z}_{b}\mathbf{u}_{b}^{\top}\mathbf{v}_{b}\right]}_{\text{solved in \eqref{eq:(ua)^2vazaubzbubvb}}}
\\&
=\frac{4\left(p-m\right)\left(\left(m+2\right)p^{2}+\left(2-3m\right)p-2m^{2}\left(p+2\right)-6m+4\right)}{\left(p-2\right)\left(p-1\right)p\left(p+1\right)\left(p+2\right)\left(p+4\right)\left(p+6\right)}
\,.
\end{align*}
And the third additive term is,
\begin{align*}
&\mathbb{E}\left[\left(\mathbf{u}_{b}^{\top}\mathbf{z}_{b}\right)^{2}\left(\left(\mathbf{u}_{a}^{\top}\mathbf{Q}_{m}\mathbf{v}_{a}\right)^{2}+2\mathbf{u}_{a}^{\top}\mathbf{Q}_{m}\mathbf{v}_{a}\mathbf{u}_{b}^{\top}\mathbf{v}_{b}+\left(\mathbf{u}_{b}^{\top}\mathbf{v}_{b}\right)^{2}\right)\right]
\\&
=\underbrace{\mathbb{E}\left[\left(\mathbf{u}_{b}^{\top}\mathbf{z}_{b}\right)^{2}\left(\mathbf{u}_{a}^{\top}\mathbf{Q}_{m}\mathbf{v}_{a}\right)^{2}\right]}_{\substack{=\mathbb{E}\left[\left(\mathbf{u}_{a}^{\top}\mathbf{Q}_{m}\mathbf{z}_{a}\right)^{2}\left(\mathbf{u}_{b}^{\top}\mathbf{v}_{b}\right)^{2}\right]\\
\text{due to the invariance (\propref{prop:invariance}),}\\
\text{already solved in \eqref{eq:(uaQmza)^2(ubvb)^2}}
}
}
+
2\underbrace{\mathbb{E}\left[\left(\mathbf{u}_{b}^{\top}\mathbf{z}_{b}\right)^{2}\mathbf{u}_{a}^{\top}\mathbf{Q}_{m}\mathbf{v}_{a}\mathbf{u}_{b}^{\top}\mathbf{v}_{b}\right]}_{=0\text{, by \corref{cor:odd-Q_m}}}
+
\underbrace{\mathbb{E}\left[\left(\mathbf{u}_{b}^{\top}\mathbf{z}_{b}\right)^{2}\left(\mathbf{u}_{b}^{\top}\mathbf{v}_{b}\right)^{2}\right]}_{\text{solved in \eqref{eq:(ubzb)^2(ubvb)^2}}}
\\&
=\tfrac{\left(p-m\right)\left(m+2\right)\left(m\left(p^{2}+5p+2\right)-6p-4\right)}{\left(p-2\right)\left(p-1\right)p\left(p+1\right)\left(p+2\right)\left(p+4\right)\left(p+6\right)}
+
\tfrac{m\left(p-m\right)\left(-m^{2}+mp+2p+4\right)}{\left(p-1\right)p\left(p+1\right)\left(p+2\right)\left(p+4\right)\left(p+6\right)}
\\&
=\frac{\left(p-m\right)\left(\left(m+2\right)\left(m\left(p^{2}+5p+2\right)-6p-4\right)+\left(p-2\right)m\left(-m^{2}+mp+2p+4\right)\right)}{\left(p-2\right)\left(p-1\right)p\left(p+1\right)\left(p+2\right)\left(p+4\right)\left(p+6\right)}
\,.
\end{align*}

By adding these three terms and after some tedious algebra, we get the required proposition.
\end{proof}

\newpage

\begin{proposition}
\label{prop:(e2Oe1)(e1Oe1)(e1Oe2)}
For $p\ge 2, m\in\{2,\dots,p\}$ and a random transformation $\mO$ sampled as described in \eqref{eq:data-model},
it holds that,
\begin{align*}
\mathbb{E}\left[\mathbf{e}_{2}^{\top}\mathbf{O}\mathbf{e}_{1}\mathbf{e}_{1}^{\top}\mathbf{O}\mathbf{e}_{1}\mathbf{e}_{1}^{\top}\mathbf{O}\mathbf{e}_{2}\right]
= 
\frac{\left(p-m\right)\left(-m^{2}+mp-m+p-2\right)}{\left(p-1\right)p\left(p+2\right)\left(p+4\right)}
\end{align*}
\end{proposition}

\begin{proof}
First, we notice that 
\begin{align*}
&
\mathbb{E}_{\mathbf{u}\perp\mathbf{v},\mathbf{Q}_{m}}\!
\left[\mathbf{v}_{a}^{\top}\mathbf{Q}_{m}\mathbf{u}_{a}\mathbf{u}_{a}^{\top}\mathbf{Q}_{m}\mathbf{u}_{a}\mathbf{u}_{b}^{\top}\mathbf{v}_{b}\right]
=\mathbb{E}_{\mathbf{u}\perp\mathbf{v}}\left[\mathbf{v}_{a}^{\top}
\mathbb{E}_{\mathbf{r}\sim\mathcal{S}^{m-1}}\!
\left[\left(\left\Vert \mathbf{u}_{a}\right\Vert \mathbf{r}\right)\left(\left\Vert \mathbf{u}_{a}\right\Vert \mathbf{r}\right)^{\top}\right]\mathbf{u}_{a}\mathbf{u}_{b}^{\top}\mathbf{v}_{b}\right]
\\&
=\mathbb{E}_{\mathbf{u}\perp\mathbf{v}}\left[\left\Vert \mathbf{u}_{a}\right\Vert ^{2}\mathbf{v}_{a}^{\top}
\mathbb{E}_{\mathbf{r}\sim\mathcal{S}^{m-1}}\!
\left[\mathbf{r}\mathbf{r}^{\top}\right]\mathbf{u}_{a}\mathbf{u}_{b}^{\top}\mathbf{v}_{b}\right]
\\&
=\frac{1}{m}\mathbb{E}_{\mathbf{u}\perp\mathbf{v}}\!\left[\left\Vert \mathbf{u}_{a}\right\Vert ^{2}\mathbf{v}_{a}^{\top}\mathbf{u}_{a}\mathbf{u}_{b}^{\top}\mathbf{v}_{b}\right]
%
=-\frac{1}{m}\mathbb{E}
\left[\left\Vert \mathbf{u}_{a}\right\Vert ^{2}\left(\mathbf{u}_{b}^{\top}\mathbf{v}_{b}\right)^{2}\right]
.
\end{align*}

Then, we focus on our quantity of interest here,
\begin{align*}
&\mathbb{E}\left(\mathbf{e}_{2}^{\top}\mathbf{O}\mathbf{e}_{1}\mathbf{e}_{1}^{\top}\mathbf{O}\mathbf{e}_{1}\mathbf{e}_{1}^{\top}\mathbf{O}\mathbf{e}_{2}\right)
\\&
=\mathbb{E}_{\mathbf{u}\perp\mathbf{v},\mathbf{Q}_{m}}\left[\mathbf{v}^{\top}\left[\begin{smallmatrix}
\mathbf{Q}_{m}\\
 & \mathbf{I}_{p-m}
\end{smallmatrix}\right]\mathbf{u}\mathbf{u}^{\top}\left[\begin{smallmatrix}
\mathbf{Q}_{m}\\
 & \mathbf{I}_{p-m}
\end{smallmatrix}\right]\mathbf{u}\mathbf{u}^{\top}\left[\begin{smallmatrix}
\mathbf{Q}_{m}\\
 & \mathbf{I}_{p-m}
\end{smallmatrix}\right]\mathbf{v}\right]
\\&
=\mathbb{E}\left[\left(\mathbf{v}_{a}^{\top}\mathbf{Q}_{m}\mathbf{u}_{a}+\mathbf{v}_{b}^{\top}\mathbf{u}_{b}\right)\left(\mathbf{u}_{a}^{\top}\mathbf{Q}_{m}\mathbf{u}_{a}+\mathbf{u}_{b}^{\top}\mathbf{u}_{b}\right)\left(\mathbf{u}_{a}^{\top}\mathbf{Q}_{m}\mathbf{v}_{a}+\mathbf{u}_{b}^{\top}\mathbf{v}_{b}\right)\right]\,.
\end{align*}

Splitting the first multiplicative term, we get,
\begin{align*}
&\mathbb{E}\left[\mathbf{v}_{a}^{\top}\mathbf{Q}_{m}\mathbf{u}_{a}\left(\mathbf{u}_{a}^{\top}\mathbf{Q}_{m}\mathbf{u}_{a}+\mathbf{u}_{b}^{\top}\mathbf{u}_{b}\right)\left(\mathbf{u}_{a}^{\top}\mathbf{Q}_{m}\mathbf{v}_{a}+\mathbf{u}_{b}^{\top}\mathbf{v}_{b}\right)\right]
\\
\text{\corref{cor:odd-Q_m}}&=\mathbb{E}\left[\mathbf{v}_{a}^{\top}\mathbf{Q}_{m}\mathbf{u}_{a}\mathbf{u}_{a}^{\top}\mathbf{Q}_{m}\mathbf{u}_{a}\mathbf{u}_{b}^{\top}\mathbf{v}_{b}\right]+\mathbb{E}\left[\mathbf{v}_{a}^{\top}\mathbf{Q}_{m}\mathbf{u}_{a}\mathbf{u}_{b}^{\top}\mathbf{u}_{b}\mathbf{u}_{a}^{\top}\mathbf{Q}_{m}\mathbf{v}_{a}\right]
\\&
=\mathbb{E}\left[\mathbf{v}_{a}^{\top}\mathbf{Q}_{m}\mathbf{u}_{a}\mathbf{u}_{a}^{\top}\mathbf{Q}_{m}\mathbf{u}_{a}\mathbf{u}_{b}^{\top}\mathbf{v}_{b}\right]+\underbrace{\mathbb{E}\left[\left\Vert \mathbf{u}_{b}\right\Vert ^{2}\mathbf{v}_{a}^{\top}\mathbf{Q}_{m}\mathbf{u}_{a}\mathbf{u}_{a}^{\top}\mathbf{Q}_{m}\mathbf{v}_{a}\right]}_{\text{solved in \eqref{eq:(ub)^2(vaQmua)(uaQmva)}}}\,,
\\
\explain{\text{above}}
&=-\frac{1}{m}\mathbb{E}_{\mathbf{u}\perp\mathbf{v}}\left[\left\Vert \mathbf{u}_{a}\right\Vert ^{2}\left(\mathbf{u}_{b}^{\top}\mathbf{v}_{b}\right)^{2}\right]+\frac{\left(p-m\right)\left(p-m+2\right)}{\left(p-1\right)p\left(p+2\right)\left(p+4\right)}\,,
\end{align*}
and
\begin{align}
\begin{split}
\label{eq:(ubvb)(uaQmua+ubub)(uaQmva+ubvb)}
&\mathbb{E}\left[\mathbf{v}_{b}^{\top}\mathbf{u}_{b}\left(\mathbf{u}_{a}^{\top}\mathbf{Q}_{m}\mathbf{u}_{a}+\mathbf{u}_{b}^{\top}\mathbf{u}_{b}\right)\left(\mathbf{u}_{a}^{\top}\mathbf{Q}_{m}\mathbf{v}_{a}+\mathbf{u}_{b}^{\top}\mathbf{v}_{b}\right)\right]
\\
\explain{\text{\corref{cor:odd-Q_m}}}
&=
\mathbb{E}\left[\mathbf{v}_{b}^{\top}\mathbf{u}_{b}\mathbf{u}_{a}^{\top}\mathbf{Q}_{m}\mathbf{u}_{a}\mathbf{u}_{a}^{\top}\mathbf{Q}_{m}\mathbf{v}_{a}\right]+\mathbb{E}\left[\mathbf{v}_{b}^{\top}\mathbf{u}_{b}\mathbf{u}_{b}^{\top}\mathbf{u}_{b}\mathbf{u}_{b}^{\top}\mathbf{v}_{b}\right]
\\&
=\mathbb{E}\left[\mathbf{v}_{b}^{\top}\mathbf{u}_{b}\mathbf{u}_{a}^{\top}\mathbf{Q}_{m}\mathbf{u}_{a}\mathbf{u}_{a}^{\top}\mathbf{Q}_{m}\mathbf{v}_{a}\right]+\mathbb{E}\left[\left\Vert \mathbf{u}_{b}\right\Vert ^{2}\left(\mathbf{u}_{b}^{\top}\mathbf{v}_{b}\right)^{2}\right]
\\&
\eqmargin
\left[\text{invariance of $\Q_m$ w.r.t. transpose (\propref{prop:invariance})}\right]
\\&
=\mathbb{E}\left[\mathbf{v}_{a}^{\top}\mathbf{Q}_{m}\mathbf{u}_{a}\mathbf{u}_{a}^{\top}\mathbf{Q}_{m}\mathbf{u}_{a}\mathbf{u}_{b}^{\top}\mathbf{v}_{b}\right]+\mathbb{E}\left[\left(1-\left\Vert \mathbf{u}_{a}\right\Vert ^{2}\right)\left(\mathbf{u}_{b}^{\top}\mathbf{v}_{b}\right)^{2}\right]
\\
\explain{\text{above}}
&=-\frac{1}{m}\mathbb{E}_{\mathbf{u}\perp\mathbf{v}}\left[\left\Vert \mathbf{u}_{a}\right\Vert ^{2}\left(\mathbf{u}_{b}^{\top}\mathbf{v}_{b}\right)^{2}\right]+\mathbb{E}\left[\left(\mathbf{u}_{b}^{\top}\mathbf{v}_{b}\right)^{2}\right]-\mathbb{E}\left[\left\Vert \mathbf{u}_{a}\right\Vert ^{2}\left(\mathbf{u}_{b}^{\top}\mathbf{v}_{b}\right)^{2}\right]\,.
\end{split}
\end{align}

Combining the above, we get
\begin{align*}
\mathbb{E}\left(\mathbf{e}_{2}^{\top}\mathbf{O}\mathbf{e}_{1}\mathbf{e}_{1}^{\top}\mathbf{O}\mathbf{e}_{1}\mathbf{e}_{1}^{\top}\mathbf{O}\mathbf{e}_{2}\right)
&
=\tfrac{\left(p-m\right)\left(p-m+2\right)}{\left(p-1\right)p\left(p+2\right)\left(p+4\right)}+\underbrace{\mathbb{E}\left[\left(\mathbf{u}_{b}^{\top}\mathbf{v}_{b}\right)^{2}\right]}_{\text{solved in \eqref{eq:(uava)2}}}-\left(1+\frac{2}{m}\right)\underbrace{\mathbb{E}\left[\left\Vert \mathbf{u}_{a}\right\Vert ^{2}\left(\mathbf{u}_{b}^{\top}\mathbf{v}_{b}\right)^{2}\right]}_{\text{solved in \eqref{eq:ua2(ubvb)2}}}
\\&
=\tfrac{\left(p-m\right)\left(p-m+2\right)}{\left(p-1\right)p\left(p+2\right)\left(p+4\right)}
+
\tfrac{\left(p-m\right)m}{\left(p-1\right)p\left(p+2\right)}-\tfrac{m+2}{m}\tfrac{\left(p-m\right)m\left(m+2\right)}{\left(p-1\right)p\left(p+2\right)\left(p+4\right)}
\\&
=\frac{\left(p-m\right)\left(-m^{2}+mp-m+p-2\right)}{\left(p-1\right)p\left(p+2\right)\left(p+4\right)}\,.
\end{align*}
\end{proof}

\newpage

\begin{proposition} 
\label{prop:(e2Oe1)(e3Oe1)(e3Oe2)}
For $p\ge 3, m\in\{2,3,\dots,p\}$ and a random transformation $\mO$ sampled as described in \eqref{eq:data-model},
it holds that,
\begin{align*}
\mathbb{E}\left[\mathbf{e}_{2}^{\top}\mathbf{O}\mathbf{e}_{1}
\mathbf{e}_{3}^{\top}\mathbf{O}\mathbf{e}_{1}
\mathbf{e}_{3}^{\top}\mathbf{O}\mathbf{e}_{2}\right]
= 
\frac{\left(p-m\right)\left(2m^{2}-3mp-2m-p+8\right)}{\left(p-2\right)\left(p-1\right)p\left(p+2\right)\left(p+4\right)}
\end{align*}
\end{proposition}

\begin{proof}
We start by decomposing the expectation into the following,
\begin{align*}
&\mathbb{E}\left[\mathbf{e}_{2}^{\top}\mathbf{O}\mathbf{e}_{1}\mathbf{e}_{3}^{\top}\mathbf{O}\mathbf{e}_{1}\mathbf{e}_{3}^{\top}\mathbf{O}\mathbf{e}_{2}\right]
\\&
=\mathbb{E}\left[\mathbf{v}^{\top}\left[\begin{smallmatrix}
\mathbf{Q}_{m}\\
 & \mathbf{I}_{p-m}
\end{smallmatrix}\right]\mathbf{u}\cdot\mathbf{z}^{\top}\left[\begin{smallmatrix}
\mathbf{Q}_{m}\\
 & \mathbf{I}_{p-m}
\end{smallmatrix}\right]\mathbf{u}\cdot\mathbf{z}^{\top}\left[\begin{smallmatrix}
\mathbf{Q}_{m}\\
 & \mathbf{I}_{p-m}
\end{smallmatrix}\right]\mathbf{v}\right]
\\&
=\mathbb{E}\left[\mathbf{v}_{a}^{\top}\mathbf{Q}_{m}\mathbf{u}_{a}\cdot\left(\mathbf{z}_{a}^{\top}\mathbf{Q}_{m}\mathbf{u}_{a}+\mathbf{z}_{b}^{\top}\mathbf{u}_{b}\right)\cdot\left(\mathbf{z}_{a}^{\top}\mathbf{Q}_{m}\mathbf{v}_{a}+\mathbf{z}_{b}^{\top}\mathbf{v}_{b}\right)\right]+
\\&
\eqmargin\,
\mathbb{E}\left[\mathbf{v}_{b}^{\top}\mathbf{u}_{b}\cdot\left(\mathbf{z}_{a}^{\top}\mathbf{Q}_{m}\mathbf{u}_{a}+\mathbf{z}_{b}^{\top}\mathbf{u}_{b}\right)\cdot\left(\mathbf{z}_{a}^{\top}\mathbf{Q}_{m}\mathbf{v}_{a}+\mathbf{z}_{b}^{\top}\mathbf{v}_{b}\right)\right]
\end{align*}

Focusing on the first additive term and by employing  \corref{cor:odd-Q_m} (in the first step below), we get,
\begin{align*}
&\mathbb{E}\left[\mathbf{v}_{a}^{\top}\mathbf{Q}_{m}\mathbf{u}_{a}\left(\mathbf{z}_{a}^{\top}\mathbf{Q}_{m}\mathbf{u}_{a}+\mathbf{z}_{b}^{\top}\mathbf{u}_{b}\right)\left(\mathbf{z}_{a}^{\top}\mathbf{Q}_{m}\mathbf{v}_{a}+\mathbf{z}_{b}^{\top}\mathbf{v}_{b}\right)\right]
\\&
=\mathbb{E}\left[\mathbf{v}_{a}^{\top}\mathbf{Q}_{m}\mathbf{u}_{a}\cdot\mathbf{z}_{a}^{\top}\mathbf{Q}_{m}\mathbf{u}_{a}\cdot\mathbf{z}_{b}^{\top}\mathbf{v}_{b}\right]+\mathbb{E}\left[\mathbf{v}_{a}^{\top}\mathbf{Q}_{m}\mathbf{u}_{a}\cdot\mathbf{z}_{b}^{\top}\mathbf{u}_{b}\cdot\mathbf{z}_{a}^{\top}\mathbf{Q}_{m}\mathbf{v}_{a}\right]
\\&
=\mathbb{E}\left[\left\Vert \mathbf{u}_{a}\right\Vert ^{2}\mathbf{v}_{a}^{\top}\mathbb{E}_{\mathbf{r}\sim\mathcal{S}^{m-1}}\left[\mathbf{r}\mathbf{r}^{\top}\right]\mathbf{z}_{a}\cdot\mathbf{z}_{b}^{\top}\mathbf{v}_{b}\right]+\mathbb{E}\left[\mathbf{z}_{b}^{\top}\mathbf{u}_{b}\cdot\left(\sum_{i,j=1}^{m}v_{i}q_{ij}u_{j}\right)\left(\sum_{k,\ell=1}^{m}z_{k}q_{k\ell}v_{\ell}\right)\right]
\\&
=\frac{1}{m}\mathbb{E}\bigg[\left\Vert \mathbf{u}_{a}\right\Vert ^{2}\underbrace{\mathbf{v}_{a}^{\top}\mathbf{z}_{a}}_{=-\mathbf{z}_{b}^{\top}\mathbf{v}_{b}}\cdot\mathbf{z}_{b}^{\top}\mathbf{v}_{b}\bigg]+\mathbb{E}\left[\left(\sum_{s=m+1}^{p}u_{s}z_{s}\right)\cdot\left(\sum_{i,j=1}^{m}v_{i}q_{ij}u_{j}\right)\left(\sum_{k,\ell=1}^{m}z_{k}q_{k\ell}v_{\ell}\right)\right]
\\&
=-\frac{1}{m}\underbrace{\mathbb{E}\left[\left\Vert \mathbf{u}_{a}\right\Vert ^{2}\left(\mathbf{v}_{b}^{\top}\mathbf{z}_{b}\right)^{2}\right]}_{\text{solved in \eqref{eq:(ua)^2(vbzb)^2}}}
+
\left(p-m\right)\sum_{i,j=1}^{m}\sum_{k,\ell=1}^{m}\mathbb{E}\left[u_{j}u_{p}v_{i}v_{\ell}z_{k}z_{p}\right]\mathbb{E}\left[q_{ij}q_{k\ell}\right]
\\
&
\eqmargin
\left[\text{By \propref{prop:odd}, most summands are zero}\right]
\\
&=-\frac{1}{m}\frac{m\left(p-m\right)\left(mp+2m-4\right)}{\left(p-2\right)\left(p-1\right)p\left(p+2\right)\left(p+4\right)}+\left(p-m\right)\sum_{i,j=1}^{m}\mathbb{E}\left[u_{j}u_{p}v_{i}v_{j}z_{i}z_{p}\right]\mathbb{E}\left[q_{ij}^{2}\right]
\\&
=-\tfrac{\left(p-m\right)\left(mp+2m-4\right)}{\left(p-2\right)\left(p-1\right)p\left(p+2\right)\left(p+4\right)}
+
\frac{p-m}{m}\left(\underbrace{m\mathbb{E}\left[u_{1}u_{p}v_{1}^{2}z_{1}z_{p}\right]}_{i=j}+\underbrace{m\left(m-1\right)\mathbb{E}\left[u_{2}u_{p}v_{1}v_{2}z_{1}z_{p}\right]}_{i\neq j}\right)
\\&
=-\frac{\left(p-m\right)\left(mp+2m-4\right)}{\left(p-2\right)\left(p-1\right)p\left(p+2\right)\left(p+4\right)}+\left(p-m\right)\left(\left\langle \begin{smallmatrix}
1 & 2 & 1\\
1 & 0 & 1\\
\overrightarrow{0} & \overrightarrow{0} & \overrightarrow{0}
\end{smallmatrix}\right\rangle +\left(m-1\right)\left\langle \begin{smallmatrix}
0 & 1 & 1\\
1 & 1 & 0\\
1 & 0 & 1\\
\overrightarrow{0} & \overrightarrow{0} & \overrightarrow{0}
\end{smallmatrix}\right\rangle \right)
\\&
=\left(p-m\right)\left(-\tfrac{mp+2m-4}{\left(p-2\right)\left(p-1\right)p\left(p+2\right)\left(p+4\right)}
+
\tfrac{-1}{\left(p-1\right)p\left(p+2\right)\left(p+4\right)}
+
\tfrac{2\left(m-1\right)}{\left(p-2\right)\left(p-1\right)p\left(p+2\right)\left(p+4\right)}\right)
\\&
=\left(p-m\right)\left(\frac{-mp+2}{\left(p-2\right)\left(p-1\right)p\left(p+2\right)\left(p+4\right)}+\frac{-1}{\left(p-1\right)p\left(p+2\right)\left(p+4\right)}\right)
\\&
=\frac{\left(p-m\right)\left(-mp-p+4\right)}{\left(p-2\right)\left(p-1\right)p\left(p+2\right)\left(p+4\right)}
\end{align*}

In addition, the second term is,
\begin{align*}
&\mathbb{E}\left[\mathbf{v}_{b}^{\top}\mathbf{u}_{b}\left(\mathbf{z}_{a}^{\top}\mathbf{Q}_{m}\mathbf{u}_{a}+\mathbf{z}_{b}^{\top}\mathbf{u}_{b}\right)\left(\mathbf{z}_{a}^{\top}\mathbf{Q}_{m}\mathbf{v}_{a}+\mathbf{z}_{b}^{\top}\mathbf{v}_{b}\right)\right]
\\&
=\mathbb{E}\left[\mathbf{v}_{b}^{\top}\mathbf{u}_{b}\cdot\mathbf{z}_{a}^{\top}\mathbf{Q}_{m}\mathbf{u}_{a}\cdot\mathbf{z}_{a}^{\top}\mathbf{Q}_{m}\mathbf{v}_{a}\right]+\mathbb{E}\left[\mathbf{v}_{b}^{\top}\mathbf{u}_{b}\cdot\mathbf{z}_{b}^{\top}\mathbf{u}_{b}\cdot\mathbf{z}_{b}^{\top}\mathbf{v}_{b}\right]
\\&
=\mathbb{E}\left[\left\Vert \mathbf{z}_{a}\right\Vert ^{2}\mathbf{v}_{b}^{\top}\mathbf{u}_{b}\mathbf{v}_{a}^{\top}\mathbb{E}_{\mathbf{r}\sim\mathcal{S}^{m-1}}\left[\mathbf{r}\mathbf{r}^{\top}\right]\mathbf{u}_{a}\right]+\mathbb{E}\left[\mathbf{v}_{b}^{\top}\mathbf{u}_{b}\mathbf{z}_{b}^{\top}\mathbf{u}_{b}\mathbf{z}_{b}^{\top}\mathbf{v}_{b}\right]
\\&
=\frac{1}{m}\mathbb{E}\bigg[\left\Vert \mathbf{z}_{a}\right\Vert ^{2}\mathbf{v}_{b}^{\top}\mathbf{u}_{b}\underbrace{\mathbf{v}_{a}^{\top}\mathbf{u}_{a}}_{=-\mathbf{v}_{b}^{\top}\mathbf{u}_{b}}\bigg]+\mathbb{E}\bigg[\mathbf{v}_{b}^{\top}\mathbf{u}_{b}\mathbf{z}_{b}^{\top}\mathbf{u}_{b}\underbrace{\mathbf{z}_{b}^{\top}\mathbf{v}_{b}}_{=-\mathbf{z}_{a}^{\top}\mathbf{v}_{a}}\bigg]
\\&
=-\frac{1}{m}\underbrace{\mathbb{E}\left[\left\Vert \mathbf{z}_{a}\right\Vert ^{2}\left(\mathbf{v}_{b}^{\top}\mathbf{u}_{b}\right)^{2}\right]}_{\text{solved in \eqref{eq:(ua)^2(vbzb)^2}}}-\mathbb{E}\left[\mathbf{v}_{b}^{\top}\mathbf{u}_{b}\mathbf{z}_{b}^{\top}\mathbf{u}_{b}\mathbf{z}_{a}^{\top}\mathbf{v}_{a}\right]
\\&
=-\frac{1}{m}\frac{m\left(p-m\right)\left(mp+2m-4\right)}{\left(p-2\right)\left(p-1\right)p\left(p+2\right)\left(p+4\right)}-\sum_{i=m+1}^{p}\sum_{k=m+1}^{p}\sum_{\ell=1}^{m}\mathbb{E}\left[u_{i}v_{i}u_{k}z_{k}v_{\ell}z_{\ell}\right]
\\&
=-\frac{\left(p-m\right)\left(mp+2m-4\right)}{\left(p-2\right)\left(p-1\right)p\left(p+2\right)\left(p+4\right)}-m\sum_{i=m+1}^{p}\sum_{k=m+1}^{p}\mathbb{E}\left[v_{1}z_{1}u_{i}v_{i}u_{k}z_{k}\right]
\\&
=-\frac{\left(p-m\right)\left(mp+2m-4\right)}{\left(p-2\right)\left(p-1\right)p\left(p+2\right)\left(p+4\right)}-\left(p-m\right)m\sum_{k=m+1}^{p}\mathbb{E}\left[v_{1}z_{1}u_{p}v_{p}u_{k}z_{k}\right]
\\&
=-\left(p-m\right)\Bigg(\tfrac{mp+2m-4}{\left(p-2\right)\left(p-1\right)p\left(p+2\right)\left(p+4\right)}
\\
& \eqmargin
+m\left(\underbrace{\mathbb{E}\left[u_{p}^{2}v_{1}v_{p}z_{1}z_{p}\right]}_{k=p}+\underbrace{\left(p-m-1\right)\mathbb{E}\left[u_{p-1}u_{p}v_{1}v_{p-1}z_{1}z_{p}\right]}_{m+1\le k\le p-1}\right)\Bigg)
\\&
=-\left(p-m\right)\left(\tfrac{mp+2m-4}{\left(p-2\right)\left(p-1\right)p\left(p+2\right)\left(p+4\right)}+m\left(\left\langle \begin{smallmatrix}
0 & 1 & 1\\
2 & 1 & 1\\
\overrightarrow{0} & \overrightarrow{0} & \overrightarrow{0}
\end{smallmatrix}\right\rangle +\left(p-m-1\right)\left\langle \begin{smallmatrix}
0 & 1 & 1\\
1 & 1 & 0\\
1 & 0 & 1\\
\overrightarrow{0} & \overrightarrow{0} & \overrightarrow{0}
\end{smallmatrix}\right\rangle \right)\right)
\\&
=-\left(p-m\right)\left(\tfrac{mp+2m-4}{\left(p-2\right)\left(p-1\right)p\left(p+2\right)\left(p+4\right)}
+
m\left(\tfrac{-1}{\left(p-1\right)p\left(p+2\right)\left(p+4\right)}
+
\tfrac{2\left(p-m-1\right)}{\left(p-2\right)\left(p-1\right)p\left(p+2\right)\left(p+4\right)}\right)\right)
\\&
=\frac{-2\left(p-m\right)\left(-m^{2}+mp+m-2\right)}{\left(p-2\right)\left(p-1\right)p\left(p+2\right)\left(p+4\right)}\,.
\end{align*}

Overall, we get that
\begin{align*}
\mathbb{E}\left[\mathbf{e}_{2}^{\top}\mathbf{O}\mathbf{e}_{1}\mathbf{e}_{3}^{\top}\mathbf{O}\mathbf{e}_{1}\mathbf{e}_{3}^{\top}\mathbf{O}\mathbf{e}_{2}\right]
%
&
=\tfrac{\left(p-m\right)\left(-mp-p+4\right)}{\left(p-2\right)\left(p-1\right)p\left(p+2\right)\left(p+4\right)}
+
\tfrac{-2\left(p-m\right)\left(-m^{2}+mp+m-2\right)}{\left(p-2\right)\left(p-1\right)p\left(p+2\right)\left(p+4\right)}
\\&
=\frac{\left(p-m\right)\left(2m^{2}-3mp-2m-p+8\right)}{\left(p-2\right)\left(p-1\right)p\left(p+2\right)\left(p+4\right)}
\,.
\end{align*}
\end{proof}

\newpage

\begin{proposition} 
\label{prop:(e1Oe1)(e1Oe2)(e2Oe1)(e2Oe2)}
For $p\ge 2, m\in\{2,\dots,p\}$ and a random transformation $\mO$ sampled as described in \eqref{eq:data-model},
it holds that,
\begin{align*}
&\mathbb{E}\left(\mathbf{e}_{1}^{\top}\mathbf{O}\mathbf{e}_{1}\cdot\mathbf{e}_{1}^{\top}\mathbf{O}\mathbf{e}_{2}\right)\left(\mathbf{e}_{2}^{\top}\mathbf{O}\mathbf{e}_{1}\cdot\mathbf{e}_{2}^{\top}\mathbf{O}\mathbf{e}_{2}\right)
\\
&=
\tfrac{-m^{4}\left(p+3\right)+m^{3}\left(3p^{2}+8p-6\right)+m^{2}\left(-3p^{3}-6p^{2}+13p+3\right)+m\left(p^{4}-7p^{2}+4p+6\right)+p^{4}-p^{3}-18p^{2}-40p-24}{\left(p-1\right)p\left(p+1\right)\left(p+2\right)\left(p+4\right)\left(p+6\right)}\,.
\end{align*}
\end{proposition}

\begin{proof}
We start by decomposing the quantity as follows,
\begin{align*}
&\mathbb{E}\left(\mathbf{e}_{1}^{\top}\mathbf{O}\mathbf{e}_{1}\cdot\mathbf{e}_{1}^{\top}\mathbf{O}\mathbf{e}_{2}\right)\left(\mathbf{e}_{2}^{\top}\mathbf{O}\mathbf{e}_{1}\cdot\mathbf{e}_{2}^{\top}\mathbf{O}\mathbf{e}_{2}\right)\\&=\mathbb{E}\left[\left(\mathbf{u}_{a}^{\top}\mathbf{Q}_{m}\mathbf{u}_{a}+\mathbf{u}_{b}^{\top}\mathbf{u}_{b}\right)\left(\mathbf{u}_{a}^{\top}\mathbf{Q}_{m}\mathbf{v}_{a}+\mathbf{u}_{b}^{\top}\mathbf{v}_{b}\right)\left(\mathbf{v}_{a}^{\top}\mathbf{Q}_{m}\mathbf{u}_{a}+\mathbf{v}_{b}^{\top}\mathbf{u}_{b}\right)\left(\mathbf{v}_{a}^{\top}\mathbf{Q}_{m}\mathbf{v}_{a}+\mathbf{v}_{b}^{\top}\mathbf{v}_{b}\right)\right]
\\
&=\mathbb{E}\left[\mathbf{u}_{a}^{\top}\mathbf{Q}_{m}\mathbf{u}_{a}\left(\mathbf{u}_{a}^{\top}\mathbf{Q}_{m}\mathbf{v}_{a}+\mathbf{u}_{b}^{\top}\mathbf{v}_{b}\right)\left(\mathbf{v}_{a}^{\top}\mathbf{Q}_{m}\mathbf{u}_{a}+\mathbf{v}_{b}^{\top}\mathbf{u}_{b}\right)\left(\mathbf{v}_{a}^{\top}\mathbf{Q}_{m}\mathbf{v}_{a}+\mathbf{v}_{b}^{\top}\mathbf{v}_{b}\right)\right]+
\\
&
\eqmargin
\mathbb{E}\left[\mathbf{u}_{b}^{\top}\mathbf{u}_{b}\left(\mathbf{u}_{a}^{\top}\mathbf{Q}_{m}\mathbf{v}_{a}+\mathbf{u}_{b}^{\top}\mathbf{v}_{b}\right)\left(\mathbf{v}_{a}^{\top}\mathbf{Q}_{m}\mathbf{u}_{a}+\mathbf{v}_{b}^{\top}\mathbf{u}_{b}\right)\left(\mathbf{v}_{a}^{\top}\mathbf{Q}_{m}\mathbf{v}_{a}+\mathbf{v}_{b}^{\top}\mathbf{v}_{b}\right)\right]
\end{align*}

We show that the first term is,
\begin{align*}
&\mathbb{E}\left[\mathbf{u}_{a}^{\top}\mathbf{Q}_{m}\mathbf{u}_{a}\left(\mathbf{u}_{a}^{\top}\mathbf{Q}_{m}\mathbf{v}_{a}+\mathbf{u}_{b}^{\top}\mathbf{v}_{b}\right)\left(\mathbf{v}_{a}^{\top}\mathbf{Q}_{m}\mathbf{u}_{a}+\mathbf{v}_{b}^{\top}\mathbf{u}_{b}\right)\left(\mathbf{v}_{a}^{\top}\mathbf{Q}_{m}\mathbf{v}_{a}+\mathbf{v}_{b}^{\top}\mathbf{v}_{b}\right)\right]
\\&
=\mathbb{E}\left[\mathbf{u}_{a}^{\top}\mathbf{Q}_{m}\mathbf{u}_{a}\mathbf{u}_{a}^{\top}\mathbf{Q}_{m}\mathbf{v}_{a}\left(\mathbf{v}_{a}^{\top}\mathbf{Q}_{m}\mathbf{u}_{a}+\mathbf{v}_{b}^{\top}\mathbf{u}_{b}\right)\left(\mathbf{v}_{a}^{\top}\mathbf{Q}_{m}\mathbf{v}_{a}+\mathbf{v}_{b}^{\top}\mathbf{v}_{b}\right)\right]+
\\&
\eqmargin
\mathbb{E}\left[\mathbf{u}_{a}^{\top}\mathbf{Q}_{m}\mathbf{u}_{a}\mathbf{u}_{b}^{\top}\mathbf{v}_{b}\left(\mathbf{v}_{a}^{\top}\mathbf{Q}_{m}\mathbf{u}_{a}+\mathbf{v}_{b}^{\top}\mathbf{u}_{b}\right)\left(\mathbf{v}_{a}^{\top}\mathbf{Q}_{m}\mathbf{v}_{a}+\mathbf{v}_{b}^{\top}\mathbf{v}_{b}\right)\right]
\\
\explain{\text{\corref{cor:odd-Q_m}}}
&=\mathbb{E}\left[\mathbf{u}_{a}^{\top}\mathbf{Q}_{m}\mathbf{u}_{a}\mathbf{u}_{a}^{\top}\mathbf{Q}_{m}\mathbf{v}_{a}\mathbf{v}_{a}^{\top}\mathbf{Q}_{m}\mathbf{u}_{a}\mathbf{v}_{a}^{\top}\mathbf{Q}_{m}\mathbf{v}_{a}\right]+\mathbb{E}\left[\mathbf{u}_{a}^{\top}\mathbf{Q}_{m}\mathbf{u}_{a}\mathbf{u}_{b}^{\top}\mathbf{v}_{b}\mathbf{v}_{b}^{\top}\mathbf{u}_{b}\mathbf{v}_{a}^{\top}\mathbf{Q}_{m}\mathbf{v}_{a}\right]+
\\&
\eqmargin
\underbrace{\mathbb{E}\left[\mathbf{u}_{a}^{\top}\mathbf{Q}_{m}\mathbf{u}_{a}\mathbf{u}_{a}^{\top}\mathbf{Q}_{m}\mathbf{v}_{a}\mathbf{v}_{b}^{\top}\mathbf{u}_{b}\mathbf{v}_{b}^{\top}\mathbf{v}_{b}\right]+\mathbb{E}\left[\mathbf{u}_{a}^{\top}\mathbf{Q}_{m}\mathbf{u}_{a}\mathbf{u}_{b}^{\top}\mathbf{v}_{b}\mathbf{v}_{a}^{\top}\mathbf{Q}_{m}\mathbf{u}_{a}\mathbf{v}_{b}^{\top}\mathbf{v}_{b}\right]}_{\text{equivalent}}
\\&
=\mathbb{E}\left[\mathbf{u}_{a}^{\top}\mathbf{Q}_{m}\mathbf{u}_{a}\mathbf{u}_{a}^{\top}\mathbf{Q}_{m}\mathbf{v}_{a}\mathbf{v}_{a}^{\top}\mathbf{Q}_{m}\mathbf{u}_{a}\mathbf{v}_{a}^{\top}\mathbf{Q}_{m}\mathbf{v}_{a}\right]+2\mathbb{E}\left[\mathbf{u}_{a}^{\top}\mathbf{Q}_{m}\mathbf{u}_{a}\mathbf{u}_{a}^{\top}\mathbf{Q}_{m}\mathbf{v}_{a}\mathbf{v}_{b}^{\top}\mathbf{u}_{b}\mathbf{v}_{b}^{\top}\mathbf{v}_{b}\right]+
\\&
\eqmargin
\mathbb{E}\left[\mathbf{u}_{a}^{\top}\mathbf{Q}_{m}\mathbf{u}_{a}\mathbf{v}_{a}^{\top}\mathbf{Q}_{m}\mathbf{v}_{a}\left(\mathbf{u}_{b}^{\top}\mathbf{v}_{b}\right)^{2}\right]
\\&
=\mathbb{E}\left[\mathbf{u}_{a}^{\top}\mathbf{Q}_{m}\mathbf{u}_{a}\mathbf{u}_{a}^{\top}\mathbf{Q}_{m}\mathbf{v}_{a}\mathbf{v}_{a}^{\top}\mathbf{Q}_{m}\mathbf{u}_{a}\mathbf{v}_{a}^{\top}\mathbf{Q}_{m}\mathbf{v}_{a}\right]+
\\&
\eqmargin
2\mathbb{E}\left[\left\Vert \mathbf{u}_{a}\right\Vert ^{2}\mathbb{E}_{\mathbf{r}\sim\mathcal{S}^{m-1}}\left(\mathbf{r}^{\top}\mathbf{u}_{a}\mathbf{r}^{\top}\mathbf{v}_{a}\right)\mathbf{v}_{b}^{\top}\mathbf{u}_{b}\mathbf{v}_{b}^{\top}\mathbf{v}_{b}\right]+
\mathbb{E}\left[\mathbf{u}_{a}^{\top}\mathbf{Q}_{m}\mathbf{u}_{a}\mathbf{v}_{a}^{\top}\mathbf{Q}_{m}\mathbf{v}_{a}\left(\mathbf{u}_{b}^{\top}\mathbf{v}_{b}\right)^{2}\right]
\\&
=\underbrace{\mathbb{E}\left[\mathbf{u}_{a}^{\top}\mathbf{Q}_{m}\mathbf{u}_{a}\mathbf{u}_{a}^{\top}\mathbf{Q}_{m}\mathbf{v}_{a}\mathbf{v}_{a}^{\top}\mathbf{Q}_{m}\mathbf{u}_{a}\mathbf{v}_{a}^{\top}\mathbf{Q}_{m}\mathbf{v}_{a}\right]}_{\text{solved in \propref{prop:(uaQmua)(uaQmva)(vaQmua)(vaQmVa)}}}
+
\tfrac{2}{m}\underbrace{\mathbb{E}\left[\left\Vert \mathbf{u}_{a}\right\Vert ^{2}\mathbf{u}_{a}^{\top}\mathbf{v}_{a}\cdot\mathbf{v}_{b}^{\top}\mathbf{u}_{b}\left\Vert \mathbf{v}_{b}\right\Vert ^{2}\right]}_{\text{solved in \eqref{eq:(ua)^2(uava)(ubvb)(vb)^2}}}+
\\
&
\eqmargin
\underbrace{\mathbb{E}\left[\mathbf{u}_{a}^{\top}\mathbf{Q}_{m}\mathbf{u}_{a}\mathbf{v}_{a}^{\top}\mathbf{Q}_{m}\mathbf{v}_{a}\left(\mathbf{v}_{b}^{\top}\mathbf{u}_{b}\right)^{2}\right]}_{\text{solved in \eqref{eq:(uaQmua)(vaQmva)(vbub)^2}}}
\\&
=\tfrac{-\left(m^{2}\left(2p+3\right)+m\left(-p^{2}+14p+30\right)-6p^{2}+4p+24\right)}{\left(p-1\right)p\left(p+1\right)\left(p+2\right)\left(p+4\right)\left(p+6\right)}+\tfrac{2m\left(p-m\right)\left(p+3\right)\left(m^{2}-mp-2\left(p+2\right)\right)}{m\left(p-1\right)p\left(p+1\right)\left(p+2\right)\left(p+4\right)\left(p+6\right)}+
\\
&
\eqmargin
\tfrac{3\left(p-m\right)\left(-m^{2}+mp+2p+4\right)}{\left(p-1\right)p\left(p+1\right)\left(p+2\right)\left(p+4\right)\left(p+6\right)}
%
%
\\&
=\tfrac{-m^{3}\left(2p+3\right)+m^{2}\left(4p^{2}+4p-3\right)-2m\left(p^{3}-p^{2}+9\right)-4\left(p^{3}+2p^{2}+4p+6\right)}{\left(p-1\right)p\left(p+1\right)\left(p+2\right)\left(p+4\right)\left(p+6\right)}\,.
\end{align*}
Moreover, we show that the second term is,
\begin{align*}
&\mathbb{E}\left[\mathbf{u}_{b}^{\top}\mathbf{u}_{b}\left(\mathbf{u}_{a}^{\top}\mathbf{Q}_{m}\mathbf{v}_{a}+\mathbf{u}_{b}^{\top}\mathbf{v}_{b}\right)\left(\mathbf{v}_{a}^{\top}\mathbf{Q}_{m}\mathbf{u}_{a}+\mathbf{v}_{b}^{\top}\mathbf{u}_{b}\right)\left(\mathbf{v}_{a}^{\top}\mathbf{Q}_{m}\mathbf{v}_{a}+\mathbf{v}_{b}^{\top}\mathbf{v}_{b}\right)\right]
\\
\explain{\text{\corref{cor:odd-Q_m}}}
&=\mathbb{E}\left[\mathbf{u}_{b}^{\top}\mathbf{u}_{b}\mathbf{u}_{b}^{\top}\mathbf{v}_{b}\mathbf{v}_{b}^{\top}\mathbf{u}_{b}\mathbf{v}_{b}^{\top}\mathbf{v}_{b}\right]+\mathbb{E}\left[\mathbf{u}_{b}^{\top}\mathbf{u}_{b}\mathbf{u}_{a}^{\top}\mathbf{Q}_{m}\mathbf{v}_{a}\mathbf{v}_{a}^{\top}\mathbf{Q}_{m}\mathbf{u}_{a}\mathbf{v}_{b}^{\top}\mathbf{v}_{b}\right]+
\\&
\eqmargin
\mathbb{E}\left[\mathbf{u}_{b}^{\top}\mathbf{u}_{b}\mathbf{u}_{a}^{\top}\mathbf{Q}_{m}\mathbf{v}_{a}\mathbf{v}_{b}^{\top}\mathbf{u}_{b}\mathbf{v}_{a}^{\top}\mathbf{Q}_{m}\mathbf{v}_{a}\right]+\mathbb{E}\left[\mathbf{u}_{b}^{\top}\mathbf{u}_{b}\mathbf{u}_{b}^{\top}\mathbf{v}_{b}\mathbf{v}_{a}^{\top}\mathbf{Q}_{m}\mathbf{u}_{a}\mathbf{v}_{a}^{\top}\mathbf{Q}_{m}\mathbf{v}_{a}\right]
\\&
=\mathbb{E}\left[\mathbf{u}_{b}^{\top}\mathbf{u}_{b}\left(\mathbf{v}_{a}^{\top}\mathbf{u}_{a}\right)^{2}\mathbf{v}_{b}^{\top}\mathbf{v}_{b}\right]+\mathbb{E}\left[\mathbf{u}_{b}^{\top}\mathbf{u}_{b}\mathbf{u}_{a}^{\top}\mathbf{Q}_{m}\mathbf{v}_{a}\mathbf{v}_{a}^{\top}\mathbf{Q}_{m}\mathbf{u}_{a}\mathbf{v}_{b}^{\top}\mathbf{v}_{b}\right]+
\\&
\eqmargin
\underbrace{\mathbb{E}\left[\mathbf{u}_{b}^{\top}\mathbf{u}_{b}\mathbf{u}_{b}^{\top}\mathbf{v}_{b}\cdot\mathbf{u}_{a}^{\top}\mathbf{Q}_{m}\mathbf{v}_{a}\mathbf{v}_{a}^{\top}\mathbf{Q}_{m}\mathbf{v}_{a}\right]+\mathbb{E}\left[\mathbf{u}_{b}^{\top}\mathbf{u}_{b}\mathbf{u}_{b}^{\top}\mathbf{v}_{b}\cdot\mathbf{v}_{a}^{\top}\mathbf{Q}_{m}\mathbf{u}_{a}\mathbf{v}_{a}^{\top}\mathbf{Q}_{m}\mathbf{v}_{a}\right]}_{\text{equal due to the invariance of \ensuremath{\mathbf{Q}_{m}} w.r.t. transpose (\propref{prop:invariance})}}
\\&
=\underbrace{\mathbb{E}\left[\mathbf{u}_{b}^{\top}\mathbf{u}_{b}\left(\mathbf{v}_{a}^{\top}\mathbf{u}_{a}\right)^{2}\mathbf{v}_{b}^{\top}\mathbf{v}_{b}\right]}_{\text{solved in \eqref{eq:(ub)^2(vb)^2(vaua)^2}}}
+
\underbrace{\mathbb{E}\left[\mathbf{u}_{b}^{\top}\mathbf{u}_{b}\mathbf{u}_{a}^{\top}\mathbf{Q}_{m}\mathbf{v}_{a}\mathbf{v}_{a}^{\top}\mathbf{Q}_{m}\mathbf{u}_{a}\mathbf{v}_{b}^{\top}\mathbf{v}_{b}\right]}_{\text{solved in \eqref{eq:(ub)^2(vb)^2(uaQmva)(vaQmua)}}}
+
\\
&
\eqmargin
\underbrace{2\mathbb{E}\left[\mathbf{u}_{b}^{\top}\mathbf{u}_{b}\mathbf{u}_{b}^{\top}\mathbf{v}_{b}\cdot\mathbf{u}_{a}^{\top}\mathbf{Q}_{m}\mathbf{v}_{a}\mathbf{v}_{a}^{\top}\mathbf{Q}_{m}\mathbf{v}_{a}\right]}_{\text{solved in \eqref{eq:(ub)^2(ubvb)(uaQmva)(vaQmva)}}}
\\&
=\tfrac{\left(p-m\right)m\left(m^{2}\left(p+3\right)-2m\left(p^{2}+5p+3\right)+p\left(p^{2}+7p+10\right)\right)}{\left(p-1\right)p\left(p+1\right)\left(p+2\right)\left(p+4\right)\left(p+6\right)}
+
\\&
\eqmargin
\tfrac{\left(p-m\right)\left(m^{2}p+3m^{2}-2mp^{2}-10mp-6m+p^{3}+7p^{2}+10p\right)}{\left(p-1\right)p\left(p+1\right)\left(p+2\right)\left(p+4\right)\left(p+6\right)}
+
\tfrac{-2\left(p-m\right)\left(p+3\right)\left(3\left(p+1\right)+\left(m-1\right)\left(p-m-1\right)\right)}{\left(p-1\right)p\left(p+1\right)\left(p+2\right)\left(p+4\right)\left(p+6\right)}
\\&
%
\\&
=\tfrac{-\left(p-m\right)\left(p-m+2\right)\left(m^{2}p+3m^{2}-mp^{2}-2mp+9m-p^{2}-p+12\right)}{\left(p-1\right)p\left(p+1\right)\left(p+2\right)\left(p+4\right)\left(p+6\right)}\,.
\end{align*}

Overall, we conclude that
\begin{align*}
&\mathbb{E}\left(\mathbf{e}_{1}^{\top}\mathbf{O}\mathbf{e}_{1}\cdot\mathbf{e}_{1}^{\top}\mathbf{O}\mathbf{e}_{2}\right)\left(\mathbf{e}_{2}^{\top}\mathbf{O}\mathbf{e}_{1}\cdot\mathbf{e}_{2}^{\top}\mathbf{O}\mathbf{e}_{2}\right)
\\
&=\tfrac{-m^{3}\left(2p+3\right)+m^{2}\left(4p^{2}+4p-3\right)-2m\left(p^{3}-p^{2}+9\right)-4\left(p^{3}+2p^{2}+4p+6\right)}{\left(p-1\right)p\left(p+1\right)\left(p+2\right)\left(p+4\right)\left(p+6\right)}+
\\
&\eqmargin
\tfrac{-\left(p-m\right)\left(p-m+2\right)\left(m^{2}p+3m^{2}-mp^{2}-2mp+9m-p^{2}-p+12\right)}{\left(p-1\right)p\left(p+1\right)\left(p+2\right)\left(p+4\right)\left(p+6\right)}
\\
&=\tfrac{-m^{4}\left(p+3\right)+m^{3}\left(3p^{2}+8p-6\right)+m^{2}\left(-3p^{3}-6p^{2}+13p+3\right)+m\left(p^{4}-7p^{2}+4p+6\right)+p^{4}-p^{3}-18p^{2}-40p-24}{\left(p-1\right)p\left(p+1\right)\left(p+2\right)\left(p+4\right)\left(p+6\right)}\,.
\end{align*}
\end{proof}

\newpage

\begin{proposition}
\label{prop:(uaQmua)(uaQmva)(vaQmua)(vaQmVa)}
For $p\ge 2, m\in\{2,\dots,p\}$ and a random transformation $\mO$ sampled as described in \eqref{eq:data-model},
it holds that,
\begin{align*}
\mathbb{E}\left[\mathbf{u}_{a}^{\top}\mathbf{Q}_{m}\mathbf{u}_{a}\mathbf{u}_{a}^{\top}\mathbf{Q}_{m}\mathbf{v}_{a}\mathbf{v}_{a}^{\top}\mathbf{Q}_{m}\mathbf{u}_{a}\mathbf{v}_{a}^{\top}\mathbf{Q}_{m}\mathbf{v}_{a}\right]=
\tfrac{-\left(m^{2}\left(2p+3\right)+m\left(-p^{2}+14p+30\right)-6p^{2}+4p+24\right)}{\left(p-1\right)p\left(p+1\right)\left(p+2\right)\left(p+4\right)\left(p+6\right)}\,.
\end{align*}
\end{proposition}

\begin{proof}
\begin{align*}
&\mathbb{E}\left[\mathbf{u}_{a}^{\top}\mathbf{Q}_{m}\mathbf{u}_{a}\mathbf{u}_{a}^{\top}\mathbf{Q}_{m}\mathbf{v}_{a}\mathbf{v}_{a}^{\top}\mathbf{Q}_{m}\mathbf{u}_{a}\mathbf{v}_{a}^{\top}\mathbf{Q}_{m}\mathbf{v}_{a}\right]
\\&
=\sum_{i,j=1}^{m}\sum_{k,\ell=1}^{m}\sum_{n,r=1}^{m}\sum_{s,t=1}^{m}\mathbb{E}\left[u_{i}u_{j}u_{k}v_{\ell}v_{n}u_{r}v_{s}v_{t}\right]\mathbb{E}\left[q_{ij}q_{k\ell}q_{nr}q_{st}\right]
\\&
=\underbrace{\sum_{i,n=1}^{m}\sum_{j,\ell,r,t=1}^{m}\mathbb{E}\left[u_{i}u_{j}u_{i}v_{\ell}v_{n}u_{r}v_{n}v_{t}\right]\mathbb{E}\left[q_{ij}q_{i\ell}q_{nr}q_{nt}\right]}_{i=k\Longrightarrow n=s}+
\\&
\eqmargin
\underbrace{\sum_{i\neq k=1}^{m}\sum_{j,\ell,r,t=1}^{m}
\mathbb{E}\left[u_{i}u_{j}u_{k}v_{\ell}v_{i}u_{r}v_{k}v_{t}\right]\left(\mathbb{E}\left[q_{ij}q_{k\ell}q_{ir}q_{kt}\right]+\mathbb{E}\left[q_{ij}q_{k\ell}q_{kr}q_{it}\right]\right)
}_{i\neq k\Longrightarrow\left(n=i\neq s=k\right)\,\vee\,\left(n=k\neq s=i\right)\text{ due to symmetry w.r.t. \ensuremath{n,s}}}
\\&
=\sum_{i,n=1}^{m}\sum_{j,\ell,r,t=1}^{m}\mathbb{E}\left[u_{i}u_{j}u_{i}v_{\ell}u_{r}v_{n}^{2}v_{t}\right]\mathbb{E}\left[q_{ij}q_{i\ell}q_{nr}q_{nt}\right]+
\\&
\eqmargin
\sum_{i\neq k=1}^{m}\sum_{j,\ell,r,t=1}^{m}\mathbb{E}\left[u_{i}u_{j}u_{k}v_{\ell}v_{i}u_{r}v_{k}v_{t}\right]\left(\mathbb{E}\left[q_{ij}q_{k\ell}q_{ir}q_{kt}\right]+\mathbb{E}\left[q_{ij}q_{k\ell}q_{kr}q_{it}\right]\right)
\\&
=m\underbrace{\sum_{n=1}^{m}\sum_{j,\ell,r,t=1}^{m}\mathbb{E}\left[u_{1}^{2}u_{j}u_{r}v_{\ell}v_{t}v_{n}^{2}\right]\mathbb{E}\left[q_{1,j}q_{1,\ell}q_{nr}q_{nt}\right]}_{\triangleq\text{(A) below}}+
\\&
\eqmargin
m\underbrace{\sum_{k=2}^{m}\sum_{j,\ell,r,t=1}^{m}\mathbb{E}\left[u_{1}v_{1}u_{r}u_{j}u_{k}v_{k}v_{\ell}v_{t}\right]\left(\mathbb{E}\left[q_{1,j}q_{k\ell}q_{1,r}q_{kt}\right]+\mathbb{E}\left[q_{1,j}q_{k\ell}q_{kr}q_{1,t}\right]\right)}_{\triangleq\text{(B) below}}
\\&
=\tfrac{-m^{3}\left(p+3\right)+m^{2}\left(p^{2}-5p-12\right)+m\left(6p^{2}-26p-60\right)+16p^{2}-8p-48}{\left(m+2\right)\left(p-1\right)p\left(p+1\right)\left(p+2\right)\left(p+4\right)\left(p+6\right)}+
\\&
\eqmargin
\tfrac{-m^{3}p-13m^{2}p-24m^{2}+2mp^{2}-6mp-24m-4p^{2}}{\left(m+2\right)\left(p-1\right)p\left(p+1\right)\left(p+2\right)\left(p+4\right)\left(p+6\right)}
%
%
\\&
=\frac{-\left(m^{2}\left(2p+3\right)+m\left(-p^{2}+14p+30\right)-6p^{2}+4p+24\right)}{\left(p-1\right)p\left(p+1\right)\left(p+2\right)\left(p+4\right)\left(p+6\right)}
\end{align*}

\newpage

\paragraph{Deriving (A).}
\begin{align*}
\text{(A)}&=\sum_{n=1}^{m}\sum_{j,\ell,r,t=1}^{m}\mathbb{E}\left[u_{1}^{2}u_{j}u_{r}v_{\ell}v_{t}v_{n}^{2}\right]\mathbb{E}\left[q_{1,j}q_{1,\ell}q_{nr}q_{nt}\right]
\\&
=\underbrace{\sum_{j,\ell,r,t=1}^{m}\mathbb{E}\left[u_{1}^{2}u_{j}u_{r}v_{\ell}v_{t}v_{1}^{2}\right]\mathbb{E}\left[q_{1,j}q_{1,\ell}q_{1,r}q_{1,t}\right]}_{n=1\text{, solved below}}+
\\
&
\eqmargin
\left(m-1\right)\underbrace{\sum_{j,\ell,r,t=1}^{m}\mathbb{E}\left[u_{1}^{2}u_{j}u_{r}v_{\ell}v_{t}v_{2}^{2}\right]\mathbb{E}\left[q_{1,j}q_{1,\ell}q_{2,r}q_{2,t}\right]}_{n\ge2\text{, solved below}}
\\&
=\tfrac{m^{2}\left(p^{2}+4p+15\right)+6m\left(p-3\right)\left(p+1\right)+4\left(5p^{2}+2p-6\right)}{m\left(m+2\right)\left(p-1\right)p\left(p+1\right)\left(p+2\right)\left(p+4\right)\left(p+6\right)}+
\\
&
\eqmargin
\tfrac{-m^{3}p-3m^{3}-9m^{2}p-27m^{2}-14mp-42m-4p^{2}-16p-24}{m\left(m+2\right)\left(p-1\right)p\left(p+1\right)\left(p+2\right)\left(p+4\right)\left(p+6\right)}
\\&
=\frac{-m^{3}\left(p+3\right)+m^{2}\left(p^{2}-5p-12\right)+m\left(6p^{2}-26p-60\right)+16p^{2}-8p-48}{m\left(m+2\right)\left(p-1\right)p\left(p+1\right)\left(p+2\right)\left(p+4\right)\left(p+6\right)}
\end{align*}
When $n=1$: The only nonzero options are $\left\langle \begin{smallmatrix}
4 & 0\\
\overrightarrow{0} & \overrightarrow{0}
\end{smallmatrix}\right\rangle$  and $\left\langle \begin{smallmatrix}
2 & 0\\
2 & 0\\
\overrightarrow{0} & \overrightarrow{0}
\end{smallmatrix}\right\rangle$. 
We have,
\begin{align*}
&\sum_{j,\ell,r,t=1}^{m}\mathbb{E}\left[u_{1}^{2}u_{j}u_{r}v_{\ell}v_{t}v_{1}^{2}\right]\mathbb{E}\left[q_{1,j}q_{1,\ell}q_{1,r}q_{1,t}\right]
\\&
=\left\langle \begin{smallmatrix}
4 & 0\\
\overrightarrow{0} & \overrightarrow{0}
\end{smallmatrix}\right\rangle _{m}\sum_{j=1}^{m}\mathbb{E}\left[u_{1}^{2}u_{j}^{2}v_{j}^{2}v_{1}^{2}\right]+
\\
&
\eqmargin
\left\langle \begin{smallmatrix}
2 & 0\\
2 & 0\\
\overrightarrow{0} & \overrightarrow{0}
\end{smallmatrix}\right\rangle _{m}\left(\sum_{j\neq\ell=1}^{m}\underbrace{\mathbb{E}\left[u_{1}^{2}u_{j}^{2}v_{\ell}^{2}v_{1}^{2}\right]}_{j=r\neq\ell=t}+\sum_{j\neq r=1}^{m}\underbrace{\mathbb{E}\left[u_{1}^{2}u_{j}u_{r}v_{j}v_{r}v_{1}^{2}\right]}_{j=\ell\neq r=t}+\sum_{j\neq\ell=1}^{m}\underbrace{\mathbb{E}\left[u_{1}^{2}u_{j}u_{\ell}v_{j}v_{\ell}v_{1}^{2}\right]}_{j=t\neq\ell=r}\right)
\\&
=\left\langle \begin{smallmatrix}
4 & 0\\
\overrightarrow{0} & \overrightarrow{0}
\end{smallmatrix}\right\rangle _{m}\sum_{j=1}^{m}\mathbb{E}\left[u_{1}^{2}u_{j}^{2}v_{1}^{2}v_{j}^{2}\right]+\left\langle \begin{smallmatrix}
2 & 0\\
2 & 0\\
\overrightarrow{0} & \overrightarrow{0}
\end{smallmatrix}\right\rangle _{m}\sum_{j\neq\ell=1}^{m}\left(\mathbb{E}\left[u_{1}^{2}u_{j}^{2}v_{1}^{2}v_{\ell}^{2}\right]+2\mathbb{E}\left[u_{1}^{2}u_{j}u_{\ell}v_{1}^{2}v_{j}v_{\ell}\right]\right)
\\&
=\left\langle \begin{smallmatrix}
4 & 0\\
\overrightarrow{0} & \overrightarrow{0}
\end{smallmatrix}\right\rangle _{m}\left(\mathbb{E}\left[u_{1}^{4}v_{1}^{4}\right]+
\left(m-1\right)
\mathbb{E}\left[u_{1}^{2}u_{2}^{2}v_{1}^{2}v_{2}^{2}\right]\right)+
\\&
\eqmargin
\left\langle \begin{smallmatrix}
2 & 0\\
2 & 0\\
\overrightarrow{0} & \overrightarrow{0}
\end{smallmatrix}\right\rangle _{m}\left(\sum_{j\neq\ell=1}^{m}\mathbb{E}\left[u_{1}^{2}u_{j}^{2}v_{1}^{2}v_{\ell}^{2}\right]+2\sum_{j\neq\ell=1}^{m}\mathbb{E}\left[u_{1}^{2}u_{j}u_{\ell}v_{1}^{2}v_{j}v_{\ell}\right]\right)
\\&
=\left\langle \begin{smallmatrix}
4 & 0\\
\overrightarrow{0} & \overrightarrow{0}
\end{smallmatrix}\right\rangle _{m}\left(\left\langle \begin{smallmatrix}
4 & 4\\
\overrightarrow{0} & \overrightarrow{0}
\end{smallmatrix}\right\rangle _{p}+\left(m-1\right)\left\langle \begin{smallmatrix}
2 & 2\\
2 & 2\\
\overrightarrow{0} & \overrightarrow{0}
\end{smallmatrix}\right\rangle _{p}\right)
+
\left\langle \begin{smallmatrix}
2 & 0\\
2 & 0\\
\overrightarrow{0} & \overrightarrow{0}
\end{smallmatrix}\right\rangle _{m}\left(
2\sum_{j\neq\ell=1}^{m}\mathbb{E}\left[u_{1}^{2}u_{j}u_{\ell}v_{j}v_{\ell}v_{1}^{2}\right]\right)
\\&
\eqmargin
\left\langle \begin{smallmatrix}
2 & 0\\
2 & 0\\
\overrightarrow{0} & \overrightarrow{0}
\end{smallmatrix}\right\rangle _{m}\left(\underbrace{\left(m-1\right)\mathbb{E}\left[u_{1}^{4}v_{2}^{2}v_{1}^{2}\right]}_{j=1,\ell\ge2}+\underbrace{\left(m-1\right)\mathbb{E}\left[u_{1}^{2}u_{2}^{2}v_{1}^{4}\right]}_{j\ge2,\ell=1}+
\underbrace{\left(m-1\right)\left(m-2\right)\mathbb{E}\left[u_{1}^{2}u_{2}^{2}v_{1}^{2}v_{3}^{2}\right]}_{j\neq\ell\ge2}
\right)
\\&
=\left\langle \begin{smallmatrix}
4 & 0\\
\overrightarrow{0} & \overrightarrow{0}
\end{smallmatrix}\right\rangle _{m}\left(\left\langle \begin{smallmatrix}
4 & 4\\
\overrightarrow{0} & \overrightarrow{0}
\end{smallmatrix}\right\rangle _{p}+\left(m-1\right)\left\langle \begin{smallmatrix}
2 & 2\\
2 & 2\\
\overrightarrow{0} & \overrightarrow{0}
\end{smallmatrix}\right\rangle _{p}\right)+
\\&
\eqmargin
\left\langle \begin{smallmatrix}
2 & 0\\
2 & 0\\
\overrightarrow{0} & \overrightarrow{0}
\end{smallmatrix}\right\rangle _{m}\left(2\left(m-1\right)\left\langle \begin{smallmatrix}
4 & 2\\
0 & 2\\
\overrightarrow{0} & \overrightarrow{0}
\end{smallmatrix}\right\rangle _{p}+\left(m-1\right)\left(m-2\right)\left\langle \begin{smallmatrix}
2 & 2\\
2 & 0\\
0 & 2\\
\overrightarrow{0} & \overrightarrow{0}
\end{smallmatrix}\right\rangle _{p}
\right)
+
\\&
\eqmargin
\left\langle \begin{smallmatrix}
2 & 0\\
2 & 0\\
\overrightarrow{0} & \overrightarrow{0}
\end{smallmatrix}\right\rangle _{m}
\left(
2\left(m-1\right)\left(\underbrace{2\mathbb{E}\left[u_{1}^{3}u_{2}v_{2}v_{1}^{3}\right]}_{j=1,\ell\ge2\,\vee\,j\ge2,\ell=1}+\underbrace{\left(m-2\right)\mathbb{E}\left[u_{1}^{2}u_{2}u_{3}v_{2}v_{3}v_{1}^{2}\right]}_{j\neq\ell\ge2}\right)\right)
\\&
=\left\langle \begin{smallmatrix}
4 & 0\\
\overrightarrow{0} & \overrightarrow{0}
\end{smallmatrix}\right\rangle _{m}\left(\left\langle \begin{smallmatrix}
4 & 4\\
\overrightarrow{0} & \overrightarrow{0}
\end{smallmatrix}\right\rangle _{p}+\left(m-1\right)\left\langle \begin{smallmatrix}
2 & 2\\
2 & 2\\
\overrightarrow{0} & \overrightarrow{0}
\end{smallmatrix}\right\rangle _{p}\right)+
\\&
\eqmargin
\left(m-1\right)\left\langle \begin{smallmatrix}
2 & 0\\
2 & 0\\
\overrightarrow{0} & \overrightarrow{0}
\end{smallmatrix}\right\rangle _{m}\left(2\left\langle \begin{smallmatrix}
4 & 2\\
0 & 2\\
\overrightarrow{0} & \overrightarrow{0}
\end{smallmatrix}\right\rangle _{p}+\left(m-2\right)\left\langle \begin{smallmatrix}
2 & 2\\
2 & 0\\
0 & 2\\
\overrightarrow{0} & \overrightarrow{0}
\end{smallmatrix}\right\rangle _{p}+4\left\langle \begin{smallmatrix}
3 & 3\\
1 & 1\\
\overrightarrow{0} & \overrightarrow{0}
\end{smallmatrix}\right\rangle _{p}+2\left(m-2\right)\left\langle \begin{smallmatrix}
2 & 2\\
1 & 1\\
1 & 1\\
\overrightarrow{0} & \overrightarrow{0}
\end{smallmatrix}\right\rangle _{p}\right)
\\&
=\frac{3}{m\left(m+2\right)}\left(\left\langle \begin{smallmatrix}
4 & 4\\
\overrightarrow{0} & \overrightarrow{0}
\end{smallmatrix}\right\rangle _{p}+\left(m-1\right)\left\langle \begin{smallmatrix}
2 & 2\\
2 & 2\\
\overrightarrow{0} & \overrightarrow{0}
\end{smallmatrix}\right\rangle _{p}\right)
+
\\
&
\eqmargin
\frac{\left(m-1\right)}{m\left(m+2\right)}\left(2\left\langle \begin{smallmatrix}
4 & 2\\
0 & 2\\
\overrightarrow{0} & \overrightarrow{0}
\end{smallmatrix}\right\rangle _{p}+4\left\langle \begin{smallmatrix}
3 & 3\\
1 & 1\\
\overrightarrow{0} & \overrightarrow{0}
\end{smallmatrix}\right\rangle _{p}+\left(m-2\right)\left(2\left\langle \begin{smallmatrix}
2 & 2\\
1 & 1\\
1 & 1\\
\overrightarrow{0} & \overrightarrow{0}
\end{smallmatrix}\right\rangle _{p}+\left\langle \begin{smallmatrix}
2 & 2\\
2 & 0\\
0 & 2\\
\overrightarrow{0} & \overrightarrow{0}
\end{smallmatrix}\right\rangle _{p}\right)\right)
\\&
=\tfrac{3}{m\left(m+2\right)}\left(\tfrac{9\left(p-1\right)\left(p+1\right)+\left(m-1\right)\left(p^{2}+4p+15\right)}{\left(p-1\right)p\left(p+1\right)\left(p+2\right)\left(p+4\right)\left(p+6\right)}\right)+
\\&
\eqmargin
\tfrac{\left(m-1\right)}{m\left(m+2\right)}\left(\tfrac{2\left(p+1\right)\cdot3\left(p+3\right)-4\cdot9\left(p+1\right)}{\left(p-1\right)p\left(p+1\right)\left(p+2\right)\left(p+4\right)\left(p+6\right)}+\tfrac{\left(m-2\right)\left(2\left(-p+3\right)+\left(p+3\right)^{2}\right)}{\left(p-1\right)p\left(p+1\right)\left(p+2\right)\left(p+4\right)\left(p+6\right)}\right)
%
%
\\&
=\frac{m^{2}\left(p^{2}+4p+15\right)+6m\left(p-3\right)\left(p+1\right)+4\left(5p^{2}+2p-6\right)}{m\left(m+2\right)\left(p-1\right)p\left(p+1\right)\left(p+2\right)\left(p+4\right)\left(p+6\right)}\,.
\end{align*}

\newpage

When $n\ge2$: The only nonzero options are $\left\langle \begin{smallmatrix}
2 & 2\\
\overrightarrow{0} & \overrightarrow{0}
\end{smallmatrix}\right\rangle ,\left\langle \begin{smallmatrix}
1 & 1\\
1 & 1\\
\overrightarrow{0} & \overrightarrow{0}
\end{smallmatrix}\right\rangle ,\left\langle \begin{smallmatrix}
2 & 0\\
0 & 2\\
\overrightarrow{0} & \overrightarrow{0}
\end{smallmatrix}\right\rangle$. 
We have,
\begin{align*}
&\sum_{j,\ell,r,t=1}^{m}\mathbb{E}\left[u_{1}^{2}u_{j}u_{r}v_{\ell}v_{t}v_{2}^{2}\right]\mathbb{E}\left[q_{1,j}q_{1,\ell}q_{2,r}q_{2,t}\right]
\\&
=
\underbrace{\left\langle \begin{smallmatrix}
2 & 2\\
\overrightarrow{0} & \overrightarrow{0}
\end{smallmatrix}\right\rangle _{m}\sum_{j=1}^{m}\mathbb{E}\left[u_{1}^{2}u_{j}^{2}v_{j}^{2}v_{2}^{2}\right]}_{j=\ell=r=t}+\underbrace{\left\langle \begin{smallmatrix}
1 & 1\\
1 & 1\\
\overrightarrow{0} & \overrightarrow{0}
\end{smallmatrix}\right\rangle _{m}\sum_{j\neq\ell=1}^{m}\left(\mathbb{E}\left[u_{1}^{2}u_{j}u_{\ell}v_{j}v_{\ell}v_{2}^{2}\right]+\mathbb{E}\left[u_{1}^{2}u_{j}^{2}v_{\ell}^{2}v_{2}^{2}\right]\right)}_{j=t\neq r=\ell\,\vee\,j=r\neq t=\ell}+
\\&
\eqmargin
\underbrace{\left\langle \begin{smallmatrix}
2 & 0\\
0 & 2\\
\overrightarrow{0} & \overrightarrow{0}
\end{smallmatrix}\right\rangle _{m}\sum_{j\neq r=1}^{m}\mathbb{E}\left[u_{1}^{2}u_{j}u_{r}v_{j}v_{r}v_{2}^{2}\right]}_{j=\ell\neq r=t}
\\&
=
\left\langle \begin{smallmatrix}
2 & 2\\
\overrightarrow{0} & \overrightarrow{0}
\end{smallmatrix}\right\rangle _{m}\sum_{j=1}^{m}\mathbb{E}\left[u_{1}^{2}u_{j}^{2}v_{j}^{2}v_{2}^{2}\right]+\left\langle \begin{smallmatrix}
1 & 1\\
1 & 1\\
\overrightarrow{0} & \overrightarrow{0}
\end{smallmatrix}\right\rangle _{m}\sum_{j\neq\ell=1}^{m}\mathbb{E}\left[u_{1}^{2}u_{j}^{2}v_{\ell}^{2}v_{2}^{2}\right]+
\\&
\eqmargin
\underbrace{\left(\left\langle \begin{smallmatrix}
1 & 1\\
1 & 1\\
\overrightarrow{0} & \overrightarrow{0}
\end{smallmatrix}\right\rangle _{m}+\left\langle \begin{smallmatrix}
2 & 0\\
0 & 2\\
\overrightarrow{0} & \overrightarrow{0}
\end{smallmatrix}\right\rangle _{m}\right)}_{=\tfrac{1}{\left(m-1\right)\left(m+2\right)}}\sum_{j\neq r=1}^{m}\mathbb{E}\left[u_{1}^{2}u_{j}u_{r}v_{j}v_{r}v_{2}^{2}\right]
\\&
=\tfrac{1}{m\left(m+2\right)}
\bigg(\underbrace{\mathbb{E}\left[u_{1}^{4}v_{1}^{2}v_{2}^{2}\right]}_{j=1}+\underbrace{\mathbb{E}\left[u_{1}^{2}u_{2}^{2}v_{2}^{4}\right]}_{j=2}+\underbrace{\left(m-2\right)\mathbb{E}\left[u_{1}^{2}u_{3}^{2}v_{3}^{2}v_{2}^{2}\right]}_{j\ge3}\bigg)+
\\&
\eqmargin
\tfrac{-1}{\left(m-1\right)m\left(m+2\right)}
\bigg(\underbrace{\mathbb{E}\left[u_{1}^{4}v_{2}^{4}\right]}_{j=1,\ell=2}+\underbrace{\mathbb{E}\left[u_{1}^{2}u_{2}^{2}v_{1}^{2}v_{2}^{2}\right]}_{j=2,\ell=1}\bigg)
+
\\
&
\eqmargin
\tfrac{-1}{\left(m-1\right)m\left(m+2\right)}
\left(m-2\right)
\bigg(
\underbrace{\mathbb{E}\left[u_{1}^{4}v_{2}^{2}v_{3}^{2}\right]}_{j=1,\ell\ge3}+\underbrace{\mathbb{E}\left[u_{1}^{2}u_{3}^{2}v_{1}^{2}v_{2}^{2}\right]}_{j\ge3,\ell=1}+\underbrace{\mathbb{E}\left[u_{1}^{2}u_{2}^{2}v_{2}^{2}v_{3}^{2}\right]}_{j=2,\ell\ge3}
+
\underbrace{\mathbb{E}\left[u_{1}^{2}u_{3}^{2}v_{2}^{4}\right]}_{j\ge3,\ell=2}+
\\
&
\hspace{3.9cm}
\underbrace{\left(m-3\right)\mathbb{E}\left[u_{1}^{2}u_{3}^{2}v_{2}^{2}v_{4}^{2}\right]}_{j\neq\ell\ge3}
\bigg)+
\\&
\eqmargin
\tfrac{1}{\left(m-1\right)\left(m+2\right)}\Bigg(\underbrace{2\mathbb{E}\left[u_{1}^{3}u_{2}v_{1}v_{2}^{3}\right]}_{j=1,r=2\,\vee\,j=2,r=1}+\underbrace{2\left(m-2\right)\mathbb{E}\left[u_{1}^{3}u_{3}v_{1}v_{2}^{2}v_{3}\right]}_{j=1,r\ge3\,\vee\,j\ge3,r=1}+\underbrace{2\left(m-2\right)\mathbb{E}\left[u_{1}^{2}u_{2}u_{3}v_{2}^{3}v_{3}\right]}_{j=2,r\ge3\,\vee\,j\ge3,r=2}
+
\\
&
\hspace{2.4cm}
\underbrace{\left(m-2\right)\left(m-3\right)\mathbb{E}\left[u_{1}^{2}u_{3}u_{4}v_{2}^{2}v_{3}v_{4}\right]}_{j\neq r\ge3}\Bigg)
\\&
=\tfrac{1}{m\left(m+2\right)}
\left(2\left\langle \begin{smallmatrix}
4 & 2\\
0 & 2\\
\overrightarrow{0} & \overrightarrow{0}
\end{smallmatrix}\right\rangle +\left(m-2\right)\left\langle \begin{smallmatrix}
2 & 0\\
0 & 2\\
2 & 2\\
\overrightarrow{0} & \overrightarrow{0}
\end{smallmatrix}\right\rangle \right)+
%
\tfrac{-1}{\left(m-1\right)m\left(m+2\right)}
\bigg(\left\langle \begin{smallmatrix}
4 & 0\\
0 & 4\\
\overrightarrow{0} & \overrightarrow{0}
\end{smallmatrix}\right\rangle +\left\langle \begin{smallmatrix}
2 & 2\\
2 & 2\\
\overrightarrow{0} & \overrightarrow{0}
\end{smallmatrix}\right\rangle \bigg)
+
\\
& 
\eqmargin
\tfrac{-\left(m-2\right)}{\left(m-1\right)m\left(m+2\right)}
\left(\left\langle \begin{smallmatrix}
4 & 0\\
0 & 2\\
0 & 2\\
\overrightarrow{0} & \overrightarrow{0}
\end{smallmatrix}\right\rangle +\left\langle \begin{smallmatrix}
2 & 2\\
0 & 2\\
2 & 0\\
\overrightarrow{0} & \overrightarrow{0}
\end{smallmatrix}\right\rangle +\left\langle \begin{smallmatrix}
2 & 0\\
2 & 2\\
0 & 2\\
\overrightarrow{0} & \overrightarrow{0}
\end{smallmatrix}\right\rangle +\left\langle \begin{smallmatrix}
2 & 0\\
0 & 4\\
2 & 0\\
\overrightarrow{0} & \overrightarrow{0}
\end{smallmatrix}\right\rangle +\left(m-3\right)\left\langle \begin{smallmatrix}
2 & 0\\
0 & 2\\
2 & 0\\
0 & 2\\
\overrightarrow{0} & \overrightarrow{0}
\end{smallmatrix}\right\rangle \right)
+
\\&
\eqmargin
\tfrac{1}{\left(m-1\right)\left(m+2\right)}
\Bigg(2\left\langle \begin{smallmatrix}
3 & 1\\
1 & 3\\
\overrightarrow{0} & \overrightarrow{0}
\end{smallmatrix}\right\rangle +2\left(m-2\right)\left\langle \begin{smallmatrix}
3 & 1\\
0 & 2\\
1 & 1\\
\overrightarrow{0} & \overrightarrow{0}
\end{smallmatrix}\right\rangle +2\left(m-2\right)\left\langle \begin{smallmatrix}
2 & 0\\
1 & 3\\
1 & 1\\
\overrightarrow{0} & \overrightarrow{0}
\end{smallmatrix}\right\rangle +
\\ &
\hspace{2.4cm}
\left(m-2\right)\left(m-3\right)\left\langle \begin{smallmatrix}
2 & 0\\
0 & 2\\
1 & 1\\
1 & 1\\
\overrightarrow{0} & \overrightarrow{0}
\end{smallmatrix}\right\rangle \Bigg)
\end{align*}

\newpage

\begin{align*}
&
=\tfrac{1}{m\left(m+2\right)}\left(2\left\langle \begin{smallmatrix}
4 & 2\\
0 & 2\\
\overrightarrow{0} & \overrightarrow{0}
\end{smallmatrix}\right\rangle +\left(m-2\right)\left\langle \begin{smallmatrix}
2 & 0\\
0 & 2\\
2 & 2\\
\overrightarrow{0} & \overrightarrow{0}
\end{smallmatrix}\right\rangle \right)+
\\& 
\eqmargin
\tfrac{-1}{\left(m-1\right)m\left(m+2\right)}
\Bigg(\left\langle \begin{smallmatrix}
4 & 0\\
0 & 4\\
\overrightarrow{0} & \overrightarrow{0}
\end{smallmatrix}\right\rangle +\left\langle \begin{smallmatrix}
2 & 2\\
2 & 2\\
\overrightarrow{0} & \overrightarrow{0}
\end{smallmatrix}\right\rangle + 
\\& 
\hspace{2.6cm}
\left(m-2\right)\bigg(2\left\langle \begin{smallmatrix}
4 & 0\\
0 & 2\\
0 & 2\\
\overrightarrow{0} & \overrightarrow{0}
\end{smallmatrix}\right\rangle +2\left\langle \begin{smallmatrix}
2 & 2\\
0 & 2\\
2 & 0\\
\overrightarrow{0} & \overrightarrow{0}
\end{smallmatrix}\right\rangle +\left(m-3\right)\left\langle \begin{smallmatrix}
2 & 0\\
0 & 2\\
2 & 0\\
0 & 2\\
\overrightarrow{0} & \overrightarrow{0}
\end{smallmatrix}\right\rangle \bigg)\Bigg)+
\\& \eqmargin
\tfrac{1}{\left(m-1\right)\left(m+2\right)}\left(2\left\langle \begin{smallmatrix}
3 & 1\\
1 & 3\\
\overrightarrow{0} & \overrightarrow{0}
\end{smallmatrix}\right\rangle +4\left(m-2\right)\left\langle \begin{smallmatrix}
3 & 1\\
0 & 2\\
1 & 1\\
\overrightarrow{0} & \overrightarrow{0}
\end{smallmatrix}\right\rangle +\left(m-2\right)\left(m-3\right)\left\langle \begin{smallmatrix}
2 & 0\\
0 & 2\\
1 & 1\\
1 & 1\\
\overrightarrow{0} & \overrightarrow{0}
\end{smallmatrix}\right\rangle \right)
\\&
=\tfrac{1}{m\left(m+2\right)}\left(\tfrac{2\left(p+1\right)\cdot3\left(p+3\right)+\left(m-2\right)\left(p+3\right)^{2}}{\left(p-1\right)p\left(p+1\right)\left(p+2\right)\left(p+4\right)\left(p+6\right)}\right)
+
\\
&\eqmargin
\tfrac{1}{\left(m-1\right)\left(m+2\right)}\left(\tfrac{-2\cdot9\left(p+3\right)+4\left(m-2\right)\left(-3\left(p+3\right)\right)
+
\left(m-2\right)\left(m-3\right)\left(-p-3\right)}{\left(p-1\right)p\left(p+1\right)\left(p+2\right)\left(p+4\right)\left(p+6\right)}\right)+
\\&
\eqmargin
\tfrac{-1}{\left(m-1\right)m\left(m+2\right)}\left(\tfrac{9\left(p+3\right)\left(p+5\right)+\left(p^{2}+4p+15\right)}{\left(p-1\right)p\left(p+1\right)\left(p+2\right)\left(p+4\right)\left(p+6\right)}+
\left(m-2\right)
\tfrac{2\cdot3\left(p+3\right)\left(p+5\right)+2\left(p+3\right)^{2}+\left(m-3\right)\left(p+3\right)\left(p+5\right)}{\left(p-1\right)p\left(p+1\right)\left(p+2\right)\left(p+4\right)\left(p+6\right)}
\right)
\\&
=\tfrac{-m^{3}p-3m^{3}-9m^{2}p-27m^{2}-14mp-42m-4p^{2}-16p-24}{\left(m-1\right)m\left(m+2\right)\left(p-1\right)p\left(p+1\right)\left(p+2\right)\left(p+4\right)\left(p+6\right)}
\end{align*}
\end{proof}

\newpage

\paragraph{Deriving (B).}
\begin{align*}
\text{(B)}&=\sum_{k=2}^{m}\sum_{j,\ell,r,t=1}^{m}\mathbb{E}\left[u_{1}v_{1}u_{r}u_{j}u_{k}v_{k}v_{\ell}v_{t}\right]\left(\mathbb{E}\left[q_{1,j}q_{k\ell}q_{1,r}q_{kt}\right]+\mathbb{E}\left[q_{1,j}q_{k\ell}q_{kr}q_{1,t}\right]\right)
\\&
=\left(m-1\right)\sum_{j,\ell,r,t=1}^{m}\mathbb{E}\left[u_{1}u_{2}u_{j}u_{r}v_{1}v_{2}v_{\ell}v_{t}\right]\left(\mathbb{E}\left[q_{1,j}q_{1,r}q_{2,\ell}q_{2,t}\right]+\mathbb{E}\left[q_{1,j}q_{1,t}q_{2,\ell}q_{2,r}\right]\right)
\\&
=\left(m-1\right)\underbrace{2\left\langle \begin{smallmatrix}
2 & 2\\
\overrightarrow{0} & \overrightarrow{0}
\end{smallmatrix}\right\rangle _{m}\sum_{j=1}^{m}\mathbb{E}\left[u_{1}u_{2}u_{j}^{2}v_{1}v_{2}v_{j}^{2}\right]}_{j=r=\ell=t}
+
\\
&
\eqmargin
\left(m-1\right)
\underbrace{\sum_{j\neq\ell=1}^{m}\mathbb{E}\left[u_{1}u_{2}u_{j}^{2}v_{1}v_{2}v_{\ell}^{2}\right]\left(\left\langle \begin{smallmatrix}
2 & 0\\
0 & 2\\
\overrightarrow{0} & \overrightarrow{0}
\end{smallmatrix}\right\rangle _{m}+\left\langle \begin{smallmatrix}
1 & 1\\
1 & 1\\
\overrightarrow{0} & \overrightarrow{0}
\end{smallmatrix}\right\rangle _{m}\right)}_{j=r\neq\ell=t}+
\\&
\eqmargin
\left(m-1\right)
\underbrace{2\left\langle \begin{smallmatrix}
1 & 1\\
1 & 1\\
\overrightarrow{0} & \overrightarrow{0}
\end{smallmatrix}\right\rangle _{m}\sum_{j\neq r=1}^{m}\mathbb{E}\left[u_{1}u_{2}u_{j}u_{r}v_{1}v_{2}v_{j}v_{r}\right]}_{j=\ell\neq r=t}
+
\\&
\eqmargin
\left(m-1\right)
\underbrace{\sum_{j\neq\ell=1}^{m}\mathbb{E}\left[u_{1}u_{2}u_{j}u_{\ell}v_{1}v_{2}v_{j}v_{\ell}\right]\left(\left\langle \begin{smallmatrix}
1 & 1\\
1 & 1\\
\overrightarrow{0} & \overrightarrow{0}
\end{smallmatrix}\right\rangle _{m}+\left\langle \begin{smallmatrix}
2 & 0\\
0 & 2\\
\overrightarrow{0} & \overrightarrow{0}
\end{smallmatrix}\right\rangle _{m}\right)}_{j=t\neq\ell=r}
\\&
=\left(m-1\right)\Bigg(2\left\langle \begin{smallmatrix}
2 & 2\\
\overrightarrow{0} & \overrightarrow{0}
\end{smallmatrix}\right\rangle _{m}\sum_{j=1}^{m}\mathbb{E}\left[u_{1}u_{2}u_{j}^{2}v_{1}v_{2}v_{j}^{2}\right]+
\\
& 
\hspace{1.9cm}
\left(\left\langle \begin{smallmatrix}
2 & 0\\
0 & 2\\
\overrightarrow{0} & \overrightarrow{0}
\end{smallmatrix}\right\rangle _{m}+\left\langle \begin{smallmatrix}
1 & 1\\
1 & 1\\
\overrightarrow{0} & \overrightarrow{0}
\end{smallmatrix}\right\rangle _{m}\right)\sum_{j\neq\ell=1}^{m}\mathbb{E}\left[u_{1}u_{2}u_{j}^{2}v_{1}v_{2}v_{\ell}^{2}\right]\Bigg)+
\\&
\eqmargin
\left(m-1\right)\left(3\left\langle \begin{smallmatrix}
1 & 1\\
1 & 1\\
\overrightarrow{0} & \overrightarrow{0}
\end{smallmatrix}\right\rangle _{m}+\left\langle \begin{smallmatrix}
2 & 0\\
0 & 2\\
\overrightarrow{0} & \overrightarrow{0}
\end{smallmatrix}\right\rangle _{m}\right)\sum_{j\neq r=1}^{m}\mathbb{E}\left[u_{1}u_{2}u_{j}u_{r}v_{1}v_{2}v_{j}v_{r}\right]
\\&
=\left(m-1\right)\left(\tfrac{2}{m\left(m+2\right)}\sum_{j=1}^{m}\mathbb{E}\left[u_{1}u_{2}u_{j}^{2}v_{1}v_{2}v_{j}^{2}\right]+\tfrac{\left(m+1\right)-1}{\left(m-1\right)m\left(m+2\right)}\sum_{j\neq\ell=1}^{m}\mathbb{E}\left[u_{1}u_{2}u_{j}^{2}v_{1}v_{2}v_{\ell}^{2}\right]\right)+
\\&
\eqmargin
\left(m-1\right)\tfrac{-3+\left(m+1\right)}{\left(m-1\right)m\left(m+2\right)}\sum_{j\neq r=1}^{m}\mathbb{E}\left[u_{1}u_{2}u_{j}u_{r}v_{1}v_{2}v_{j}v_{r}\right]
\\&
=\tfrac{\left(m-1\right)}{\left(m-1\right)m\left(m+2\right)}\Bigg(2\left(m-1\right)\sum_{j=1}^{m}\mathbb{E}\left[u_{1}u_{2}u_{j}^{2}v_{1}v_{2}v_{j}^{2}\right]+m\sum_{j\neq\ell=1}^{m}\mathbb{E}\left[u_{1}u_{2}u_{j}^{2}v_{1}v_{2}v_{\ell}^{2}\right]
\\
& 
\eqmargin
+\left(m-2\right)\sum_{j\neq r=1}^{m}\mathbb{E}\left[u_{1}u_{2}u_{j}u_{r}v_{1}v_{2}v_{j}v_{r}\right]\Bigg)
\\&
=\tfrac{1}{m\left(m+2\right)}\bigg(2\left(m-1\right)\sum_{j=1}^{m}\mathbb{E}\left[u_{1}u_{2}u_{j}^{2}v_{1}v_{2}v_{j}^{2}\right]+m\sum_{j\neq\ell=1}^{m}\mathbb{E}\left[u_{1}u_{2}u_{j}^{2}v_{1}v_{2}v_{\ell}^{2}\right]+
\\
& 
\eqmargin
\hspace{1.5cm}
\left(m-2\right)\sum_{j\neq r=1}^{m}\mathbb{E}\left[u_{1}u_{2}u_{j}u_{r}v_{1}v_{2}v_{j}v_{r}\right]\bigg)
\\&
=\tfrac{1}{m\left(m+2\right)}\Bigg(2\left(m-1\right)
\bigg(\underbrace{2\mathbb{E}\left[u_{1}^{3}u_{2}v_{1}^{3}v_{2}\right]}_{j=1,2}+\underbrace{\left(m-2\right)\mathbb{E}\left[u_{1}u_{2}u_{3}^{2}v_{1}v_{2}v_{3}^{2}\right]}_{j\ge3}\bigg)
+
\\
& 
\hspace{1.9cm}
m\sum_{j\neq\ell=1}^{m}\mathbb{E}\left[u_{1}u_{2}u_{j}^{2}v_{1}v_{2}v_{\ell}^{2}\right]\Bigg)+
\\
& 
\eqmargin
\tfrac{\left(m-2\right)}{m\left(m+2\right)}\Bigg(\underbrace{2\mathbb{E}\left[u_{1}^{2}u_{2}^{2}v_{1}^{2}v_{2}^{2}\right]}_{j=1,r=2\,\vee\,r=1,j=2}+\underbrace{4\left(m-2\right)\mathbb{E}\left[u_{1}^{2}u_{2}u_{3}v_{1}^{2}v_{2}v_{3}\right]}_{j\le2,r\ge3\,\vee\,r\le2,j\ge3}+
\\
& 
\hspace{1.7cm}
\underbrace{\left(m-2\right)\left(m-3\right)\mathbb{E}\left[u_{1}u_{2}u_{3}u_{4}v_{1}v_{2}v_{3}v_{4}\right]}_{j\neq r\ge3}\Bigg)
\\&
=\tfrac{2\left(m-1\right)}{m\left(m+2\right)}\Bigg(2\left\langle \begin{smallmatrix}
3 & 3\\
1 & 1\\
\overrightarrow{0} & \overrightarrow{0}
\end{smallmatrix}\right\rangle +\left(m-2\right)\left\langle \begin{smallmatrix}
1 & 1\\
1 & 1\\
2 & 2\\
\overrightarrow{0} & \overrightarrow{0}
\end{smallmatrix}\right\rangle \Bigg)
+
\\
& \eqmargin
\tfrac{\left(m-2\right)}{m\left(m+2\right)}\left(2\left\langle \begin{smallmatrix}
2 & 2\\
2 & 2\\
\overrightarrow{0} & \overrightarrow{0}
\end{smallmatrix}\right\rangle +4\left(m-2\right)\left\langle \begin{smallmatrix}
2 & 2\\
1 & 1\\
1 & 1\\
\overrightarrow{0} & \overrightarrow{0}
\end{smallmatrix}\right\rangle +\left(m-2\right)\left(m-3\right)\left\langle \begin{smallmatrix}
1 & 1\\
1 & 1\\
1 & 1\\
1 & 1\\
\overrightarrow{0} & \overrightarrow{0}
\end{smallmatrix}\right\rangle \right)+
\\&
\eqmargin
\tfrac{m}{m\left(m+2\right)}
\Bigg(\underbrace{2\mathbb{E}\left[u_{1}^{3}u_{2}v_{1}v_{2}^{3}\right]}_{j=1,\ell=2\,\vee\,j=2,\ell=1}+\underbrace{2\left(m-2\right)\mathbb{E}\left[u_{1}^{3}u_{2}v_{1}v_{2}v_{3}^{2}\right]}_{j\in\left\{ 1,2\right\} ,\ell\ge3}+\underbrace{2\left(m-2\right)\mathbb{E}\left[u_{1}u_{2}u_{3}^{2}v_{1}^{3}v_{2}\right]}_{\ell\in\left\{ 1,2\right\} ,j\ge3}
+
\\
& 
\hspace{1.7cm}
\underbrace{\left(m-2\right)\left(m-3\right)\mathbb{E}\left[u_{1}u_{2}u_{3}^{2}v_{1}v_{2}v_{4}^{2}\right]}_{\ell\neq j\ge3}\Bigg)
\\&
=\tfrac{2\left(m-1\right)}{m\left(m+2\right)}\left(2\left\langle \begin{smallmatrix}
3 & 3\\
1 & 1\\
\overrightarrow{0} & \overrightarrow{0}
\end{smallmatrix}\right\rangle +\left(m-2\right)\left\langle \begin{smallmatrix}
1 & 1\\
1 & 1\\
2 & 2\\
\overrightarrow{0} & \overrightarrow{0}
\end{smallmatrix}\right\rangle \right)+
\\
&
\eqmargin
\tfrac{\left(m-2\right)}{m\left(m+2\right)}\left(2\left\langle \begin{smallmatrix}
2 & 2\\
2 & 2\\
\overrightarrow{0} & \overrightarrow{0}
\end{smallmatrix}\right\rangle +4\left(m-2\right)\left\langle \begin{smallmatrix}
2 & 2\\
1 & 1\\
1 & 1\\
\overrightarrow{0} & \overrightarrow{0}
\end{smallmatrix}\right\rangle +\left(m-2\right)\left(m-3\right)\left\langle \begin{smallmatrix}
1 & 1\\
1 & 1\\
1 & 1\\
1 & 1\\
\overrightarrow{0} & \overrightarrow{0}
\end{smallmatrix}\right\rangle \right)+
\\&
\eqmargin
\tfrac{m}{m\left(m+2\right)}\Bigg(2\left\langle \begin{smallmatrix}
3 & 1\\
1 & 3\\
\overrightarrow{0} & \overrightarrow{0}
\end{smallmatrix}\right\rangle +2\left(m-2\right)\left\langle \begin{smallmatrix}
3 & 1\\
1 & 1\\
0 & 2\\
\overrightarrow{0} & \overrightarrow{0}
\end{smallmatrix}\right\rangle +
\\
& 
\hspace{1.7cm}
2\left(m-2\right)\left\langle \begin{smallmatrix}
1 & 3\\
1 & 1\\
2 & 0\\
\overrightarrow{0} & \overrightarrow{0}
\end{smallmatrix}\right\rangle +\left(m-2\right)\left(m-3\right)\left\langle \begin{smallmatrix}
1 & 1\\
1 & 1\\
2 & 0\\
0 & 2\\
\overrightarrow{0} & \overrightarrow{0}
\end{smallmatrix}\right\rangle \Bigg)
\\&
=\tfrac{4\left(m-1\right)}{m\left(m+2\right)}\left\langle \begin{smallmatrix}
3 & 3\\
1 & 1\\
\overrightarrow{0} & \overrightarrow{0}
\end{smallmatrix}\right\rangle +\left(\tfrac{2\left(m-1\right)\left(m-2\right)}{m\left(m+2\right)}+\tfrac{4\left(m-2\right)^{2}}{m\left(m+2\right)}\right)\left\langle \begin{smallmatrix}
2 & 2\\
1 & 1\\
1 & 1\\
\overrightarrow{0} & \overrightarrow{0}
\end{smallmatrix}\right\rangle +\\&
\hphantom{=}\tfrac{\left(m-2\right)}{m\left(m+2\right)}\left(2\left\langle \begin{smallmatrix}
2 & 2\\
2 & 2\\
\overrightarrow{0} & \overrightarrow{0}
\end{smallmatrix}\right\rangle +\left(m-2\right)\left(m-3\right)\left\langle \begin{smallmatrix}
1 & 1\\
1 & 1\\
1 & 1\\
1 & 1\\
\overrightarrow{0} & \overrightarrow{0}
\end{smallmatrix}\right\rangle \right)+
\\&
\hphantom{=}\tfrac{m}{m\left(m+2\right)}\left(2\left\langle \begin{smallmatrix}
3 & 1\\
1 & 3\\
\overrightarrow{0} & \overrightarrow{0}
\end{smallmatrix}\right\rangle +4\left(m-2\right)\left\langle \begin{smallmatrix}
3 & 1\\
1 & 1\\
0 & 2\\
\overrightarrow{0} & \overrightarrow{0}
\end{smallmatrix}\right\rangle +\left(m-2\right)\left(m-3\right)\left\langle \begin{smallmatrix}
1 & 1\\
1 & 1\\
2 & 0\\
0 & 2\\
\overrightarrow{0} & \overrightarrow{0}
\end{smallmatrix}\right\rangle \right)
\\&
=\tfrac{-4\left(m-1\right)\cdot9\left(p+1\right)+2\left(m-2\right)\left(3m-5\right)\left(-p+3\right)}{m\left(m+2\right)\left(p-1\right)p\left(p+1\right)\left(p+2\right)\left(p+4\right)\left(p+6\right)}+
\tfrac{\left(m-2\right)}{m\left(m+2\right)}\left(\tfrac{2\left(p^{2}+4p+15\right)+3\left(m-2\right)\left(m-3\right)}{\left(p-1\right)p\left(p+1\right)\left(p+2\right)\left(p+4\right)\left(p+6\right)}\right)+
\\&
\eqmargin
\tfrac{m}{m\left(m+2\right)}\left(\tfrac{-2\cdot9\left(p+3\right)-4\left(m-2\right)\cdot3\left(p+3\right)-\left(m-2\right)\left(m-3\right)\left(p+3\right)}{\left(p-1\right)p\left(p+1\right)\left(p+2\right)\left(p+4\right)\left(p+6\right)}\right)
\\&
=\frac{-m^{3}p-13m^{2}p-24m^{2}+2mp^{2}-6mp-24m-4p^{2}}{m\left(m+2\right)\left(p-1\right)p\left(p+1\right)\left(p+2\right)\left(p+4\right)\left(p+6\right)}
\end{align*}

\newpage

\begin{proposition} 
\label{prop:(e1Oe1)(e1Oe2)(e3Oe1)(e3Oe2)}
For $p\ge 3, m\in\{2,3,\dots,p\}$ and a random transformation $\mO$ sampled as described in \eqref{eq:data-model},
it holds that,
\begin{align*}
&\mathbb{E}\left(\mathbf{e}_{1}^{\top}\mathbf{O}\mathbf{e}_{1}\cdot\mathbf{e}_{1}^{\top}\mathbf{O}\mathbf{e}_{2}\right)\left(\mathbf{e}_{3}^{\top}\mathbf{O}\mathbf{e}_{1}\cdot\mathbf{e}_{3}^{\top}\mathbf{O}\mathbf{e}_{2}\right)
\\
&
=\tfrac{2m^{4}p+4m^{4}-7m^{3}p^{2}-18m^{3}p+8m^{2}p^{3}+25m^{2}p^{2}+24m^{2}p}{\left(p-2\right)\left(p-1\right)p\left(p+1\right)\left(p+2\right)\left(p+4\right)\left(p+6\right)}
+
\\
&\eqmargin
\tfrac{20m^{2}-3mp^{4}-11mp^{3}-44mp^{2}-64mp-24m-p^{4}+11p^{3}+32p^{2}+68p+48}{\left(p-2\right)\left(p-1\right)p\left(p+1\right)\left(p+2\right)\left(p+4\right)\left(p+6\right)}
\end{align*}
\end{proposition}

\begin{proof}
Like in previous proofs, we decompose the expression into,
\begin{align}
&\mathbb{E}\left(\mathbf{e}_{1}^{\top}\mathbf{O}\mathbf{e}_{1}\cdot\mathbf{e}_{1}^{\top}\mathbf{O}\mathbf{e}_{2}\right)\left(\mathbf{e}_{3}^{\top}\mathbf{O}\mathbf{e}_{1}\cdot\mathbf{e}_{3}^{\top}\mathbf{O}\mathbf{e}_{2}\right)
\nonumber
\\&
=\mathbb{E}\left[\left(\mathbf{u}_{a}^{\top}\mathbf{Q}_{m}\mathbf{u}_{a}+\mathbf{u}_{b}^{\top}\mathbf{u}_{b}\right)\left(\mathbf{u}_{a}^{\top}\mathbf{Q}_{m}\mathbf{v}_{a}+\mathbf{u}_{b}^{\top}\mathbf{v}_{b}\right)\left(\mathbf{z}_{a}^{\top}\mathbf{Q}_{m}\mathbf{u}_{a}+\mathbf{z}_{b}^{\top}\mathbf{u}_{b}\right)\left(\mathbf{z}_{a}^{\top}\mathbf{Q}_{m}\mathbf{v}_{a}+\mathbf{z}_{b}^{\top}\mathbf{v}_{b}\right)\right]
\nonumber
\\&
=\mathbb{E}\left[\mathbf{u}_{a}^{\top}\mathbf{Q}_{m}\mathbf{u}_{a}\left(\mathbf{u}_{a}^{\top}\mathbf{Q}_{m}\mathbf{v}_{a}+\mathbf{u}_{b}^{\top}\mathbf{v}_{b}\right)\left(\mathbf{z}_{a}^{\top}\mathbf{Q}_{m}\mathbf{u}_{a}+\mathbf{z}_{b}^{\top}\mathbf{u}_{b}\right)\left(\mathbf{z}_{a}^{\top}\mathbf{Q}_{m}\mathbf{v}_{a}+\mathbf{z}_{b}^{\top}\mathbf{v}_{b}\right)\right]+
\nonumber
\\&
\eqmargin
\mathbb{E}\left[\mathbf{u}_{b}^{\top}\mathbf{u}_{b}\left(\mathbf{u}_{a}^{\top}\mathbf{Q}_{m}\mathbf{v}_{a}+\mathbf{u}_{b}^{\top}\mathbf{v}_{b}\right)\left(\mathbf{z}_{a}^{\top}\mathbf{Q}_{m}\mathbf{u}_{a}+\mathbf{z}_{b}^{\top}\mathbf{u}_{b}\right)\left(\mathbf{z}_{a}^{\top}\mathbf{Q}_{m}\mathbf{v}_{a}+\mathbf{z}_{b}^{\top}\mathbf{v}_{b}\right)\right]
\label{eq:intermediate_in_prop}
\end{align}
Focusing on the first term, and employing \corref{cor:odd-Q_m},
we see that,
\begin{align*}
&\mathbb{E}\left[\mathbf{u}_{a}^{\top}\mathbf{Q}_{m}\mathbf{u}_{a}\left(\mathbf{u}_{a}^{\top}\mathbf{Q}_{m}\mathbf{v}_{a}+\mathbf{u}_{b}^{\top}\mathbf{v}_{b}\right)\left(\mathbf{z}_{a}^{\top}\mathbf{Q}_{m}\mathbf{u}_{a}+\mathbf{z}_{b}^{\top}\mathbf{u}_{b}\right)\left(\mathbf{z}_{a}^{\top}\mathbf{Q}_{m}\mathbf{v}_{a}+\mathbf{z}_{b}^{\top}\mathbf{v}_{b}\right)\right]
\\
&=
\mathbb{E}\left[\mathbf{u}_{a}^{\top}\mathbf{Q}_{m}\mathbf{u}_{a}\mathbf{u}_{a}^{\top}\mathbf{Q}_{m}\mathbf{v}_{a}\mathbf{z}_{a}^{\top}\mathbf{Q}_{m}\mathbf{u}_{a}\mathbf{z}_{a}^{\top}\mathbf{Q}_{m}\mathbf{v}_{a}\right]
+
\mathbb{E}\left[\mathbf{u}_{a}^{\top}\mathbf{Q}_{m}\mathbf{u}_{a}\mathbf{u}_{a}^{\top}\mathbf{Q}_{m}\mathbf{v}_{a}\mathbf{z}_{b}^{\top}\mathbf{u}_{b}\mathbf{z}_{b}^{\top}\mathbf{v}_{b}\right]
+
\\
& \eqmargin
\mathbb{E}\left[\mathbf{u}_{a}^{\top}\mathbf{Q}_{m}\mathbf{u}_{a}\mathbf{u}_{b}^{\top}\mathbf{v}_{b}\mathbf{z}_{a}^{\top}\mathbf{Q}_{m}\mathbf{u}_{a}\mathbf{z}_{b}^{\top}\mathbf{v}_{b}\right]+\mathbb{E}\left[\mathbf{u}_{a}^{\top}\mathbf{Q}_{m}\mathbf{u}_{a}\mathbf{u}_{b}^{\top}\mathbf{v}_{b}\mathbf{z}_{b}^{\top}\mathbf{u}_{b}\mathbf{z}_{a}^{\top}\mathbf{Q}_{m}\mathbf{v}_{a}\right]
\end{align*}
The polynomial in the first summand can be derived tediously, very much like in the proof of the former \propref{prop:(uaQmua)(uaQmva)(vaQmua)(vaQmVa)},
and shown to hold
\begin{align*}
\mathbb{E}\left[\mathbf{u}_{a}^{\top}\mathbf{Q}_{m}\mathbf{u}_{a}\mathbf{u}_{a}^{\top}\mathbf{Q}_{m}\mathbf{v}_{a}\mathbf{z}_{a}^{\top}\mathbf{Q}_{m}\mathbf{u}_{a}\mathbf{z}_{a}^{\top}\mathbf{Q}_{m}\mathbf{v}_{a}\right]
=
\tfrac{-m^{2}\left(2p^{2}+9p+6\right)+m\left(p-2\right)\left(p^{2}-6\right)+2p^{3}+32p+48}{\left(p-2\right)\left(p-1\right)p\left(p+1\right)\left(p+2\right)\left(p+4\right)\left(p+6\right)}
\,.
\end{align*}
The sum of the three rightmost summands is,
\begin{align*}
&\mathbb{E}\left[\mathbf{u}_{a}^{\top}\mathbf{Q}_{m}\mathbf{u}_{a}\mathbf{u}_{a}^{\top}\mathbf{Q}_{m}\mathbf{v}_{a}\mathbf{z}_{b}^{\top}\mathbf{u}_{b}\mathbf{z}_{b}^{\top}\mathbf{v}_{b}\right]+\mathbb{E}\left[\mathbf{u}_{a}^{\top}\mathbf{Q}_{m}\mathbf{u}_{a}\mathbf{u}_{b}^{\top}\mathbf{v}_{b}\mathbf{z}_{a}^{\top}\mathbf{Q}_{m}\mathbf{u}_{a}\mathbf{z}_{b}^{\top}\mathbf{v}_{b}\right]+
\\&
\eqmargin+
\mathbb{E}\left[\mathbf{u}_{a}^{\top}\mathbf{Q}_{m}\mathbf{u}_{a}\mathbf{u}_{b}^{\top}\mathbf{v}_{b}\mathbf{z}_{b}^{\top}\mathbf{u}_{b}\mathbf{z}_{a}^{\top}\mathbf{Q}_{m}\mathbf{v}_{a}\right]
\\&
=\sum_{i,j,k,\ell=1}^{m}\left(\mathbb{E}\left[u_{i}u_{j}q_{ij}u_{k}v_{\ell}q_{k\ell}\mathbf{z}_{b}^{\top}\mathbf{u}_{b}\mathbf{z}_{b}^{\top}\mathbf{v}_{b}\right]+\mathbb{E}\left[u_{i}u_{j}q_{ij}\mathbf{u}_{b}^{\top}\mathbf{v}_{b}z_{k}u_{\ell}q_{k\ell}\mathbf{z}_{b}^{\top}\mathbf{v}_{b}\right]\right)+
\\&
\eqmargin
\sum_{i,j,k,\ell=1}^{m}\left(\mathbb{E}\left[
u_{i}u_{j}q_{ij}
z_{k}v_{\ell}q_{k\ell}
\mathbf{u}_{b}^{\top}\mathbf{v}_{b}
\mathbf{z}_{b}^{\top}\mathbf{u}_{b}
\right]\right)
\\&
=
\sum_{i,j,k,\ell=1}^{m}
\mathbb{E}\left[q_{ij}q_{k\ell}\right]
\left(\mathbb{E}\left[u_{i}u_{j}u_{k}v_{\ell}\mathbf{z}_{b}^{\top}\mathbf{u}_{b}\mathbf{z}_{b}^{\top}\mathbf{v}_{b}\right]
+
\mathbb{E}\left[u_{i}u_{j}\mathbf{u}_{b}^{\top}\mathbf{v}_{b}z_{k}u_{\ell}\mathbf{z}_{b}^{\top}\mathbf{v}_{b}\right]\right)+
\\&
\eqmargin
\sum_{i,j,k,\ell=1}^{m}
\mathbb{E}\left[q_{ij}q_{k\ell}\right]
\left(\mathbb{E}\left[u_{i}u_{j}z_{k}v_{\ell}
\mathbf{u}_{b}^{\top}\mathbf{v}_{b}\mathbf{z}_{b}^{\top}\mathbf{u}_{b}\right]\right)
\\&
=\sum_{i,j=1}^{m}\underbrace{\mathbb{E}\left[q_{ij}^{2}\right]}_{=1/m}\left(\mathbb{E}\left[u_{i}u_{j}u_{i}v_{j}\mathbf{z}_{b}^{\top}\mathbf{u}_{b}\mathbf{z}_{b}^{\top}\mathbf{v}_{b}\right]+\mathbb{E}\left[u_{i}u_{j}z_{i}u_{j}\mathbf{u}_{b}^{\top}\mathbf{v}_{b}\mathbf{z}_{b}^{\top}\mathbf{v}_{b}\right]+\mathbb{E}\left[u_{i}u_{j}z_{i}v_{j}\mathbf{u}_{b}^{\top}\mathbf{v}_{b}\mathbf{z}_{b}^{\top}\mathbf{u}_{b}\right]\right)
\\&
=\frac{1}{m}\left(\mathbb{E}\!\left[\left\Vert \mathbf{u}_{a}\right\Vert ^{2}\mathbf{u}_{a}^{\top}\mathbf{v}_{a}\mathbf{z}_{b}^{\top}\mathbf{u}_{b}\mathbf{z}_{b}^{\top}\mathbf{v}_{b}\right]+
\mathbb{E}\!\left[\left\Vert \mathbf{u}_{a}\right\Vert ^{2}\mathbf{u}_{a}^{\top}\mathbf{z}_{a}\mathbf{u}_{b}^{\top}\mathbf{v}_{b}\mathbf{z}_{b}^{\top}\mathbf{v}_{b}\right]+
\mathbb{E}\!\left[\mathbf{u}_{a}^{\top}\mathbf{v}_{a}\mathbf{u}_{a}^{\top}\mathbf{z}_{a}\mathbf{u}_{b}^{\top}\mathbf{v}_{b}\mathbf{z}_{b}^{\top}\mathbf{u}_{b}\right]\right)
\\&
=
\frac{1}{m}
\Bigg(\underbrace{2\mathbb{E}\left[\left\Vert \mathbf{u}_{a}\right\Vert ^{2}\mathbf{v}_{b}^{\top}\mathbf{z}_{b}\mathbf{u}_{b}^{\top}\mathbf{z}_{b}\mathbf{u}_{a}^{\top}\mathbf{v}_{a}\right]}_{\text{obtained from \eqref{eq:(ua)^2vazaubzbubvb} by plugging in \ensuremath{m\leftarrow p\!-\!m}}}
+
\underbrace{
\mathbb{E}\left[\mathbf{u}_{a}^{\top}\mathbf{v}_{a}\mathbf{u}_{a}^{\top}\mathbf{z}_{a}\mathbf{u}_{b}^{\top}\mathbf{v}_{b}\mathbf{z}_{b}^{\top}\mathbf{u}_{b}\right]
}_{\substack{\text{solved in \eqref{eq:(ubzb)^2(ubvb)^2}}
\\
\text{since $\mathbf{u}_{a}^{\top}\mathbf{v}_{a}
=-\mathbf{u}_{b}^{\top}\mathbf{v}_{b}$
and
$\mathbf{u}_{a}^{\top}\mathbf{z}_{a}=
-\mathbf{z}_{b}^{\top}\mathbf{u}_{b}$
}
}
}\Bigg)
\\&
=2\tfrac{\left(p-m\right)\left(m+2\right)\left(2\left(p-m\right)p+4\left(p-m\right)-p^{2}-3p+2\right)}{\left(p-2\right)\left(p-1\right)p\left(p+1\right)\left(p+2\right)\left(p+4\right)\left(p+6\right)}
+
\tfrac{\left(p-m\right)\left(-m^{2}+mp+2p+4\right)}{\left(p-1\right)p\left(p+1\right)\left(p+2\right)\left(p+4\right)\left(p+6\right)}
\\&
=\frac{-\left(p-m\right)\left(m+2\right)\left(5mp+6m-3p^{2}-2p\right)}{\left(p-2\right)\left(p-1\right)p\left(p+1\right)\left(p+2\right)\left(p+4\right)\left(p+6\right)}\,.
\end{align*}

Overall, the first term of \eqref{eq:intermediate_in_prop} equals,
\begin{align*}
&\mathbb{E}\left[\mathbf{u}_{a}^{\top}\mathbf{Q}_{m}\mathbf{u}_{a}\left(\mathbf{u}_{a}^{\top}\mathbf{Q}_{m}\mathbf{v}_{a}+\mathbf{u}_{b}^{\top}\mathbf{v}_{b}\right)\left(\mathbf{z}_{a}^{\top}\mathbf{Q}_{m}\mathbf{u}_{a}+\mathbf{z}_{b}^{\top}\mathbf{u}_{b}\right)\left(\mathbf{z}_{a}^{\top}\mathbf{Q}_{m}\mathbf{v}_{a}+\mathbf{z}_{b}^{\top}\mathbf{v}_{b}\right)\right]
\\
&=\tfrac{6\left(m^{3}+m^{2}+2m+8\right)+4\left(m+2\right)p^{3}-2\left(5m^{2}+8m-2\right)p^{2}+\left(m-2\right)\left(5m^{2}+3m-16\right)p}{\left(p-2\right)\left(p-1\right)p\left(p+1\right)\left(p+2\right)\left(p+4\right)\left(p+6\right)}\,.
\end{align*}

\medskip

The second term holds,
\begin{align*}
&\mathbb{E}\left[\mathbf{u}_{b}^{\top}\mathbf{u}_{b}\left(\mathbf{u}_{a}^{\top}\mathbf{Q}_{m}\mathbf{v}_{a}+\mathbf{u}_{b}^{\top}\mathbf{v}_{b}\right)\left(\mathbf{z}_{a}^{\top}\mathbf{Q}_{m}\mathbf{u}_{a}+\mathbf{z}_{b}^{\top}\mathbf{u}_{b}\right)\left(\mathbf{z}_{a}^{\top}\mathbf{Q}_{m}\mathbf{v}_{a}+\mathbf{z}_{b}^{\top}\mathbf{v}_{b}\right)\right]
\\&
=\mathbb{E}\left[\mathbf{u}_{b}^{\top}\mathbf{u}_{b}\mathbf{u}_{b}^{\top}\mathbf{v}_{b}\mathbf{z}_{b}^{\top}\mathbf{u}_{b}\mathbf{z}_{b}^{\top}\mathbf{v}_{b}\right]+\mathbb{E}\left[\mathbf{u}_{b}^{\top}\mathbf{u}_{b}\mathbf{u}_{a}^{\top}\mathbf{Q}_{m}\mathbf{v}_{a}\mathbf{z}_{a}^{\top}\mathbf{Q}_{m}\mathbf{u}_{a}\mathbf{z}_{b}^{\top}\mathbf{v}_{b}\right]+
\\&
\underbrace{\mathbb{E}\left[\mathbf{u}_{b}^{\top}\mathbf{u}_{b}\mathbf{u}_{a}^{\top}\mathbf{Q}_{m}\mathbf{v}_{a}\mathbf{z}_{b}^{\top}\mathbf{u}_{b}\mathbf{z}_{a}^{\top}\mathbf{Q}_{m}\mathbf{v}_{a}\right]+\mathbb{E}\left[\mathbf{u}_{b}^{\top}\mathbf{u}_{b}\mathbf{u}_{b}^{\top}\mathbf{v}_{b}\mathbf{z}_{a}^{\top}\mathbf{Q}_{m}\mathbf{u}_{a}\mathbf{z}_{a}^{\top}\mathbf{Q}_{m}\mathbf{v}_{a}\right]}_{\text{equal (swap \ensuremath{\mathbf{v},\mathbf{z}} and map \ensuremath{\mathbf{Q}\to\mathbf{Q}^{\top}})}}
\\&
=\mathbb{E}\left[\mathbf{u}_{b}^{\top}\mathbf{u}_{b}\mathbf{u}_{b}^{\top}\mathbf{v}_{b}\mathbf{z}_{b}^{\top}\mathbf{u}_{b}\mathbf{z}_{b}^{\top}\mathbf{v}_{b}\right]+\mathbb{E}\left[\mathbf{u}_{b}^{\top}\mathbf{u}_{b}\mathbf{z}_{b}^{\top}\mathbf{v}_{b}\cdot\mathbf{u}_{a}^{\top}\mathbf{Q}_{m}\mathbf{v}_{a}\mathbf{z}_{a}^{\top}\mathbf{Q}_{m}\mathbf{u}_{a}\right]+
\\
&\eqmargin
2\mathbb{E}\left[\mathbf{u}_{b}^{\top}\mathbf{u}_{b}\mathbf{u}_{b}^{\top}\mathbf{v}_{b}\cdot\mathbf{z}_{a}^{\top}\mathbf{Q}_{m}\mathbf{u}_{a}\mathbf{z}_{a}^{\top}\mathbf{Q}_{m}\mathbf{v}_{a}\right]
\\&
=\underbrace{\mathbb{E}\left[\mathbf{u}_{b}^{\top}\mathbf{u}_{b}\mathbf{u}_{b}^{\top}\mathbf{v}_{b}\mathbf{z}_{a}^{\top}\mathbf{u}_{a}\mathbf{z}_{a}^{\top}\mathbf{v}_{a}\right]}_{\text{solved in \eqref{eq:(ub)^2(vaza)(uaza)(ubvb)}}}
+
\underbrace{\mathbb{E}\left[\mathbf{u}_{b}^{\top}\mathbf{u}_{b}\mathbf{z}_{b}^{\top}\mathbf{v}_{b}\cdot\mathbf{u}_{a}^{\top}\mathbf{Q}_{m}\mathbf{v}_{a}\mathbf{z}_{a}^{\top}\mathbf{Q}_{m}\mathbf{u}_{a}\right]}_{\text{solved in \eqref{eq:(ub)^2(vbzb)(uaQmva)(zaQmua)}}}+
\\
&\eqmargin
2\underbrace{\mathbb{E}\left[\mathbf{u}_{b}^{\top}\mathbf{u}_{b}\mathbf{u}_{b}^{\top}\mathbf{v}_{b}\cdot\mathbf{z}_{a}^{\top}\mathbf{Q}_{m}\mathbf{u}_{a}\mathbf{z}_{a}^{\top}\mathbf{Q}_{m}\mathbf{v}_{a}\right]}_{\text{solved in \eqref{eq:(ub)^2(ubvb)(zaQmua)(zaQmva)}}}
\\&
=\tfrac{\left(p-m\right)\left(\left(1+\tfrac{1}{m}\right)m\left(p-m+2\right)\left(2mp+4m-p^{2}-3p+2\right)\right)}{\left(p-2\right)\left(p-1\right)p\left(p+1\right)\left(p+2\right)\left(p+4\right)\left(p+6\right)}
+
\\
&\eqmargin
\tfrac{\left(p-m\right)\left(2\left(m^{2}p^{2}+5m^{2}p+2m^{2}-mp^{3}-7mp^{2}-16mp-12m+4p^{2}+16p+16\right)\right)}{\left(p-2\right)\left(p-1\right)p\left(p+1\right)\left(p+2\right)\left(p+4\right)\left(p+6\right)}
\\&
=\tfrac{\left(p-m\right)\left(p-m+2\right)\left(2m^{2}p+4m^{2}-3mp^{2}-11mp+2m-p^{2}+5p+18\right)}{\left(p-2\right)\left(p-1\right)p\left(p+1\right)\left(p+2\right)\left(p+4\right)\left(p+6\right)}
\,.
\end{align*}

\medskip

Finally, the overall expression holds,
\begin{align*}
&\mathbb{E}\left(\mathbf{e}_{1}^{\top}\mathbf{O}\mathbf{e}_{1}\cdot\mathbf{e}_{1}^{\top}\mathbf{O}\mathbf{e}_{2}\right)\left(\mathbf{e}_{3}^{\top}\mathbf{O}\mathbf{e}_{1}\cdot\mathbf{e}_{3}^{\top}\mathbf{O}\mathbf{e}_{2}\right)
\\
&=\tfrac{6\left(m^{3}+m^{2}+2m+8\right)+4\left(m+2\right)p^{3}-2\left(5m^{2}+8m-2\right)p^{2}+\left(m-2\right)\left(5m^{2}+3m-16\right)p}{\left(p-2\right)\left(p-1\right)p\left(p+1\right)\left(p+2\right)\left(p+4\right)\left(p+6\right)}+
\\
&\eqmargin
\tfrac{\left(p-m\right)\left(p-m+2\right)\left(2m^{2}p+4m^{2}-3mp^{2}-11mp+2m-p^{2}+5p+18\right)}{\left(p-2\right)\left(p-1\right)p\left(p+1\right)\left(p+2\right)\left(p+4\right)\left(p+6\right)}
\\
&=
\tfrac{2m^{4}p+4m^{4}-7m^{3}p^{2}-18m^{3}p+8m^{2}p^{3}+25m^{2}p^{2}+24m^{2}p}{\left(p-2\right)\left(p-1\right)p\left(p+1\right)\left(p+2\right)\left(p+4\right)\left(p+6\right)}
+
\\
&
\eqmargin
\tfrac{20m^{2}-3mp^{4}-11mp^{3}-44mp^{2}-64mp-24m-p^{4}+11p^{3}+32p^{2}+68p+48}{\left(p-2\right)\left(p-1\right)p\left(p+1\right)\left(p+2\right)\left(p+4\right)\left(p+6\right)}\,.
\end{align*}
\end{proof}

\newpage

\begin{proposition} 
\label{prop:(e3Oe4)(e3Oe1)(e2Oe4)(e2Oe1)}
For $p\ge 4, m\in\{2,3,\dots,p\}$ and a random transformation $\mO$ sampled as described in \eqref{eq:data-model},
it holds that,
\begin{align*}
&\mathbb{E}\left(\mathbf{e}_{3}^{\top}\mathbf{O}\mathbf{e}_{4}\cdot\mathbf{e}_{3}^{\top}\mathbf{O}\mathbf{e}_{1}\right)\left(\mathbf{e}_{2}^{\top}\mathbf{O}\mathbf{e}_{4}\cdot\mathbf{e}_{2}^{\top}\mathbf{O}\mathbf{e}_{1}\right)
\\
&=
\tfrac{-5m^{4}p-6m^{4}+18m^{3}p^{2}+34m^{3}p-12m^{3}-20m^{2}p^{3}-46m^{2}p^{2}-39m^{2}p-42m^{2}+6mp^{4}+10mp^{3}+96mp^{2}}{\left(p-3\right)\left(p-2\right)\left(p-1\right)p\left(p+1\right)\left(p+2\right)\left(p+4\right)\left(p+6\right)}+
\\
&
\eqmargin
\tfrac{154mp+60m+2p^{4}-30p^{3}-32p^{2}-144p-144}{\left(p-3\right)\left(p-2\right)\left(p-1\right)p\left(p+1\right)\left(p+2\right)\left(p+4\right)\left(p+6\right)}\,.
\end{align*}
\end{proposition}

\begin{proof}
We begin by decomposing the expression into four terms,
\begin{align*}
&\mathbb{E}\left(\mathbf{e}_{3}^{\top}\mathbf{O}\mathbf{e}_{p}\cdot\mathbf{e}_{3}^{\top}\mathbf{O}\mathbf{e}_{1}\right)\left(\mathbf{e}_{2}^{\top}\mathbf{O}\mathbf{e}_{p}\cdot\mathbf{e}_{2}^{\top}\mathbf{O}\mathbf{e}_{1}\right)
\\&
=\mathbb{E}\left[\left(\mathbf{u}_{a}^{\top}\mathbf{Q}_{m}\mathbf{z}_{a}+\mathbf{u}_{b}^{\top}\mathbf{z}_{b}\right)\left(\mathbf{u}_{a}^{\top}\mathbf{Q}_{m}\mathbf{v}_{a}+\mathbf{u}_{b}^{\top}\mathbf{v}_{b}\right)\left(\mathbf{x}_{a}^{\top}\mathbf{Q}_{m}\mathbf{z}_{a}+\mathbf{x}_{b}^{\top}\mathbf{z}_{b}\right)\left(\mathbf{x}_{a}^{\top}\mathbf{Q}_{m}\mathbf{v}_{a}+\mathbf{x}_{b}^{\top}\mathbf{v}_{b}\right)\right]
\\&
=\mathbb{E}\left[\mathbf{u}_{a}^{\top}\mathbf{Q}_{m}\mathbf{z}_{a}\cdot\mathbf{u}_{a}^{\top}\mathbf{Q}_{m}\mathbf{v}_{a}\left(\mathbf{x}_{a}^{\top}\mathbf{Q}_{m}\mathbf{z}_{a}+\mathbf{x}_{b}^{\top}\mathbf{z}_{b}\right)\left(\mathbf{x}_{a}^{\top}\mathbf{Q}_{m}\mathbf{v}_{a}+\mathbf{x}_{b}^{\top}\mathbf{v}_{b}\right)\right]+
\\&
\eqmargin
\mathbb{E}\left[\mathbf{u}_{b}^{\top}\mathbf{z}_{b}\cdot\mathbf{u}_{b}^{\top}\mathbf{v}_{b}\left(\mathbf{x}_{a}^{\top}\mathbf{Q}_{m}\mathbf{z}_{a}+\mathbf{x}_{b}^{\top}\mathbf{z}_{b}\right)\left(\mathbf{x}_{a}^{\top}\mathbf{Q}_{m}\mathbf{v}_{a}+\mathbf{x}_{b}^{\top}\mathbf{v}_{b}\right)\right]+
\\&
\eqmargin
\mathbb{E}\left[\mathbf{u}_{a}^{\top}\mathbf{Q}_{m}\mathbf{z}_{a}\cdot\mathbf{u}_{b}^{\top}\mathbf{v}_{b}\left(\mathbf{x}_{a}^{\top}\mathbf{Q}_{m}\mathbf{z}_{a}+\mathbf{x}_{b}^{\top}\mathbf{z}_{b}\right)\left(\mathbf{x}_{a}^{\top}\mathbf{Q}_{m}\mathbf{v}_{a}+\mathbf{x}_{b}^{\top}\mathbf{v}_{b}\right)\right]+
\\&
\eqmargin
\mathbb{E}\left[\mathbf{u}_{b}^{\top}\mathbf{z}_{b}\cdot\mathbf{u}_{a}^{\top}\mathbf{Q}_{m}\mathbf{v}_{a}\left(\mathbf{x}_{a}^{\top}\mathbf{Q}_{m}\mathbf{z}_{a}+\mathbf{x}_{b}^{\top}\mathbf{z}_{b}\right)\left(\mathbf{x}_{a}^{\top}\mathbf{Q}_{m}\mathbf{v}_{a}+\mathbf{x}_{b}^{\top}\mathbf{v}_{b}\right)\right]
\end{align*}

Below we compute each of these terms separately. The result in the proposition is given by summing these 4 terms.

\paragraph{Term 1.}
Employing \ref{cor:odd-Q_m} once again, the term decomposes as
\begin{align*}
&\mathbb{E}\left[\mathbf{u}_{a}^{\top}\mathbf{Q}_{m}\mathbf{z}_{a}\cdot\mathbf{u}_{a}^{\top}\mathbf{Q}_{m}\mathbf{v}_{a}\left(\mathbf{x}_{a}^{\top}\mathbf{Q}_{m}\mathbf{z}_{a}+\mathbf{x}_{b}^{\top}\mathbf{z}_{b}\right)\left(\mathbf{x}_{a}^{\top}\mathbf{Q}_{m}\mathbf{v}_{a}+\mathbf{x}_{b}^{\top}\mathbf{v}_{b}\right)\right]\\&=\mathbb{E}\left[\mathbf{u}_{a}^{\top}\mathbf{Q}_{m}\mathbf{z}_{a}\mathbf{u}_{a}^{\top}\mathbf{Q}_{m}\mathbf{v}_{a}\mathbf{x}_{a}^{\top}\mathbf{Q}_{m}\mathbf{z}_{a}\mathbf{x}_{a}^{\top}\mathbf{Q}_{m}\mathbf{v}_{a}\right]+\underbrace{\mathbb{E}\left[\mathbf{u}_{a}^{\top}\mathbf{Q}_{m}\mathbf{z}_{a}\mathbf{u}_{a}^{\top}\mathbf{Q}_{m}\mathbf{v}_{a}\cdot\mathbf{x}_{b}^{\top}\mathbf{z}_{b}\mathbf{x}_{b}^{\top}\mathbf{v}_{b}\right]}_{\text{solved in \eqref{eq:(uaQmza)(uaQmva)(xbzb)(xbvb)}}}\,.
\end{align*}
The polynomial in the left inner term (the first summand) can be derived tediously, very much like in the proof of \propref{prop:(uaQmua)(uaQmva)(vaQmua)(vaQmVa)},
and shown to hold
\begin{align*}
&
\mathbb{E}\left[\mathbf{u}_{a}^{\top}\mathbf{Q}_{m}\mathbf{z}_{a}\mathbf{u}_{a}^{\top}\mathbf{Q}_{m}\mathbf{v}_{a}\mathbf{x}_{a}^{\top}\mathbf{Q}_{m}\mathbf{z}_{a}\mathbf{x}_{a}^{\top}\mathbf{Q}_{m}\mathbf{v}_{a}\right]
\\
&=
\tfrac{-m^{3}\left(2p^{3}+11p^{2}+p-30\right)+m^{2}\left(p^{4}+p^{3}+2p^{2}+60p+72\right)+2m\left(p^{4}+p^{3}+18p^{2}+14p-60\right)-8\left(2p+3\right)\left(p^{2}+12\right)}{\left(m+2\right)\left(p-3\right)\left(p-2\right)\left(p-1\right)p\left(p+1\right)\left(p+2\right)\left(p+4\right)\left(p+6\right)}
\,.
\end{align*}

Overall, the first term is,
\begin{align*}
&\mathbb{E}\left[\mathbf{u}_{a}^{\top}\mathbf{Q}_{m}\mathbf{z}_{a}\cdot\mathbf{u}_{a}^{\top}\mathbf{Q}_{m}\mathbf{v}_{a}\left(\mathbf{x}_{a}^{\top}\mathbf{Q}_{m}\mathbf{z}_{a}+\mathbf{x}_{b}^{\top}\mathbf{z}_{b}\right)\left(\mathbf{x}_{a}^{\top}\mathbf{Q}_{m}\mathbf{v}_{a}+\mathbf{x}_{b}^{\top}\mathbf{v}_{b}\right)\right]
\\
&=\tfrac{-m^{3}\left(2p^{3}+11p^{2}+p-30\right)+m^{2}\left(p^{4}+p^{3}+2p^{2}+60p+72\right)+2m\left(p^{4}+p^{3}+18p^{2}+14p-60\right)-8\left(2p+3\right)\left(p^{2}+12\right)}{\left(m+2\right)\left(p-3\right)\left(p-2\right)\left(p-1\right)p\left(p+1\right)\left(p+2\right)\left(p+4\right)\left(p+6\right)}+
\\
&\tfrac{\left(p-m\right)\left(-2m^{2}p^{2}-8m^{2}p+mp^{3}+4mp^{2}+15mp+18m-6p^{2}-6p-12\right)}{\left(p-3\right)\left(p-2\right)\left(p-1\right)p\left(p+1\right)\left(p+2\right)\left(p+4\right)\left(p+6\right)}
\\
&=\tfrac{2m^{3}p^{2}+8m^{3}p-5m^{2}p^{3}-23m^{2}p^{2}-16m^{2}p+12m^{2}+2mp^{4}+9mp^{3}+45mp^{2}+86mp+24m-14p^{3}-18p^{2}-108p-144}{\left(p-3\right)\left(p-2\right)\left(p-1\right)p\left(p+1\right)\left(p+2\right)\left(p+4\right)\left(p+6\right)}.
\end{align*}

\paragraph{Term 2.} 
Notice that $\mathbf{u},\mathbf{v},\mathbf{z},\mathbf{x}$ are exchangeable in the sense that we can swap them freely (see \propref{prop:invariance}). Therefore,
\begin{align*}
&\mathbb{E}\left[\mathbf{u}_{b}^{\top}\mathbf{z}_{b}\cdot\mathbf{u}_{b}^{\top}\mathbf{v}_{b}\left(\mathbf{x}_{a}^{\top}\mathbf{Q}_{m}\mathbf{z}_{a}+\mathbf{x}_{b}^{\top}\mathbf{z}_{b}\right)\left(\mathbf{x}_{a}^{\top}\mathbf{Q}_{m}\mathbf{v}_{a}+\mathbf{x}_{b}^{\top}\mathbf{v}_{b}\right)\right]
\\
\explain{\text{\corref{cor:odd-Q_m}}}
&=
\mathbb{E}\left[\mathbf{u}_{b}^{\top}\mathbf{z}_{b}\cdot\mathbf{u}_{b}^{\top}\mathbf{v}_{b}\left(\mathbf{x}_{a}^{\top}\mathbf{Q}_{m}\mathbf{z}_{a}\cdot\mathbf{x}_{a}^{\top}\mathbf{Q}_{m}\mathbf{v}_{a}+\mathbf{x}_{b}^{\top}\mathbf{z}_{b}\cdot\mathbf{x}_{b}^{\top}\mathbf{v}_{b}\right)\right]
\\&
=\mathbb{E}\left[\mathbf{u}_{b}^{\top}\mathbf{z}_{b}\mathbf{u}_{b}^{\top}\mathbf{v}_{b}\mathbf{x}_{a}^{\top}\mathbf{Q}_{m}\mathbf{z}_{a}\mathbf{x}_{a}^{\top}\mathbf{Q}_{m}\mathbf{v}_{a}\right]+\mathbb{E}\left[\mathbf{u}_{b}^{\top}\mathbf{z}_{b}\mathbf{u}_{b}^{\top}\mathbf{v}_{b}\mathbf{x}_{b}^{\top}\mathbf{z}_{b}\mathbf{x}_{b}^{\top}\mathbf{v}_{b}\right]
\\
\explain{\text{swap}}
&=\underbrace{\mathbb{E}\left[\mathbf{u}_{a}^{\top}\mathbf{Q}_{m}\mathbf{z}_{a}\mathbf{u}_{a}^{\top}\mathbf{Q}_{m}\mathbf{v}_{a}\mathbf{x}_{b}^{\top}\mathbf{z}_{b}\mathbf{x}_{b}^{\top}\mathbf{v}_{b}\right]}_{\text{swapped \ensuremath{\mathbf{x},\mathbf{u}}; solved in \eqref{eq:(uaQmza)(uaQmva)(xbzb)(xbvb)}}}+
\underbrace{\mathbb{E}\left[\mathbf{u}_{b}^{\top}\mathbf{v}_{b}\mathbf{x}_{b}^{\top}\mathbf{z}_{b}\mathbf{u}_{b}^{\top}\mathbf{x}_{b}\mathbf{v}_{b}^{\top}\mathbf{z}_{b}\right]}_{\text{swapped \ensuremath{\mathbf{v},\mathbf{u}}; solved in \eqref{eq:(ubvb)(xbzb)(ubxb)(vbzb)}}}
\\&
=\frac{\left(p-m\right)\left(-2m^{2}p^{2}-8m^{2}p+mp^{3}+4mp^{2}+15mp+18m-6p^{2}-6p-12\right)}{\left(p-3\right)\left(p-2\right)\left(p-1\right)p\left(p+1\right)\left(p+2\right)\left(p+4\right)\left(p+6\right)}+
\\&
\eqmargin
\frac{m\left(p-m\right)\left(5m^{2}p+6m^{2}-5mp^{2}-6mp+p^{3}+p^{2}+2p\right)}{\left(p-3\right)\left(p-2\right)\left(p-1\right)p\left(p+1\right)\left(p+2\right)\left(p+4\right)\left(p+6\right)}
\\&
=\left(p-m\right)\left(\tfrac{5m^{3}p+6m^{3}-7m^{2}p^{2}-14m^{2}p+2mp^{3}+5mp^{2}+17mp+18m-6p^{2}-6p-12}{\left(p-3\right)\left(p-2\right)\left(p-1\right)p\left(p+1\right)\left(p+2\right)\left(p+4\right)\left(p+6\right)}\right)\,.
\end{align*}

\newpage

\paragraph{Terms 3 and 4.} 
First, we notice that subterms 3 and 4 are equivalent (we can swap $\mathbf{z},\mathbf{v}$ due to the invariance; \propref{prop:invariance}), that is
\begin{align*}
&\mathbb{E}\left[\mathbf{u}_{a}^{\top}\mathbf{Q}_{m}\mathbf{z}_{a}\cdot\mathbf{u}_{b}^{\top}\mathbf{v}_{b}\left(\mathbf{x}_{a}^{\top}\mathbf{Q}_{m}\mathbf{z}_{a}+\mathbf{x}_{b}^{\top}\mathbf{z}_{b}\right)\left(\mathbf{x}_{a}^{\top}\mathbf{Q}_{m}\mathbf{v}_{a}+\mathbf{x}_{b}^{\top}\mathbf{v}_{b}\right)\right]\\&=\mathbb{E}\left[\mathbf{u}_{b}^{\top}\mathbf{z}_{b}\cdot\mathbf{u}_{a}^{\top}\mathbf{Q}_{m}\mathbf{v}_{a}\left(\mathbf{x}_{a}^{\top}\mathbf{Q}_{m}\mathbf{z}_{a}+\mathbf{x}_{b}^{\top}\mathbf{z}_{b}\right)\left(\mathbf{x}_{a}^{\top}\mathbf{Q}_{m}\mathbf{v}_{a}+\mathbf{x}_{b}^{\top}\mathbf{v}_{b}\right)\right]
\\
\explain{\text{below}}
&=
\tfrac{\left(p-m\right)\left(-2m^{2}p^{2}-3m^{2}p+6m^{2}+mp^{3}-mp^{2}+9mp+18m+p^{3}-5p^{2}-4p-12\right)}{\left(p-3\right)\left(p-2\right)\left(p-1\right)p\left(p+1\right)\left(p+2\right)\left(p+4\right)\left(p+6\right)}
\,,
\end{align*}
and so we focus on just one of them.

By employing \corref{cor:odd-Q_m}, we see that
\begin{align*}
&\mathbb{E}\left[\mathbf{u}_{a}^{\top}\mathbf{Q}_{m}\mathbf{z}_{a}\cdot\mathbf{u}_{b}^{\top}\mathbf{v}_{b}\left(\mathbf{x}_{a}^{\top}\mathbf{Q}_{m}\mathbf{z}_{a}+\mathbf{x}_{b}^{\top}\mathbf{z}_{b}\right)\left(\mathbf{x}_{a}^{\top}\mathbf{Q}_{m}\mathbf{v}_{a}+\mathbf{x}_{b}^{\top}\mathbf{v}_{b}\right)\right]
\\&
=\mathbb{E}\left[\mathbf{u}_{a}^{\top}\mathbf{Q}_{m}\mathbf{z}_{a}\cdot\mathbf{u}_{b}^{\top}\mathbf{v}_{b}\left(\mathbf{x}_{a}^{\top}\mathbf{Q}_{m}\mathbf{z}_{a}\cdot\mathbf{x}_{b}^{\top}\mathbf{v}_{b}+\mathbf{x}_{b}^{\top}\mathbf{z}_{b}\cdot\mathbf{x}_{a}^{\top}\mathbf{Q}_{m}\mathbf{v}_{a}\right)\right]
\\
&=\tfrac{\left(p-m\right)\left(\left(5m^{2}p+6m^{2}-5mp^{2}-6mp+p^{3}+p^{2}+2p\right)+\left(-2m^{2}p^{2}-8m^{2}p+mp^{3}+4mp^{2}+15mp+18m-6p^{2}-6p-12\right)\right)}{\left(p-3\right)\left(p-2\right)\left(p-1\right)p\left(p+1\right)\left(p+2\right)\left(p+4\right)\left(p+6\right)}
\\&
=\tfrac{\left(p-m\right)\left(-2m^{2}p^{2}-3m^{2}p+6m^{2}+mp^{3}-mp^{2}+9mp+18m+p^{3}-5p^{2}-4p-12\right)}{\left(p-3\right)\left(p-2\right)\left(p-1\right)p\left(p+1\right)\left(p+2\right)\left(p+4\right)\left(p+6\right)}
\,,
\end{align*}
where we used the following two derivations, \ie
\begin{align*}
&\mathbb{E}\left[\mathbf{u}_{a}^{\top}\mathbf{Q}_{m}\mathbf{z}_{a}\cdot\mathbf{u}_{b}^{\top}\mathbf{v}_{b}\cdot\mathbf{x}_{b}^{\top}\mathbf{z}_{b}\cdot\mathbf{x}_{a}^{\top}\mathbf{Q}_{m}\mathbf{v}_{a}\right]=\mathbb{E}\left[\left(\mathbf{u}_{b}^{\top}\mathbf{v}_{b}\cdot\mathbf{x}_{b}^{\top}\mathbf{z}_{b}\right)\cdot\mathbf{u}_{a}^{\top}\mathbf{Q}_{m}\mathbf{z}_{a}\cdot\mathbf{x}_{a}^{\top}\mathbf{Q}_{m}\mathbf{v}_{a}\right]
\\&
=\mathbb{E}\left[\left(\mathbf{u}_{b}^{\top}\mathbf{v}_{b}\cdot\mathbf{x}_{b}^{\top}\mathbf{z}_{b}\right)\cdot\left(\sum_{i=1}^{m}\sum_{j=1}^{m}u_{i}q_{ij}z_{j}\right)\cdot\left(\sum_{k=1}^{m}\sum_{\ell=1}^{m}x_{k}q_{k\ell}v_{\ell}\right)\right]
\\&
=\sum_{i,j,k,\ell=1}^{m}\mathbb{E}\left[\left(\mathbf{u}_{b}^{\top}\mathbf{v}_{b}\cdot\mathbf{x}_{b}^{\top}\mathbf{z}_{b}\right)\cdot u_{i}v_{\ell}x_{k}z_{j}q_{ij}q_{k\ell}\right]
\\&
=
\sum_{i,j,k,\ell=1}^{m}\mathbb{E}\left[\left(\mathbf{u}_{b}^{\top}\mathbf{v}_{b}\cdot\mathbf{x}_{b}^{\top}\mathbf{z}_{b}\right)\cdot u_{i}v_{\ell}x_{k}z_{j}\mathbb{E}_{\mathbf{Q}_{m}}\left[q_{ij}q_{k\ell}\right]\right]
\\
\explain{\text{\propref{prop:odd}}}
&=\sum_{i,j=1}^{m}\underbrace{\mathbb{E}_{\mathbf{Q}_{m}}\left[q_{ij}^{2}\right]}_{=1/m}\mathbb{E}\left[\left(\mathbf{u}_{b}^{\top}\mathbf{v}_{b}\cdot\mathbf{x}_{b}^{\top}\mathbf{z}_{b}\right)\cdot u_{i}v_{j}x_{i}z_{j}\right]
\\&
=\frac{1}{m}\underbrace{\mathbb{E}\left[\mathbf{u}_{b}^{\top}\mathbf{v}_{b}\cdot\mathbf{x}_{b}^{\top}\mathbf{z}_{b}\cdot\mathbf{u}_{b}^{\top}\mathbf{x}_{b}\cdot\mathbf{v}_{b}^{\top}\mathbf{z}_{b}\right]}_{\text{solved in \eqref{eq:(ubvb)(xbzb)(ubxb)(vbzb)}}}
\\&
=\frac{\left(p-m\right)\left(5m^{2}p+6m^{2}-5mp^{2}-6mp+p^{3}+p^{2}+2p\right)}{\left(p-3\right)\left(p-2\right)\left(p-1\right)p\left(p+1\right)\left(p+2\right)\left(p+4\right)\left(p+6\right)}\,,
\end{align*}
and
\begin{align*}   
&\mathbb{E}\left[\mathbf{u}_{a}^{\top}\mathbf{Q}_{m}\mathbf{z}_{a}\cdot\mathbf{u}_{b}^{\top}\mathbf{v}_{b}\mathbf{x}_{a}^{\top}\mathbf{Q}_{m}\mathbf{z}_{a}\cdot\mathbf{x}_{b}^{\top}\mathbf{v}_{b}\right]
\\&
=\mathbb{E}\left[\left\Vert \mathbf{z}_{a}\right\Vert ^{2}\mathbb{E}_{\mathbf{r}\sim\mathcal{S}^{m-1}}\left(\mathbf{u}_{a}^{\top}\mathbf{r}\mathbf{r}^{\top}\mathbf{x}_{a}\right)\cdot\mathbf{u}_{b}^{\top}\mathbf{v}_{b}\mathbf{x}_{b}^{\top}\mathbf{v}_{b}\right]=\frac{1}{m}\underbrace{\mathbb{E}\left[\left\Vert \mathbf{z}_{a}\right\Vert ^{2}\mathbf{u}_{a}^{\top}\mathbf{x}_{a}\mathbf{u}_{b}^{\top}\mathbf{v}_{b}\mathbf{x}_{b}^{\top}\mathbf{v}_{b}\right]}_{\text{solved in \eqref{eq:(za)^2(uaxa)(ubvb)(xbvb)}}}
\\&
=\frac{\left(p-m\right)\left(-2m^{2}p^{2}-8m^{2}p+mp^{3}+4mp^{2}+15mp+18m-6p^{2}-6p-12\right)}{\left(p-3\right)\left(p-2\right)\left(p-1\right)p\left(p+1\right)\left(p+2\right)\left(p+4\right)\left(p+6\right)}\,.
\end{align*}
\end{proof}

\newpage

\newpage

\begin{proposition}
\label{prop:u_norms}
Let $\vu\sim\mathcal{S}^{p-1}$ and let
$\vu_a$ consists of its first $m$ coordinates.
The expected $n$\nth power of the squared norm is,
$$
\mathbb{E}_{\vu\sim \mathcal{S}^{p-1}}
\norm{\vu_a}^{2n}
=
\prod_{r=0}^{n-1}\frac{m+2r}{p+2r}\,.
$$

Specifically, it holds that
\begin{align*}
\begin{split}
\mathbb{E}_{\vu\sim \mathcal{S}^{p-1}}
\norm{\vu_a}^2
&
=\frac{m}{p}
\\
&
\\
\mathbb{E}_{\vu\sim \mathcal{S}^{p-1}}
\norm{\vu_a}^4
&
=\frac{m}{p}\cdot\frac{m+2}{p+2}
\\
&
\\
\mathbb{E}_{\vu\sim \mathcal{S}^{p-1}}
\norm{\vu_a}^6
&
=\frac{m}{p}\cdot\frac{m+2}{p+2}\cdot\frac{m+4}{p+4}
\\
&
\\
\mathbb{E}_{\vu\sim \mathcal{S}^{p-1}}
\norm{\vu_a}^8
&
=\frac{m}{p}\cdot\frac{m+2}{p+2}\cdot\frac{m+4}{p+4}\cdot\frac{m+6}{p+6}
\\
\Longrightarrow
\mathbb{E}_{\vu\sim \mathcal{S}^{p-1}}
\norm{\vu_b}^8
&
=
\frac{p-m}{p}\cdot\frac{p-m+2}{p+2}\cdot\frac{p-m+4}{p+4}\cdot\frac{p-m+6}{p+6}
\end{split}
\end{align*}
\end{proposition}

\begin{proof}
Notice that the squared norm can be parameterized as
\hfill
$\norm{\vu_a}^2=\vu_a^\top \vu_a\triangleq \frac{X}{X+Y}\,,$
\linebreak
where $X\sim\chi_{m}^{2}, Y\sim\chi_{p-m}^{2}$.
Consequently, it is distributed as $\norm{\vu_a}^2 \sim B\left(\frac{m}{2},\frac{p-m}{2}\right)$.
Moreover, given $n\in\naturals$, the $n$\nth raw moment is given by 
(Chapter 25 in \citet{johnson1995continuous}),
$$
\mathbb{E}_{\vu\sim \mathcal{S}^{p-1}}
\norm{\vu_a}^{2n}
=
\prod_{r=0}^{n-1}\frac{\frac{m}{2}+r}{\frac{p}{2}+r}
=
\prod_{r=0}^{n-1}\frac{m+2r}{p+2r}\,.
$$
\end{proof}

\begin{align}
\label{eq:ua4(u1v1)^2}
\begin{split}
&\mathbb{E}\left[\left\Vert \mathbf{u}_{a}\right\Vert ^{4}u_{1}^{2}v_{1}^{2}\right]
=
\mathbb{E}\left[\left(\sum_{i=1}^{m}u_{i}^{2}\right)\left(\sum_{j=1}^{m}u_{j}^{2}\right)u_{1}^{2}v_{1}^{2}\right]
=
\sum_{i=1}^{m}\sum_{j=1}^{m}\mathbb{E}\left[u_{i}^{2}u_{j}^{2}u_{1}^{2}v_{1}^{2}\right]
\\&
=\underbrace{\mathbb{E}\left[u_{1}^{6}v_{1}^{2}\right]}_{i=j=1}
\!
+
2\left(m\!-\!1\right)\!\!
\underbrace{\mathbb{E}\left[u_{1}^{4}v_{1}^{2}u_{2}^{2}\right]}_{i=1\neq j\,\vee\,i\neq1=j}
+
\left(m\!-\!1\right)
\underbrace{\mathbb{E}\left[u_{1}^{2}u_{2}^{4}v_{1}^{2}\right]}_{i=j\ge2
}
+\left(m\!-\!1\right)\left(m\!-\!2\right)
\underbrace{\mathbb{E}\left[u_{1}^{2}u_{2}^{2}u_{3}^{2}v_{1}^{2}\right]}_{i\neq j\ge2}
\\&
=\left\langle \begin{smallmatrix}
6 & 2\\
\overrightarrow{0} & \overrightarrow{0}
\end{smallmatrix}\right\rangle +2\left(m-1\right)\left\langle \begin{smallmatrix}
4 & 2\\
2 & 0\\
\overrightarrow{0} & \overrightarrow{0}
\end{smallmatrix}\right\rangle +\left(m-1\right)\left\langle \begin{smallmatrix}
2 & 2\\
4 & 0\\
\overrightarrow{0} & \overrightarrow{0}
\end{smallmatrix}\right\rangle +\left(m-1\right)\left(m-2\right)\left\langle \begin{smallmatrix}
2 & 2\\
2 & 0\\
2 & 0\\
\overrightarrow{0} & \overrightarrow{0}
\end{smallmatrix}\right\rangle 
\\&
=\frac{15\left(p-1\right)+\left(m-1\right)\left(6\left(p+1\right)+3\left(p+3\right)+\left(m-2\right)\left(p+3\right)\right)}{\left(p-1\right)p\left(p+2\right)\left(p+4\right)\left(p+6\right)}
\\&
=\frac{m^{2}\left(p+3\right)+m\left(6p+6\right)+8p-24}{\left(p-1\right)p\left(p+2\right)\left(p+4\right)\left(p+6\right)}
\end{split}
\end{align}

\begin{align}
\label{eq:ua4(u1v2)^2}
\begin{split}
&\mathbb{E}\left[\left\Vert \mathbf{u}_{a}\right\Vert ^{4}u_{1}^{2}v_{2}^{2}\right]=\mathbb{E}\left[\left(\sum_{i=1}^{m}u_{i}^{2}\right)\left(\sum_{j=1}^{m}u_{j}^{2}\right)u_{1}^{2}v_{2}^{2}\right]=\sum_{i=1}^{m}\sum_{j=1}^{m}\mathbb{E}\left[u_{i}^{2}u_{j}^{2}u_{1}^{2}v_{2}^{2}\right]
\\&
=\underbrace{\mathbb{E}\left[u_{1}^{6}v_{2}^{2}\right]}_{i=j=1}+\underbrace{\mathbb{E}\left[u_{1}^{2}u_{2}^{4}v_{2}^{2}\right]}_{i=j=2}+2\underbrace{\mathbb{E}\left[u_{1}^{4}u_{2}^{2}v_{2}^{2}\right]}_{i=2,j=1\,\vee\,i=1,j=2}+2\left(m-2\right)\underbrace{\mathbb{E}\left[u_{1}^{4}v_{2}^{2}u_{3}^{2}\right]}_{i=1,j\ge3\,\vee\,j=1,i\ge3}
+
\\&
\eqmargin
2\left(m-2\right)
\underbrace{\mathbb{E}\left[u_{1}^{2}u_{2}^{2}v_{2}^{2}u_{3}^{2}\right]}_{i=2,j\ge3\,\vee\,j=2,i\ge3}+
\left(m-2\right)\left(m-3\right)\underbrace{\mathbb{E}\left[u_{1}^{2}v_{2}^{2}u_{3}^{2}u_{4}^{2}\right]}_{i\neq j\ge3}+\left(m-2\right)\underbrace{\mathbb{E}\left[u_{1}^{2}v_{2}^{2}u_{3}^{4}\right]}_{i=j\ge3}
\\&
=\left\langle \begin{smallmatrix}
6 & 0\\
0 & 2\\
\overrightarrow{0} & \overrightarrow{0}
\end{smallmatrix}\right\rangle +\left\langle \begin{smallmatrix}
2 & 0\\
4 & 2\\
\overrightarrow{0} & \overrightarrow{0}
\end{smallmatrix}\right\rangle +2\left\langle \begin{smallmatrix}
4 & 0\\
2 & 2\\
\overrightarrow{0} & \overrightarrow{0}
\end{smallmatrix}\right\rangle +2\left(m-2\right)\left\langle \begin{smallmatrix}
4 & 0\\
0 & 2\\
2 & 0\\
\overrightarrow{0} & \overrightarrow{0}
\end{smallmatrix}\right\rangle +2\left(m-2\right)\left\langle \begin{smallmatrix}
2 & 0\\
2 & 2\\
2 & 0\\
\overrightarrow{0} & \overrightarrow{0}
\end{smallmatrix}\right\rangle +
\\&
\eqmargin
\left(m-2\right)\left(m-3\right)\left\langle \begin{smallmatrix}
2 & 0\\
0 & 2\\
2 & 0\\
2 & 0\\
\overrightarrow{0} & \overrightarrow{0}
\end{smallmatrix}\right\rangle +\left(m-2\right)\left\langle \begin{smallmatrix}
2 & 0\\
0 & 2\\
4 & 0\\
\overrightarrow{0} & \overrightarrow{0}
\end{smallmatrix}\right\rangle 
\\&
=\left\langle \begin{smallmatrix}
6 & 0\\
0 & 2\\
\overrightarrow{0} & \overrightarrow{0}
\end{smallmatrix}\right\rangle +\left\langle \begin{smallmatrix}
2 & 0\\
4 & 2\\
\overrightarrow{0} & \overrightarrow{0}
\end{smallmatrix}\right\rangle +2\left\langle \begin{smallmatrix}
4 & 0\\
2 & 2\\
\overrightarrow{0} & \overrightarrow{0}
\end{smallmatrix}\right\rangle +3\left(m-2\right)\left\langle \begin{smallmatrix}
4 & 0\\
0 & 2\\
2 & 0\\
\overrightarrow{0} & \overrightarrow{0}
\end{smallmatrix}\right\rangle +2\left(m-2\right)\left\langle \begin{smallmatrix}
2 & 0\\
2 & 2\\
2 & 0\\
\overrightarrow{0} & \overrightarrow{0}
\end{smallmatrix}\right\rangle +
\\&
\eqmargin
\left(m-2\right)\left(m-3\right)\left\langle \begin{smallmatrix}
2 & 0\\
0 & 2\\
2 & 0\\
2 & 0\\
\overrightarrow{0} & \overrightarrow{0}
\end{smallmatrix}\right\rangle 
\\&
=\tfrac{15\left(p+5\right)+3\left(p+1\right)+2\left(3p+9\right)+\left(m-2\right)\left(9\left(p+5\right)+2\left(p+3\right)+\left(m-3\right)\left(p+5\right)\right)}{\left(p-1\right)p\left(p+2\right)\left(p+4\right)\left(p+6\right)}
\\&
=\frac{m^{2}\left(p+5\right)+2m\left(3p+13\right)+8\left(p+3\right)}{\left(p-1\right)p\left(p+2\right)\left(p+4\right)\left(p+6\right)}
\end{split}
\end{align}

\bigskip

We notice that the following derivation is symmetric in $i,j$. 
Thus, we assume $j\ge i$ and multiply everything by $2$.
\begin{align}
\label{eq:ua4(u1u2v1v2)}
\begin{split}
&\mathbb{E}\left[\left\Vert \mathbf{u}_{a}\right\Vert ^{4}u_{1}u_{2}v_{1}v_{2}\right]=\mathbb{E}\left[\left(\sum_{i=1}^{m}u_{i}^{2}\right)\left(\sum_{j=1}^{m}u_{j}^{2}\right)u_{1}u_{2}v_{1}v_{2}\right]=\sum_{i=1}^{m}\sum_{j=1}^{m}\mathbb{E}\left[u_{i}^{2}u_{j}^{2}u_{1}u_{2}v_{1}v_{2}\right]
\\&
=2\underbrace{\mathbb{E}\left[u_{1}^{5}u_{2}v_{1}v_{2}\right]}_{i=j=1\,\vee\,i=j=2}+2\underbrace{\mathbb{E}\left[u_{1}^{3}u_{2}^{3}v_{1}v_{2}\right]}_{i=1,j=2}+4\left(m-2\right)\underbrace{\mathbb{E}\left[u_{1}^{3}u_{2}u_{3}^{2}v_{1}v_{2}\right]}_{i=1,j\ge3\,\vee\,i=2,j\ge3}
+
\\&
\eqmargin
\left(m-2\right)\left(m-3\right)
\underbrace{\mathbb{E}\left[u_{1}u_{2}u_{3}^{2}u_{4}^{2}v_{1}v_{2}\right]}_{i\neq j\ge3}+\left(m-2\right)\underbrace{\mathbb{E}\left[u_{1}u_{2}u_{3}^{4}v_{1}v_{2}\right]}_{i=j\ge3}
\\&
=2\left\langle \begin{smallmatrix}
5 & 1\\
1 & 1\\
\overrightarrow{0} & \overrightarrow{0}
\end{smallmatrix}\right\rangle +2\left\langle \begin{smallmatrix}
3 & 1\\
3 & 1\\
\overrightarrow{0} & \overrightarrow{0}
\end{smallmatrix}\right\rangle +4\left(m-2\right)\left\langle \begin{smallmatrix}
3 & 1\\
1 & 1\\
2 & 0\\
\overrightarrow{0} & \overrightarrow{0}
\end{smallmatrix}\right\rangle +\left(m-2\right)\left(m-3\right)\left\langle \begin{smallmatrix}
1 & 1\\
1 & 1\\
2 & 0\\
2 & 0\\
\overrightarrow{0} & \overrightarrow{0}
\end{smallmatrix}\right\rangle +
\\&
\eqmargin
\left(m-2\right)\left\langle \begin{smallmatrix}
1 & 1\\
1 & 1\\
4 & 0\\
\overrightarrow{0} & \overrightarrow{0}
\end{smallmatrix}\right\rangle 
\\&
=\frac{-30-18+\left(m-2\right)\left(-12-\left(m-3\right)-3\right)}{\left(p-1\right)p\left(p+2\right)\left(p+4\right)\left(p+6\right)}
\\&
=\frac{-\left(m+6\right)\left(m+4\right)}{\left(p-1\right)p\left(p+2\right)\left(p+4\right)\left(p+6\right)}
\end{split}
\end{align}

\newpage

\subsection{Auxiliary derivations with two vectors}
Below we attached many auxiliary derivations of simple polynomials that we need in our main propositions and lemmas.

\begin{align} 
\label{eq:(uaQmua)2(uaQmva)2}
\begin{split}
&\mathbb{E}_{\mathbf{u}\perp\mathbf{v}, \Q_m}\left(\mathbf{u}_{a}^{\top}\mathbf{Q}_{m}\mathbf{u}_{a}\right)^{2}\left(\mathbf{u}_{a}^{\top}\mathbf{Q}_{m}\mathbf{v}_{a}\right)^{2}
\\&
=\mathbb{E}_{\mathbf{u}\perp\mathbf{v}}\left[\mathbb{E}_{\mathbf{r}\sim\mathcal{S}^{m-1}}\left[\left\Vert \mathbf{u}_{a}\right\Vert ^{4}\left(\tsum_{i=1}^{m}r_{i}u_{i}\right)^{2}\left(\tsum_{j=1}^{m}r_{j}v_{j}\right)^{2}\right]\right]
\\&
=\sum_{i=1}^{m}\sum_{j=1}^{m}\sum_{k=1}^{m}\sum_{\ell=1}^{m}\mathbb{E}_{\mathbf{u}\perp\mathbf{v}}\left[\mathbb{E}_{\mathbf{r}\sim\mathcal{S}^{m-1}}\left[\left\Vert \mathbf{u}_{a}\right\Vert ^{4}r_{i}u_{i}r_{j}u_{j}r_{k}v_{k}r_{\ell}v_{\ell}\right]\right]
\\&
=\sum_{i=1}^{m}\sum_{j=1}^{m}\sum_{k=1}^{m}\sum_{\ell=1}^{m}\mathbb{E}_{\mathbf{u}\perp\mathbf{v}}\left[\left\Vert \mathbf{u}_{a}\right\Vert ^{4}u_{i}u_{j}v_{k}v_{\ell}\mathbb{E}_{\mathbf{r}\sim\mathcal{S}^{m-1}}\left[r_{i}r_{j}r_{k}r_{\ell}\right]\right]
\\
\explain{
\text{\propref{prop:odd}}
}
&
=\underbrace{\sum_{i=1}^{m}\mathbb{E}_{\mathbf{u}\perp\mathbf{v}}\left[\left\Vert \mathbf{u}_{a}\right\Vert ^{4}u_{i}^{2}v_{i}^{2}\mathbb{E}_{\mathbf{r}}\left[r_{i}^{4}\right]\right]}_{i=j=k=\ell}+\underbrace{\sum_{i\neq k=1}^{m}\mathbb{E}_{\mathbf{u}\perp\mathbf{v}}\left[\left\Vert \mathbf{u}_{a}\right\Vert ^{4}u_{i}^{2}v_{k}^{2}\mathbb{E}_{\mathbf{r}}\left[r_{i}^{2}r_{k}^{2}\right]\right]}_{i=j\neq k=\ell}
+
\\
&
\eqmargin
\underbrace{2\sum_{i\neq j=1}^{m}\mathbb{E}_{\mathbf{u}\perp\mathbf{v}}\left[\left\Vert \mathbf{u}_{a}\right\Vert ^{4}u_{i}u_{j}v_{i}v_{j}\mathbb{E}_{\mathbf{r}}\left[r_{i}^{2}r_{j}^{2}\right]\right]}_{i=k\neq j=\ell\,\vee\,i=\ell\neq j=k}
\\&
=\left\langle \begin{smallmatrix}
4\\
\overrightarrow{0}
\end{smallmatrix}\right\rangle _{m}\sum_{i=1}^{m}\mathbb{E}\left[\left\Vert \mathbf{u}_{a}\right\Vert ^{4}u_{i}^{2}v_{i}^{2}\right]
+
\\&
\eqmargin
\left\langle \begin{smallmatrix}
2\\
2\\
\overrightarrow{0}
\end{smallmatrix}\right\rangle _{m}\left(\sum_{i\neq k=1}^{m}\mathbb{E}\left[\left\Vert \mathbf{u}_{a}\right\Vert ^{4}u_{i}^{2}v_{k}^{2}\right]+2\sum_{i\neq j=1}^{m}\mathbb{E}\left[\left\Vert \mathbf{u}_{a}\right\Vert ^{4}u_{i}u_{j}v_{i}v_{j}\right]\right)
\\&
=m\left\langle \begin{smallmatrix}
4\\
\overrightarrow{0}
\end{smallmatrix}\right\rangle _{m}\mathbb{E}\left[\left\Vert \mathbf{u}_{a}\right\Vert ^{4}u_{1}^{2}v_{1}^{2}\right]
+
\\&
\eqmargin
m
\left(m-1\right)\left\langle \begin{smallmatrix}
2\\
2\\
\overrightarrow{0}
\end{smallmatrix}\right\rangle _{m}\left(\mathbb{E}\left[\left\Vert \mathbf{u}_{a}\right\Vert ^{4}u_{1}^{2}v_{2}^{2}\right]+2\mathbb{E}\left[\left\Vert \mathbf{u}_{a}\right\Vert ^{4}u_{1}u_{2}v_{1}v_{2}\right]\right)
\\
&
=\frac{3}{m+2}
\underbrace{\mathbb{E}\left[\left\Vert \mathbf{u}_{a}\right\Vert ^{4}u_{1}^{2}v_{1}^{2}\right]}_{
\text{solved in \eqref{eq:ua4(u1v1)^2}}
}
+
\frac{m-1}{m+2}
\Bigprn{
    \underbrace{\mathbb{E}\left[\left\Vert \mathbf{u}_{a}\right\Vert ^{4}u_{1}^{2}v_{2}^{2}\right]}_{
    \text{solved in \eqref{eq:ua4(u1v2)^2}}
    }
    +2
    \underbrace{
    \mathbb{E}\left[\left\Vert \mathbf{u}_{a}\right\Vert ^{4}u_{1}u_{2}v_{1}v_{2}\right]
    }_{
    \text{solved in \eqref{eq:ua4(u1u2v1v2)}}
    }
}
\\
&
=
\tfrac{3\left(m^{2}\left(p+3\right)+m\left(6p+6\right)+8p-24\right)}{\left(m+2\right)\left(p-1\right)p\left(p+2\right)\left(p+4\right)\left(p+6\right)}
+
\tfrac{\left(m-1\right)
\left(m^{2}\left(p+5\right)+2m\left(3p+13\right)+8\left(p+3\right)\right)}{\left(m+2\right)\left(p-1\right)p\left(p+2\right)\left(p+4\right)\left(p+6\right)}
-
\\
&
\eqmargin
\tfrac{2\left(m-1\right)\cdot\left(m+6\right)\left(m+4\right)}{\left(m+2\right)\left(p-1\right)p\left(p+2\right)\left(p+4\right)\left(p+6\right)}
\\
&
=
\frac{\left(m+2\right)\left(m+4\right)\left(m\left(p+3\right)+2\left(p-3\right)\right)}{\left(m+2\right)\left(p-1\right)p\left(p+2\right)\left(p+4\right)\left(p+6\right)}
=\frac{\left(m+4\right)\left(m\left(p+3\right)+2\left(p-3\right)\right)}{\left(p-1\right)p\left(p+2\right)\left(p+4\right)\left(p+6\right)}
\end{split}
\end{align}

\newpage

\begin{align} %
\label{eq:(ua)4(ubvb)2}
\begin{split}
&
\mathbb{E}_{\mathbf{u}\perp\mathbf{v}}\left[\left\Vert \mathbf{u}_{a}\right\Vert ^{4}\left(\mathbf{u}_{b}^{\top}\mathbf{v}_{b}\right)^{2}\right]
=
\mathbb{E}\left[\left\Vert \mathbf{u}_{a}\right\Vert ^{4}\left(-\mathbf{u}_{a}^{\top}\mathbf{v}_{a}\right)^{2}\right]
=\sum_{i,j=1}^{m}\mathbb{E}\left[\left\Vert \mathbf{u}_{a}\right\Vert ^{4}u_{i}u_{j}v_{i}v_{j}\right]
\\&
=\underbrace{\sum_{i=1}^{m}\mathbb{E}\left[\left\Vert \mathbf{u}_{a}\right\Vert ^{4}u_{i}^{2}v_{i}^{2}\right]}_{i=j}+\underbrace{\sum_{i\neq j=1}^{m}\mathbb{E}\left[\left\Vert \mathbf{u}_{a}\right\Vert ^{4}u_{i}u_{j}v_{i}v_{j}\right]}_{i\neq j}
\\&
=
m\underbrace{\mathbb{E}\left[\left\Vert \mathbf{u}_{a}\right\Vert ^{4}u_{1}^{2}v_{1}^{2}\right]}_{\text{solved in \eqref{eq:ua4(u1v1)^2}}}
+
m\left(m-1\right)\underbrace{\mathbb{E}\left[\left\Vert \mathbf{u}_{a}\right\Vert ^{4}u_{1}u_{2}v_{1}v_{2}\right]}_{\text{solved in \eqref{eq:ua4(u1u2v1v2)}}}
\\&
=
m\prn{
\frac{m^{2}\left(p+3\right)+m\left(6p+6\right)+8p-24}{\left(p-1\right)p\left(p+2\right)\left(p+4\right)\left(p+6\right)}-\frac{\left(m-1\right)\cdot\left(m+4\right)\left(m+6\right)}{\left(p-1\right)p\left(p+2\right)\left(p+4\right)\left(p+6\right)}
}
\\&
=
m
\frac{\left(m+2\right)\left(m+4\right)\left(p-m\right)}{\left(p-1\right)p\left(p+2\right)\left(p+4\right)\left(p+6\right)}
\end{split}
\end{align}

\bigskip

\begin{align} 
\label{eq:(uaQmua)2(ubvb)2}
\begin{split}
&\mathbb{E}_{\mathbf{u}\perp\mathbf{v},\mathbf{Q}_{m}}\left[\left(\mathbf{u}_{a}^{\top}\mathbf{Q}_{m}\mathbf{u}_{a}\right)^{2}\left(\mathbf{u}_{b}^{\top}\mathbf{v}_{b}\right)^{2}\right]
=
\mathbb{E}_{\mathbf{u}\perp\mathbf{v}}\left[\mathbb{E}_{\mathbf{r}\sim\mathcal{S}^{m-1}}\left(\left\Vert \mathbf{u}_{a}\right\Vert \mathbf{r}^{\top}\mathbf{u}_{a}\right)^{2}\left(\mathbf{u}_{b}^{\top}\mathbf{v}_{b}\right)^{2}\right]
\\&
=\mathbb{E}_{\mathbf{u}\perp\mathbf{v}}\left[\left\Vert \mathbf{u}_{a}\right\Vert ^{2}\mathbb{E}_{\mathbf{r}\sim\mathcal{S}^{m-1}}\mathbf{u}_{a}^{\top}\mathbf{r}\mathbf{r}^{\top}\mathbf{u}_{a}\left(\mathbf{u}_{b}^{\top}\mathbf{v}_{b}\right)^{2}\right]
=
\frac{1}{m}
\underbrace{\mathbb{E}_{\mathbf{u}\perp\mathbf{v}}\left[\left\Vert \mathbf{u}_{a}\right\Vert ^{4}\left(\mathbf{u}_{b}^{\top}\mathbf{v}_{b}\right)^{2}\right]
}_{\text{solved in \eqref{eq:(ua)4(ubvb)2}}}
\\
&
=
\frac{\left(m+2\right)\left(m+4\right)\left(p-m\right)}{\left(p-1\right)p\left(p+2\right)\left(p+4\right)\left(p+6\right)}
\end{split}
\end{align}

\newpage

\begin{align}
\label{eq:ua2(ubvb)2}
\begin{split}
\mathbb{E}\left[\left\Vert \vu_{a}\right\Vert ^{2}\left(\vu_{b}^{\top}\mathbf{v}_{b}\right)^{2}\right]&=\mathbb{E}\left[\sum_{i=1}^{m}u_{i}^{2}\left(\sum_{j=m+1}^{p}u_{j}v_{j}\right)^{2}\right]
=
m\mathbb{E}\left[u_{1}^{2}\sum_{j=m+1}^{p}\sum_{k=m+1}^{p}u_{j}v_{j}u_{k}v_{k}\right]
\\&
=m\left(p-m\right)\sum_{k=m+1}^{p}\mathbb{E}\left[u_{1}^{2}u_{p}v_{p}u_{k}v_{k}\right]
\\&
=m\left(p-m\right)\left(\underbrace{\mathbb{E}\left[u_{1}^{2}u_{p}^{2}v_{p}^{2}\right]}_{k=p}+\underbrace{\left(p-m-1\right)\mathbb{E}\left[u_{1}^{2}u_{p-1}v_{p-1}u_{p}v_{p}\right]}_{m+1\le k\le p-1}\right)
\\&
=m\left(p-m\right)\left(\left\langle \begin{smallmatrix}
2 & 0\\
2 & 2\\
\overrightarrow{0} & \overrightarrow{0}
\end{smallmatrix}\right\rangle +\left(p-m-1\right)\left\langle \begin{smallmatrix}
2 & 0\\
1 & 1\\
1 & 1\\
\overrightarrow{0} & \overrightarrow{0}
\end{smallmatrix}\right\rangle \right)
\\&
=m\left(p-m\right)\left(\tfrac{p+1}{\left(p-1\right)p\left(p+2\right)\left(p+4\right)}+\tfrac{-\left(p-m-1\right)}{\left(p-1\right)p\left(p+2\right)\left(p+4\right)}\right)
\\&
=\frac{\left(p-m\right)m\left(m+2\right)}{\left(p-1\right)p\left(p+2\right)\left(p+4\right)}
\end{split}
\end{align}

Consequently,

\begin{align}
\label{eq:ub2(uava)2}
\begin{split}
\mathbb{E}\left[\left\Vert \vu_{b}\right\Vert ^{2}\left(\vu_{a}^{\top}\mathbf{v}_{a}\right)^{2}\right]
=\frac{\left(p-m\right)m\left(p-m+2\right)}{\left(p-1\right)p\left(p+2\right)\left(p+4\right)}
\end{split}
\end{align}

\bigskip

\begin{align}
\label{eq:ua2va2ub2}
\begin{split}
\mathbb{E}\left[\left\Vert \vu_{a}\right\Vert ^{2}\left\Vert \mathbf{v}_{a}\right\Vert ^{2}\left\Vert \vu_{b}\right\Vert ^{2}\right]
&=
\sum_{i=1}^{m}\sum_{j=1}^{m}\sum_{k=m+1}^{p}\mathbb{E}\left[u_{i}^{2}u_{k}^{2}v_{j}^{2}\right]
=
\left(p-m\right)\sum_{i=1}^{m}\sum_{j=1}^{m}\mathbb{E}\left[u_{i}^{2}u_{p}^{2}v_{j}^{2}\right]
\\&
=
m\left(p-m\right)\sum_{j=1}^{m}\mathbb{E}\left[u_{1}^{2}u_{p}^{2}v_{j}^{2}\right]
\\&
=
m\left(p-m\right)\left(\underbrace{\mathbb{E}\left[u_{1}^{2}u_{p}^{2}v_{1}^{2}\right]}_{j=1}+\underbrace{\left(m-1\right)\mathbb{E}\left[u_{1}^{2}u_{p}^{2}v_{2}^{2}\right]}_{j\ge2}\right)
\\&
=
m\left(p-m\right)\left(\left\langle \begin{smallmatrix}
2 & 2\\
2 & 0\\
\overrightarrow{0} & \overrightarrow{0}
\end{smallmatrix}\right\rangle +\left(m-1\right)\left\langle \begin{smallmatrix}
2 & 0\\
0 & 2\\
2 & 0\\
\overrightarrow{0} & \overrightarrow{0}
\end{smallmatrix}\right\rangle \right)
\\&
=
m\left(p-m\right)\left(\frac{p+1+\left(m-1\right)\left(p+3\right)}{\left(p-1\right)p\left(p+2\right)\left(p+4\right)}
\right)
\\&
=\frac{m\left(p-m\right)\left(m\left(p+3\right)-2\right)}{\left(p-1\right)p\left(p+2\right)\left(p+4\right)}
\end{split}
\end{align}

\newpage

\begin{align}
\label{eq:ua4va4}
\begin{split}
&
\mathbb{E}\left[\left\Vert \mathbf{u}_{a}\right\Vert ^{4}\left\Vert \mathbf{v}_{a}\right\Vert ^{4}\right]
=
\mathbb{E}\left[\left\Vert \mathbf{u}_{a}\right\Vert ^{4}\left\Vert \mathbf{v}_{a}\right\Vert ^{4}\right]=\mathbb{E}_{\mathbf{u}\perp\mathbf{v}}\left[\left(\sum_{i=1}^{m}u_{i}^{2}\right)^{2}\left(\sum_{i=1}^{m}v_{i}^{2}\right)^{2}\right]
\\
&
=
\mathbb{E}\left[\left(\sum_{i=1}^{m}u_{i}^{4}+\sum_{i\neq j}u_{i}^{2}u_{j}^{2}\right)\left(\sum_{k=1}^{m}v_{k}^{4}+\sum_{k\neq\ell}v_{k}^{2}v_{\ell}^{2}\right)\right]
\\&
=\sum_{i=1}^{m}\sum_{k=1}^{m}\mathbb{E}\left[u_{i}^{4}v_{k}^{4}\right]+\underbrace{\sum_{i=1}^{m}\sum_{k\neq\ell}\mathbb{E}\left[u_{i}^{4}v_{k}^{2}v_{\ell}^{2}\right]+\sum_{k=1}^{m}\sum_{i\neq j}\mathbb{E}\left[u_{i}^{2}u_{j}^{2}v_{k}^{4}\right]}_{\text{same, due to the identical distributions (see \propref{prop:invariance})}}+\sum_{i\neq j}\sum_{k\neq\ell}\mathbb{E}\left[u_{i}^{2}u_{j}^{2}v_{k}^{2}v_{\ell}^{2}\right]
\\&
=\sum_{i=1}^{m}\sum_{k=1}^{m}\mathbb{E}\left[u_{i}^{4}v_{k}^{4}\right]+2\sum_{i=1}^{m}\sum_{k\neq\ell}\mathbb{E}_{\mathbf{u}\perp\mathbf{v}}\left[u_{i}^{4}v_{k}^{2}v_{\ell}^{2}\right]+\sum_{i\neq j}\sum_{k\neq\ell}\mathbb{E}_{\mathbf{u}\perp\mathbf{v}}\left[u_{i}^{2}u_{j}^{2}v_{k}^{2}v_{\ell}^{2}\right]
\\&
=m\sum_{k=1}^{m}\mathbb{E}\left[u_{1}^{4}v_{k}^{4}\right]+2m\sum_{k\neq\ell}\mathbb{E}\left[u_{1}^{4}v_{k}^{2}v_{\ell}^{2}\right]+m\left(m-1\right)\sum_{k\neq\ell}\mathbb{E}\left[u_{1}^{2}u_{2}^{2}v_{k}^{2}v_{\ell}^{2}\right]
\\&
=m\left(\underbrace{\mathbb{E}\!\left[u_{1}^{4}v_{1}^{4}\right]}_{k=1}+\underbrace{\left(m\!-\!1\right)\mathbb{E}\!\left[u_{1}^{4}v_{2}^{4}\right]}_{k\ge2}\right)
\!+\!
\!2m\!\left(\underbrace{2\left(m\!-\!1\right)\mathbb{E}\!\left[u_{1}^{4}v_{1}^{2}v_{2}^{2}\right]}_{k=1<\ell\,\vee\,\ell=1<k}
\!+\!
\underbrace{\left(m\!-\!1\right)\left(m\!-\!2\right)\mathbb{E}\!\left[u_{1}^{4}v_{2}^{2}v_{3}^{2}\right]}_{k\neq\ell\ge2}\right)
\\&
\eqmargin
+
m\left(m-1\right)
\!\left(\underbrace{2\mathbb{E}\left[u_{1}^{2}u_{2}^{2}v_{1}^{2}v_{2}^{2}\right]}_{k=1,\ell=2\,\vee\,\ell=1,k=2}
\!\!+\!\underbrace{4\left(m-2\right)\mathbb{E}\left[u_{1}^{2}u_{2}^{2}v_{2}^{2}v_{3}^{2}\right]}_{k\le2,\ell\ge3\,\vee\,\ell\le2,k\ge3}+\underbrace{\left(m-2\right)\left(m-3\right)\mathbb{E}\left[u_{1}^{2}u_{2}^{2}v_{3}^{2}v_{4}^{2}\right]}_{k\neq\ell\ge3}\right)
\\&
=m\left(\left\langle \begin{smallmatrix}
4 & 4\\
\overrightarrow{0} & \overrightarrow{0}
\end{smallmatrix}\right\rangle +\left(m-1\right)\left\langle \begin{smallmatrix}
4 & 0\\
0 & 4\\
\overrightarrow{0} & \overrightarrow{0}
\end{smallmatrix}\right\rangle \right)+2m\left(2\left(m-1\right)\left\langle \begin{smallmatrix}
4 & 2\\
0 & 2\\
\overrightarrow{0} & \overrightarrow{0}
\end{smallmatrix}\right\rangle +\left(m-1\right)\left(m-2\right)\left\langle \begin{smallmatrix}
4 & 0\\
0 & 2\\
0 & 2\\
\overrightarrow{0} & \overrightarrow{0}
\end{smallmatrix}\right\rangle \right)+
\\&
\eqmargin
m\left(m-1\right)\left(2\left\langle \begin{smallmatrix}
2 & 2\\
2 & 2\\
\overrightarrow{0} & \overrightarrow{0}
\end{smallmatrix}\right\rangle +4\left(m-2\right)\left\langle \begin{smallmatrix}
2 & 2\\
2 & 0\\
0 & 2\\
\overrightarrow{0} & \overrightarrow{0}
\end{smallmatrix}\right\rangle +\left(m-2\right)\left(m-3\right)\left\langle \begin{smallmatrix}
2 & 0\\
2 & 0\\
0 & 2\\
0 & 2\\
\overrightarrow{0} & \overrightarrow{0}
\end{smallmatrix}\right\rangle \right)
\\&
\\&
=\tfrac{9m}{p\left(p+2\right)\left(p+4\right)\left(p+6\right)}
\!+\!
\tfrac{9m\left(m-1\right)\left(p+3\right)\left(p+5\right)}{\left(p-1\right)p\left(p+1\right)\left(p+2\right)\left(p+4\right)\left(p+6\right)}
\!+\!
2m\tfrac{6\left(m-1\right)\left(p+1\right)\left(p+3\right) 
+
3\left(m-1\right)\left(m-2\right)\left(p+3\right)\left(p+5\right)
}{\left(p-1\right)p\left(p+1\right)\left(p+2\right)\left(p+4\right)\left(p+6\right)}
+
\\&
\eqmargin
m\left(m-1\right)\left(\tfrac{2\left(p^{2}+4p+15\right)+4\left(m-2\right)\left(p+3\right)^{2}}{\left(p-1\right)p\left(p+1\right)\left(p+2\right)\left(p+4\right)\left(p+6\right)}
+
\tfrac{\left(m-2\right)\left(m-3\right)\left(p+3\right)\left(p+5\right)}{\left(p-1\right)p\left(p+1\right)\left(p+2\right)\left(p+4\right)\left(p+6\right)}\right)
\\&
\\&
=\tfrac{9m\left(p-1\right)\left(p+1\right)+9m\left(m-1\right)\left(p+3\right)\left(p+5\right)+2m\left(6\left(m-1\right)\left(p+1\right)\left(p+3\right)+3\left(m-1\right)\left(m-2\right)\left(p+3\right)\left(p+5\right)\right)}{\left(p-1\right)p\left(p+1\right)\left(p+2\right)\left(p+4\right)\left(p+6\right)}
+
\\
&
\eqmargin
\tfrac{m\left(m-1\right)\left(2\left(p^{2}+4p+15\right)+4\left(m-2\right)\left(p+3\right)^{2}+\left(m-2\right)\left(m-3\right)\left(p+3\right)\left(p+5\right)\right)}{\left(p-1\right)p\left(p+1\right)\left(p+2\right)\left(p+4\right)\left(p+6\right)}
\\&
=\frac{m\left(m+2\right)\left(m^{2}\left(p+3\right)\left(p+5\right)+2m\left(p+1\right)\left(p+3\right)-8\left(2p+3\right)\right)}{\left(p-1\right)p\left(p+1\right)\left(p+2\right)\left(p+4\right)\left(p+6\right)}
\end{split}
\end{align}


\newpage

\begin{align} 
\label{eq:(ua)^2(va)^2}
\begin{split}
\mathbb{E}
\left[\left\Vert \mathbf{u}_{a}\right\Vert ^{2}\left\Vert \mathbf{v}_{a}\right\Vert ^{2}\right]
&=
\sum_{i,j=1}^{m}\mathbb{E}\left[u_{i}^{2}v_{j}^{2}\right]=m\sum_{j=1}^{m}\mathbb{E}\left[u_{1}^{2}v_{j}^{2}\right]
=
m\underbrace{\mathbb{E}\left[u_{1}^{2}v_{1}^{2}\right]}_{j=1}+m\underbrace{\left(m-1\right)\mathbb{E}\left[u_{1}^{2}v_{2}^{2}\right]}_{j\ge2}
\\&
=m\left\langle \begin{smallmatrix}
2 & 2\\
\overrightarrow{0} & \overrightarrow{0}
\end{smallmatrix}\right\rangle +m\left(m-1\right)\left\langle \begin{smallmatrix}
2 & 0\\
0 & 2\\
\overrightarrow{0} & \overrightarrow{0}
\end{smallmatrix}\right\rangle =\frac{m}{p\left(p+2\right)}+\frac{m\left(m-1\right)\left(p+1\right)}{\left(p-1\right)p\left(p+2\right)}
\\&
=\frac{m\left(mp+m-2\right)}{\left(p-1\right)p\left(p+2\right)}
\end{split}
\end{align}

\bigskip

\begin{align} 
\label{eq:(uaQmva)^2}
\begin{split}
\mathbb{E}_{\mathbf{u}\perp\mathbf{v},\mathbf{Q}_{m}}\left(\mathbf{u}_{a}^{\top}\mathbf{Q}_{m}\mathbf{v}_{a}\right)^{2}&=\mathbb{E}_{\mathbf{u}\perp\mathbf{v},\mathbf{r}\sim\mathcal{S}^{m-1}}\left(\left\Vert \mathbf{u}_{a}\right\Vert \mathbf{r}^{\top}\mathbf{v}_{a}\right)^{2}
\\&
=
\mathbb{E}_{\mathbf{u}\perp\mathbf{v}}\left[\left\Vert \mathbf{u}_{a}\right\Vert ^{2}\mathbb{E}_{\mathbf{r}\sim\mathcal{S}^{m-1}}\left(\mathbf{v}_{a}^{\top}\mathbf{r}\mathbf{r}^{\top}\mathbf{v}_{a}\right)\right]
\\&
=\frac{1}{m}\underbrace{\mathbb{E}_{\mathbf{u}\perp\mathbf{v}}\left[\left\Vert \mathbf{u}_{a}\right\Vert ^{2}\left\Vert \mathbf{v}_{a}\right\Vert ^{2}\right]}_{\text{solved in \eqref{eq:(ua)^2(va)^2}}}
=\frac{mp+m-2}{\left(p-1\right)p\left(p+2\right)}
\end{split}
\end{align}

\bigskip

\begin{align}
\label{eq:(uava)2}
\begin{split}
\mathbb{E}
\left(\mathbf{u}_{a}^{\top}\mathbf{v}_{a}\right)^{2}&=\mathbb{E}_{\mathbf{u}\perp\mathbf{v}}\left(-\mathbf{u}_{b}^{\top}\mathbf{v}_{b}\right)^{2}=\mathbb{E}_{\mathbf{u}\perp\mathbf{v}}\left(\mathbf{u}_{b}^{\top}\mathbf{v}_{b}\right)^{2}
=\sum_{i,j=1}^{m}\mathbb{E}\left[u_{i}u_{j}v_{i}v_{j}\right]
\\&
=m\sum_{j=1}^{m}\mathbb{E}\left[u_{1}u_{j}v_{1}v_{j}\right]
=m\left(\mathbb{E}\left[u_{1}^{2}v_{1}^{2}\right]+\left(m-1\right)\mathbb{E}\left[u_{1}u_{2}v_{1}v_{2}\right]\right)
\\&
=m\left(\left\langle \begin{smallmatrix}
2 & 2\\
\overrightarrow{0} & \overrightarrow{0}
\end{smallmatrix}\right\rangle +
\left(m-1\right)\left\langle \begin{smallmatrix}
1 & 1\\
1 & 1\\
\overrightarrow{0} & \overrightarrow{0}
\end{smallmatrix}\right\rangle \right)=m\left(\frac{1}{p\left(p+2\right)}+\frac{-1\cdot\left(m-1\right)}{\left(p-1\right)p\left(p+2\right)}\right)
\\&
=\frac{m\left(p-1+1-m\right)}{\left(p-1\right)p\left(p+2\right)}=\frac{m\left(p-m\right)}{\left(p-1\right)p\left(p+2\right)}    
\end{split}
\end{align}

\newpage

\begin{align} 
\label{eq:(uava)4}
\begin{split}
&\mathbb{E}_{\mathbf{u}\perp\mathbf{v}}\left(\mathbf{u}_{a}^{\top}\mathbf{v}_{a}\right)^{4}
=
\mathbb{E}_{\mathbf{u}\perp\mathbf{v}}\left(-\mathbf{u}_{b}^{\top}\mathbf{v}_{b}\right)^{4}
=
\mathbb{E}\left(\mathbf{u}_{a}^{\top}\mathbf{v}_{a}\right)^{2}\left(\mathbf{u}_{b}^{\top}\mathbf{v}_{b}\right)^{2}
\\
&
=
\mathbb{E}\left(\sum_{i=1}^{m}u_{i}v_{i}\right)^{2}\left(\sum_{k=m+1}^{p}u_{k}v_{k}\right)^{2}
\\&
=\sum_{i,j=1}^{m}\sum_{k,\ell=m+1}^{p}\mathbb{E}\left[u_{i}v_{i}u_{j}v_{j}u_{k}v_{k}u_{\ell}v_{\ell}\right]
=m\!\left(p\!-\!m\right)\!
\sum_{i=1}^{m}\sum_{k=m+1}^{p}\mathbb{E}\left[u_{1}v_{1}u_{i}v_{i}u_{k}v_{k}u_{p}v_{p}\right]
\\&
=m\!\left(p\!-\!m\right)\!
\bigg(\underbrace{\mathbb{E}\left[u_{1}^{2}v_{1}^{2}u_{p}^{2}v_{p}^{2}\right]}_{i=1,k=p}
+\underbrace{\left(p-m-1\right)\mathbb{E}\left[u_{1}^{2}v_{1}^{2}u_{p-1}v_{p-1}u_{p}v_{p}\right]}_{i=1,m+1\le k\le p-1}
+
\\
&
\hspace{2.1cm}
\underbrace{\left(m\!-\!1\right)
\mathbb{E}\!\left[u_{1}v_{1}u_{2}v_{2}u_{p}^{2}v_{p}^{2}\right]}_{i\ge2,k=p}
\!+\!
\underbrace{\left(p\!-\!m\!-\!1\right)\!\left(m\!-\!1\right)
\mathbb{E}\!\left[u_{1}v_{1}u_{2}v_{2}u_{p-1}v_{p-1}u_{p}v_{p}\right]}_{i\ge2,m+1\le k\le p-1}
\!
\bigg)
\\&
=m\!\left(p\!-\!m\right)\!\left(\!\left\langle \begin{smallmatrix}
2 & 2\\
2 & 2\\
\overrightarrow{0} & \overrightarrow{0}
\end{smallmatrix}\right\rangle \!+\!
\Big(\!\left(p\!-\!m\!-\!1\right)+\left(m\!-\!1\right)\!\Big)
\!
\left\langle \begin{smallmatrix}
2 & 2\\
1 & 1\\
1 & 1\\
\overrightarrow{0} & \overrightarrow{0}
\end{smallmatrix}\right\rangle 
\!+\!
\left(p\!-\!m\!-\!1\right)
\!
\left(m\!-\!1\right)
\!
\left\langle \begin{smallmatrix}
1 & 1\\
1 & 1\\
1 & 1\\
1 & 1\\
\overrightarrow{0} & \overrightarrow{0}
\end{smallmatrix}\right\rangle \right)
\\&
=m\left(p-m\right)\left(\left\langle \begin{smallmatrix}
2 & 2\\
2 & 2\\
\overrightarrow{0} & \overrightarrow{0}
\end{smallmatrix}\right\rangle +\left(p-2\right)\left\langle \begin{smallmatrix}
2 & 2\\
1 & 1\\
1 & 1\\
\overrightarrow{0} & \overrightarrow{0}
\end{smallmatrix}\right\rangle +\left(p-m-1\right)\left(m-1\right)\left\langle \begin{smallmatrix}
1 & 1\\
1 & 1\\
1 & 1\\
1 & 1\\
\overrightarrow{0} & \overrightarrow{0}
\end{smallmatrix}\right\rangle \right)
\\&
=m\left(p-m\right)\left(\frac{\left(p^{2}+4p+15\right)+\left(p-2\right)\left(-p+3\right)+3\left(p-m-1\right)\left(m-1\right)}{\left(p-1\right)p\left(p+1\right)\left(p+2\right)\left(p+4\right)\left(p+6\right)}\right)
\\&
=\frac{3m^{4}-6m^{3}p+3m^{2}p^{2}-6m^{2}p-12m^{2}+6mp^{2}+12mp}{\left(p-1\right)p\left(p+1\right)\left(p+2\right)\left(p+4\right)\left(p+6\right)}
\end{split}
\end{align}

\bigskip

\begin{align}
\begin{split}
\label{eq:(uaQmva)4}
\mathbb{E}_{\mathbf{u}\perp\mathbf{v}}\left(\mathbf{u}_{a}^{\top}\mathbf{Q}_{m}\mathbf{v}_{a}\right)^{4}&=\mathbb{E}_{\mathbf{u}\perp\mathbf{v},\mathbf{r}\sim\mathcal{S}^{m-1}}\left(\left\Vert \mathbf{u}_{a}\right\Vert \mathbf{r}^{\top}\mathbf{v}_{a}\right)^{4}=\mathbb{E}_{\mathbf{u}\perp\mathbf{v}}\left[\left\Vert \mathbf{u}_{a}\right\Vert ^{4}\mathbb{E}_{\mathbf{r}\sim\mathcal{S}^{m-1}}\left(\mathbf{r}^{\top}\mathbf{v}_{a}\right)^{4}\right]
\\
\explain{
\substack{
\text{reparameterize $\vect{r}\mapsto \A^{\top} \vr$}
\\
\text{where $\A\vv_{a} = \vect{e}_1$}
}
}
&=
\mathbb{E}_{\mathbf{u}\perp\mathbf{v}}\left[\left\Vert \mathbf{u}_{a}\right\Vert ^{4}\mathbb{E}_{\mathbf{r}\sim\mathcal{S}^{m-1}}\left(\left\Vert \mathbf{v}_{a}\right\Vert \mathbf{r}^{\top}\ve_{1}\right)^{4}\right]
\\&
=
\mathbb{E}_{\mathbf{u}\perp\mathbf{v}}\left[\left\Vert \mathbf{u}_{a}\right\Vert ^{4}\left\Vert \mathbf{v}_{a}\right\Vert ^{4}\mathbb{E}_{\mathbf{r}\sim\mathcal{S}^{m-1}}r_{1}^{4}\right]
=\left\langle 
\begin{smallmatrix}
4\\
\overrightarrow{0}
\end{smallmatrix}\right\rangle _{m}\mathbb{E}\left[\left\Vert \mathbf{u}_{a}\right\Vert ^{4}\left\Vert \mathbf{v}_{a}\right\Vert ^{4}\right]
\\
&
=
\tfrac{3}{m\left(m+2\right)}\cdot\tfrac{m\left(m+2\right)\left(m^{2}\left(p+3\right)\left(p+5\right)+2m\left(p+1\right)\left(p+3\right)-8\left(2p+3\right)\right)}{\left(p-1\right)p\left(p+1\right)\left(p+2\right)\left(p+4\right)\left(p+6\right)}
\\
&
=\frac{3\left(m^{2}\left(p+3\right)\left(p+5\right)+2m\left(p+1\right)\left(p+3\right)-8\left(2p+3\right)\right)}{\left(p-1\right)p\left(p+1\right)\left(p+2\right)\left(p+4\right)\left(p+6\right)}\,.
\end{split}
\end{align}

\newpage

\begin{align}
\label{eq:(uaQmva)2(uv)2}
\begin{split}
&
\mathbb{E}\left(\mathbf{u}_{a}^{\top}\mathbf{Q}_{m}\mathbf{v}_{a}\right)^{2}\left(\mathbf{u}_{b}^{\top}\mathbf{v}_{b}\right)^{2}
=
\mathbb{E}\left(\mathbf{u}_{a}^{\top}\mathbf{Q}_{m}\mathbf{v}_{a}\right)^{2}\left(\mathbf{u}_{a}^{\top}\mathbf{v}_{a}\right)^{2}
\\&
=
\mathbb{E}_{\mathbf{u}\perp\mathbf{v}}\left[\left(\mathbf{u}_{a}^{\top}\mathbf{v}_{a}\right)^{2}\mathbb{E}_{\mathbf{r}\sim\mathcal{S}^{p-1}}\left(\left\Vert \mathbf{u}_{a}\right\Vert \mathbf{r}^{\top}\mathbf{v}_{a}\right)^{2}\right]
\\&
=
\frac{1}{m}\mathbb{E}_{\mathbf{u}\perp\mathbf{v}}\left[\left(\mathbf{u}_{a}^{\top}\mathbf{v}_{a}\right)^{2}\left\Vert \mathbf{u}_{a}\right\Vert ^{2}\left\Vert \mathbf{v}_{a}\right\Vert ^{2}\right]
=
\frac{1}{m}\mathbb{E}_{\mathbf{u}\perp\mathbf{v}}\left[\left(-\mathbf{u}_{b}^{\top}\mathbf{v}_{b}\right)^{2}\left\Vert \mathbf{u}_{a}\right\Vert ^{2}\left\Vert \mathbf{v}_{a}\right\Vert ^{2}\right]
\\&
=\frac{1}{m}\mathbb{E}\left[\left\Vert \mathbf{u}_{a}\right\Vert ^{2}\left\Vert \mathbf{v}_{a}\right\Vert ^{2}\left(\mathbf{u}_{b}^{\top}\mathbf{v}_{b}\right)^{2}\right]
=
\frac{1}{m}\sum_{i,j=1}^{m}\sum_{k,\ell=m+1}^{p}\mathbb{E}\left[u_{i}^{2}v_{j}^{2}u_{k}v_{k}u_{\ell}v_{\ell}\right]
\\&
=
\sum_{j=1}^{m}\sum_{k,\ell=m+1}^{p}\mathbb{E}\left[u_{1}^{2}v_{j}^{2}u_{k}v_{k}u_{\ell}v_{\ell}\right]
=
\left(p-m\right)\sum_{j=1}^{m}\sum_{k=m+1}^{p}\mathbb{E}\left[u_{1}^{2}v_{j}^{2}u_{k}v_{k}u_{p}v_{p}\right]
\\&
=\left(p-m\right)
\bigg(\underbrace{\mathbb{E}\left[u_{1}^{2}v_{1}^{2}u_{p}^{2}v_{p}^{2}\right]}_{j=1,k=p}+\underbrace{\left(p-m-1\right)\mathbb{E}\left[u_{1}^{2}v_{1}^{2}u_{p-1}v_{p-1}u_{p}v_{p}\right]}_{j=1,m+1\le k\le p-1}+
\\
&
\hspace{5.4em}
\underbrace{\left(m-1\right)\mathbb{E}\left[u_{1}^{2}v_{2}^{2}u_{p}^{2}v_{p}^{2}\right]}_{j\ge2,k=p}+
\underbrace{\left(p-m-1\right)\left(m-1\right)\mathbb{E}\left[u_{1}^{2}v_{2}^{2}u_{p-1}v_{p-1}u_{p}v_{p}\right]}_{j\ge2,m+1\le k\le p-1}\bigg)
\\&
=\left(p-m\right)\bigg(\left\langle \begin{smallmatrix}
2 & 2\\
2 & 2\\
\overrightarrow{0} & \overrightarrow{0}
\end{smallmatrix}\right\rangle 
+
\left(p-m-1\right)\left\langle \begin{smallmatrix}
2 & 2\\
1 & 1\\
1 & 1\\
\overrightarrow{0} & \overrightarrow{0}
\end{smallmatrix}\right\rangle +\left(m-1\right)\left\langle \begin{smallmatrix}
2 & 0\\
0 & 2\\
2 & 2\\
\overrightarrow{0} & \overrightarrow{0}
\end{smallmatrix}\right\rangle 
+
\\
&
\hspace{5.4em}
\left(p-m-1\right)\left(m-1\right)\left\langle \begin{smallmatrix}
2 & 0\\
0 & 2\\
1 & 1\\
1 & 1\\
\overrightarrow{0} & \overrightarrow{0}
\end{smallmatrix}\right\rangle
\bigg)
\\&
=\left(p-m\right)\left(\tfrac{p^{2}+4p+15+\left(p-m-1\right)\left(-p+3\right)+\left(m-1\right)\left(p+3\right)^{2}+\left(m-1\right)\left(p-m-1\right)\left(-p-3\right)}{\left(p-1\right)p\left(p+1\right)\left(p+2\right)\left(p+4\right)\left(p+6\right)}\right)
\\&
=\frac{\left(m+2\right)\left(mp+2p+3m\right)\left(p-m\right)}{\left(p-1\right)p\left(p+1\right)\left(p+2\right)\left(p+4\right)\left(p+6\right)}
\end{split}
\end{align}

\bigskip

\begin{align}
\label{eq:(ub)^2(vaQmua)(uaQmva)}
\begin{split}
&\mathbb{E}\left[\left\Vert \mathbf{u}_{b}\right\Vert ^{2}\mathbf{v}_{a}^{\top}\mathbf{Q}_{m}\mathbf{u}_{a}\mathbf{u}_{a}^{\top}\mathbf{Q}_{m}\mathbf{v}_{a}\right]
=
\mathbb{E}\left[\left\Vert \mathbf{u}_{b}\right\Vert ^{2}\left(\sum_{i,j=1}^{m}v_{i}q_{ij}u_{j}\right)\left(\sum_{k,\ell=1}^{m}u_{k}q_{k\ell}v_{\ell}\right)\right]
\\&
=\sum_{i,j=1}^{m}\sum_{k,\ell=1}^{m}\mathbb{E}\left[\left\Vert \mathbf{u}_{b}\right\Vert ^{2}v_{i}u_{j}u_{k}v_{\ell}\right]\mathbb{E}\left[q_{ij}q_{k\ell}\right]
\\
\explain{\text{\propref{prop:odd}}}
&=\sum_{i,j=1}^{m}\mathbb{E}\left[\left\Vert \mathbf{u}_{b}\right\Vert ^{2}u_{i}u_{j}v_{i}v_{j}\right]\mathbb{E}\left[q_{ij}^{2}\right]
\\&
=\frac{1}{m}\left(\underbrace{m\mathbb{E}\left[\left\Vert \mathbf{u}_{b}\right\Vert ^{2}u_{1}^{2}v_{1}^{2}\right]}_{i=j}+\underbrace{m\left(m-1\right)\mathbb{E}\left[\left\Vert \mathbf{u}_{b}\right\Vert ^{2}u_{1}u_{2}v_{1}v_{2}\right]}_{i\neq j}\right)
\\&
=\mathbb{E}\left[\left(\sum_{k=m+1}^{p}u_{k}^{2}\right)u_{1}^{2}v_{1}^{2}\right]+\left(m-1\right)\mathbb{E}\left[\left(\sum_{k=m+1}^{p}u_{k}^{2}\right)u_{1}u_{2}v_{1}v_{2}\right]
\\&
=\left(p-m\right)\mathbb{E}\left[u_{p}^{2}u_{1}^{2}v_{1}^{2}\right]+\left(m-1\right)\left(p-m\right)\mathbb{E}\left[u_{p}^{2}u_{1}u_{2}v_{1}v_{2}\right]
\\&
=\left(p-m\right)\left\langle \begin{smallmatrix}
2 & 2\\
2 & 0\\
\overrightarrow{0} & \overrightarrow{0}
\end{smallmatrix}\right\rangle +\left(m-1\right)\left(p-m\right)\left\langle \begin{smallmatrix}
1 & 1\\
1 & 1\\
2 & 0\\
\overrightarrow{0} & \overrightarrow{0}
\end{smallmatrix}\right\rangle 
\\&
=\frac{\left(p-m\right)\left(p+1\right)
-
\left(m-1\right)\left(p-m\right)}{\left(p-1\right)p\left(p+2\right)\left(p+4\right)}
%
=\frac{\left(p-m\right)\left(p-m+2\right)}{\left(p-1\right)p\left(p+2\right)\left(p+4\right)}
\end{split}
\end{align}


\begin{align}
\label{eq:(ub)4(uaQmva)2}
\begin{split}
&\mathbb{E}\left[\left\Vert \mathbf{u}_{b}\right\Vert ^{4}\left(\mathbf{u}_{a}^{\top}\mathbf{Q}_{m}\mathbf{v}_{a}\right)^{2}\right]=\mathbb{E}_{\mathbf{u}\perp\mathbf{v}}\left[\left\Vert \mathbf{u}_{b}\right\Vert ^{4}\mathbf{v}_{a}^{\top}\mathbb{E}_{\mathbf{Q}_{m}}\left(\mathbf{Q}_{m}\mathbf{u}_{a}\mathbf{u}_{a}^{\top}\mathbf{Q}_{m}\right)\mathbf{v}_{a}\right]
\\&
=\mathbb{E}_{\mathbf{u}\perp\mathbf{v}}\left[\left\Vert \mathbf{u}_{b}\right\Vert ^{4}\mathbf{v}_{a}^{\top}\mathbb{E}_{\mathbf{r}\sim\mathcal{S}^{m-1}}\left(\left\Vert \mathbf{u}_{a}\right\Vert ^{2}\mathbf{r}\mathbf{r}^{\top}\right)\mathbf{v}_{a}\right]=\frac{1}{m}\mathbb{E}_{\mathbf{u}\perp\mathbf{v}}\left[\left\Vert \mathbf{u}_{a}\right\Vert ^{2}\left\Vert \mathbf{u}_{b}\right\Vert ^{4}\mathbf{v}_{a}^{\top}\mathbf{v}_{a}\right]
\\&
=\frac{1}{m}\mathbb{E}\left[\left\Vert \mathbf{u}_{b}\right\Vert ^{4}\left\Vert \mathbf{u}_{a}\right\Vert ^{2}\left\Vert \mathbf{v}_{a}\right\Vert ^{2}\right]=\frac{1}{m}\mathbb{E}\left[\left(\sum_{i=m+1}^{p}u_{i}^{2}\right)\left(\sum_{j=m+1}^{p}u_{j}^{2}\right)\left(\sum_{k=1}^{m}u_{k}^{2}\right)\left(\sum_{\ell=1}^{m}v_{\ell}^{2}\right)\right]
\\&
=\frac{p-m}{m}\mathbb{E}\left[u_{p}^{2}\left(
\sum_{j=m+1}^{p}
\!
u_{j}^{2}\right)\left(\sum_{k=1}^{m}u_{k}^{2}\right)\left(\sum_{\ell=1}^{m}v_{\ell}^{2}\right)\right]
\\&
=
\frac{p-m}{m}\sum_{j=m+1}^{p}\mathbb{E}\left[u_{p}^{2}u_{j}^{2}\left(\sum_{k=1}^{m}u_{k}^{2}\right)\left(\sum_{i=1}^{m}v_{\ell}^{2}\right)\right]
=
\left(p-m\right)\!
\sum_{j=m+1}^{p}
\!
\mathbb{E}\left[u_{1}^{2}u_{p}^{2}u_{j}^{2}\left(\sum_{i=1}^{m}v_{\ell}^{2}\right)\right]
\\&
=\left(p-m\right)\left(\underbrace{\mathbb{E}\left[u_{1}^{2}u_{p}^{4}\left(\sum_{i=1}^{m}v_{\ell}^{2}\right)\right]}_{j=p}
+
\underbrace{\left(p-m-1\right)\mathbb{E}\left[u_{1}^{2}u_{p-1}^{2}u_{p}^{2}\left(\sum_{i=1}^{m}v_{\ell}^{2}\right)\right]}_{m+1\le j\le p-1}\right)
\\&
=\left(p-m\right)
\bigg(
\bigprn{\underbrace{\mathbb{E}\left[u_{1}^{2}u_{p}^{4}v_{1}^{2}\right]}_{\ell=1}+\underbrace{\left(m-1\right)\mathbb{E}\left[u_{1}^{2}u_{p}^{4}v_{2}^{2}\right]}_{2\le\ell\le m}}
+
\\&
\eqmargin
\left(p-m-1\right)
\bigprn{\underbrace{\mathbb{E}\left[u_{1}^{2}u_{p-1}^{2}u_{p}^{2}v_{1}^{2}\right]}_{\ell=1}
+
\underbrace{\left(m-1\right)\mathbb{E}\left[u_{1}^{2}u_{p-1}^{2}u_{p}^{2}v_{2}^{2}\right]}_{2\le\ell\le m}
}
\bigg)
\\&
=\left(p-m\right)
\bigg(
\mathbb{E}\left[u_{1}^{2}u_{2}^{4}v_{1}^{2}\right]+\left(m-1\right)\mathbb{E}\left[u_{1}^{2}u_{2}^{4}v_{3}^{2}\right]
+
\\
&
\eqmargin
\left(p-m-1\right)\left(\mathbb{E}\left[u_{1}^{2}u_{2}^{2}u_{3}^{2}v_{1}^{2}\right]+\left(m-1\right)\mathbb{E}\left[u_{1}^{2}u_{2}^{2}u_{3}^{2}v_{4}^{2}\right]\right)
\bigg)
\\&
=\left(p-m\right)
\Bigg(\left\langle \begin{smallmatrix}
2 & 2\\
4 & 0\\
\overrightarrow{0} & \overrightarrow{0}
\end{smallmatrix}\right\rangle +
\left(m-1\right)\left\langle \begin{smallmatrix}
2 & 0\\
4 & 0\\
0 & 2\\
\overrightarrow{0} & \overrightarrow{0}
\end{smallmatrix}\right\rangle 
+
\\
&\hspace{2cm}
\left(p-m-1\right)\bigg(\left\langle \begin{smallmatrix}
2 & 2\\
2 & 0\\
2 & 0\\
\overrightarrow{0} & \overrightarrow{0}
\end{smallmatrix}\right\rangle +
\left(m-1\right)\left\langle \begin{smallmatrix}
2 & 0\\
2 & 0\\
2 & 0\\
0 & 2\\
\overrightarrow{0} & \overrightarrow{0}
\end{smallmatrix}\right\rangle \bigg)\Bigg)
\\&
=\left(p-m\right)\left(\left(\tfrac{3\left(p+3\right)+3\left(m-1\right)\left(p+5\right)}{\left(p-1\right)p\left(p+2\right)\left(p+4\right)\left(p+6\right)}\right)
+
\left(p-m-1\right)\left(\tfrac{p+3+\left(m-1\right)\left(p+5\right)}{\left(p-1\right)p\left(p+2\right)\left(p+4\right)\left(p+6\right)}\right)\right)
\\&
=\frac{\left(p-m\right)\left(mp^{2}-m^{2}p-5m^{2}+7mp+12m-2p-4\right)}{\left(p-1\right)p\left(p+2\right)\left(p+4\right)\left(p+6\right)}
\end{split}
\end{align}

\newpage

\begin{align}
\label{eq:(ua)^2(uava)(ubvb)(vb)^2}
\begin{split}
&\mathbb{E}\left[\left\Vert \mathbf{u}_{a}\right\Vert ^{2}\mathbf{u}_{a}^{\top}\mathbf{v}_{a}\cdot\mathbf{v}_{b}^{\top}\mathbf{u}_{b}\left\Vert \mathbf{v}_{b}\right\Vert ^{2}\right]
=\sum_{i=1}^{m}\sum_{j=1}^{m}\sum_{k=m+1}^{p}\sum_{\ell=m+1}^{p}\mathbb{E}\left[u_{i}^{2}\cdot u_{j}v_{j}\cdot u_{k}v_{k}\cdot v_{\ell}^{2}\right]
\\&
=
m\left(p-m\right)\sum_{j=1}^{m}\sum_{k=m+1}^{p}\mathbb{E}\left[u_{1}^{2}\cdot u_{j}v_{j}\cdot u_{k}v_{k}\cdot v_{p}^{2}\right]
\\&
=m\left(p-m\right)\left(\sum_{k=m+1}^{p}\mathbb{E}\left[u_{1}^{3}v_{1}\cdot u_{k}v_{k}\cdot v_{p}^{2}\right]+\left(m-1\right)\sum_{k=m+1}^{p}\mathbb{E}\left[u_{1}^{2}\cdot u_{2}v_{2}\cdot u_{k}v_{k}\cdot v_{p}^{2}\right]\right)
\\&
=m\left(p-m\right)\Bigg(\mathbb{E}\left[u_{1}^{3}v_{1}u_{p}v_{p}^{3}\right]+\left(p-m-1\right)\mathbb{E}\left[u_{1}^{3}v_{1}u_{p-1}v_{p-1}v_{p}^{2}\right]
+
\\
& \eqmargin
\hspace{2cm}
\left(m-1\right)
\bigg(\mathbb{E}\left[u_{1}^{2}u_{2}v_{2}u_{p}v_{p}^{3}\right]
+\left(p-m-1\right)
\mathbb{E}\left[u_{1}^{2}u_{2}v_{2}u_{p-1}v_{p-1}v_{p}^{2}\right]\bigg)\Bigg)
\\&
=
m\left(p-m\right)\Bigg(\left\langle \begin{smallmatrix}
3 & 1\\
1 & 3\\
\overrightarrow{0} & \overrightarrow{0}
\end{smallmatrix}\right\rangle +
\left(p-m-1\right)\left\langle \begin{smallmatrix}
3 & 1\\
1 & 1\\
0 & 2\\
\overrightarrow{0} & \overrightarrow{0}
\end{smallmatrix}\right\rangle 
+
\\
& \eqmargin
\hspace{2cm}
\left(m-1\right)\left(\left\langle \begin{smallmatrix}
2 & 0\\
1 & 1\\
1 & 3\\
\overrightarrow{0} & \overrightarrow{0}
\end{smallmatrix}\right\rangle +\left(p-m-1\right)\left\langle \begin{smallmatrix}
2 & 0\\
1 & 1\\
1 & 1\\
0 & 2\\
\overrightarrow{0} & \overrightarrow{0}
\end{smallmatrix}\right\rangle \right)\Bigg)
\\&
=m\left(p-m\right)\left(\left\langle \begin{smallmatrix}
3 & 1\\
1 & 3\\
\overrightarrow{0} & \overrightarrow{0}
\end{smallmatrix}\right\rangle +\left(p-2\right)\left\langle \begin{smallmatrix}
3 & 1\\
1 & 1\\
0 & 2\\
\overrightarrow{0} & \overrightarrow{0}
\end{smallmatrix}\right\rangle +\left(m-1\right)\left(p-m-1\right)\left\langle \begin{smallmatrix}
2 & 0\\
1 & 1\\
1 & 1\\
0 & 2\\
\overrightarrow{0} & \overrightarrow{0}
\end{smallmatrix}\right\rangle \right)
\\&
=m\left(p-m\right)\Bigg(\tfrac{-9\left(p+3\right)}{\left(p-1\right)p\left(p+1\right)\left(p+2\right)\left(p+4\right)\left(p+6\right)}
\\
& \eqmargin
+
\tfrac{-3\left(p-2\right)\left(p+3\right)}{\left(p-1\right)p\left(p+1\right)\left(p+2\right)\left(p+4\right)\left(p+6\right)}
+
\tfrac{\left(m-1\right)\left(p-m-1\right)\left(-p-3\right)}{\left(p-1\right)p\left(p+1\right)\left(p+2\right)\left(p+4\right)\left(p+6\right)}\Bigg)
\\&
=\frac{m\left(p-m\right)\left(p+3\right)\left(m^{2}-mp-2\left(p+2\right)\right)}{\left(p-1\right)p\left(p+1\right)\left(p+2\right)\left(p+4\right)\left(p+6\right)}
\end{split}
\end{align}

\newpage

\begin{align}
\label{eq:(uaQmua)(vaQmva)(vbub)^2}
\begin{split}
    &\mathbb{E}\left[\mathbf{u}_{a}^{\top}\mathbf{Q}_{m}\mathbf{u}_{a}\mathbf{v}_{a}^{\top}\mathbf{Q}_{m}\mathbf{v}_{a}\left(\mathbf{v}_{b}^{\top}\mathbf{u}_{b}\right)^{2}\right]
\\&
=\mathbb{E}\left[\left(\mathbf{v}_{b}^{\top}\mathbf{u}_{b}\right)^{2}\sum_{i,j,k,\ell}u_{i}q_{ij}u_{j}v_{k}q_{k\ell}v_{\ell}\right]=\sum_{i,j,k,\ell=1}^{m}\mathbb{E}\left[\left(\mathbf{v}_{b}^{\top}\mathbf{u}_{b}\right)^{2}u_{i}u_{j}v_{k}v_{\ell}\right]\mathbb{E}\left[q_{ij}q_{k\ell}\right]
\\
&\eqmargin
\left[{\text{By \propref{prop:odd}, most summands are zero}}\right]
\\
&=\sum_{i,j=1}^{m}\mathbb{E}\left[\left(\mathbf{v}_{b}^{\top}\mathbf{u}_{b}\right)^{2}u_{i}u_{j}v_{i}v_{j}\right]\mathbb{E}\left[q_{ij}^{2}\right]=\frac{1}{m}\sum_{i,j=1}^{m}\mathbb{E}\left[\left(\sum_{k=m+1}^{p}u_{k}v_{k}\right)^{2}u_{i}u_{j}v_{i}v_{j}\right]
\\
&=\frac{1}{m}\sum_{i,j=1}^{m}\sum_{k,\ell=m+1}^{p}
\!\!
\mathbb{E}\left[u_{i}u_{j}u_{k}u_{\ell}v_{i}v_{j}v_{k}v_{\ell}\right]
=\frac{m\left(p-m\right)}{m}\sum_{j=1}^{m}
\sum_{\ell=m+1}^{p}
\!\!
\mathbb{E}\left[u_{1}u_{j}u_{p}u_{\ell}v_{1}v_{j}v_{p}v_{\ell}\right]
\\&
=
\left(p-m\right)\sum_{j=1}^{m}\sum_{\ell=m+1}^{p}\mathbb{E}\left[u_{1}u_{j}u_{p}u_{\ell}v_{1}v_{j}v_{p}v_{\ell}\right]
\\&
=\left(p-m\right)\left(\sum_{\ell=m+1}^{p}\mathbb{E}\left[u_{1}^{2}u_{p}u_{\ell}v_{1}^{2}v_{p}v_{\ell}\right]+\left(m-1\right)\sum_{\ell=m+1}^{p}\mathbb{E}\left[u_{1}u_{2}u_{p}u_{\ell}v_{1}v_{2}v_{p}v_{\ell}\right]\right)
\\&
=\left(p-m\right)\Bigg(\mathbb{E}\left[u_{1}^{2}u_{p}^{2}v_{1}^{2}v_{p}^{2}\right]+\left(p-m-1\right)\mathbb{E}\left[u_{1}^{2}u_{p-1}u_{p}v_{1}^{2}v_{p-1}v_{p}\right]
+
\\
& \eqmargin
\hspace{1.7cm}
\left(m-1\right)\!
\bigg(\mathbb{E}\!\left[u_{1}u_{2}u_{p}^{2}v_{1}v_{2}v_{p}^{2}\right]
+
\left(p\!-\!m\!-\!1\right)\mathbb{E}\left[u_{1}u_{2}u_{p-1}u_{p}v_{1}v_{2}v_{p-1}v_{p}\right]\bigg)\Bigg)
\\&
=\left(p-m\right)\Bigg(\left\langle \begin{smallmatrix}
2 & 2\\
2 & 2\\
\overrightarrow{0} & \overrightarrow{0}
\end{smallmatrix}\right\rangle +
\left(p-m-1\right)\left\langle \begin{smallmatrix}
2 & 2\\
1 & 1\\
1 & 1\\
\overrightarrow{0} & \overrightarrow{0}
\end{smallmatrix}\right\rangle 
+
\\
& \eqmargin
\hspace{1.7cm}
\left(m-1\right)\left(\left\langle \begin{smallmatrix}
1 & 1\\
1 & 1\\
2 & 2\\
\overrightarrow{0} & \overrightarrow{0}
\end{smallmatrix}\right\rangle +\left(p-m-1\right)\left\langle \begin{smallmatrix}
1 & 1\\
1 & 1\\
1 & 1\\
1 & 1\\
\overrightarrow{0} & \overrightarrow{0}
\end{smallmatrix}\right\rangle \right)\Bigg)
\\&
=
\left(p-m\right)\left(\left\langle \begin{smallmatrix}
2 & 2\\
2 & 2\\
\overrightarrow{0} & \overrightarrow{0}
\end{smallmatrix}\right\rangle +\left(p-2\right)\left\langle \begin{smallmatrix}
2 & 2\\
1 & 1\\
1 & 1\\
\overrightarrow{0} & \overrightarrow{0}
\end{smallmatrix}\right\rangle +\left(m-1\right)\left(p-m-1\right)\left\langle \begin{smallmatrix}
1 & 1\\
1 & 1\\
1 & 1\\
1 & 1\\
\overrightarrow{0} & \overrightarrow{0}
\end{smallmatrix}\right\rangle \right)
\\&
=
\left(p-m\right)\left(\tfrac{p^{2}+4p+15}{\left(p-1\right)p\left(p+1\right)\left(p+2\right)\left(p+4\right)\left(p+6\right)}+
\tfrac{\left(p-2\right)\left(-p+3\right)
+
\left(m-1\right)\left(p-m-1\right)\cdot 3
}{\left(p-1\right)p\left(p+1\right)\left(p+2\right)\left(p+4\right)\left(p+6\right)}\right)
%
%
\\&
=\frac{3\left(p-m\right)\left(-m^{2}+mp+2p+4\right)}{\left(p-1\right)p\left(p+1\right)\left(p+2\right)\left(p+4\right)\left(p+6\right)}
\end{split}
\end{align}

\bigskip

\begin{align}
\label{eq:(ub)^2(vb)^2(vaua)^2}
\begin{split}
\mathbb{E}\left[\mathbf{u}_{b}^{\top}\mathbf{u}_{b}\left(\mathbf{v}_{a}^{\top}\mathbf{u}_{a}\right)^{2}\mathbf{v}_{b}^{\top}\mathbf{v}_{b}\right]&=\mathbb{E}\left[\left\Vert \mathbf{u}_{b}\right\Vert ^{2}\left(\mathbf{v}_{a}^{\top}\mathbf{u}_{a}\right)^{2}\left\Vert \mathbf{v}_{b}\right\Vert ^{2}\right]
\\&
=\mathbb{E}\left[\left(1-\left\Vert \mathbf{u}_{a}\right\Vert ^{2}\right)\left(\mathbf{v}_{a}^{\top}\mathbf{u}_{a}\right)^{2}\left\Vert \mathbf{v}_{b}\right\Vert ^{2}\right]
\\&
=\mathbb{E}\left[\left\Vert \mathbf{v}_{b}\right\Vert ^{2}\left(\mathbf{v}_{a}^{\top}\mathbf{u}_{a}\right)^{2}\right]-\mathbb{E}\left[\left\Vert \mathbf{u}_{a}\right\Vert ^{2}\left\Vert \mathbf{v}_{b}\right\Vert ^{2}\left(\mathbf{v}_{a}^{\top}\mathbf{u}_{a}\right)^{2}\right]
\\&
=\underbrace{\mathbb{E}\left[\left\Vert \mathbf{v}_{b}\right\Vert ^{2}\left(\mathbf{v}_{a}^{\top}\mathbf{u}_{a}\right)^{2}\right]}_{\text{solved in \eqref{eq:ub2(uava)2}}}
+
\underbrace{\mathbb{E}\left[\left\Vert \mathbf{u}_{a}\right\Vert ^{2}\left\Vert \mathbf{v}_{b}\right\Vert ^{2}\cdot\mathbf{u}_{a}^{\top}\mathbf{v}_{a}\cdot\mathbf{v}_{b}^{\top}\mathbf{u}_{b}\right]}_{\text{solved in \eqref{eq:(ua)^2(uava)(ubvb)(vb)^2}}}
\\&
=\tfrac{\left(p-m\right)m\left(p-m+2\right)}{\left(p-1\right)p\left(p+2\right)\left(p+4\right)}+\tfrac{m\left(p-m\right)\left(p+3\right)\left(m^{2}-mp-2\left(p+2\right)\right)}{\left(p-1\right)p\left(p+1\right)\left(p+2\right)\left(p+4\right)\left(p+6\right)}
\\&
=\tfrac{\left(p-m\right)m\left(m^{2}\left(p+3\right)-2m\left(p^{2}+5p+3\right)+p\left(p^{2}+7p+10\right)\right)}{\left(p-1\right)p\left(p+1\right)\left(p+2\right)\left(p+4\right)\left(p+6\right)}
\end{split}
\end{align}

\bigskip

\begin{align}
\label{eq:(ub)^2(vb)^2(uaQmva)(vaQmua)}
\begin{split}
&\mathbb{E}\left[\mathbf{u}_{b}^{\top}\mathbf{u}_{b}\mathbf{u}_{a}^{\top}\mathbf{Q}_{m}\mathbf{v}_{a}\mathbf{v}_{a}^{\top}\mathbf{Q}_{m}\mathbf{u}_{a}\mathbf{v}_{b}^{\top}\mathbf{v}_{b}\right]=\sum_{i,j=1}^{m}\sum_{k,\ell=1}^{m}\mathbb{E}\left[\mathbf{u}_{b}^{\top}\mathbf{u}_{b}\mathbf{v}_{b}^{\top}\mathbf{v}_{b}u_{i}q_{ij}v_{j}v_{k}q_{k,\ell}u_{\ell}\right]
\\&
=\sum_{i,j=1}^{m}\sum_{k,\ell=1}^{m}\mathbb{E}\left[\mathbf{u}_{b}^{\top}\mathbf{u}_{b}\mathbf{v}_{b}^{\top}\mathbf{v}_{b}\cdot u_{i}v_{j}v_{k}u_{\ell}\right]\mathbb{E}\left[q_{ij}q_{k,\ell}\right]
\\
\explain{\text{\propref{prop:odd}}}
&=\sum_{i,j=1}^{m}\mathbb{E}\left[\left\Vert \mathbf{u}_{b}\right\Vert ^{2}\left\Vert \mathbf{v}_{b}\right\Vert ^{2}\cdot u_{i}u_{j}v_{i}v_{j}\right]\underbrace{\mathbb{E}\left[q_{ij}^{2}\right]}_{=1/m}
\\&
=\frac{1}{m}\left(\underbrace{m\mathbb{E}\left[\left\Vert \mathbf{u}_{b}\right\Vert ^{2}\left\Vert \mathbf{v}_{b}\right\Vert ^{2}\cdot u_{1}^{2}v_{1}^{2}\right]}_{i=j}+\underbrace{m\left(m-1\right)\mathbb{E}\left[\left\Vert \mathbf{u}_{b}\right\Vert ^{2}\left\Vert \mathbf{v}_{b}\right\Vert ^{2}\cdot u_{1}u_{2}v_{1}v_{2}\right]}_{i\neq j}\right)
\\&
=\sum_{i=m+1}^{p}\sum_{j=m+1}^{p}\mathbb{E}\left[u_{1}^{2}v_{1}^{2}\cdot u_{i}^{2}v_{j}^{2}\right]+\left(m-1\right)\sum_{i=m+1}^{p}\sum_{j=m+1}^{p}\mathbb{E}\left[u_{1}u_{2}v_{1}v_{2}\cdot u_{i}^{2}v_{j}^{2}\right]
\\&
=\left(p-m\right)\left(\sum_{j=m+1}^{p}\mathbb{E}\left[u_{1}^{2}v_{1}^{2}\cdot u_{p}^{2}v_{j}^{2}\right]+\left(m-1\right)\sum_{j=m+1}^{p}\mathbb{E}\left[u_{1}u_{2}v_{1}v_{2}\cdot u_{p}^{2}v_{j}^{2}\right]\right)
\\&
=\left(p-m\right)\bigg(\underbrace{\mathbb{E}\left[u_{1}^{2}v_{1}^{2}\cdot u_{p}^{2}v_{p}^{2}\right]}_{j=p}+\underbrace{\left(p-m-1\right)\mathbb{E}\left[u_{1}^{2}v_{1}^{2}\cdot u_{p}^{2}v_{p-1}^{2}\right]}_{m+1\le j\le p-1}+
\\
& \eqmargin
\left(m-1\right)
\bigg(\underbrace{\mathbb{E}\left[u_{1}u_{2}v_{1}v_{2}\cdot u_{p}^{2}v_{p}^{2}\right]}_{j=p}+\underbrace{\left(p-m-1\right)\mathbb{E}\left[u_{1}u_{2}v_{1}v_{2}\cdot u_{p}^{2}v_{p-1}^{2}\right]}_{m+1\le j\le p-1}\bigg)\bigg)
\\&
=
\left(p-m\right)\Bigg(\left\langle \begin{smallmatrix}
2 & 2\\
2 & 2\\
\overrightarrow{0} & \overrightarrow{0}
\end{smallmatrix}\right\rangle 
+
\left(p-m-1\right)\left\langle \begin{smallmatrix}
2 & 2\\
0 & 2\\
2 & 0\\
\overrightarrow{0} & \overrightarrow{0}
\end{smallmatrix}\right\rangle 
+
\\
&
\hspace{1.9cm}
\left(m-1\right)\left(\left\langle \begin{smallmatrix}
1 & 1\\
1 & 1\\
2 & 2\\
\overrightarrow{0} & \overrightarrow{0}
\end{smallmatrix}\right\rangle 
+
\left(p-m-1\right)\left\langle \begin{smallmatrix}
1 & 1\\
1 & 1\\
0 & 2\\
2 & 0\\
\overrightarrow{0} & \overrightarrow{0}
\end{smallmatrix}\right\rangle \right)\Bigg)
\\&
=
\left(p-m\right)\left(\tfrac{\left(p^{2}+4p+15\right)+\left(p-m-1\right)\left(p+3\right)^{2}}{\left(p-1\right)p\left(p+1\right)\left(p+2\right)\left(p+4\right)\left(p+6\right)}
+
\tfrac{\left(m-1\right)\left(-p+3-\left(p-m-1\right)\left(p+3\right)\right)}{\left(p-1\right)p\left(p+1\right)\left(p+2\right)\left(p+4\right)\left(p+6\right)}\right)
%
%
\\&
=\frac{\left(p-m\right)\left(m^{2}p+3m^{2}-2mp^{2}-10mp-6m+p^{3}+7p^{2}+10p\right)}{\left(p-1\right)p\left(p+1\right)\left(p+2\right)\left(p+4\right)\left(p+6\right)}
\end{split}
\end{align}

\bigskip

\begin{align}
\label{eq:(ub)^2(ubvb)(uaQmva)(vaQmva)}
\begin{split}
&\mathbb{E}\left[\mathbf{u}_{b}^{\top}\mathbf{u}_{b}\mathbf{u}_{b}^{\top}\mathbf{v}_{b}\cdot\mathbf{u}_{a}^{\top}\mathbf{Q}_{m}\mathbf{v}_{a}\mathbf{v}_{a}^{\top}\mathbf{Q}_{m}\mathbf{v}_{a}\right]=\mathbb{E}\left[\left\Vert \mathbf{v}_{a}\right\Vert ^{2}\mathbf{u}_{b}^{\top}\mathbf{u}_{b}\mathbf{u}_{b}^{\top}\mathbf{v}_{b}\cdot\mathbb{E}_{\mathbf{r}\sim\mathcal{S}^{m-1}}\left(\mathbf{u}_{a}^{\top}\mathbf{r}\mathbf{v}_{a}^{\top}\mathbf{r}\right)\right]
\\&
=\frac{1}{m}\mathbb{E}\left[\left\Vert \mathbf{v}_{a}\right\Vert ^{2}\mathbf{u}_{b}^{\top}\mathbf{u}_{b}\mathbf{u}_{b}^{\top}\mathbf{v}_{b}\mathbf{u}_{a}^{\top}\mathbf{v}_{a}\right]=-\frac{1}{m}\mathbb{E}\left[\left\Vert \mathbf{v}_{a}\right\Vert ^{2}\left\Vert \mathbf{u}_{b}\right\Vert ^{2}\left(\mathbf{u}_{a}^{\top}\mathbf{v}_{a}\right)^{2}\right]
\\&
=-\frac{1}{m}\mathbb{E}\left[\left(1-\left\Vert \mathbf{v}_{b}\right\Vert ^{2}\right)\left\Vert \mathbf{u}_{b}\right\Vert ^{2}\left(\mathbf{u}_{a}^{\top}\mathbf{v}_{a}\right)^{2}\right]
\\
&
=
-\frac{1}{m}\underbrace{\mathbb{E}\left[\left\Vert \mathbf{u}_{b}\right\Vert ^{2}\left(\mathbf{u}_{a}^{\top}\mathbf{v}_{a}\right)^{2}\right]}_{\text{solved in  \eqref{eq:ub2(uava)2}}}+\frac{1}{m}\underbrace{\mathbb{E}\left[\left\Vert \mathbf{u}_{b}\right\Vert ^{2}\left\Vert \mathbf{v}_{b}\right\Vert ^{2}\left(\mathbf{u}_{a}^{\top}\mathbf{v}_{a}\right)^{2}\right]}_{\text{solved in \eqref{eq:(ub)^2(vb)^2(vaua)^2}}}
\\&
=-\tfrac{\left(p-m\right)\left(p-m+2\right)}{\left(p-1\right)p\left(p+2\right)\left(p+4\right)}+\tfrac{\left(p-m\right)\left(m^{2}\left(p+3\right)-2m\left(p^{2}+5p+3\right)+p\left(p^{2}+7p+10\right)\right)}{\left(p-1\right)p\left(p+1\right)\left(p+2\right)\left(p+4\right)\left(p+6\right)}
\\&
=-\frac{\left(p-m\right)\left(p+3\right)\left(3\left(p+1\right)+\left(m-1\right)\left(p-m-1\right)\right)}{\left(p-1\right)p\left(p+1\right)\left(p+2\right)\left(p+4\right)\left(p+6\right)}
\end{split}
\end{align}



\subsection{Auxiliary derivations with three vectors}

\begin{align} 
\label{eq:(uaQmza)^2(uaQmva)^2}
\begin{split}
&\mathbb{E}_{\mathbf{u}\perp\mathbf{v}\perp\mathbf{z},\mathbf{Q}_{m}}\left[\left(\mathbf{u}_{a}^{\top}\mathbf{Q}_{m}\mathbf{z}_{a}\right)^{2}\left(\mathbf{u}_{a}^{\top}\mathbf{Q}_{m}\mathbf{v}_{a}\right)^{2}\right]
\\&
=\mathbb{E}_{\mathbf{u}\perp\mathbf{v}\perp\mathbf{z}}\left[\mathbb{E}_{\mathbf{r}\sim\mathcal{S}^{m-1}}\left[\left(\left\Vert \mathbf{u}_{a}\right\Vert \mathbf{r}^{\top}\mathbf{z}_{a}\right)^{2}\left(\left\Vert \mathbf{u}_{a}\right\Vert \mathbf{r}^{\top}\mathbf{v}_{a}\right)^{2}\right]\right]
\\&
=\sum_{i,j=1}^{m}\sum_{k,\ell=1}^{m}\mathbb{E}_{\mathbf{u}\perp\mathbf{v}\perp\mathbf{z}}\left[\left\Vert \mathbf{u}_{a}\right\Vert ^{4}\mathbb{E}_{\mathbf{r}\sim\mathcal{S}^{m-1}}\left[r_{i}z_{i}r_{j}z_{j}r_{k}v_{k}r_{\ell}v_{\ell}\right]\right]
\\&
=
\mathbb{E}_{\mathbf{u}\perp\mathbf{v}\perp\mathbf{z}}\left[\left\Vert \mathbf{u}_{a}\right\Vert ^{4}\sum_{i,j=1}^{m}\sum_{k,\ell=1}^{m}z_{i}z_{j}v_{k}v_{\ell}\mathbb{E}_{\mathbf{r}\sim\mathcal{S}^{m-1}}\left[r_{i}r_{j}r_{k}r_{\ell}\right]\right]
\\
&
\eqmargin
[\text{by \propref{prop:odd}, most terms become zero}]
\\
&=
\left\langle \begin{smallmatrix}
4\\
\overrightarrow{0}
\end{smallmatrix}\right\rangle _{m}\underbrace{\mathbb{E}\left[\left\Vert \mathbf{u}_{a}\right\Vert ^{4}\sum_{i=1}^{m}z_{i}^{2}v_{i}^{2}\right]}_{i=j=k=\ell}+
\\
&\eqmargin
\left\langle \begin{smallmatrix}
2\\
2\\
\overrightarrow{0}
\end{smallmatrix}\right\rangle _{m}\left(\underbrace{\mathbb{E}\left[\left\Vert \mathbf{u}_{a}\right\Vert ^{4}\sum_{i\neq k=1}^{m}z_{i}^{2}v_{k}^{2}\right]}_{i=j\neq k=\ell}+\underbrace{2\mathbb{E}\left[\left\Vert \mathbf{u}_{a}\right\Vert ^{4}\sum_{i\neq j=1}^{m}z_{i}z_{j}v_{i}v_{j}\right]}_{i=k\neq j=\ell\,\vee\,i=\ell\neq j=k}\right)
\\&
=\frac{3}{m\left(m+2\right)}\sum_{i=1}^{m}\mathbb{E}\left[\left\Vert \mathbf{u}_{a}\right\Vert ^{4}z_{i}^{2}v_{i}^{2}\right]
+
\\
&\eqmargin
\frac{1}{m\left(m+2\right)}\left(\sum_{i\neq k=1}^{m}\mathbb{E}\left[\left\Vert \mathbf{u}_{a}\right\Vert ^{4}z_{i}^{2}v_{k}^{2}\right]+2\sum_{i\neq j=1}^{m}\mathbb{E}\left[\left\Vert \mathbf{u}_{a}\right\Vert ^{4}z_{i}z_{j}v_{i}v_{j}\right]\right)
\\&
=\frac{3m}{m\left(m+2\right)}\mathbb{E}\left[\left\Vert \mathbf{u}_{a}\right\Vert ^{4}z_{1}^{2}v_{1}^{2}\right]+\frac{m\left(m-1\right)}{m\left(m+2\right)}\left(\mathbb{E}\left[\left\Vert \mathbf{u}_{a}\right\Vert ^{4}z_{1}^{2}v_{2}^{2}\right]+2\mathbb{E}\left[\left\Vert \mathbf{u}_{a}\right\Vert ^{4}z_{1}z_{2}v_{1}v_{2}\right]\right)
\\&
=\frac{3}{m+2}\underbrace{\mathbb{E}\left[\left\Vert \mathbf{u}_{a}\right\Vert ^{4}z_{1}^{2}v_{1}^{2}\right]}_{\text{solved in \eqref{eq:(ua^4)(z1v1)^2}}}
+
\frac{m-1}{m+2}\left(\underbrace{\mathbb{E}\left[\left\Vert \mathbf{u}_{a}\right\Vert ^{4}z_{1}^{2}v_{2}^{2}\right]}_{\text{solved in \eqref{eq:(ua^4)(z1v2)^2}}}
+
2\underbrace{\mathbb{E}\left[\left\Vert \mathbf{u}_{a}\right\Vert ^{4}z_{1}z_{2}v_{1}v_{2}\right]}_{\text{solved in \eqref{eq:(ua^4)(z1v1z2v2)}}}\right)
\\&
=\tfrac{3}{m+2}
\tfrac{\left(p+3\right)\left(m^{2}\left(p+5\right)+2m\left(p+1\right)\right)-16p-24}{\left(p-1\right)p\left(p+1\right)\left(p+2\right)\left(p+4\right)\left(p+6\right)}
+
\\
&
\eqmargin
\tfrac{m-1}{m+2}\tfrac{m^{2}\left(p^{3}+8p^{2}+13p-2\right)+2m\left(p^{3}+4p^{2}-7p-10\right)-8\left(2p^{2}+9p+6\right)}{\left(p-2\right)\left(p-1\right)p\left(p+1\right)\left(p+2\right)\left(p+4\right)\left(p+6\right)}
+
\\
&
\eqmargin
\tfrac{2\left(m-1\right)}{m+2}
\tfrac{16p-8p^{2}-\left(m-2\right)
\left(mp^{2}+7mp+14m+4p^{2}+12p+24\right)}{\left(p-2\right)\left(p-1\right)p\left(p+1\right)\left(p+2\right)\left(p+4\right)\left(p+6\right)}
%
%
%
\\&
=\tfrac{\left(m+2\right)\left(m\left(p+3\right)\left(\left(m+2\right)p+5m+2\right)-16p-24\right)}{\left(m+2\right)\left(p-1\right)p\left(p+1\right)\left(p+2\right)\left(p+4\right)\left(p+6\right)}=\frac{m\left(p+3\right)\left(\left(m+2\right)p+5m+2\right)-16p-24}{\left(p-1\right)p\left(p+1\right)\left(p+2\right)\left(p+4\right)\left(p+6\right)}
\end{split}
\end{align}

\newpage

\begin{align} 
\label{eq:(uaQmza)^2(ubvb)^2}
\begin{split}
&\mathbb{E}_{\mathbf{u}\perp\mathbf{v}\perp\mathbf{z},\mathbf{Q}_{m}}\left[\left(\mathbf{u}_{a}^{\top}\mathbf{Q}_{m}\mathbf{z}_{a}\right)^{2}\left(\mathbf{u}_{b}^{\top}\mathbf{v}_{b}\right)^{2}\right]=\mathbb{E}_{\mathbf{u}\perp\mathbf{v}\perp\mathbf{z}}\left[\mathbb{E}_{\mathbf{r}\sim\mathcal{S}^{m-1}}\left(\left\Vert \mathbf{u}_{a}\right\Vert \mathbf{r}^{\top}\mathbf{z}_{a}\right)^{2}\left(\mathbf{u}_{b}^{\top}\mathbf{v}_{b}\right)^{2}\right]
\\&
=\mathbb{E}_{\mathbf{u}\perp\mathbf{v}\perp\mathbf{z}}\left[\left\Vert \mathbf{u}_{a}\right\Vert ^{2}\mathbb{E}_{\mathbf{r}\sim\mathcal{S}^{m-1}}\left(\mathbf{z}_{a}^{\top}\mathbf{r}\mathbf{r}^{\top}\mathbf{z}_{a}\right)\left(\mathbf{u}_{b}^{\top}\mathbf{v}_{b}\right)^{2}\right]
\\&
=\frac{1}{m}\mathbb{E}_{\mathbf{u}\perp\mathbf{v}\perp\mathbf{z}}\left[\left\Vert \mathbf{u}_{a}\right\Vert ^{2}\left\Vert \mathbf{z}_{a}\right\Vert ^{2}\left(\mathbf{u}_{b}^{\top}\mathbf{v}_{b}\right)^{2}\right]=\frac{1}{m}\sum_{i,j=1}^{m}\sum_{k,\ell=m+1}^{p}\mathbb{E}\left[u_{i}^{2}z_{j}^{2}u_{k}u_{\ell}v_{k}v_{\ell}\right]
\\&
=\frac{m\left(p-m\right)}{m}\sum_{j=1}^{m}\sum_{k=m+1}^{p}\mathbb{E}\left[u_{1}^{2}z_{j}^{2}u_{k}u_{p}v_{k}v_{p}\right]
\\&
=\left(p-m\right)\left(\underbrace{\sum_{k=m+1}^{p}\mathbb{E}\left[u_{1}^{2}z_{1}^{2}u_{k}u_{p}v_{k}v_{p}\right]}_{j=1}+\underbrace{\left(m-1\right)\sum_{k=m+1}^{p}\mathbb{E}\left[u_{1}^{2}z_{2}^{2}u_{k}u_{p}v_{k}v_{p}\right]}_{j\ge2}\right)
\\&
=\left(p-m\right)\Bigg(\underbrace{\mathbb{E}\left[u_{1}^{2}z_{1}^{2}u_{p}^{2}v_{p}^{2}\right]}_{k=p}+\underbrace{\left(p-m-1\right)\mathbb{E}\left[u_{1}^{2}z_{1}^{2}u_{p-1}u_{p}v_{p-1}v_{p}\right]}_{m+1\le k\le p-1}
+
\\
& \hspace{2cm}
\left(m-1\right)
\bigg(\underbrace{\mathbb{E}\left[u_{1}^{2}z_{2}^{2}u_{p}^{2}v_{p}^{2}\right]}_{k=p}+\underbrace{\left(p-m-1\right)\mathbb{E}\left[u_{1}^{2}z_{2}^{2}u_{p-1}u_{p}v_{p-1}v_{p}\right]}_{m+1\le k\le p-1}\bigg)\Bigg)
\\&
=\left(p-m\right)\Bigg(\left\langle \begin{smallmatrix}
2 & 2 & 0\\
2 & 0 & 2\\
\overrightarrow{0} & \overrightarrow{0} & \overrightarrow{0}
\end{smallmatrix}\right\rangle +\left(p-m-1\right)\left\langle \begin{smallmatrix}
2 & 0 & 2\\
1 & 1 & 0\\
1 & 1 & 0\\
\overrightarrow{0} & \overrightarrow{0} & \overrightarrow{0}
\end{smallmatrix}\right\rangle 
+
\\
& \hspace{2cm}
\left(m-1\right)\bigg(\left\langle \begin{smallmatrix}
2 & 0 & 0\\
0 & 2 & 0\\
2 & 0 & 2\\
\overrightarrow{0} & \overrightarrow{0} & \overrightarrow{0}
\end{smallmatrix}\right\rangle +\left(p-m-1\right)\left\langle \begin{smallmatrix}
2 & 0 & 0\\
0 & 0 & 2\\
1 & 1 & 0\\
1 & 1 & 0\\
\overrightarrow{0} & \overrightarrow{0} & \overrightarrow{0}
\end{smallmatrix}\right\rangle \bigg)\Bigg)
\\&
=\left(p-m\right)\left(\tfrac{\left(p-2\right)\left(p+3\right)^{2}-\left(p-m-1\right)\left(p^{2}+3p+6\right)}{\left(p-2\right)\left(p-1\right)p\left(p+1\right)\left(p+2\right)\left(p+4\right)\left(p+6\right)}+
\tfrac{
\left(m-1\right)
\left(p^{3}+6p^{2}+3p-6-\left(p-m-1\right)\left(p^{2}+5p+2\right)
\right)
}{\left(p-2\right)\left(p-1\right)p\left(p+1\right)\left(p+2\right)\left(p+4\right)\left(p+6\right)}\right)
%
%
\\&
=\left(p-m\right)\left(\tfrac{mp^{2}+3mp+6m+2p^{2}-6p-12+\left(m-1\right)\left(mp^{2}+5mp+2m+2p^{2}+6p-4\right)}{\left(p-2\right)\left(p-1\right)p\left(p+1\right)\left(p+2\right)\left(p+4\right)\left(p+6\right)}\right)
\\&
=\frac{\left(p-m\right)\left(m+2\right)\left(m\left(p^{2}+5p+2\right)-6p-4\right)}{\left(p-2\right)\left(p-1\right)p\left(p+1\right)\left(p+2\right)\left(p+4\right)\left(p+6\right)}
\end{split}
\end{align}

\newpage

\begin{align}
\label{eq:(ua^4)(z1v1)^2}
\begin{split}
&\mathbb{E}\left[\left\Vert \mathbf{u}_{a}\right\Vert ^{4}z_{1}^{2}v_{1}^{2}\right]=\sum_{i=1}^{m}\sum_{j=1}^{m}\mathbb{E}\left[u_{i}^{2}u_{j}^{2}z_{1}^{2}v_{1}^{2}\right]
\\&
=\underbrace{\mathbb{E}\left[u_{1}^{4}z_{1}^{2}v_{1}^{2}\right]}_{i=j=1}+\underbrace{2\left(m-1\right)\mathbb{E}\left[u_{1}^{2}u_{2}^{2}z_{1}^{2}v_{1}^{2}\right]}_{i=1\neq j\,\vee\,i\neq1=j}+\underbrace{\left(m-1\right)\mathbb{E}\left[u_{2}^{4}z_{1}^{2}v_{1}^{2}\right]}_{i=j\ge2}+
\\
&\eqmargin
\underbrace{\left(m-1\right)\left(m-2\right)\mathbb{E}\left[u_{2}^{2}u_{3}^{2}z_{1}^{2}v_{1}^{2}\right]}_{i\neq j\ge2}
\\&
=\left\langle \begin{smallmatrix}
4 & 2 & 2\\
\overrightarrow{0} & \overrightarrow{0} & \overrightarrow{0}
\end{smallmatrix}\right\rangle +
2\left(m-1\right)\left\langle \begin{smallmatrix}
2 & 2 & 2\\
2 & 0 & 0\\
\overrightarrow{0} & \overrightarrow{0} & \overrightarrow{0}
\end{smallmatrix}\right\rangle +
\left(m-1\right)\left\langle \begin{smallmatrix}
0 & 2 & 2\\
4 & 0 & 0\\
\overrightarrow{0} & \overrightarrow{0} & \overrightarrow{0}
\end{smallmatrix}\right\rangle +
\\
&
\eqmargin
\left(m-1\right)\left(m-2\right)\left\langle \begin{smallmatrix}
0 & 2 & 2\\
2 & 0 & 0\\
2 & 0 & 0\\
\overrightarrow{0} & \overrightarrow{0} & \overrightarrow{0}
\end{smallmatrix}\right\rangle 
\\&
=\tfrac{3}{p\left(p+2\right)\left(p+4\right)\left(p+6\right)}+
\tfrac{2\left(m-1\right)\left(p+3\right)}{\left(p-1\right)p\left(p+2\right)\left(p+4\right)\left(p+6\right)}+
\tfrac{3\left(m-1\right)\left(p+3\right)\left(p+5\right)+\left(m-1\right)\left(m-2\right)\left(p+3\right)\left(p+5\right)}{\left(p-1\right)p\left(p+1\right)\left(p+2\right)\left(p+4\right)\left(p+6\right)}
\\&
=\frac{3\left(p-1\right)\left(p+1\right)+2\left(m-1\right)\left(p+1\right)\left(p+3\right)+\left(m^{2}-1\right)\left(p+3\right)\left(p+5\right)}{\left(p-1\right)p\left(p+1\right)\left(p+2\right)\left(p+4\right)\left(p+6\right)}
\\&
=\frac{\left(p+3\right)\left(m^{2}\left(p+5\right)+2m\left(p+1\right)\right)-16p-24}{\left(p-1\right)p\left(p+1\right)\left(p+2\right)\left(p+4\right)\left(p+6\right)}
\end{split}
\end{align} 

\bigskip

\begin{align}
\label{eq:(ua^4)(z1v2)^2}
\begin{split}
&\mathbb{E}\left[\left\Vert \mathbf{u}_{a}\right\Vert ^{4}z_{1}^{2}v_{2}^{2}\right]=\sum_{i=1}^{m}\sum_{j=1}^{m}\mathbb{E}\left[u_{i}^{2}u_{j}^{2}z_{1}^{2}v_{2}^{2}\right]
\\&
=2\underbrace{\mathbb{E}\!\left[u_{1}^{4}z_{1}^{2}v_{2}^{2}\right]}_{i=j\in\left\{ 1,2\right\} }
+
2\!\!\underbrace{\mathbb{E}\!\left[u_{1}^{2}u_{2}^{2}z_{1}^{2}v_{2}^{2}\right]}_{i=1,j=2\,\vee\,i=2,j=1}
\!\!\!
+
\underbrace{2\left(m-2\right)
\!\!
\underbrace{\mathbb{E}\!\left[u_{1}^{2}u_{3}^{2}z_{1}^{2}v_{2}^{2}\right]}_{i=1,j\ge3\,\vee\,j=1,i\ge3}
\!\!
+
2\left(m-2\right)\underbrace{\mathbb{E}\!\left[u_{2}^{2}u_{3}^{2}z_{1}^{2}v_{2}^{2}\right]}_{i=2<j\,\vee\,i>2=j}}_{\text{equal}}
+
\\
&\eqmargin
\left(m-2\right)\underbrace{\mathbb{E}\left[u_{3}^{4}z_{1}^{2}v_{2}^{2}\right]}_{i=j\ge3}+
\left(m-2\right)\left(m-3\right)\underbrace{\mathbb{E}\left[u_{3}^{2}u_{4}^{2}z_{1}^{2}v_{2}^{2}\right]}_{i\neq j\ge3}
\\&
=
2\left\langle \begin{smallmatrix}
4 & 2 & 0\\
0 & 0 & 2\\
\overrightarrow{0} & \overrightarrow{0} & \overrightarrow{0}
\end{smallmatrix}\right\rangle 
+
2\left\langle \begin{smallmatrix}
2 & 2 & 0\\
2 & 0 & 2\\
\overrightarrow{0} & \overrightarrow{0} & \overrightarrow{0}
\end{smallmatrix}\right\rangle 
+
\\
&\eqmargin
\left(m-2\right)
\left(
4\left\langle \begin{smallmatrix}
2 & 2 & 0\\
0 & 0 & 2\\
2 & 0 & 0\\
\overrightarrow{0} & \overrightarrow{0} & \overrightarrow{0}
\end{smallmatrix}\right\rangle +
\left\langle \begin{smallmatrix}
0 & 2 & 0\\
0 & 0 & 2\\
4 & 0 & 0\\
\overrightarrow{0} & \overrightarrow{0} & \overrightarrow{0}
\end{smallmatrix}\right\rangle 
+
\left(m-3\right)\left\langle \begin{smallmatrix}
0 & 2 & 0\\
0 & 0 & 2\\
2 & 0 & 0\\
2 & 0 & 0\\
\overrightarrow{0} & \overrightarrow{0} & \overrightarrow{0}
\end{smallmatrix}\right\rangle 
\right)
\\&
=\tfrac{6\left(p+1\right)\left(p+5\right)+2\left(p+3\right)^{2}}{\left(p-1\right)p\left(p+1\right)\left(p+2\right)\left(p+4\right)\left(p+6\right)}
+
\tfrac{4\left(m-2\right)\left(p^{3}+6p^{2}+3p-6\right)+3\left(m-2\right)\left(p^{3}+8p^{2}+13p-2\right)}{\left(p-2\right)\left(p-1\right)p\left(p+1\right)\left(p+2\right)\left(p+4\right)\left(p+6\right)}
+
\\
&
\eqmargin
\tfrac{\left(m-2\right)\left(m-3\right)\left(p^{3}+8p^{2}+13p-2\right)}{\left(p-2\right)\left(p-1\right)p\left(p+1\right)\left(p+2\right)\left(p+4\right)\left(p+6\right)}
%
%
\\&
=\frac{m^{2}\left(p^{3}+8p^{2}+13p-2\right)+2m\left(p^{3}+4p^{2}-7p-10\right)-8\left(2p^{2}+9p+6\right)}{\left(p-2\right)\left(p-1\right)p\left(p+1\right)\left(p+2\right)\left(p+4\right)\left(p+6\right)}
\end{split}
\end{align}

\newpage

\begin{align}
\label{eq:(ua^4)(z1v1z2v2)}
\begin{split}
&\mathbb{E}\left[\left\Vert \mathbf{u}_{a}\right\Vert ^{4}z_{1}z_{2}v_{1}v_{2}\right]=\sum_{i=1}^{m}\sum_{j=1}^{m}\mathbb{E}\left[u_{i}^{2}u_{j}^{2}z_{1}z_{2}v_{1}v_{2}\right]
\\&
=2\underbrace{\mathbb{E}\left[u_{1}^{4}z_{1}z_{2}v_{1}v_{2}\right]}_{i=j\in\left\{ 1,2\right\} }+2\underbrace{\mathbb{E}\left[u_{1}^{2}u_{2}^{2}z_{1}z_{2}v_{1}v_{2}\right]}_{i=1,j=2\,\vee\,i=2,j=1}
+
\\
&
\eqmargin
\underbrace{2\left(m-2\right)\underbrace{\mathbb{E}\left[u_{1}^{2}u_{3}^{2}z_{1}z_{2}v_{1}v_{2}\right]}_{i=1,j\ge3\,\vee\,j=1,i\ge3}+2\left(m-2\right)\underbrace{\mathbb{E}\left[u_{2}^{2}u_{3}^{2}z_{1}z_{2}v_{1}v_{2}\right]}_{i=2<j\,\vee\,i>2=j}}_{\text{equal}}
+
\\
&
\eqmargin
\left(m-2\right)\underbrace{\mathbb{E}\left[u_{3}^{4}z_{1}z_{2}v_{1}v_{2}\right]}_{i=j\ge3}+
\left(m-2\right)\left(m-3\right)\underbrace{\mathbb{E}\left[u_{3}^{2}u_{4}^{2}z_{1}z_{2}v_{1}v_{2}\right]}_{i\neq j\ge3}
\\&
=2\left\langle \begin{smallmatrix}
4 & 1 & 1\\
0 & 1 & 1\\
\overrightarrow{0} & \overrightarrow{0} & \overrightarrow{0}
\end{smallmatrix}\right\rangle 
+
2\left\langle \begin{smallmatrix}
2 & 1 & 1\\
2 & 1 & 1\\
\overrightarrow{0} & \overrightarrow{0} & \overrightarrow{0}
\end{smallmatrix}\right\rangle 
+
\\
&\eqmargin
\left(m-2\right)
\left(
4\left\langle \begin{smallmatrix}
2 & 1 & 1\\
0 & 1 & 1\\
2 & 0 & 0\\
\overrightarrow{0} & \overrightarrow{0} & \overrightarrow{0}
\end{smallmatrix}\right\rangle 
+
\left\langle \begin{smallmatrix}
0 & 1 & 1\\
0 & 1 & 1\\
4 & 0 & 0\\
\overrightarrow{0} & \overrightarrow{0} & \overrightarrow{0}
\end{smallmatrix}\right\rangle 
+
\left(m-3\right)\left\langle \begin{smallmatrix}
0 & 1 & 1\\
0 & 1 & 1\\
2 & 0 & 0\\
2 & 0 & 0\\
\overrightarrow{0} & \overrightarrow{0} & \overrightarrow{0}
\end{smallmatrix}\right\rangle 
\right)
\\&
=\tfrac{-6}{\left(p-1\right)p\left(p+2\right)\left(p+4\right)\left(p+6\right)}
+
\tfrac{-2\left(p-3\right)}{\left(p-1\right)p\left(p+1\right)\left(p+2\right)\left(p+4\right)\left(p+6\right)}
+
\\
&
\eqmargin
\tfrac{-4\left(m-2\right)\left(p^{2}+3p+6\right)-3\left(m-2\right)\left(p^{2}+7p+14\right)-\left(m-2\right)\left(m-3\right)\left(p^{2}+7p+14\right)}{\left(p-2\right)\left(p-1\right)p\left(p+1\right)\left(p+2\right)\left(p+4\right)\left(p+6\right)}
\\&
=\tfrac{-6\left(p-2\right)\left(p+1\right)-2\left(p-3\right)\left(p-2\right)-4\left(m-2\right)\left(p^{2}+3p+6\right)-3\left(m-2\right)\left(p^{2}+7p+14\right)-\left(m-2\right)\left(m-3\right)\left(p^{2}+7p+14\right)}{\left(p-2\right)\left(p-1\right)p\left(p+1\right)\left(p+2\right)\left(p+4\right)\left(p+6\right)}
\\&
=\frac{16p-8p^{2}-\left(m-2\right)\left(mp^{2}+7mp+14m+4p^{2}+12p+24\right)}{\left(p-2\right)\left(p-1\right)p\left(p+1\right)\left(p+2\right)\left(p+4\right)\left(p+6\right)}
\end{split}
\end{align}

\newpage

\begin{align}
\label{eq:(ua)^2vazaubzbubvb}
\begin{split}
&\mathbb{E}_{\mathbf{u}\perp\mathbf{v}\perp\mathbf{z}}\left[\left\Vert \mathbf{u}_{a}\right\Vert ^{2}\mathbf{v}_{a}^{\top}\mathbf{z}_{a}\mathbf{u}_{b}^{\top}\mathbf{z}_{b}\mathbf{u}_{b}^{\top}\mathbf{v}_{b}\right]=\sum_{i,j=1}^{m}\sum_{k,\ell=m+1}^{p}\mathbb{E}\left[u_{i}^{2}v_{j}z_{j}u_{k}z_{k}u_{\ell}v_{\ell}\right]
\\&
=m\left(p-m\right)\sum_{j=1}^{m}\sum_{k=m+1}^{p}\mathbb{E}\left[u_{1}^{2}v_{j}z_{j}u_{k}z_{k}u_{p}v_{p}\right]
\\&
=m\left(p-m\right)\left(\underbrace{\sum_{k=m+1}^{p}\mathbb{E}\left[u_{1}^{2}v_{1}z_{1}u_{k}z_{k}u_{p}v_{p}\right]}_{j=1}+\underbrace{\left(m-1\right)\sum_{k=m+1}^{p}\mathbb{E}\left[u_{1}^{2}v_{2}z_{2}u_{k}z_{k}u_{p}v_{p}\right]}_{j\ge2}\right)
\\&
=m\left(p-m\right)\Bigg(\underbrace{\mathbb{E}\left[u_{1}^{2}v_{1}z_{1}u_{p}^{2}v_{p}z_{p}\right]}_{k=p}
+
\underbrace{\left(p\!-\!m\!-\!1\right)
\mathbb{E}\left[u_{1}^{2}v_{1}z_{1}u_{p-1}z_{p-1}u_{p}v_{p}\right]}_{m+1\le k\le p-1}
+
\\
& \hspace{2.2cm}
\left(m-1\right)\bigg(\underbrace{\mathbb{E}\left[u_{1}^{2}v_{2}z_{2}u_{p}^{2}v_{p}z_{p}\right]}_{k=p}
+
\underbrace{\left(p\!-\!m\!-\!1\right)
\mathbb{E}\left[u_{1}^{2}v_{2}z_{2}u_{p-1}z_{p-1}u_{p}v_{p}\right]}_{m+1\le k\le p-1}\bigg)\Bigg)
\\&
=m\left(p-m\right)\Biggl(\left\langle \begin{smallmatrix}
2 & 1 & 1\\
2 & 1 & 1\\
\overrightarrow{0} & \overrightarrow{0} & \overrightarrow{0}
\end{smallmatrix}\right\rangle +\left(p-m-1\right)\left\langle \begin{smallmatrix}
2 & 1 & 1\\
1 & 1 & 0\\
1 & 0 & 1\\
\overrightarrow{0} & \overrightarrow{0} & \overrightarrow{0}
\end{smallmatrix}\right\rangle 
+
\\&
\hspace{2.2cm} 
\left(m-1\right)\bigg(\left\langle \begin{smallmatrix}
2 & 0 & 0\\
0 & 1 & 1\\
2 & 1 & 1\\
\overrightarrow{0} & \overrightarrow{0} & \overrightarrow{0}
\end{smallmatrix}\right\rangle +\left(p-m-1\right)\left\langle \begin{smallmatrix}
2 & 0 & 0\\
0 & 1 & 1\\
1 & 1 & 0\\
1 & 0 & 1\\
\overrightarrow{0} & \overrightarrow{0} & \overrightarrow{0}
\end{smallmatrix}\right\rangle \bigg)\Biggr)
\\&
=m\left(p-m\right)\left(
\tfrac{-\left(p-3\right)\left(p-2\right)+4p\left(p-m-1\right)}{\left(p-2\right)\left(p-1\right)p\left(p+1\right)\left(p+2\right)\left(p+4\right)\left(p+6\right)}
+
\tfrac{
\left(m-1\right)
\left(
-\left(p^{2}+3p+6\right)+2\left(p-m-1\right)\left(p+2\right)
\right)
}{\left(p-2\right)\left(p-1\right)p\left(p+1\right)\left(p+2\right)\left(p+4\right)\left(p+6\right)}
\right)
%
%
\\&
=\frac{m\left(p-m\right)\left(\left(m+2\right)p^{2}+\left(2-3m\right)p-2m^{2}\left(p+2\right)-6m+4\right)}{\left(p-2\right)\left(p-1\right)p\left(p+1\right)\left(p+2\right)\left(p+4\right)\left(p+6\right)}
\end{split}
\end{align}

\newpage

\begin{align} 
\label{eq:(ubzb)^2(ubvb)^2}
\begin{split}
&\mathbb{E}\left[\left(\mathbf{u}_{b}^{\top}\mathbf{z}_{b}\right)^{2}\left(\mathbf{u}_{b}^{\top}\mathbf{v}_{b}\right)^{2}\right]=\mathbb{E}\left[\left(-\mathbf{u}_{a}^{\top}\mathbf{z}_{a}\right)^{2}\left(\mathbf{u}_{b}^{\top}\mathbf{v}_{b}\right)^{2}\right]=\mathbb{E}\left[\left(\mathbf{u}_{a}^{\top}\mathbf{z}_{a}\right)^{2}\left(\mathbf{u}_{b}^{\top}\mathbf{v}_{b}\right)^{2}\right]
\\&
=\sum_{i,j=1}^{m}\sum_{k,\ell=m+1}^{p-m}\mathbb{E}\left[u_{i}u_{j}u_{k}u_{\ell}z_{i}z_{j}v_{k}v_{\ell}\right]=m\left(p-m\right)\sum_{j=1}^{m}\sum_{k=m+1}^{p-m}\mathbb{E}\left[u_{1}u_{j}u_{k}u_{p}z_{1}z_{j}v_{k}v_{p}\right]
\\&
=m\left(p-m\right)\Biggl(\underbrace{\mathbb{E}\left[u_{1}^{2}u_{p}^{2}z_{1}^{2}v_{p}^{2}\right]}_{j=1,\,k=p}+\underbrace{\left(p-m-1\right)\mathbb{E}\left[u_{1}^{2}u_{p-1}u_{p}z_{1}^{2}v_{p-1}v_{p}\right]}_{j=1,\,m+1\le k\le p-1}
+
\\
&
\hspace{2.3cm}
\underbrace{\left(m-1\right)\mathbb{E}\left[u_{1}u_{2}u_{p}^{2}z_{1}z_{2}v_{p}^{2}\right]}_{2\le j\le m,\,k=p}
+
\\
&
\hspace{2.3cm}
\underbrace{\left(m-1\right)\left(p-m-1\right)\mathbb{E}\left[u_{1}u_{2}u_{p-1}u_{p}z_{1}z_{2}v_{p-1}v_{p}\right]}_{2\le j\le m,\,m+1\le k\le p-1}\Biggr)
\\&
=m\left(p-m\right)
\Bigg(\left\langle \begin{smallmatrix}
2 & 0 & 2\\
2 & 2 & 0\\
\overrightarrow{0} & \overrightarrow{0} & \overrightarrow{0}
\end{smallmatrix}\right\rangle +\underbrace{\left(p-m-1\right)\left\langle \begin{smallmatrix}
2 & 0 & 2\\
1 & 1 & 0\\
1 & 1 & 0\\
\overrightarrow{0} & \overrightarrow{0} & \overrightarrow{0}
\end{smallmatrix}\right\rangle +\left(m-1\right)\left\langle \begin{smallmatrix}
1 & 0 & 1\\
1 & 0 & 1\\
2 & 2 & 0\\
\overrightarrow{0} & \overrightarrow{0} & \overrightarrow{0}
\end{smallmatrix}\right\rangle }_{
\text{equal due to invariance (\propref{prop:invariance})}}
+
\\
&
\hspace{2.3cm}
\left(m-1\right)\left(p-m-1\right)\left\langle \begin{smallmatrix}
1 & 0 & 1\\
1 & 0 & 1\\
1 & 1 & 0\\
1 & 1 & 0\\
\overrightarrow{0} & \overrightarrow{0} & \overrightarrow{0}
\end{smallmatrix}\right\rangle \Bigg)
\\&
=m\left(p-m\right)\left(\left\langle \begin{smallmatrix}
2 & 0 & 2\\
2 & 2 & 0\\
\overrightarrow{0} & \overrightarrow{0} & \overrightarrow{0}
\end{smallmatrix}\right\rangle 
+
\left(p-2\right)\left\langle \begin{smallmatrix}
2 & 0 & 2\\
1 & 1 & 0\\
1 & 1 & 0\\
\overrightarrow{0} & \overrightarrow{0} & \overrightarrow{0}
\end{smallmatrix}\right\rangle 
+
\left(m-1\right)\left(p-m-1\right)\left\langle \begin{smallmatrix}
1 & 0 & 1\\
1 & 0 & 1\\
1 & 1 & 0\\
1 & 1 & 0\\
\overrightarrow{0} & \overrightarrow{0} & \overrightarrow{0}
\end{smallmatrix}\right\rangle \right)
\\&
=m\left(p-m\right)\bigg(\tfrac{\left(p+3\right)^{2}}{\left(p-1\right)p\left(p+1\right)\left(p+2\right)\left(p+4\right)\left(p+6\right)}+
\tfrac{-\left(p-2\right)\left(p^{2}+3p+6\right)}{\left(p-2\right)\left(p-1\right)p\left(p+1\right)\left(p+2\right)\left(p+4\right)\left(p+6\right)}+
\\
&
\hspace{2.3cm}
\tfrac{\left(m-1\right)\left(p-m-1\right)}{\left(p-1\right)p\left(p+1\right)\left(p+2\right)\left(p+4\right)\left(p+6\right)}\bigg)
\\&
=m\left(p-m\right)\left(\frac{\left(p+3\right)^{2}-\left(p^{2}+3p+6\right)+\left(m-1\right)\left(p-m-1\right)}{\left(p-1\right)p\left(p+1\right)\left(p+2\right)\left(p+4\right)\left(p+6\right)}\right)
\\&
=\frac{m\left(p-m\right)\left(-m^{2}+mp+2p+4\right)}{\left(p-1\right)p\left(p+1\right)\left(p+2\right)\left(p+4\right)\left(p+6\right)}
\end{split}
\end{align}

\bigskip

\begin{align}
\label{eq:(ua)^2(vbzb)^2}
\begin{split}
&\mathbb{E}_{\mathbf{u}\perp\mathbf{v}\perp\mathbf{z}}\left[\left\Vert \mathbf{z}_{a}\right\Vert ^{2}\left(\mathbf{v}_{b}^{\top}\mathbf{u}_{b}\right)^{2}\right]=\mathbb{E}\left[\left\Vert \mathbf{u}_{a}\right\Vert ^{2}\left(\mathbf{v}_{b}^{\top}\mathbf{z}_{b}\right)^{2}\right]=\mathbb{E}\left[\sum_{i=1}^{m}u_{i}^{2}\left(\sum_{j=m+1}^{p}v_{j}z_{j}\right)^{2}\right]
\\&
=\sum_{i=1}^{m}\sum_{j,k=m+1}^{p}\mathbb{E}\left[u_{i}^{2}v_{j}z_{j}v_{k}z_{k}\right]=m\sum_{j,k=m+1}^{p}\mathbb{E}\left[u_{1}^{2}v_{j}z_{j}v_{k}z_{k}\right]
\\&
=m\underbrace{\left(p-m\right)\mathbb{E}\left[u_{1}^{2}v_{p}^{2}z_{p}^{2}\right]}_{j=k}+m\underbrace{\left(p-m\right)\left(p-m-1\right)\mathbb{E}\left[u_{1}^{2}v_{p-1}z_{p-1}v_{p}z_{p}\right]}_{j\neq k}
\\&
=m\left(p-m\right)\left(\left\langle \begin{smallmatrix}
2 & 2 & 0\\
0 & 0 & 2\\
\overrightarrow{0} & \overrightarrow{0} & \overrightarrow{0}
\end{smallmatrix}\right\rangle +\left(p-m-1\right)\left\langle \begin{smallmatrix}
2 & 0 & 0\\
0 & 1 & 1\\
0 & 1 & 1\\
\overrightarrow{0} & \overrightarrow{0} & \overrightarrow{0}
\end{smallmatrix}\right\rangle \right)
\\&
=m\left(p-m\right)\left(\tfrac{\left(p+3\right)}{\left(p-1\right)p\left(p+2\right)\left(p+4\right)}
-
\tfrac{\left(p-m-1\right)\left(p+2\right)}{\left(p-2\right)\left(p-1\right)p\left(p+2\right)\left(p+4\right)}\right)
\\&
=\frac{m\left(p-m\right)\left(mp+2m-4\right)}{\left(p-2\right)\left(p-1\right)p\left(p+2\right)\left(p+4\right)}
\end{split}
\end{align}


\begin{align} 
\label{eq:(ub)^2(vaza)(uaza)(ubvb)}
\begin{split}
&\mathbb{E}\left[\left\Vert \mathbf{u}_{b}\right\Vert ^{2}\mathbf{v}_{a}^{\top}\mathbf{z}_{a}\mathbf{u}_{a}^{\top}\mathbf{z}_{a}\mathbf{u}_{b}^{\top}\mathbf{v}_{b}\right]=\mathbb{E}\left[\mathbf{u}_{b}^{\top}\mathbf{u}_{b}\mathbf{u}_{b}^{\top}\mathbf{v}_{b}\mathbf{z}_{a}^{\top}\mathbf{u}_{a}\mathbf{z}_{a}^{\top}\mathbf{v}_{a}\right]
\\&
=\sum_{i,j=1}^{m}\sum_{k,\ell=m+1}^{p}\mathbb{E}\left[u_{\ell}^{2}u_{k}v_{k}u_{i}z_{i}z_{j}v_{j}\right]=\left(p-m\right)m\sum_{j=1}^{m}\sum_{k=m+1}^{p}\mathbb{E}\left[u_{1}u_{p}^{2}u_{k}v_{k}z_{1}z_{j}v_{j}\right]
\\&
=\left(p-m\right)m\sum_{k=m+1}^{p}\left(\mathbb{E}\left[u_{1}u_{p}^{2}v_{1}z_{1}^{2}\left(u_{k}v_{k}\right)\right]+\left(m-1\right)\mathbb{E}\left[u_{1}u_{p}^{2}v_{2}z_{1}z_{2}\left(u_{k}v_{k}\right)\right]\right)
\\&
=\left(p-m\right)m\left(\mathbb{E}\left[u_{1}u_{p}^{3}v_{1}v_{p}z_{1}^{2}\right]+\left(p-m-1\right)\mathbb{E}\left[u_{1}u_{p-1}u_{p}^{2}v_{1}v_{p-1}z_{1}^{2}\right]\right)+
\\&
\eqmargin
\left(p-m\right)m\left(m-1\right)
\left(\mathbb{E}\left[u_{1}u_{p}^{3}v_{2}v_{p}z_{1}z_{2}\right]+\left(p-m-1\right)\mathbb{E}\left[u_{1}u_{p-1}u_{p}^{2}v_{2}v_{p-1}z_{1}z_{2}\right]\right)
\\&
=\left(p-m\right)m\left(\left\langle \begin{smallmatrix}
1 & 1 & 2\\
3 & 1 & 0\\
\overrightarrow{0} & \overrightarrow{0} & \overrightarrow{0}
\end{smallmatrix}\right\rangle +\left(p-m-1\right)\left\langle \begin{smallmatrix}
1 & 1 & 2\\
1 & 1 & 0\\
2 & 0 & 0\\
\overrightarrow{0} & \overrightarrow{0} & \overrightarrow{0}
\end{smallmatrix}\right\rangle\right)
+
\\&
\eqmargin
\left(p-m\right)m\left(m-1\right)
\left(\left\langle \begin{smallmatrix}
1 & 0 & 1\\
0 & 1 & 1\\
3 & 1 & 0\\
\overrightarrow{0} & \overrightarrow{0} & \overrightarrow{0}
\end{smallmatrix}\right\rangle +\left(p-m-1\right)\left\langle \begin{smallmatrix}
1 & 0 & 1\\
0 & 1 & 1\\
1 & 1 & 0\\
2 & 0 & 0\\
\overrightarrow{0} & \overrightarrow{0} & \overrightarrow{0}
\end{smallmatrix}\right\rangle \right)
\\&
=\left(p-m\right)m\left(\tfrac{-3\left(p+3\right)-\left(p+3\right)\left(p-m-1\right)}{\left(p-1\right)p\left(p+1\right)\left(p+2\right)\left(p+4\right)\left(p+6\right)}+\left(m-1\right)\left(\tfrac{6\left(p+2\right)+2\left(p+2\right)\left(p-m-1\right)}{\left(p-2\right)\left(p-1\right)p\left(p+1\right)\left(p+2\right)\left(p+4\right)\left(p+6\right)}\right)\right)
\\&
=\frac{\left(p-m\right)m\left(p-m+2\right)\left(2mp+4m-p^{2}-3p+2\right)}{\left(p-2\right)\left(p-1\right)p\left(p+1\right)\left(p+2\right)\left(p+4\right)\left(p+6\right)}
\end{split}
\end{align}

\medskip

\begin{align} 
\label{eq:(ua)^2(vaza)(vbxb)(xbzb)}
\begin{split}
&\mathbb{E}\left[\left\Vert \mathbf{u}_{a}\right\Vert ^{2}\mathbf{v}_{a}^{\top}\mathbf{z}_{a}\cdot\mathbf{v}_{b}^{\top}\mathbf{x}_{b}\cdot\mathbf{x}_{b}^{\top}\mathbf{z}_{b}\right]=\sum_{i,j=1}^{m}\sum_{k,\ell=m+1}^{p}\mathbb{E}\left[u_{i}^{2}v_{j}z_{j}v_{k}x_{k}x_{\ell}z_{\ell}\right]
\\&
=m\left(p-m\right)\sum_{j=1}^{m}\sum_{k=m+1}^{p}\mathbb{E}\left[u_{1}^{2}v_{j}z_{j}v_{k}x_{k}x_{p}z_{p}\right]
\\&
=m\left(p-m\right)\left(\underbrace{\mathbb{E}\left[u_{1}^{2}v_{1}z_{1}v_{p}x_{p}^{2}z_{p}\right]}_{j=1,\,k=p}+\underbrace{\left(m-1\right)\mathbb{E}\left[u_{1}^{2}v_{2}z_{2}v_{p}x_{p}^{2}z_{p}\right]}_{2\le j\le m,\,k=p}\right)+
\\&
\eqmargin
m\left(p\!-\!m\right)\left(p\!-\!m\!-\!1\right)
\!
\left(\underbrace{\mathbb{E}\!\left[u_{1}^{2}v_{1}z_{1}v_{p-1}x_{p-1}x_{p}z_{p}\right]}_{j=1,\,m+1\le k\le p-1}+\underbrace{\left(m\!-\!1\right)\mathbb{E}\!\left[u_{1}^{2}v_{2}z_{2}v_{p-1}x_{p-1}x_{p}z_{p}\right]}_{2\le j\le m,\,m+1\le k\le p-1}\right)
\\&
=m\left(p-m\right)\left(\left\langle \begin{smallmatrix}
2 & 1 & 1 & 0\\
0 & 1 & 1 & 2\\
\overrightarrow{0} & \overrightarrow{0} & \overrightarrow{0} & \overrightarrow{0}
\end{smallmatrix}\right\rangle +\left(m-1\right)\left\langle \begin{smallmatrix}
2 & 0 & 0 & 0\\
0 & 1 & 1 & 0\\
0 & 1 & 1 & 2\\
\overrightarrow{0} & \overrightarrow{0} & \overrightarrow{0} & \overrightarrow{0}
\end{smallmatrix}\right\rangle \right)+
\\&
\eqmargin
m\left(p-m\right)\left(\left(p-m-1\right)\left\langle \begin{smallmatrix}
2 & 1 & 1 & 0\\
0 & 1 & 0 & 1\\
0 & 0 & 1 & 1\\
\overrightarrow{0} & \overrightarrow{0} & \overrightarrow{0} & \overrightarrow{0}
\end{smallmatrix}\right\rangle +\left(m-1\right)\left(p-m-1\right)\left\langle \begin{smallmatrix}
2 & 0 & 0 & 0\\
0 & 1 & 1 & 0\\
0 & 1 & 0 & 1\\
0 & 0 & 1 & 1\\
\overrightarrow{0} & \overrightarrow{0} & \overrightarrow{0} & \overrightarrow{0}
\end{smallmatrix}\right\rangle \right)
\\&
=m\left(p\!-\!m\right)
\!
\left(\!\tfrac{-\left(p-2\right)\left(p+3\right)+2\left(p-m-1\right)\left(p+2\right)}{\left(p-2\right)\left(p-1\right)p\left(p+1\right)\left(p+2\right)\left(p+4\right)\left(p+6\right)}
\!+\!
\tfrac{
\left(m-1\right)
\left(
2\left(p-m-1\right)p\left(p+4\right)-\left(p^{2}+5p+2\right)\left(p-3\right)
\right)}{\left(p-3\right)\left(p-2\right)\left(p-1\right)p\left(p+1\right)\left(p+2\right)\left(p+4\right)\left(p+6\right)}\!\right)
\\&
=\tfrac{m\left(p-m\right)\left(-2m^{2}p^{2}-8m^{2}p+mp^{3}+4mp^{2}+15mp+18m-6p^{2}-6p-12\right)}{\left(p-3\right)\left(p-2\right)\left(p-1\right)p\left(p+1\right)\left(p+2\right)\left(p+4\right)\left(p+6\right)}
\end{split}
\end{align}

\medskip

\begin{align} 
\label{eq:(uaQmza)(uaQmva)(xbzb)(xbvb)}
\begin{split}
&\mathbb{E}\left[\mathbf{u}_{a}^{\top}\mathbf{Q}_{m}\mathbf{z}_{a}\mathbf{u}_{a}^{\top}\mathbf{Q}_{m}\mathbf{v}_{a}\mathbf{x}_{b}^{\top}\mathbf{z}_{b}\mathbf{x}_{b}^{\top}\mathbf{v}_{b}\right]
\\&
=\mathbb{E}_{\mathbf{r}\sim\mathcal{S}^{m-1}}\left[\left\Vert \mathbf{u}_{a}\right\Vert ^{2}\left(\mathbf{z}_{a}^{\top}\mathbf{r}\mathbf{r}^{\top}\mathbf{v}_{a}\right)\left(\mathbf{x}_{b}^{\top}\mathbf{z}_{b}\cdot\mathbf{x}_{b}^{\top}\mathbf{v}_{b}\right)\right]
=\frac{1}{m}\underbrace{\mathbb{E}\left[\left\Vert \mathbf{u}_{a}\right\Vert ^{2}\mathbf{z}_{a}^{\top}\mathbf{v}_{a}\cdot\mathbf{x}_{b}^{\top}\mathbf{z}_{b}\cdot\mathbf{x}_{b}^{\top}\mathbf{v}_{b}\right]}_{\text{solved in \eqref{eq:(ua)^2(vaza)(vbxb)(xbzb)}}}
\\&
=\frac{\left(p-m\right)\left(-2m^{2}p^{2}-8m^{2}p+mp^{3}+4mp^{2}+15mp+18m-6p^{2}-6p-12\right)}{\left(p-3\right)\left(p-2\right)\left(p-1\right)p\left(p+1\right)\left(p+2\right)\left(p+4\right)\left(p+6\right)}
\end{split}
\end{align}

\bigskip

\begin{align} 
\label{eq:(ub)^2(vbzb)(uaQmva)(zaQmua)}
\begin{split}
&\mathbb{E}\left[\mathbf{u}_{b}^{\top}\mathbf{u}_{b}\mathbf{z}_{b}^{\top}\mathbf{v}_{b}\cdot\mathbf{u}_{a}^{\top}\mathbf{Q}_{m}\mathbf{v}_{a}\mathbf{z}_{a}^{\top}\mathbf{Q}_{m}\mathbf{u}_{a}\right]=\sum_{i,j=1}^{m}\sum_{k,\ell=1}^{m}\mathbb{E}\left[\mathbf{u}_{b}^{\top}\mathbf{u}_{b}\mathbf{z}_{b}^{\top}\mathbf{v}_{b}\cdot u_{i}q_{ij}v_{j}z_{k}q_{k\ell}u_{\ell}\right]
\\&
=\sum_{i,j=1}^{m}\sum_{k,\ell=1}^{m}\mathbb{E}\left[\mathbf{u}_{b}^{\top}\mathbf{u}_{b}\mathbf{z}_{b}^{\top}\mathbf{v}_{b}\cdot u_{i}v_{j}z_{k}u_{\ell}\right]\mathbb{E}\left[q_{ij}q_{k\ell}\right]\\
\explain{\text{\propref{prop:odd}}}
&=\frac{1}{m}\sum_{i,j=1}^{m}\mathbb{E}\left[\mathbf{u}_{b}^{\top}\mathbf{u}_{b}\mathbf{z}_{b}^{\top}\mathbf{v}_{b}\cdot u_{i}v_{j}z_{i}u_{j}\right]\underbrace{\mathbb{E}\left[q_{ij}^{2}\right]}_{=1/m}
\\&
=\frac{1}{m}\mathbb{E}\left[\mathbf{u}_{b}^{\top}\mathbf{u}_{b}\mathbf{z}_{b}^{\top}\mathbf{v}_{b}\mathbf{u}_{a}^{\top}\mathbf{v}_{a}\mathbf{u}_{a}^{\top}\mathbf{z}_{a}\right]=\frac{1}{m}\mathbb{E}\left[\left\Vert \mathbf{u}_{b}\right\Vert ^{2}\left(-\mathbf{z}_{a}^{\top}\mathbf{v}_{a}\right)\left(-\mathbf{u}_{b}^{\top}\mathbf{v}_{b}\right)\mathbf{u}_{a}^{\top}\mathbf{z}_{a}\right]
\\&
=\frac{1}{m}\underbrace{\mathbb{E}\left[\left\Vert \mathbf{u}_{b}\right\Vert ^{2}\mathbf{v}_{a}^{\top}\mathbf{z}_{a}\mathbf{u}_{a}^{\top}\mathbf{z}_{a}\mathbf{u}_{b}^{\top}\mathbf{v}_{b}\right]}_{\text{solved in \eqref{eq:(ub)^2(vaza)(uaza)(ubvb)}}}
\\&
=\frac{\left(p-m\right)\left(p-m+2\right)\left(2mp+4m-p^{2}-3p+2\right)}{\left(p-2\right)\left(p-1\right)p\left(p+1\right)\left(p+2\right)\left(p+4\right)\left(p+6\right)}
\end{split}
\end{align}

\bigskip

\begin{align} 
\label{eq:(ub)^2(ubvb)(zaQmua)(zaQmva)}
\begin{split}
&\mathbb{E}\left[\left\Vert \mathbf{u}_{b}\right\Vert ^{2}\mathbf{u}_{b}^{\top}\mathbf{v}_{b}\cdot\mathbf{z}_{a}^{\top}\mathbf{Q}_{m}\mathbf{u}_{a}\mathbf{z}_{a}^{\top}\mathbf{Q}_{m}\mathbf{v}_{a}\right]
\\&
=\mathbb{E}\left[\left\Vert \mathbf{z}_{a}\right\Vert ^{2}\left\Vert \mathbf{u}_{b}\right\Vert ^{2}\mathbf{u}_{b}^{\top}\mathbf{v}_{b}
\mathbf{u}_{a}^{\top}
\mathbb{E}_{\mathbf{r}\sim\mathcal{S}^{m-1}}\!
\left(\mathbf{r}^{\top}\mathbf{r}\right)\mathbf{v}_{a}\right]
=
\tfrac{1}{m}\mathbb{E}\left[\left\Vert \mathbf{z}_{a}\right\Vert ^{2}\left(\sum_{\ell=m+1}^{p}u_{\ell}^{2}\right)\mathbf{u}_{b}^{\top}\mathbf{v}_{b}\mathbf{u}_{a}^{\top}\mathbf{v}_{a}\right]
\\&
=\tfrac{\left(p-m\right)}{m}\mathbb{E}\left[u_{p}^{2}\left\Vert \mathbf{z}_{a}\right\Vert ^{2}\left(\sum_{i=m+1}^{p}u_{i}v_{i}\right)\cdot\left(\sum_{j=1}^{m}u_{j}v_{j}\right)\right]
\\&
=
\left(p-m\right)\mathbb{E}\left[u_{1}v_{1}u_{p}^{2}\left\Vert \mathbf{z}_{a}\right\Vert ^{2}\left(\sum_{i=m+1}^{p}u_{i}v_{i}\right)\right]
\\&
=\left(p-m\right)
\left(\mathbb{E}\!\left[u_{1}u_{p}^{3}v_{1}v_{p}\left\Vert \mathbf{z}_{a}\right\Vert ^{2}\right]+
\left(p\!-\!m\!-\!1\right)\mathbb{E}\left[u_{1}u_{p-1}u_{p}^{2}v_{1}v_{p-1}\left\Vert \mathbf{z}_{a}\right\Vert ^{2}\right]\right)
\\&
=\left(p-m\right)
\left(\mathbb{E}\!\left[u_{1}u_{p}^{3}v_{1}v_{p}
\!\left(\sum_{i=1}^{m}z_{i}^{2}\right)\right]+\left(p-m-1\right)\mathbb{E}\left[u_{1}u_{p-1}u_{p}^{2}v_{1}v_{p-1}\left(\sum_{i=1}^{m}z_{i}^{2}\right)\right]\right)
\\&
=\left(p-m\right)\left(\mathbb{E}\left[u_{1}u_{p}^{3}v_{1}v_{p}z_{1}^{2}\right]+\left(m-1\right)\mathbb{E}\left[u_{1}u_{p}^{3}v_{1}v_{p}z_{2}^{2}\right]\right)+
\\&
\eqmargin
\left(p-m\right)\left(p-m-1\right)\left(\mathbb{E}\left[u_{1}u_{p-1}u_{p}^{2}v_{1}v_{p-1}z_{1}^{2}\right]+\left(m-1\right)\mathbb{E}\left[u_{1}u_{p-1}u_{p}^{2}v_{1}v_{p-1}z_{2}^{2}\right]\right)
\\&
=\left(p-m\right)\left(\left\langle \begin{smallmatrix}
1 & 1 & 2\\
3 & 1 & 0\\
\overrightarrow{0} & \overrightarrow{0} & \overrightarrow{0}
\end{smallmatrix}\right\rangle +\left(m-1\right)\left\langle \begin{smallmatrix}
1 & 1 & 0\\
0 & 0 & 2\\
3 & 1 & 0\\
\overrightarrow{0} & \overrightarrow{0} & \overrightarrow{0}
\end{smallmatrix}\right\rangle\right)
+
\\
&
\eqmargin
\left(p-m\right)
\left(p-m-1\right)
\left(\left\langle \begin{smallmatrix}
1 & 1 & 2\\
1 & 1 & 0\\
2 & 0 & 0\\
\overrightarrow{0} & \overrightarrow{0} & \overrightarrow{0}
\end{smallmatrix}\right\rangle +\left(m-1\right)\left\langle \begin{smallmatrix}
1 & 1 & 0\\
0 & 0 & 2\\
1 & 1 & 0\\
2 & 0 & 0\\
\overrightarrow{0} & \overrightarrow{0} & \overrightarrow{0}
\end{smallmatrix}\right\rangle \right)
\\&
=\left(p-m\right)\tfrac{-3\left(p+3\right)}{\left(p-1\right)p\left(p+1\right)\left(p+2\right)\left(p+4\right)\left(p+6\right)}+
\\&
\eqmargin
\left(p-m\right)\tfrac{-3\left(m-1\right)\left(p^{2}+5p+2\right)+\left(p-m-1\right)\left(-\left(p-2\right)\left(p+3\right)-\left(m-1\right)\left(p^{2}+5p+2\right)\right)}{\left(p-2\right)\left(p-1\right)p\left(p+1\right)\left(p+2\right)\left(p+4\right)\left(p+6\right)}
\\&
=\left(p-m\right)\left(\tfrac{\left(p-2\right)\left(-3\left(p+3\right)\right)+m^{2}p^{2}+5m^{2}p+2m^{2}-mp^{3}-7mp^{2}-16mp-12m+7p^{2}+19p-2}{\left(p-2\right)\left(p-1\right)p\left(p+1\right)\left(p+2\right)\left(p+4\right)\left(p+6\right)}\right)
\\&
=\tfrac{\left(p-m\right)\left(m^{2}p^{2}+5m^{2}p+2m^{2}-mp^{3}-7mp^{2}-16mp-12m+4p^{2}+16p+16\right)}{\left(p-2\right)\left(p-1\right)p\left(p+1\right)\left(p+2\right)\left(p+4\right)\left(p+6\right)}
\end{split}
\end{align}

\newpage

\subsection{Auxiliary derivations with four vectors}

\begin{align}
\label{eq:(ubvb)(xbzb)(ubxb)(vbzb)}
\begin{split}
&\mathbb{E}\left[\mathbf{u}_{b}^{\top}\mathbf{v}_{b}\cdot\mathbf{x}_{b}^{\top}\mathbf{z}_{b}\cdot\mathbf{u}_{b}^{\top}\mathbf{x}_{b}\cdot\mathbf{v}_{b}^{\top}\mathbf{z}_{b}\right]
=
\mathbb{E}\left[\mathbf{u}_{a}^{\top}\mathbf{v}_{a}\cdot\mathbf{x}_{a}^{\top}\mathbf{z}_{a}\cdot\mathbf{u}_{b}^{\top}\mathbf{x}_{b}\cdot\mathbf{v}_{b}^{\top}\mathbf{z}_{b}\right]
\\&
=\sum_{i,j=1}^{m}\sum_{k,\ell=m+1}^{p}\mathbb{E}\left[u_{i}v_{i}x_{j}z_{j}u_{k}x_{k}v_{\ell}z_{\ell}\right]=m\left(p-m\right)\sum_{j=1}^{m}\sum_{k=m+1}^{p}\mathbb{E}\left[u_{1}v_{1}x_{j}z_{j}u_{k}x_{k}v_{p}z_{p}\right]
\\&
=m\left(p-m\right)\bigg(\underbrace{\mathbb{E}\left[u_{1}v_{1}x_{1}z_{1}u_{p}x_{p}v_{p}z_{p}\right]}_{j=1,\,k=p}+\underbrace{\left(m-1\right)\mathbb{E}\left[u_{1}v_{1}x_{2}z_{2}u_{p}x_{p}v_{p}z_{p}\right]}_{2\le j\le m,\,k=p}\bigg)+
\\&
\eqmargin
m\left(p-m\right)\bigg(\underbrace{\left(p-m-1\right)\mathbb{E}\left[u_{1}v_{1}x_{1}z_{1}u_{p-1}x_{p-1}v_{p}z_{p}\right]}_{j=1,\,m+1\le k\le p-1}+
\\
&
\eqmargin\hspace{2cm}
\underbrace{\left(m-1\right)\left(p-m-1\right)\mathbb{E}\left[u_{1}v_{1}x_{2}z_{2}u_{p-1}x_{p-1}v_{p}z_{p}\right]}_{2\le j\le m,\,m+1\le k\le p-1}\bigg)
\\&
=m\left(p-m\right)\bigg(\left\langle \begin{smallmatrix}
1 & 1 & 1 & 1\\
1 & 1 & 1 & 1\\
\overrightarrow{0} & \overrightarrow{0} & \overrightarrow{0} & \overrightarrow{0}
\end{smallmatrix}\right\rangle +\left(m-1\right)\left\langle \begin{smallmatrix}
1 & 1 & 1 & 1\\
1 & 0 & 1 & 0\\
0 & 1 & 0 & 1\\
\overrightarrow{0} & \overrightarrow{0} & \overrightarrow{0} & \overrightarrow{0}
\end{smallmatrix}\right\rangle +
\\
&
\eqmargin\hspace{2cm}
\left(p-m-1\right)\left\langle \begin{smallmatrix}
1 & 1 & 1 & 1\\
1 & 1 & 0 & 0\\
0 & 0 & 1 & 1\\
\overrightarrow{0} & \overrightarrow{0} & \overrightarrow{0} & \overrightarrow{0}
\end{smallmatrix}\right\rangle +\left(m-1\right)\left(p-m-1\right)\left\langle \begin{smallmatrix}
1 & 0 & 1 & 0\\
0 & 1 & 0 & 1\\
1 & 1 & 0 & 0\\
0 & 0 & 1 & 1\\
\overrightarrow{0} & \overrightarrow{0} & \overrightarrow{0} & \overrightarrow{0}
\end{smallmatrix}\right\rangle \bigg)
\\&
=m\left(p-m\right)\Bigg(\left\langle \begin{smallmatrix}
1 & 1 & 1 & 1\\
1 & 1 & 1 & 1\\
\overrightarrow{0} & \overrightarrow{0} & \overrightarrow{0} & \overrightarrow{0}
\end{smallmatrix}\right\rangle +\left(p-2\right)\left\langle \begin{smallmatrix}
1 & 1 & 1 & 1\\
1 & 1 & 0 & 0\\
0 & 0 & 1 & 1\\
\overrightarrow{0} & \overrightarrow{0} & \overrightarrow{0} & \overrightarrow{0}
\end{smallmatrix}\right\rangle 
+
\\
&
\hspace{2.3cm}
\left(m-1\right)\left(p-m-1\right)\left\langle \begin{smallmatrix}
1 & 0 & 1 & 0\\
0 & 1 & 0 & 1\\
1 & 1 & 0 & 0\\
0 & 0 & 1 & 1\\
\overrightarrow{0} & \overrightarrow{0} & \overrightarrow{0} & \overrightarrow{0}
\end{smallmatrix}\right\rangle \Bigg)
\\&
=m\left(p-m\right)\left(\tfrac{3\left(p-3\right)}{\left(p-3\right)\left(p-1\right)p\left(p+1\right)\left(p+2\right)\left(p+4\right)\left(p+6\right)}+\tfrac{\left(p-2\right)}{\left(p-1\right)p\left(p+1\right)\left(p+2\right)\left(p+4\right)\left(p+6\right)}\right)+
\\&
\eqmargin
m\left(p-m\right)\left(\tfrac{\left(m-1\right)\left(p-m-1\right)\left(-5p-6\right)}{\left(p-3\right)\left(p-2\right)\left(p-1\right)p\left(p+1\right)\left(p+2\right)\left(p+4\right)\left(p+6\right)}\right)
\\&
=\frac{m\left(p-m\right)\left(5m^{2}p+6m^{2}-5mp^{2}-6mp+p^{3}+p^{2}+2p\right)}{\left(p-3\right)\left(p-2\right)\left(p-1\right)p\left(p+1\right)\left(p+2\right)\left(p+4\right)\left(p+6\right)}
\end{split}
\end{align}

\newpage

\begin{align} 
\label{eq:(za)^2(uaxa)(ubvb)(xbvb)}
\begin{split}
&\mathbb{E}\left[\left\Vert \mathbf{z}_{a}\right\Vert ^{2}\mathbf{u}_{a}^{\top}\mathbf{x}_{a}\mathbf{u}_{b}^{\top}\mathbf{v}_{b}\mathbf{x}_{b}^{\top}\mathbf{v}_{b}\right]
\\&
=\sum_{i,j=1}^{m}\sum_{k,\ell=m+1}^{p}\mathbb{E}\left[z_{i}^{2}u_{j}x_{j}u_{k}v_{k}x_{\ell}v_{\ell}\right]=m\left(p-m\right)\sum_{j=1}^{m}\sum_{k=m+1}^{p}\mathbb{E}\left[z_{1}^{2}u_{j}x_{j}u_{k}v_{k}x_{p}v_{p}\right]
\\&
=m\left(p\!-\!m\right)
\bigg(\underbrace{\mathbb{E}\left[u_{1}x_{1}z_{1}^{2}u_{p}v_{p}^{2}x_{p}\right]}_{j=1,\,k=p}+\underbrace{\left(m-1\right)\mathbb{E}\left[u_{2}x_{2}z_{1}^{2}u_{p}v_{p}^{2}x_{p}\right]}_{2\le j\le m,\,k=p}\bigg)+
\\&
\eqmargin
m\left(p\!-\!m\right)\left(p\!-\!m\!-\!1\right)
\!
\left(\underbrace{\mathbb{E}\!\left[u_{1}x_{1}z_{1}^{2}u_{p-1}v_{p-1}x_{p}v_{p}\right]}_{j=1,\,m+1\le k\le p-1}+\underbrace{\left(m\!-\!1\right)\mathbb{E}\!\left[u_{2}x_{2}z_{1}^{2}u_{p-1}v_{p-1}x_{p}v_{p}\right]}_{2\le j\le m,\,m+1\le k\le p-1}\right)
\\&
=m\left(p-m\right)
\bigg(\left\langle \begin{smallmatrix}
1 & 0 & 1 & 2\\
1 & 2 & 1 & 0\\
\overrightarrow{0} & \overrightarrow{0} & \overrightarrow{0} & \overrightarrow{0}
\end{smallmatrix}\right\rangle +\left(m-1\right)\left\langle \begin{smallmatrix}
0 & 0 & 0 & 2\\
1 & 0 & 1 & 0\\
1 & 2 & 1 & 0\\
\overrightarrow{0} & \overrightarrow{0} & \overrightarrow{0} & \overrightarrow{0}
\end{smallmatrix}\right\rangle +
\\
&
\eqmargin\hspace{2cm}
\left(p-m-1\right)\left\langle \begin{smallmatrix}
1 & 0 & 1 & 2\\
1 & 1 & 0 & 0\\
0 & 1 & 1 & 0\\
\overrightarrow{0} & \overrightarrow{0} & \overrightarrow{0} & \overrightarrow{0}
\end{smallmatrix}\right\rangle +\left(m-1\right)\left(p-m-1\right)\left\langle \begin{smallmatrix}
0 & 0 & 0 & 2\\
1 & 0 & 1 & 0\\
1 & 1 & 0 & 0\\
0 & 1 & 1 & 0\\
\overrightarrow{0} & \overrightarrow{0} & \overrightarrow{0} & \overrightarrow{0}
\end{smallmatrix}\right\rangle \bigg)
\\&
=m\left(p-m\right)\Bigg(\frac{-\left(p-2\right)\left(p+3\right)+2\left(p-m-1\right)\left(p+2\right)}{\left(p-2\right)\left(p-1\right)p\left(p+1\right)\left(p+2\right)\left(p+4\right)\left(p+6\right)}
+
\\
&
\hspace{2.3cm}
\frac{
\left(m-1\right)
\left(
-\left(p-3\right)\left(p^{2}+5p+2\right)+2\left(p-m-1\right)p\left(p+4\right)\right)}{\left(p-3\right)\left(p-2\right)\left(p-1\right)p\left(p+1\right)\left(p+2\right)\left(p+4\right)\left(p+6\right)}\Bigg)
\\&
=\frac{m\left(p-m\right)\left(-2m^{2}p^{2}-8m^{2}p+mp^{3}+4mp^{2}+15mp+18m-6p^{2}-6p-12\right)}{\left(p-3\right)\left(p-2\right)\left(p-1\right)p\left(p+1\right)\left(p+2\right)\left(p+4\right)\left(p+6\right)}
\end{split}
\end{align}

\end{document}